\documentclass[11pt,a4paper,reqno]{amsart}
\pdfoutput=1
\usepackage{times}
\usepackage[margin=1.0in,tmargin=0.9in,bmargin=0.80in]{geometry}
\usepackage[dvipsnames,table]{xcolor}
\usepackage[english]{babel} 

\definecolor{bred}{rgb}{0.8,0,0}
\definecolor{Gray}{gray}{0.85}

\usepackage{amsmath,amssymb,graphicx,epsfig}
\usepackage{mathtools}
\usepackage{mathrsfs}
\usepackage{enumerate,amsthm,epstopdf,mathabx}
\usepackage[shortlabels]{enumitem}
\setlist[enumerate]{label={\upshape(\roman*)}}
\usepackage[flushleft]{threeparttable}
\usepackage{multirow}
\usepackage{longtable}
\usepackage[utf8]{inputenc} 
\usepackage[T1]{fontenc}    
\usepackage{bbm}
\usepackage{xargs}
\usepackage{latexsym}
\usepackage{wasysym}
\usepackage{float}
\usepackage{hyperref}       
\usepackage{url}            
\usepackage{booktabs}       
\usepackage{amsfonts}       
\usepackage{nicefrac}       
\usepackage{microtype}      
\usepackage{cleveref}
\usepackage{dsfont}
\usepackage{footmisc}
\usepackage{array}
\usepackage{subcaption}
\usepackage[linesnumbered,ruled,vlined]{algorithm2e}
\usepackage[section]{placeins}
\usepackage[labelfont=bf]{caption}
\usepackage[textsize=footnotesize]{todonotes}
\hypersetup{colorlinks,linkcolor={blue},citecolor={bred},urlcolor={blue}}
\usepackage[numbers,sort]{natbib}

\SetCommentSty{mycommfont}

\SetKwInput{KwInput}{Input}                
\SetKwInput{KwOutput}{Output}              

\DeclareMathOperator{\re}{Re}
\DeclareMathOperator{\im}{Im}

\DeclareMathOperator{\argmin}{arg\,min}

\DeclareMathOperator{\linspan}{span}
\DeclareMathOperator{\supp}{supp}

\usepackage{stackengine}

\newtheorem{theorem}{Theorem}[section]
\newtheorem{proposition}[theorem]{Proposition}
\newtheorem{lemma}[theorem]{Lemma}
\newtheorem{corollary}[theorem]{Corollary}
\newtheorem{remark}[theorem]{Remark}
\newtheorem{definition}[theorem]{Definition}
\newtheorem{example}[theorem]{Example}
\newtheorem{assumption}[theorem]{Assumption}

\newcolumntype{L}[1]{>{\raggedright\let\newline\\\arraybackslash\hspace{0pt}}m{#1}}
\newcolumntype{C}[1]{>{\centering\let\newline\\\arraybackslash\hspace{0pt}}m{#1}}
\newcolumntype{R}[1]{>{\raggedleft\let\newline\\\arraybackslash\hspace{0pt}}m{#1}}

\usepackage{scalerel,stackengine}
\stackMath
\newcommand\reallywidehat[1]{%
	\savestack{\tmpbox}{\stretchto{%
			\scaleto{%
				\scalerel*[\widthof{\ensuremath{#1}}]{\kern.1pt\mathchar"0362\kern.1pt}%
				{\rule{0ex}{\textheight}}
			}{\textheight}%
		}{2.4ex}}%
	\stackon[-6.9pt]{#1}{\tmpbox}%
}
\newcommand\reallywidetilde[1]{%
	\savestack{\tmpbox}{\stretchto{%
			\scaleto{%
				\scalerel*[\widthof{\ensuremath{#1}}]{\kern.1pt\mathchar"0366\kern.1pt}%
				{\rule{0ex}{\textheight}}
			}{\textheight}%
		}{2.4ex}}%
	\stackon[-6.9pt]{#1}{\tmpbox}%
}

\makeatletter
\newcommand{\labeltext}[3][]{%
	\@bsphack%
	\csname phantomsection\endcsname
	\def\tst{#1}%
	\def\labelmarkup{\emph}
	\def\refmarkup{}%
	\ifx\tst\empty\def\@currentlabel{\refmarkup{#2}}{\label{#3}}%
	\else\def\@currentlabel{\refmarkup{#1}}{\label{#3}}\fi%
	\@esphack%
	\labelmarkup{#2}
}
\makeatother

\allowdisplaybreaks

\begin{document}

\title[Universal approximation property of random feature models]{Universal approximation property of Banach space-valued\\random feature models including random neural networks}

\author[]{Ariel Neufeld}

\address{Nanyang Technological University, Division of Mathematical Sciences, 21 Nanyang Link, Singapore}
\email{ariel.neufeld@ntu.edu.sg}

\author[]{Philipp Schmocker}

\address{ETH Zurich, Department of Mathematics, R\"amistrasse 101, Z\"urich, Switzerland}
\email{philipp.schmocker@math.ethz.ch}

\date{\today}
\thanks{The first author was supported by the Nanyang Assistant Professorship Grant (NAP Grant) \emph{Machine Learning based Algorithms in Finance and Insurance}. The second author was partly supported by the FinsureTech Hub of ETH Zurich.}

\begin{abstract}
	We introduce a Banach space-valued extension of random feature learning, a data-driven supervised machine learning technique for large-scale kernel approximation. By randomly initializing the feature maps, only the linear readout needs to be trained, which reduces the computational complexity substantially. Viewing random feature models as Banach space-valued random variables, we prove a universal approximation result in the corresponding Bochner space. Moreover, we derive approximation rates and an explicit algorithm to learn an element of the given Banach space by such models.
	
	The framework of this paper includes random trigonometric/Fourier regression and in particular random neural networks which are single-hidden-layer feedforward neural networks whose weights and biases are randomly initialized, whence only the linear readout needs to be trained. For the latter, we can then lift the universal approximation property of deterministic neural networks to random neural networks, even within function spaces over non-compact domains, e.g., weighted spaces, $L^p$-spaces, and (weighted) Sobolev spaces, where the latter includes the approximation of the (weak) derivatives.
	
	In addition, we analyze when the training costs for approximating a given function grow polynomially in both the input/output dimension and the reciprocal of a pre-specified tolerated approximation error. Furthermore, we demonstrate in a numerical example the empirical advantages of random feature models over their deterministic counterparts.
\end{abstract}

\maketitle

\vspace{-0.3cm}

\section{Introduction}

The \emph{random feature model} is an architecture for the data-driven approximation of functions between finite dimensional Euclidean spaces, which was introduced by Rahimi and Recht in \cite{rahimi07,rahimi08,rahimi08b} building on earlier instances in \cite{barron93,igelnik95,neal96,williams96}. It can be seen as one of the simplest supervised machine learning technique: By randomly initializing the inner parameters of the model, only the linear readout needs to be trained, which reduces the computational complexity substantially. In this paper, we introduce a Banach space-valued extension of this architecture, which returns for every random initialization the corresponding model as an element of the given Banach space, allowing us to learn infinite dimensional objects with random features. Some examples include, but are not limited to, random trigonometric/Fourier regression (see \cite{rahimi07,avron17}), kernel regression tasks (see \cite{aronszajn50,rahimi07,bach17b}), Gaussian processes (see \cite{neal96,williams96,rasmussen06,chen21}), random neural networks (see \cite{huang06,gonon20,gonon21}), operator-valued kernels (see \cite{zhang09,micchelli05,carmeli10,alvarez12}), and random operator learning (see \cite{brault16,nelsen21,nelsen24}).

Originally, random feature learning was introduced to overcome the computational limitations of traditional kernel methods. These kernel methods map the input data into a high-dimensional feature space to capture the nonlinear input/output relation. Even though the explicit form of this feature map is often unknown, one can still compute the Gram matrix whose entries are given as the inner products of features between all pairs of data points. These inner products can be efficiently calculated using the ``kernel trick'', even for infinite dimensional feature spaces, but the computational costs increase quadratically in the number of samples. This motivated Rahimi and Recht to explore kernel approximation through random features (see \cite{rahimi07}) and to extend this approach to shallow architectures (see \cite{rahimi08,rahimi08b}). Indeed, by replacing the optimization of the non-linear feature maps with randomization, the explicit calculation of the Gram matrix can be avoided, which reduces the computational complexity and enables the application of kernel-based methods to large-scale datasets.

Subsequently, different works have contributed to the mathematical theory of random feature learning. Rahimi and Recht established in \cite{rahimi07,rahimi08,rahimi08b} the connection to reproducing kernel Hilbert spaces (RKHS) and proved the approximation rate $\mathcal{O}(1/\sqrt{N})$, where $N \in \mathbb{N}$ denotes the number of features. Subsequently, \cite{rudi17} showed that $N := \mathcal{O}(\sqrt{J} \ln(J))$ random features lead to an $L^2$-generalization error of $\mathcal{O}(1/J^{1/4})$ when approximating functions between Euclidean spaces, where $J \in \mathbb{N}$ denotes the number of samples. Moreover, \cite{carratino18} learned random features with the stochastic gradient descent algorithm instead of least squares, while \cite{jacot18} connected the infinite-width case to neural tangent kernels (see also \cite{chizat24} for an extension to deep linear neural networks). In addition, \cite{chen24} studied the generalization error of random features in a non-asymptotic setting by using random matrix theory, while \cite{mei22,mei23} showed precise asymptotics of the generalization error including the double descent curve as well as a sharp generalization error under a hypercontractivity assumption.

Our first contribution consists of a comprehensive \emph{universal approximation theorem} for Banach space-valued random feature models, presented in Theorem~\ref{ThmUAT}. In contrast to traditional kernel approximation, this result ensures the convergence of random features to any (random) element of the given Banach space. Indeed, by assuming that the deterministic feature maps are universal (i.e.~the linear span of their image can approximate any element of the Banach space), we can apply the strong law of large numbers for Banach space-valued random variables (see \cite[Theorem~3.1.10]{hytoenen16}) to lift the universality to random feature models. We apply this framework to the following three instances of random feature learning: Random trigonometric features, random Fourier regression, and random neural networks.

Random neural networks are single-hidden-layer feed-forward neural networks whose weights and biases inside the activation function are randomly initialized (see the work \cite{huang06} on extreme learning machines and in particular the work \cite{gonon20} on random neural networks with ReLU activation). By training only the linear readout, one avoids the non-convex optimization problem for training deterministic neural networks (caused by the training of the weights and biases inside the activation function, see \cite[p.~282]{goodfellow16}) and one can replace the computationally expensive backpropagation (see \cite[p.~13]{montavon12}) by, e.g., the more efficient least squares method. Using the universal approximation property of deterministic neural networks (first proven in \cite{cybenko89,hornik89}, see also \cite{leshno93,chen95,pinkus99,cuchiero23,neufeld24}), we obtain a universal approximation theorem for random neural networks which significantly generalizes the results in \cite{gonon20} from the case of ReLU activation function and $L^2$-spaces (resp.~$C^0$-spaces) to more general non-polynomial activation functions and more general function spaces over non-compact domains, e.g., weighted spaces, $L^p$-spaces, and (weighted) Sobolev spaces over unbounded domains, where the latter \emph{includes the approximation of the (weak) derivatives}.

Our second contribution are \emph{approximation rates} for learning a (possibly infinite dimensional) element of the given Banach space by a random feature model, presented in Theorem~\ref{ThmAR}. To this end, we assume that the element belongs to a specific Barron space in order to represent it as expectation of the random features (see also \cite{barron93,rahimi08b,klusowski16,bach17,e20,e22}). Then, by using a symmetrization argument with Rademacher averages and the concept of Banach space types, we obtain the desired approximation rates. In $L^2$-spaces, these rates allow us then to derive a generalization error for learning via the least squares method.

As a corollary, we obtain approximation rates and generalization errors for learning a function by a random neural network, which turn out to be similar to the approximation rates for \emph{deterministic} neural networks (see e.g.~\cite{barron92,barron93,darken93,mhaskar95,darken97,kurkova12,bgkp17,siegel20,neufeld24}). To this end, we use the ridgelet transform (see \cite{candes98}) and its distributional extension (see \cite{sonoda17}) to represent the function to be approximated as expectation of a random neuron. This approach generalizes the approximation rates and generalization errors in \cite[Section~4.2]{gonon20} from random neural networks with ReLU activation to more general activation functions and by including the approximation of the (weak) derivatives. In addition, we analyze the situation when random neural networks overcome the curse of dimensionality in the sense that the computational costs (measured as number of neurons) grow polynomially in both the input/output dimensions and the reciprocal of a pre-specified tolerated approximation error.

The theoretical foundations of this paper are also relevant in scientific computing. In particular, random neural networks have been successfully applied for solving partial differential equations (PDEs) in mathematical physics (see \cite{dwivedi20,dong21,wang23}), for quantum neural networks and quantum reservoirs \cite{jacquier23}, for solving Black-Scholes-type PDE in mathematical finance \cite{gonon21}, for optimal stopping \cite{herrera21}, for learning the hedging strategy via chaos expansion \cite{neufeld22}, for solving path-dependent PDEs in the context of rough volatility \cite{zuric23}, for pricing American options \cite{krach23}, and for solving non-linear parabolic PDEs in finance \cite{wu24}. Moreover, during the revision process, \cite{deryck25} recently established approximation bounds for random neural networks with sigmoid or $\tanh$ activation functions in Sobolev spaces, with applications to high-dimensional PDEs, while \cite{liao26} applied random features to solve PDEs by minimizing the corresponding PDE residuals.

We complement these numerical examples by learning the heat equation, which shows the empirical advantages of random feature learning over their deterministic counterparts.

\subsection{Outline} In Section~\ref{SecRFL}, we introduce a Banach space-valued extension of random feature learning. In Section~\ref{SecUAT}, we show a universal approximation result for random feature models, which is applied to random trigonometric/Fourier features and random neural networks, followed by some approximation rates in Section~\ref{SecAR}. In Section~\ref{SecLSGE}, we use the least squares method to learn a random feature model and prove a generalization error. In Section~\ref{SecNE}, we provide a numerical example, while all proofs are given in Section~\ref{SecProofs}-\ref{SecProofsNE}.

\subsection{Notation}
\label{SecNotation}

As usual, $\mathbb{N} := \lbrace 1, 2, 3, ... \rbrace$ and $\mathbb{N}_0 := \mathbb{N} \cup \lbrace 0 \rbrace$ denote the sets of natural numbers, while $\mathbb{R}$ and $\mathbb{C}$ (with imaginary unit $\mathbf{i} := \sqrt{-1} \in \mathbb{C}$) represent the sets of real and complex numbers, respectively. In addition, we define $\lceil r \rceil := \min\left\lbrace k \in \mathbb{N}_0: k \geq r \right\rbrace$ for all $r \in [0,\infty)$. Furthermore, for any $z \in \mathbb{C}$, we denote its real and imaginary part as $\re(z)$ and $\im(z)$, respectively, whereas its complex conjugate is defined as $\overline{z} := \re(z) - \im(z) \mathbf{i}$.

Moreover, for any $m \in \mathbb{N}$, we denote by $\mathbb{R}^m$ (and $\mathbb{C}^m$) the $m$-dimensional (complex) Euclidean space, equipped with the norm $\Vert u \Vert = \sqrt{\sum_{i=1}^m \vert u_i \vert^2}$. Hereby, we denote by $e_i \in \mathbb{R}^d$ denotes the $i$-th unit vector of $\mathbb{R}^d$, $i = 1,...,d$, and define $\re(z) := (\re(z_1),...,\re(z_m))^\top$, for $z := (z_1,...,z_m)^\top \in \mathbb{C}^m$. Furthermore, for any $m,n \in \mathbb{N}$, we denote by $\mathbb{R}^{m \times n}$ the vector space of matrices $A := (a_{i,j})_{i=1,...,m}^{j=1,...,n} \in \mathbb{R}^{m \times n}$, equipped with the matrix $2$-norm $\Vert A \Vert = \sup_{x \in \mathbb{R}^n, \, \Vert x \Vert \leq 1} \Vert Ax \Vert$, where $I_m \in \mathbb{R}^{m \times m}$ represents the identity matrix. Hereby, we denote by $\mathbb{S}^m_{++} \subseteq \mathbb{R}^{m \times m}$ the subspace of symmetric and strictly positive definite matrices $A \in \mathbb{R}^{m \times m}$.

In addition, for a metric space $(\Theta,d_\Theta)$ and a Banach space $(X,\Vert \cdot \Vert_X)$, we denote by $C^0(\Theta;X)$ the vector space of continuous maps $g: \Theta \rightarrow X$, equipped with the topology of compact convergence (see \cite[p.~283]{munkres14}), while $\mathcal{B}(\Theta)$ is the Borel $\sigma$-algebra of $(\Theta,d_\Theta)$. Moreover, $du: \mathcal{L}(U) \rightarrow [0,\infty]$ denotes the Lebesgue measure on $U$, with $\mathcal{L}(U)$ being the $\sigma$-algebra of Lebesgue-measurable subsets of $U \in \mathcal{B}(\mathbb{R}^m)$, where a property is said to hold almost everywhere (a.e.)~if it holds everywhere except on a set of Lebesgue measure zero.

Furthermore, for every fixed $m,d \in \mathbb{N}$, $k \in \mathbb{N}_0$, $U \subseteq \mathbb{R}^m$ (open, if $k \geq 1$), and $p \in [1,\infty)$, we introduce the following function spaces:
\begin{itemize}	
	\item $C^k(U;\mathbb{R}^d)$ denotes the vector space of $k$-times continuously differentiable functions $f: U \rightarrow \mathbb{R}^d$ such that the partial derivative $U \ni u \mapsto \partial_\alpha f(u) := \frac{\partial^{\vert\alpha\vert} f}{\partial u_1^{\alpha_1} \cdots \partial u_m^{\alpha_m}}(u) \in \mathbb{R}^d$ is continuous for all $\alpha \in \mathbb{N}^m_{0,k} := \left\lbrace \alpha := (\alpha_1,...,\alpha_m) \in \mathbb{N}_0^m: \vert\alpha\vert := \alpha_1 + ... + \alpha_m \leq k \right\rbrace$. If $m = 1$, we denote the derivatives by $f^{(j)} := \frac{\partial^j f}{\partial u^j}: U \rightarrow \mathbb{R}^d$, $j = 0,...,k$.
	
	\item $C^k_b(U;\mathbb{R}^d)$ denotes the vector space of functions $f \in C^k(U;\mathbb{R}^d)$ such that $\partial_\alpha f: U \rightarrow \mathbb{R}^d$ is bounded for all $\alpha \in \mathbb{N}^m_{0,k}$. Then, the norm $\Vert f \Vert_{C^k_b(U;\mathbb{R}^d)} := \max_{\alpha \in \mathbb{N}^m_{0,k}} \sup_{u \in U} \Vert \partial_\alpha f(u) \Vert$ turns $C^k_b(U;\mathbb{R}^d)$ into a Banach space. Note that for $k = 0$ and $U \subset \mathbb{R}^m$ compact, we obtain the Banach space of continuous functions $(C^0(U;\mathbb{R}^d),\Vert \cdot \Vert_{C^0(U;\mathbb{R}^d)})$ equipped with the supremum norm $\Vert f \Vert_{C^0(U;\mathbb{R}^d)} := \Vert f \Vert_{C^0_b(U;\mathbb{R}^d)} = \sup_{u \in U} \Vert f(u) \Vert$.
	
	\item $C^k_{pol,\gamma}(U;\mathbb{R}^d)$, with $\gamma \in [0,\infty)$, denotes the vector space of functions $f \in C^k(U;\mathbb{R})$ such that $\Vert f \Vert_{C^k_{pol,\gamma}(U;\mathbb{R}^d)} := \max_{\alpha \in \mathbb{N}^m_{0,k}} \sup_{u \in U} \frac{\Vert \partial_\alpha f(u) \Vert}{(1 + \Vert u \Vert)^\gamma} < \infty$.
	
	\item $\overline{C^k_b(U;\mathbb{R}^d)}^\gamma$, with $\gamma \in (0,\infty)$, is defined as the closure of $C^k_b(U;\mathbb{R}^d)$ with respect to $\Vert \cdot \Vert_{C^k_{pol,\gamma}(U;\mathbb{R}^d)}$. Then, $(\overline{C^k_b(U;\mathbb{R}^d)}^\gamma,\Vert \cdot \Vert_{C^k_{pol,\gamma}(U;\mathbb{R}^d)})$ is by definition a Banach space. If $U \subseteq \mathbb{R}^m$ is bounded, then $\overline{C^k_b(U;\mathbb{R}^d)}^\gamma = C^k_b(U;\mathbb{R}^d)$. Otherwise, $f \in \overline{C^k_b(U;\mathbb{R}^d)}^\gamma$ if and only if $f \in C^k(U;\mathbb{R}^d)$ and $\lim_{r \rightarrow \infty} \max_{\alpha \in \mathbb{N}^m_{0,k}} \sup_{u \in U, \, \Vert u \Vert \geq r} \frac{\Vert \partial_\alpha f(u) \Vert}{(1+\Vert u \Vert)^\gamma} = 0$ (see \cite[Lemma~4.1]{neufeld24}).
	
	\item $C^\infty_c(U;\mathbb{R}^d)$, with $U \subseteq \mathbb{R}^m$ open, denotes the vector space of smooth functions $f: U \rightarrow \mathbb{R}^d$ with $\supp(f) \subseteq U$, where $\supp(f)$ is defined as the closure of $\left\lbrace u \in U: f(u) \neq 0 \right\rbrace$ in $\mathbb{R}^m$.
	
	\item $\mathcal{S}(\mathbb{R}^m;\mathbb{C})$ denotes the Schwartz space consisting of smooth functions $f: \mathbb{R}^m \rightarrow \mathbb{C}$ such that $\max_{\alpha \in \mathbb{N}^m_{0,n}} \sup_{u \in \mathbb{R}^m} \left( 1 + \Vert u \Vert^2 \right)^n \left\vert \partial_\alpha f(u) \right\vert < \infty$, for all $n \in \mathbb{N}_0$. Moreover, its dual space $\mathcal{S}'(\mathbb{R}^m;\mathbb{C})$ consists of continuous linear functionals $T: \mathcal{S}(\mathbb{R}^m;\mathbb{C}) \rightarrow \mathbb{C}$ called tempered distributions (see \cite[p.~332]{folland92}). For example, every $f \in C^k_{pol,\gamma}(\mathbb{R}^m)$ defines a tempered distribution $\big( g \mapsto T_f(g) := \int_{\mathbb{R}} f(u) g(u) du \big) \in \mathcal{S}'(\mathbb{R}^m;\mathbb{C})$ (see \cite[Equation~9.26]{folland92}).
	
	\item $\mathcal{S}_0(\mathbb{R};\mathbb{C}) \subseteq \mathcal{S}(\mathbb{R};\mathbb{C})$ denotes the vector subspace of functions $f \in \mathcal{S}(\mathbb{R};\mathbb{C})$ such that $\int_{\mathbb{R}} u^j f(u) du = 0$ for all $j \in \mathbb{N}_0$ (see \cite[Definition~1.1.1]{grafakos14}).
	
	\item $L^p(U,\Sigma,\mu;\mathbb{R}^d)$, with (possibly non-finite) measure space $(U,\Sigma,\mu)$, denotes the vector space of (equivalence classes of) $\Sigma/\mathcal{B}(\mathbb{R}^d)$-measurable functions $f: U \rightarrow \mathbb{R}^d$ such that $\Vert f \Vert_{L^p(U,\Sigma,\mu;\mathbb{R}^d)} := \left( \int_U \Vert f(u) \Vert^p \mu(du) \right)^{1/p} < \infty$. Then, $(L^p(U,\Sigma,\mu;\mathbb{R}^d),\Vert \cdot \Vert_{L^p(U,\Sigma,\mu;\mathbb{R}^d)})$ is a Banach space (see \cite[Chapter~3]{rudin87}).
	
	\item $W^{k,p}(U,\mathcal{L}(U),du;\mathbb{R}^d)$ denotes the Sobolev space of (equivalence classes of) $k$-times weakly differentiable functions $f: U \rightarrow \mathbb{R}^d$ such that $\partial_\alpha f \in L^p(U,\mathcal{L}(U),du;\mathbb{R}^d)$ for all $\alpha \in \mathbb{N}^m_{0,k}$. Then, the norm $\Vert f \Vert_{W^{k,p}(U,\mathcal{L}(U),du;\mathbb{R}^d)} := \big( \sum_{\alpha \in \mathbb{N}^m_{0,k}} \int_U \Vert \partial_\alpha f(u) \Vert^p du \big)^{1/p}$ turns $W^{k,p}(U,\mathcal{L}(U),du;\mathbb{R}^d)$ into a Banach space (see \cite[Chapter~3]{adams75}).
	
	\item $W^{k,p}(U,\mathcal{L}(U),w;\mathbb{R}^d)$, with $\mathcal{L}(U)/\mathcal{B}(\mathbb{R})$-measurable $w: U \rightarrow [0,\infty)$, denotes the weighted Sobolev space of (equivalence classes of) $k$-times weakly differentiable functions $f: U \rightarrow \mathbb{R}^d$ such that $\partial_\alpha f \in L^p(U,\mathcal{L}(U),w(u)du;\mathbb{R}^d)$ for all $\alpha \in \mathbb{N}^m_{0,k}$. Hereby, $w: U \rightarrow [0,\infty)$ is called a \emph{weight} if it is $\mathcal{L}(U)/\mathcal{B}(\mathbb{R})$-measurable and a.e.~strictly positive such that $w^{-1/(p-1)}$ is locally Lebesgue-integrable. Then, the norm $\Vert f \Vert_{W^{k,p}(U,\mathcal{L}(U),w;\mathbb{R}^d)} := \big( \sum_{\alpha \in \mathbb{N}^m_{0,k}} \int_U \Vert \partial_\alpha f(u) \Vert^p w(u) du \big)^{1/p}$ turns $W^{k,p}(U,\mathcal{L}(U),w;\mathbb{R}^d)$ into a Banach space (see \cite{kufner84}). Moreover, we define $L^p(U,\mathcal{L}(U),w;\mathbb{R}^d) := W^{0,p}(U,\mathcal{L}(U),w;\mathbb{R}^d)$.
\end{itemize}
Moreover, we follow \cite[Chapter~7]{folland92} and define the (multi-dimensional) Fourier transform of any function $f \in L^1(\mathbb{R}^m,\mathcal{L}(\mathbb{R}^m),du;\mathbb{C}^d)$ as
\begin{equation}
	\label{EqDefFourier}
	\mathbb{R}^m \ni \xi \quad \mapsto \quad \widehat{f}(\xi) := \int_{\mathbb{R}^m} e^{-\mathbf{i} \xi^\top u} f(u) du \in \mathbb{C}^d.
\end{equation}
Then, by using \cite[Proposition~1.2.2]{hytoenen16}, it holds that
\begin{equation}
	\label{EqFourierBound}
	\sup_{\xi \in \mathbb{R}^m} \left\Vert \widehat{f}(\xi) \right\Vert = \sup_{\xi \in \mathbb{R}^m} \left\Vert \int_{\mathbb{R}^m} e^{-\mathbf{i} \xi^\top u} f(u) du \right\Vert \leq \int_{\mathbb{R}^m} \Vert f(u) \Vert du = \Vert f \Vert_{L^1(\mathbb{R}^m,\mathcal{L}(\mathbb{R}^m),du;\mathbb{R}^d)}.
\end{equation}
In addition, the Fourier transform of any tempered distribution $T \in \mathcal{S}'(\mathbb{R}^m;\mathbb{C})$ is defined by $\widehat{T}(g) := T(\widehat{g})$, for $g \in \mathcal{S}(\mathbb{R}^m;\mathbb{C})$ (see \cite[Equation~9.28]{folland92}).

Furthermore, for $r \in [1,\infty)$, a probability space $(\Omega,\mathcal{F},\mathbb{P})$, and a Banach space $(X,\Vert \cdot \Vert_X)$, we denote by $L^r(\Omega,\mathcal{F},\mathbb{P};X)$ the Bochner space of (equivalence classes of) strongly $(\mathbb{P},\mathcal{F})$-measurable maps $F: \Omega \rightarrow X$ such that $\Vert F \Vert_{L^r(\Omega,\mathcal{F},\mathbb{P};X)} := \mathbb{E}\big[ \Vert F \Vert_X^r \big]^{1/r} < \infty$. This norm turns $L^r(\Omega,\mathcal{F},\mathbb{P};X)$ into a Banach space. For more details, we refer to \cite[Section~1.2.b]{hytoenen16}.

Moreover, we use the Landau notation: $a_n = \mathcal{O}(b_n)$ (as $n \rightarrow \infty$) if $\limsup_{n \rightarrow \infty} \frac{\vert a_n \vert}{\vert b_n \vert} < \infty$.

\section{Random feature learning}
\label{SecRFL}

We now present a Banach space-valued extension of the random feature learning architecture introduced by Rahimi and Recht in \cite{rahimi07,rahimi08,rahimi08b}. Our approach imposes no specific structure on the random features (e.g.~sine/cosine or Fourier), nor does it assume that the Banach space is a particular function space.

To this end, we fix throughout this paper a probability space $(\Omega,\mathcal{F},\mathbb{P})$, a metric space $(\Theta,d_\Theta)$ representing the parameter space, and a separable Banach space $(X,\Vert \cdot \Vert_X)$ over a field $\mathbb{K}$ (either $\mathbb{K} := \mathbb{R}$ or $\mathbb{K} := \mathbb{C}$), which contains the elements to learn. Moreover, we assume the existence of an independent identically distributed (i.i.d.) sequence of $\Theta$-valued random variables $(\theta_n)_{n \in \mathbb{N}}: \Omega \rightarrow \Theta$. Then, by inserting these random initializations $(\theta_n)_{n \in \mathbb{N}}$ into the feature maps taken from a given set $\mathcal{G} \subseteq C^0(\Theta;X)$, we only need to train the linear readout that is assumed to be measurable with respect to the $\sigma$-algebra $\mathcal{F}_\theta := \sigma\left( \left\lbrace \theta_n: n \in \mathbb{N} \right\rbrace \right)$.

\begin{definition}
	\label{DefRF}
	For given $\mathcal{G} \subseteq C^0(\Theta;X)$, a \emph{random feature model (RF) (with respect to $\mathcal{G}$)} is of the form
	\begin{equation}
		\label{EqDefRF}
		\Omega \ni \omega \quad \mapsto \quad G(\omega) := \sum_{n=1}^N y_n(\omega) g_n(\theta_n(\omega)) \in X
	\end{equation}
	with respect to some $N \in \mathbb{N}$ denoting the \emph{number of features}, where $g_1,...,g_N \in \mathcal{G}$ are the \emph{feature maps}, and where the \emph{linear readouts} $y_1,...,y_N: \Omega \rightarrow \mathbb{K}$ are assumed to be $\mathcal{F}_\theta/\mathcal{B}(\mathbb{K})$-measurable.
\end{definition}

For a given set of feature maps $\mathcal{G} \subseteq C^0(\Theta;X)$, we denote by $\mathcal{RG}$ the set of all random feature models (RFs) of the form \eqref{EqDefRF}.

\begin{remark}
	\label{RemImpl}
	Let us briefly explain how the random feature model $G \in \mathcal{RG}$ in \eqref{EqDefRF} can be implemented. For the random initialization of $(\theta_n)_{n=1,...,N}$, we draw some $\omega \in \Omega$ and fix the values of $\theta_1(\omega),...,\theta_N(\omega) \in \Theta$. Thus, by using that $y_1,...,y_N: \Omega \rightarrow \mathbb{K}$ are $\mathcal{F}_\theta/\mathcal{B}(\mathbb{K})$-measurable, the training of $G \in \mathcal{RG}$ consists of finding the optimal $y_1(\omega),...,y_N(\omega) \in \mathbb{K}$ given $g_1(\theta_1(\omega)),...,g_N(\theta_N(\omega)) \in X$. This can be achieved, e.g., by the least squares method.
\end{remark}

In the following, we give an overview of several applications of this general framework, including random trigonometric/Fourier regression and random neural networks.

\subsection{Random trigonometric features}
\label{SecExRandTrigo}

Introduced in \cite{rahimi07,rahimi08b}, random trigonometric regression uses trigonometric functions (i.e.~sines and cosines) in the feature maps. More precisely, for a compact subset $U \subset \mathbb{R}^m$, we consider the real Banach space $(X,\Vert \cdot \Vert_X) := (C^0(U),\Vert \cdot \Vert_{C^0(U)})$ and the parameter space $\Theta := \mathbb{R}^m$, where $(\theta_n)_{n \in \mathbb{N}}: \Omega \rightarrow \mathbb{R}^m$ denotes the i.i.d.~sequence. Then, by choosing $\mathcal{G} := \big\lbrace \mathbb{R}^m \ni \vartheta \mapsto h\big( \vartheta^\top \cdot \big) \in C^0(U): h \in \lbrace \cos, \sin \rbrace \big\rbrace$, we obtain the following random trigonometric feature model.\footnote{The element $h\big( \vartheta^\top \cdot \big)$ denotes the function $U \ni u \mapsto h\big( \vartheta^\top u \big) \in \mathbb{R}$.}

\begin{definition}
	\label{DefRT}
	A \emph{random trigonometric feature model (RTF)} is of the form
	\begin{equation}
		\label{EqDefRT}
		\begin{aligned}
			\Omega \ni \omega \quad \mapsto \quad G(\omega) := \sum_{n=1}^N \bigg( y^{(1)}_n(\omega) \cos\left( \theta_n(\omega)^\top \cdot \right) + y^{(2)}_n(\omega) \sin\left( \theta_n(\omega)^\top \cdot \right) \bigg) \in C^0(U)
		\end{aligned}
	\end{equation}
	with respect to some $N \in \mathbb{N}$ denoting the \emph{number of trigonometric features}, where the \emph{linear readouts} $y^{(1)}_1,...,y^{(1)}_N,y^{(2)}_1,...,y^{(2)}_N: \Omega \rightarrow \mathbb{R}$ are assumed to be $\mathcal{F}_\theta/\mathcal{B}(\mathbb{R})$-measurable.
\end{definition}

We denote by $\mathcal{RT}_{U,1}$ the set of all RTFs of the form \eqref{EqDefRT}. Moreover, we could also consider multidimensional extensions $\mathcal{RT}_{U,d}$ with $\mathbb{R}^d$-valued linear readouts.

\subsection{Random Fourier features}
\label{SecExRandFouier}

Introduced in \cite{rahimi07,rahimi08b,avron17}, random Fourier regression uses the Fourier transform as feature map. For a compact subset $U \subset \mathbb{R}^m$, we consider the complex Banach space $(X,\Vert \cdot \Vert_X) := (C^0(U;\mathbb{C}),\Vert \cdot \Vert_{C^0(U;\mathbb{C})})$ and the parameter space $\Theta := \mathbb{R}^m$, where $(\theta_n)_{n \in \mathbb{N}}: \Omega \rightarrow \mathbb{R}^m$ denotes the i.i.d.~sequence. Moreover, we let $\mathcal{G}$ consist of the single feature map $\mathbb{R}^m \ni \vartheta \mapsto \exp\big( \mathbf{i} \vartheta^\top \cdot \big) \in C^0(U;\mathbb{C})$ to obtain random Fourier features.\footnote{The element $\exp\big( \mathbf{i} \vartheta^\top \cdot \big) \in C^0(U;\mathbb{C})$ denotes the function $U \ni u \mapsto \exp\big( \mathbf{i} \vartheta^\top u \big) \in \mathbb{C}$.}

\begin{definition}
	\label{DefRFourier}
	A \emph{random Fourier feature model (RFF)} is of the form
	\vspace{-0.05cm}
	\begin{equation}
		\label{EqDefRFourier}
		\Omega \ni \omega \quad \mapsto \quad G(\omega) := \sum_{n=1}^N y_n(\omega) \exp\left( \mathbf{i} \theta_n(\omega)^\top \cdot \right) \in C^0(U;\mathbb{C})
	\end{equation}
	with respect to some $N \in \mathbb{N}$ denoting the \emph{number of Fourier features}, where the \emph{linear readouts} $y_1,...,y_N: \Omega \rightarrow \mathbb{C}$ are assumed to be $\mathcal{F}_\theta/\mathcal{B}(\mathbb{C})$-measurable.
\end{definition}

We denote by $\mathcal{RF}_{U,1}$ the set of RFFs of the form \eqref{EqDefRFourier}. Moreover, we could consider vector-valued versions $\mathcal{RF}_{U,d}$ or Banach spaces containing $C^0(U;\mathbb{C})$ (e.g.~certain $L^2$-spaces).

\subsection{Random neural networks}
\label{SecExRandRN}

As third particular instance of random feature learning, we consider random neural networks that are defined as single-hidden-layer feed-forward neural networks whose weights and biases inside the activation function are randomly initialized. Hence, only the linear readout needs to be trained (see \cite{huang06,gonon21}).

To this end, we fix the input and output dimension $m,d \in \mathbb{N}$, the order of differentiability $k \in \mathbb{N}_0$, the domain $U \subseteq \mathbb{R}^m$ (open, if $k \geq 1$), and some $\gamma \in (0,\infty)$. Then, we consider the Banach space $(X,\Vert \cdot \Vert_X) := (\overline{C^k_b(U;\mathbb{R}^d)}^\gamma,\Vert \cdot \Vert_{C^k_{pol,\gamma}(U;\mathbb{R}^d)})$ introduced in Section~\ref{SecNotation} and the parameter space $\Theta := \mathbb{R}^m \times \mathbb{R}$, where $(\theta_n)_{n \in \mathbb{N}} := (a_n,b_n)_{n \in \mathbb{N}}: \Omega \rightarrow \mathbb{R}^m \times \mathbb{R}$ denotes the i.i.d.~sequence of random initializations, which are used for the network weights and biases. Hence, by choosing deterministic (i.e.~fully trained) neural networks as feature maps, i.e.~by setting $\mathcal{G} := \big\lbrace \mathbb{R}^m \times \mathbb{R} \ni (\vartheta_1,\vartheta_2) \mapsto e_i \rho\big( \vartheta_1^\top \cdot - \vartheta_2 \big) \in \overline{C^k_b(U;\mathbb{R}^d)}^\gamma: i = 1,...,d \big\rbrace$, we obtain random neural networks as particular instance of random feature learning.\footnote{The element $y \rho\big( \vartheta_1^\top \cdot - \vartheta_2 \big) \in \overline{C^k_b(U;\mathbb{R}^d)}^\gamma$ denotes the function $U \ni u \mapsto y \rho\big( \vartheta_1^\top u - \vartheta_2 \big) \in \mathbb{R}^d$, where $y \in \mathbb{R}^d$. Moreover, $e_i \in \mathbb{R}^d$ denotes the $i$-th unit vector of $\mathbb{R}^d$.}

\begin{definition}
	\label{DefRN}
	A \emph{random neural network (RN)} is of the form
	\vspace{-0.05cm}
	\begin{equation}
		\label{EqDefRN}
		\Omega \ni \omega \quad \mapsto \quad G(\omega) = \sum_{n=1}^N y_n(\omega) \rho\left( a_n(\omega)^\top \cdot - b_n(\omega) \right) \in \overline{C^k_b(U;\mathbb{R}^d)}^\gamma
	\end{equation}
	with respect to some $N \in \mathbb{N}$ denoting the \emph{number of neurons} and $\rho \in \overline{C^k_b(\mathbb{R})}^\gamma$ representing the \emph{activation function}. Hereby, $a_1,...,a_N: \Omega \rightarrow \mathbb{R}^m$ and $b_1,...,b_N: \Omega \rightarrow \mathbb{R}$ are the \emph{random network weights} and \emph{random network biases}, respectively, and the \emph{linear readouts} $y_1,...,y_N: \Omega \rightarrow \mathbb{R}^d$ are assumed to be $\mathcal{F}_{a,b}/\mathcal{B}(\mathbb{R}^d)$-measurable, with $\mathcal{F}_{a,b} := \sigma\left( \left\lbrace a_n, b_n: n \in \mathbb{N} \right\rbrace \right)$.
\end{definition}

For a given activation function $\rho \in \overline{C^k_b(\mathbb{R})}^\gamma$, we denote by $\mathcal{RN}^\rho_{U,d}$ the set of all random neural networks (RNs) of the form \eqref{EqDefRN}. We refer to Remark~\ref{RemImpl} for the implementation and training of such a random neural network.

\subsection{Further random feature models}

Besides the examples in Section~\ref{SecExRandTrigo}-\ref{SecExRandRN}, the random feature learning model could also be applied, e.g., to kernel regression tasks (see \cite{aronszajn50,rahimi07,bach17b}), Gaussian processes (see \cite{neal96,williams96,rasmussen06}), and operator learning (see \cite{brault16,minh16,nelsen21}). However, in this paper, we focus on the three particular instances in Section~\ref{SecExRandTrigo}-\ref{SecExRandRN}.

\section{Universal approximation}
\label{SecUAT}

In this section, we present our universal approximation results for the Banach space-valued random feature models introduced in Definition~\ref{DefRF}. To this end, we consider for every $r \in [1,\infty)$ the Bochner (sub-)space $L^r(\Omega,\mathcal{F}_\theta,\mathbb{P};X) \subseteq L^r(\Omega,\mathcal{F},\mathbb{P};X)$ of $\mathcal{F}_\theta$-strongly measurable maps $F: \Omega \rightarrow X$ such that $\mathbb{E}\left[ \Vert F \Vert_X^r \right]^{1/r} < \infty$. For more details on Bochner spaces, we refer to \cite[Section~1.2.b]{hytoenen16}.

Moreover, we impose the following condition on the distribution of the i.i.d.~sequence of random initializations $(\theta_n)_{n \in \mathbb{N}}: \Omega \rightarrow \Theta$. This condition ensures that every feature model can be reached with positive probability.

\begin{assumption}[Full support]
	\label{AssCDF}
	Let $(\theta_n)_{n \in \mathbb{N}}: \Omega \rightarrow \Theta$ be an i.i.d.~sequence such that for every $\vartheta \in \Theta$ and $r > 0$ it holds that $\mathbb{P}\left[ \left\lbrace \omega \in \Omega: d_\Theta(\theta_1(\omega),\vartheta) < r \right\rbrace \right] > 0$.
\end{assumption}

In addition, we assume that the feature maps $\mathcal{G} \subseteq C^0(\Theta;X)$ are universal in the sense that the linear span $\linspan_\mathbb{K}(\mathcal{G}(\Theta)) := \big\lbrace \sum_{n=1}^N y_n g_n(\vartheta_n): N \in \mathbb{N}, \, g_1,...,g_N \in \mathcal{G}, \, \vartheta_1,...,\vartheta_N \in \Theta, \, y_1,...,y_N \in \mathbb{K} \big\rbrace$ over a field $\mathbb{K}$ is dense in $X$. Then, by using the law of large numbers for Banach space-valued random variables (see \cite[Theorem~3.3.10]{hytoenen16}), random feature models inherit the universality from the deterministic feature maps. The proof is given in Section~\ref{SecProofsUAT}.

\begin{theorem}[Universal approximation]
	\label{ThmUAT}
	Let Assumption~\ref{AssCDF} hold and let $\mathcal{G} \subseteq C^0(\Theta;X)$ such that $\linspan_{\mathbb{K}}(\mathcal{G}(\Theta))$ is dense in $X$. Moreover, let $F \in L^r(\Omega,\mathcal{F}_\theta,\mathbb{P};X)$ for some $r \in [1,\infty)$. Then, for every $\varepsilon > 0$ there exists some $G \in \mathcal{RG} \cap L^r(\Omega,\mathcal{F}_\theta,\mathbb{P};X)$ such that
	\begin{equation*}
		\Vert F - G \Vert_{L^r(\Omega,\mathcal{F},\mathbb{P};X)} := \mathbb{E}\left[ \Vert F-G \Vert_X^r \right]^\frac{1}{r} < \varepsilon.
	\end{equation*}
\end{theorem}

In particular, every element $x \in X$ can be approximated arbitrarily well by a random feature model $G \in \mathcal{RG}$ with respect to the Bochner norm $\Vert \cdot \Vert_{L^r(\Omega,\mathcal{F},\mathbb{P};X)}$. 

Now, we apply Theorem~\ref{ThmUAT} to random trigonometric/Fourier regression and random neural networks considered in Section~\ref{SecExRandTrigo}-\ref{SecExRandRN}. The corresponding proofs are given in Section~\ref{SecProofsUAT}.

\subsection{Random trigonometric features}

Assume the setting of Section~\ref{SecExRandTrigo} with Banach space $(X,\Vert \cdot \Vert_X) := (C^0(U),\Vert \cdot \Vert_{C^0(U)})$ and parameter space $\Theta := \mathbb{R}^m$, where $U \subseteq \mathbb{R}^m$ is compact. Since $\linspan_\mathbb{R}(\mathcal{G}(\Theta)) = \linspan_\mathbb{R}\big(\big\lbrace U \ni u \mapsto h\big( \vartheta^\top u \big) \in \mathbb{R}: h \in \lbrace \cos, \sin \rbrace, \, \vartheta \in \mathbb{R}^m \big\rbrace\big)$ forms the trigonometric algebra on $U$ which by the Stone-Weierstrass theorem is dense in $C^0(U)$, we obtain the following corollary of Theorem~\ref{ThmUAT}.

\begin{corollary}[Universal approximation]
	\label{CorUATTrigo}
	Let Assumption~\ref{AssCDF} hold and let $F \in L^r(\Omega,\mathcal{F}_\theta,\mathbb{P};C^0(U))$ for some $r \in [1,\infty)$. Then, for every $\varepsilon > 0$ there exists a random trigonometric feature model $G \in \mathcal{RT}_{U,1} \cap L^r(\Omega,\mathcal{F}_\theta,\mathbb{P};C^0(U))$ such that
	\begin{equation*}
		\Vert F - G \Vert_{L^r(\Omega,\mathcal{F},\mathbb{P};C^0(U))} := \mathbb{E}\left[ \Vert F-G \Vert_{C^0(U)}^r \right]^\frac{1}{r} < \varepsilon.
	\end{equation*}
\end{corollary}

\subsection{Random Fourier features}

Assume the setting of Section~\ref{SecExRandTrigo} with Banach space $(X,\Vert \cdot \Vert_X) := (C^0(U;\mathbb{C}),\Vert \cdot \Vert_{C^0(U;\mathbb{C})})$ and parameter space $\Theta := \mathbb{R}^m$, where $U \subseteq \mathbb{R}^m$ is compact. Then, by using that $\linspan_\mathbb{C}(\mathcal{G}(\Theta)) = \linspan_\mathbb{C}\big(\big\lbrace U \ni u \mapsto \exp\big( \mathbf{i} \vartheta^\top u \big) \in \mathbb{C}: \vartheta \in \mathbb{R}^m \big\rbrace\big)$ is dense in $C^0(U;\mathbb{C})$, we obtain the following corollary.

\begin{corollary}[Universal approximation]
	\label{CorUATFourier}
	Let Assumption~\ref{AssCDF} hold and let $F \in L^r(\Omega,\mathcal{F}_\theta,\mathbb{P};C^0(U;\mathbb{C}))$ for some $r \in [1,\infty)$. Then, for every $\varepsilon > 0$ there exists a random Fourier feature model $G \in \mathcal{RF}_{U,1} \cap L^r(\Omega,\mathcal{F}_\theta,\mathbb{P};C^0(U;\mathbb{C}))$ such that
	\begin{equation*}
		\Vert F - G \Vert_{L^r(\Omega,\mathcal{F},\mathbb{P};C^0(U;\mathbb{C}))} := \mathbb{E}\left[ \Vert F-G \Vert_{C^0(U;\mathbb{C})}^r \right]^\frac{1}{r} < \varepsilon.
	\end{equation*}
\end{corollary}

\subsection{Random neural networks}

In view of Theorem~\ref{ThmUAT}, we obtain a universal approximation result for random neural networks from the universal approximation property of deterministic (i.e.~fully trained) neural networks. To this end, we fix the input and output dimension $m,d \in \mathbb{N}$ and consider the following type of function spaces $(X,\Vert \cdot \Vert_X)$.

\begin{assumption}
	\label{AssEmb}
	For $k \in \mathbb{N}_0$, $U \subseteq \mathbb{R}^m$ (open, if $k \geq 1$), and $\gamma \in (0,\infty)$, let \mbox{$(X,\Vert \cdot \Vert_X)$} be a Banach space consisting of functions $f: U \rightarrow \mathbb{R}^d$ such that the restriction map $(C^k_b(\mathbb{R}^m;\mathbb{R}^d),\Vert \cdot \Vert_{C^k_{pol,\gamma}(\mathbb{R}^m;\mathbb{R}^d)}) \ni f \mapsto f\vert_U \in X$ is a continuous dense embedding.
\end{assumption}

\begin{example}[{\cite[Example~2.6]{neufeld24}}]
	The following function spaces $(X,\Vert \cdot \Vert_X)$ satisfy Assumption~\ref{AssEmb}:
	\begin{enumerate}
		\item The $C^k_b$-space $(X,\Vert \cdot \Vert_X) := (C^k_b(U;\mathbb{R}^d),\Vert \cdot \Vert_{C^k_b(U;\mathbb{R}^d)})$ if $U \subseteq \mathbb{R}^m$ is bounded.
		\item The weighted $C^k$-space $(X,\Vert \cdot \Vert_X) := (\overline{C^k_b(U;\mathbb{R}^d)}^\gamma,\Vert \cdot \Vert_{C^k_{pol,\gamma}(U;\mathbb{R}^d)})$.
		\item The $L^p$-space $(X,\Vert \cdot \Vert_X) := (L^p(U,\mathcal{B}(U),\mu;\mathbb{R}^d),\Vert \cdot \Vert_{L^p(U,\mathcal{B}(U),\mu;\mathbb{R}^d)})$ if $p \in [1,\infty)$ and $\mu: \mathcal{B}(U) \rightarrow [0,\infty)$ is a Borel measure with $\int_U (1+\Vert u \Vert)^{\gamma p} \mu(du) < \infty$.
		\item The Sobolev space $(X,\Vert \cdot \Vert_X) := (W^{k,p}(U,\mathcal{L}(U),du;\mathbb{R}^d),\Vert \cdot \Vert_{W^{k,p}(U,\mathcal{L}(U),du;\mathbb{R}^d)})$ if $p \in [1,\infty)$ and $U \subseteq \mathbb{R}^m$ is bounded having the segment property.\footnote{\label{FootnoteSegmentProp}An open subset $U \subseteq \mathbb{R}^m$ is said to have the \emph{segment property} if for every $u \in \partial U := \overline{U} \setminus U$ there exists an open subset $V \subseteq \mathbb{R}^m$ with $u \in V$ and some vector $y \in \mathbb{R}^m \setminus \lbrace 0 \rbrace$ such that for every $z \in \overline{U} \cap V$ and $t \in (0,1)$ it holds that $z + t y \in U$ (see \cite[p.~54]{adams75}).}
		\item The weighted Sobolev space $(X,\Vert \cdot \Vert_X) := (W^{k,p}(U,\mathcal{L}(U),w;\mathbb{R}^d),\Vert \cdot \Vert_{W^{k,p}(U,\mathcal{L}(U),w;\mathbb{R}^d)})$ if $p \in (1,\infty)$, $U \subseteq \mathbb{R}^m$ has the segment property$^{\text{\ref{FootnoteSegmentProp}}}$, the weight $w: U \rightarrow [0,\infty)$ is bounded, $\inf_{u \in B} w(u) > 0$ for all bounded $B \subseteq U$, and $\int_U (1+\Vert u \Vert)^{\gamma p} w(u) du < \infty$.
	\end{enumerate}
	For the precise definition of these function spaces, we refer to Section~\ref{SecNotation}.
\end{example}

Moreover, by using the parameter space $\Theta := \mathbb{R}^m \times \mathbb{R}$ we assume that the random initializations $(\theta_n)_{n \in \mathbb{N}} := (a_n,b_n)_{n \in \mathbb{N}}: \Omega \rightarrow \mathbb{R}^m \times \mathbb{R}$ have full support (see also Assumption~\ref{AssCDF}).

\begin{assumption}[Full support]
	\label{AssCDFRN}
	Let $(a_n,b_n)_{n \in \mathbb{N}}: \Omega \rightarrow \mathbb{R}^m \times \mathbb{R}$ be i.i.d.~such that for every $(a,b) \in \mathbb{R}^m \times \mathbb{R}$ and $r > 0$ we have $\mathbb{P}\left[ \left\lbrace \omega \in \Omega: \Vert (a_1(\omega),b_1(\omega))-(a,b) \Vert < r \right\rbrace \right] > 0$.
\end{assumption}

Then, by using the universal approximation property of deterministic neural networks, i.e. that $\linspan_\mathbb{R}(\mathcal{G}(\Theta)) = \linspan_\mathbb{R}\big(\big\lbrace U \ni u \mapsto e_i \rho\big( \vartheta_1^\top u - \vartheta_2 \big) \in \mathbb{R}^d: (\vartheta_1,\vartheta_2) \in \Theta, \, i = 1,...,d \big\rbrace\big)$ with non-polynomial activation function $\rho \in \overline{C^k_b(\mathbb{R})}^\gamma$ is dense in $X$ (see \cite[Theorem~2.8]{neufeld24}), we obtain a universal approximation result for random neural networks.

\begin{corollary}[Universal approximation]
	\label{CorUATRN}
	Let Assumption~\ref{AssEmb}+\ref{AssCDFRN} hold and let $\rho \in \overline{C^k_b(\mathbb{R})}^\gamma$ be non-polynomial. Moreover, let $F \in L^r(\Omega,\mathcal{F}_{a,b},\mathbb{P};X)$ for some $r \in [1,\infty)$. Then, for every $\varepsilon > 0$ there exists some random neural network $G \in \mathcal{RN}^\rho_{U,d} \cap L^r(\Omega,\mathcal{F}_{a,b},\mathbb{P};X)$ such that
	\begin{equation*}
		\Vert F - G \Vert_{L^r(\Omega,\mathcal{F},\mathbb{P};X)} := \mathbb{E}\left[ \Vert F-G \Vert_X^r \right]^\frac{1}{r} < \varepsilon.
	\end{equation*}
\end{corollary}

In particular, every function $f \in X$ can be approximated arbitrarily well by a random neural network $G \in \mathcal{RG}$ with respect to the Bochner norm $\Vert \cdot \Vert_{L^r(\Omega,\mathcal{F},\mathbb{P};X)}$. 

\begin{remark}
	Corollary~\ref{CorUATRN} extends the universal approximation results in \cite[Theorem~3.1]{rahimi08b}, \cite[Corollary~2.3]{heiss19}, \cite[Theorem~2.4.3]{hart20}, and \cite[Corollary~3]{gonon21} from particular activation functions and $L^2$-spaces (resp.~$C^0$-spaces) to more general non-polynomial activation functions and function spaces over non-compact domains, e.g.,~weighted Sobolev spaces.
\end{remark}

\section{Approximation rates}
\label{SecAR}

In this section, we derive some approximation rates to learn an element $x \in X$ by a random feature model, which relates the number of features needed for a pre-given approximation error. To this end, we assume that the set of feature maps $\mathcal{G} := \lbrace g_1,...,g_e \rbrace$ consists of finitely many maps $g_1,...,g_e: \Theta \rightarrow X$, where $e \in \mathbb{N}$.

In order to derive the approximation rates, we recall the notion of the \emph{type} of a Banach space $(X,\Vert \cdot \Vert_X)$ and refer to \cite[Section~6.2]{albiac06} and \cite[Chapter~9]{ledoux91} for more details.

\begin{definition}
	A Banach space $(X,\Vert \cdot \Vert_X)$ is called of \emph{type $t \in [1,2]$} if there exists a constant $C_X > 0$ such that for every $N \in \mathbb{N}$, $(x_n)_{n=1,...,N} \subseteq X$, and every Rademacher sequence\footnote{A \emph{Rademacher sequence} $(\epsilon_n)_{n=1,...,N}$ on $(\widetilde{\Omega},\widetilde{\mathcal{F}},\widetilde{\mathbb{P}})$ are i.i.d.~random variables with $\widetilde{\mathbb{P}}[\epsilon_1 = \pm 1] = 1/2$.} $(\epsilon_n)_{n=1,...,N}$ on a (possibly different) probability space $(\widetilde{\Omega},\widetilde{\mathcal{F}},\widetilde{\mathbb{P}})$ it holds that
	\begin{equation*}
		\widetilde{\mathbb{E}}\left[ \left\Vert \sum_{n=1}^N \epsilon_n x_n \right\Vert_X^t \right]^\frac{1}{t} \leq C_X \left( \sum_{n=1}^N \Vert x_n \Vert_X^t \right)^\frac{1}{t}.
	\end{equation*}
\end{definition}

\begin{remark}
	Every Banach space $(X,\Vert \cdot \Vert_X)$ is of type $t = 1$ (with $C_X = 1$), every Hilbert space $(X,\Vert \cdot \Vert_X)$ is of type $t = 2$ (with $C_X = 1$), and every $L^p$-space as well as (weighted) $W^{k,p}$-Sobolev space (with $p \in [1,\infty)$ and $k \in \mathbb{N}$) are of type $t = \min(2,p)$ (with constant $C_X$ depending only on $p \in [1,\infty)$), see \cite[Remark~6.2.11]{albiac06} and \cite[Lemma~4.9]{neufeld24}.
\end{remark}

In addition, we define the Barron space $\mathbb{B}^r_{\mathcal{G},\theta}(X)$ as all elements $x \in X$ having a representation as expectation of the random feature maps, which is similar to the Barron spaces introduced in \cite{barron93,rahimi08b,klusowski16,bach17,e20,e22} in the context of neural networks.

\begin{definition}
	\label{DefBarronSpace}
	For $r \in [1,\infty)$, $e \in \mathbb{N}$, and $\mathcal{G} := \lbrace g_1,...,g_e \rbrace$, we define the \emph{Barron space} $\mathbb{B}^r_{\mathcal{G},\theta}(X) \subseteq X$ as the subset of all elements $x \in X$ such that
	\begin{equation}
		\label{EqDefBarronSpace1}
		\Vert x \Vert_{\mathbb{B}^r_{\mathcal{G},\theta}(X)} := \inf_{y: \Theta \rightarrow \mathbb{R}^e}  \mathbb{E}\left[ \left\Vert \sum_{i=1}^e y_i(\theta_1) g_i(\theta_1) \right\Vert_X^r \right]^\frac{1}{r} < \infty,
	\end{equation}
	where the infimum is taken over all $\mathcal{B}(\Theta)/\mathcal{B}(\mathbb{R}^e)$-measurable maps $y := (y_1,...,y_e)^\top: \Theta \rightarrow \mathbb{R}^e$ satisfying $x = \mathbb{E}\left[ \sum_{i=1}^e y_i(\theta_1) g_i(\theta_1) \right]$. Then, we equip the vector space $\mathbb{B}^r_{\mathcal{G},\theta}(X)$ with the \emph{Barron norm} $\Vert \cdot \Vert_{\mathbb{B}^r_{\mathcal{G},\theta}(X)}$ defined in \eqref{EqDefBarronSpace1}.
\end{definition}

\begin{remark}
	Note that $\Vert \cdot \Vert_{\mathbb{B}^r_{\mathcal{G},\theta}(X)}$ satisfies the norm axioms. Moreover, by using H\"older's inequality, we observe that $\mathbb{B}^{r_2}_{\mathcal{G},\theta}(X) \subseteq \mathbb{B}^{r_1}_{\mathcal{G},\theta}(X)$ for all $1 \leq r_1 \leq r_2 < \infty$.
\end{remark}

Now, we are able to derive the following approximation rates which are based on Rademacher averages and the Banach space type. The proof can be found in Section~\ref{SecProofsAR}.

\begin{theorem}[Approximation rates]
	\label{ThmAR}
	Let $(X,\Vert \cdot \Vert_X)$ be a separable Banach space of type $t \in [1,2]$ (with constant $C_X > 0$), let $(\theta_n)_{n \in \mathbb{N}}: \Omega \rightarrow \Theta$ be i.i.d., let $\mathcal{G} := \lbrace g_1,...,g_e \rbrace$ consist of $\mathcal{B}(\Theta)/\mathcal{B}(X)$-measurable maps $g_1,...,g_e: \Theta \rightarrow X$, and let $r \in [1,\infty)$. Then, there exists a constant $C_{r,t} > 0$ (depending only on $r \in [1,\infty)$ and $t \in [1,2]$) such that for every $x \in \mathbb{B}^r_{\mathcal{G},\theta}(X)$ and $N \in \mathbb{N}$ there exists $G_N \in \mathcal{RG} \cap L^r(\Omega,\mathcal{F}_\theta,\mathbb{P};X)$ with $N$ features satisfying
	\begin{equation}
		\label{EqThmAR1}
		\mathbb{E}\left[ \Vert x - G_N \Vert_X^r \right]^\frac{1}{r} \leq C_{r,t} C_X \frac{\Vert x \Vert_{\mathbb{B}^r_{\mathcal{G},\theta}(X)}}{N^{1-\frac{1}{\min(r,t)}}}.
	\end{equation}
\end{theorem}

Hence, Theorem~\ref{ThmAR} relates the approximation error (right-hand side of \eqref{EqThmAR1}) to the number of features $N \in \mathbb{N}$ needed for the random feature model $G_N \in \mathcal{RG} \cap L^r(\Omega,\mathcal{F}_\theta,\mathbb{P};X)$.

In the following, we apply Theorem~\ref{ThmAR} to random trigonometric/Fourier regression and random neural networks considered in Section~\ref{SecExRandTrigo}-\ref{SecExRandRN}. The proofs are given in Section~\ref{SecProofsARCor}.

\subsection{Random trigonometric features}
\label{SecARRandTrigo}

To obtain rates with random trigonometric regression in a weighted Sobolev space $W^{k,p}(U,\mathcal{L}(U),w)$, we fix the following quantities.

\begin{assumption}
	\label{AssSob}
	Let $k \in \mathbb{N}_0$, $p \in (1,\infty)$, $r \in [1,\infty)$, $U \subseteq \mathbb{R}^m$ (open, if $k \geq 1$), and let $w: U \rightarrow [0,\infty)$ be a weight satisfying $w(U) := \int_U w(u) du < \infty$.
\end{assumption}

Moreover, we let $\Theta := \mathbb{R}^m$ be the parameter space and consider the two features maps $\mathbb{R}^m \ni \vartheta \mapsto g_1(\vartheta) := \cos\big( \vartheta^\top \cdot \big) \in W^{k,p}(U,\mathcal{L}(U),w)$ and $\mathbb{R}^m \ni \vartheta \mapsto g_2(\vartheta) := \sin\big( \vartheta^\top \cdot \big) \in W^{k,p}(U,\mathcal{L}(U),w)$. In addition, we impose the following condition on $(\theta_n)_{n \in \mathbb{N}}$.

\begin{assumption}
	\label{AssPDF}
	Let $(\theta_n)_{n \in \mathbb{N}}: \Omega \rightarrow \mathbb{R}^m$ be an i.i.d.~sequence of random variables, each of them having a strictly positive probability density function $p_\theta: \mathbb{R}^m \rightarrow (0,\infty)$.
\end{assumption}

Then, we use the real and imaginary part of the Fourier transform as linear readouts to obtain a representation of a given $f \in L^1(\mathbb{R}^m,\mathcal{L}(\mathbb{R}^m),du)$ in terms of the random sine and cosine features. This then implies the following result as a corollary of Theorem~\ref{ThmAR}.

\begin{corollary}[Approximation rates]
	\label{CorARTrigo}
	Let Assumption~\ref{AssSob}+\ref{AssPDF} hold. Then, there exists a constant $C_{p,r} > 0$ (depending only on $p \in (1,\infty)$ and $r \in [1,\infty)$) such that for every $f \in W^{k,p}(U,\mathcal{L}(U),w) \cap L^1(\mathbb{R}^m,\mathcal{L}(\mathbb{R}^m),du)$ with $C_f := \big( \int_{\mathbb{R}^m} \frac{\vert \widehat{f}(\vartheta) \vert^r \left( 1+\Vert \vartheta \Vert^2 \right)^{kr/2}}{p_\theta(\vartheta)^{r-1}} d\vartheta \big)^{1/r} < \infty$ and every $N \in \mathbb{N}$ there exists some random trigonometric feature model $G_N \in \mathcal{RT}_{U,1} \cap L^r(\Omega,\mathcal{F}_\theta,\mathbb{P};W^{k,p}(U,\mathcal{L}(U),w))$ with $N$ features satisfying
	\begin{equation*}
		\mathbb{E}\left[ \Vert f - G_N \Vert_{W^{k,p}(U,\mathcal{L}(U),w)}^r \right]^\frac{1}{r} \leq C_{p,r} \frac{m^\frac{k}{p} w(U)^\frac{1}{p}}{(2\pi)^m} \frac{C_f}{N^{1-\frac{1}{\min(2,p,r)}}}.
	\end{equation*}
\end{corollary}

\subsection{Random Fourier features}

For approximation rates with random Fourier regression in a weighted Sobolev space as above, we let $\Theta := \mathbb{R}^m$ be the parameter space and consider the single feature map $\mathbb{R}^m \ni \vartheta \mapsto g(\vartheta) := \exp\big( \mathbf{i} \vartheta^\top \cdot \big) \in W^{k,p}(U,\mathcal{L}(U),w;\mathbb{C})$.

\begin{corollary}[Approximation rates]
	\label{CorARFourier}
	Let Assumption~\ref{AssSob}+\ref{AssPDF} hold. Then, there exists a constant $C_{p,r} > 0$ (depending only on $p \in (1,\infty)$ and $r \in [1,\infty)$) such that for every $f \in W^{k,p}(U,\mathcal{L}(U),w;\mathbb{C}) \cap L^1(\mathbb{R}^m,\mathcal{L}(\mathbb{R}^m),du;\mathbb{C})$ with $C_f \!:=\! \big( \int_{\mathbb{R}^m} \frac{\vert \widehat{f}(\vartheta) \vert^r \left( 1+\Vert \vartheta \Vert^2 \right)^{kr/2}}{p_\theta(\vartheta)^{r-1}} d\vartheta \big)^\frac{1}{r} \!<\! \infty$ and for every $N \in \mathbb{N}$ there exists some random Fourier feature model $G_N \in \mathcal{RF}_{U,1} \cap L^r(\Omega,\mathcal{F}_\theta,\mathbb{P};W^{k,p}(U,\mathcal{L}(U),w;\mathbb{C}))$ with $N$ features satisfying
	\begin{equation*}
		\mathbb{E}\left[ \Vert f - G_N \Vert_{W^{k,p}(U,\mathcal{L}(U),w;\mathbb{C})}^r \right]^\frac{1}{r} \leq C_{p,r} \frac{m^\frac{k}{p} w(U)^\frac{1}{p}}{(2\pi)^m} \frac{C_f}{N^{1-\frac{1}{\min(2,p,r)}}}.
	\end{equation*}
\end{corollary}

\begin{remark}
	For $k = 0$ and $p = r = 2$, the approximation rates in Corollary~\ref{CorARTrigo}+\ref{CorARFourier} coincide with the rate $\mathcal{O}\big( 1/\sqrt{N} \big)$ proven in \cite{rahimi08b} (see also \cite{barron93}).
\end{remark}

\subsection{Random neural networks}

Finally, we apply Theorem~\ref{ThmAR} to obtain some approximation rates for learning a given function $f: U \rightarrow \mathbb{R}^d$ by a random neural network in a weighted Sobolev space $W^{k,p}(U,\mathcal{L}(U),w;\mathbb{R}^d)$, where we fix the following quantities.

\begin{assumption}
	\label{AssSobRN}
	Let $k \in \mathbb{N}_0$, $p \in (1,\infty)$, $r \in [1,\infty)$, $U \subseteq \mathbb{R}^m$ (open, if $k \geq 1$), $\gamma \in [0,\infty)$, and let $w: U \rightarrow [0,\infty)$ be a weight.
\end{assumption}

Moreover, we recall that $\Theta := \mathbb{R}^m \times \mathbb{R}$ and impose the following condition on the random initializations $(\theta_n)_{n \in \mathbb{N}} := (a_n,b_n)_{n \in \mathbb{N}}: \Omega \rightarrow \mathbb{R}^m \times \mathbb{R}$ (see also Assumption~\ref{AssPDF}).

\begin{assumption}
	\label{AssPDFRN}
	Let $(a_n,b_n)_{n \in \mathbb{N}}: \Omega \rightarrow \mathbb{R}^m \times \mathbb{R}$ be an i.i.d.~sequence, each of the random variables with strictly positive probability density function $p_{a,b}: \mathbb{R}^m \times \mathbb{R} \rightarrow (0,\infty)$.
\end{assumption}

To obtain rates for random neural networks, we apply the reconstruction formula in \cite[Theorem~5.6]{sonoda17} to express a given function as expectation of a random neuron. To this end, we consider admissible pairs $(\psi,\rho) \in \mathcal{S}_0(\mathbb{R};\mathbb{C}) \times C^k_{pol,\gamma}(\mathbb{R})$ of a ridgelet function $\psi$ and an activation function $\rho$, which is a special case of \cite[Definition~5.1]{sonoda17} (see also \cite{candes98}).

\begin{definition}
	\label{DefAdm}
	A pair $(\psi,\rho) \in \mathcal{S}_0(\mathbb{R};\mathbb{C}) \times C^k_{pol,\gamma}(\mathbb{R})$ is called \emph{$m$-admissible} if the Fourier transform $\widehat{T_\rho} \in \mathcal{S}'(\mathbb{R};\mathbb{C})$ of $\rho \in C^k_{pol,\gamma}(\mathbb{R})$ (in the sense of distribution) coincides\footnote{This means that $\widehat{T_\rho}(g) = \int_{\mathbb{R} \setminus \lbrace 0 \rbrace} f_{\widehat{T_\rho}}(\xi) g(\xi) d\xi$ for all $g \in C^\infty_c(\mathbb{R} \setminus \lbrace 0 \rbrace;\mathbb{C})$.} on $\mathbb{R} \setminus \lbrace 0 \rbrace$ with a locally Lebesgue-integrable function $f_{\widehat{T_\rho}}: \mathbb{R} \setminus \lbrace 0 \rbrace \rightarrow \mathbb{C}$ such that
	\vspace{-0.1cm}
	\begin{equation*}
		C^{(\psi,\rho)}_m := (2\pi)^{m-1} \int_{\mathbb{R} \setminus \lbrace 0 \rbrace} \frac{\overline{\widehat{\psi}(\xi)} f_{\widehat{T_\rho}}(\xi)}{\vert \xi \vert^m} d\xi \in \mathbb{C} \setminus \lbrace 0 \rbrace.
	\end{equation*}
\end{definition}

\begin{remark}
	If $(\psi,\rho) \in \mathcal{S}_0(\mathbb{R};\mathbb{C}) \times C^k_{pol,\gamma}(\mathbb{R})$ is $m$-admissible, then $\rho \in C^k_{pol,\gamma}(\mathbb{R})$ has to be non-polynomial (see \cite[Remark~3.2]{neufeld24}).
\end{remark}

Together with some appropriate $\psi \in \mathcal{S}_0(\mathbb{R};\mathbb{C})$, the most common activation functions such as the sigmoid function and the ReLU function satisfy Definition~\ref{DefAdm}.

\begin{example}[{\cite[Example~3.3]{neufeld24}}]
	\label{ExAdm}
	Let $\psi \in \mathcal{S}_0(\mathbb{R};\mathbb{C})$ with non-negative $\widehat{\psi} \in C^\infty_c(\mathbb{R})$ and $\supp(\widehat{\psi}) = [\zeta_1,\zeta_2]$ for some $0 < \zeta_1 < \zeta_2 < \infty$. Then, for every $m \in \mathbb{N}$ and the following activation functions $\rho \in C^k_{pol,\gamma}(\mathbb{R})$ the pair $(\psi,\rho) \in \mathcal{S}_0(\mathbb{R};\mathbb{C}) \times C^k_{pol,\gamma}(\mathbb{R})$ is $m$-admissible:
	\begin{enumerate}
		\item The sigmoid function $\rho(s) := \frac{1}{1+\exp(-s)}$ if $k \in \mathbb{N}_0$ and $\gamma \geq 0$.
		\item The tangens hyperbolicus $\rho(s) := \tanh(s)$ if $k \in \mathbb{N}_0$ and $\gamma \geq 0$.
		\item The softplus function $\rho(s) := \ln(1+\exp(s))$ if $k \in \mathbb{N}_0$ and $\gamma \geq 1$.
		\item The ReLU function $\rho(s) = \max(s,0)$ if $k = 0$ and $\gamma \geq 1$.
	\end{enumerate}
	Moreover, there exists a constant $C_{\psi,\rho} > 0$ (being independent of $m \in \mathbb{N}$) such that for every $m \in \mathbb{N}$ it holds that $\big\vert C^{(\psi,\rho)}_m \big\vert \geq C_{\psi,\rho} (2\pi/\zeta_2)^m$.
\end{example}

In addition, we follow \cite{candes98,sonoda17} and define for every $\psi \in \mathcal{S}_0(\mathbb{R};\mathbb{C})$ the (multi-dimensional) \emph{Ridgelet transform} of any function $f \in L^1(\mathbb{R}^m,\mathcal{L}(\mathbb{R}^m),du;\mathbb{R}^d)$ as
\begin{equation}
	\label{EqDefRidgelet}
	\mathbb{R}^m \times \mathbb{R} \ni (a,b) \quad \mapsto \quad (\mathfrak{R}_\psi f)(a,b) := \int_{\mathbb{R}^m} \psi\left( a^\top u - b \right) f(u) \Vert a \Vert du \in \mathbb{C}^d.
\end{equation}
Then, we can apply the reconstruction formula in \cite[Theorem~5.6]{sonoda17} to obtain a representation of any sufficiently integrable function as expectation of a random neuron.

\begin{proposition}[Reconstruction, {\cite[Proposition~3.3]{neufeld24}}]
	\label{PropIntRepr}
	Let Assumption~\ref{AssSobRN}+\ref{AssPDFRN} hold, let $(\psi,\rho) \in \mathcal{S}_0(\mathbb{R};\mathbb{C}) \times C^k_{pol,\gamma}(\mathbb{R})$ be $m$-admissible, and let $f \in L^1(\mathbb{R}^m,\mathcal{L}(\mathbb{R}^m),du;\mathbb{R}^d)$ with $\widehat{f} \in L^1(\mathbb{R}^m,\mathcal{L}(\mathbb{R}^m),du;\mathbb{C}^d)$. Then, for a.e.~$u \in \mathbb{R}^m$, it holds that
	\begin{equation*}
		\mathbb{E}\left[ \frac{(\mathfrak{R}_\psi f)(a_1,b_1)}{p_{a,b}(a_1,b_1)} \rho\left( a_1^\top u - b_1 \right) \right] = \int_{\mathbb{R}^m} \int_{\mathbb{R}} (\mathfrak{R}_\psi f)(a,b) \rho\left( a^\top u - b \right) db da = C^{(\psi,\rho)}_m f(u).
	\end{equation*}
\end{proposition}

\begin{remark}
	\label{RemIntRepr}
	Recall that the set $\mathcal{G}$ consists of the feature maps $\mathbb{R}^m \times \mathbb{R} \ni (a,b) \mapsto g_i(a,b) := e_i \rho\big( a^\top \cdot - b \big) \in W^{k,p}(U,\mathcal{L}(U),w;\mathbb{R}^d)$, $i = 1,...,d$. Hence, for every function $f \in W^{k,p}(U,\mathcal{L}(U),w;\mathbb{R}^d)$ satisfying the conditions of Proposition~\ref{PropIntRepr}, we choose the linear readout $\mathbb{R}^m \times \mathbb{R} \ni (a,b) \mapsto y(a,b) := (y_i(a,b))_{i=1,...,d}^\top := \re\left( \frac{(\mathfrak{R}_\psi f)(a,b)}{C^{(\psi,\rho)}_m p_{a,b}(a,b)} \right) \in \mathbb{R}^d$ to obtain $\mathbb{E}\big[ \sum_{i=1}^d y_i(a_1,b_1) g_i(a_1,b_1) \big] = f$ a.e.~on $U$, showing that $f \in \mathbb{B}^r_{\mathcal{G},\theta}(W^{k,p}(U,\mathcal{L}(U),w;\mathbb{R}^d))$.
\end{remark}

In order to extend the reconstruction also to other more general functions, we adapt the Barron spaces in Definition~\ref{DefBarronSpace} to this setting with ridgelet transform introduced in \eqref{EqDefRidgelet}.

\begin{definition}
	\label{DefBarronRidgeletSpace}
	Let Assumption~\ref{AssSobRN}+\ref{AssPDFRN} hold and let $\psi \in \mathcal{S}(\mathbb{R};\mathbb{C})$. Then, we define the Barron-ridgelet space $\widetilde{\mathbb{B}}^{k,r,\gamma}_{\psi,a,b}(U;\mathbb{R}^d)$ as vector space of $\mathcal{L}(U)/\mathcal{B}(\mathbb{R}^d)$-measurable functions $f: U \rightarrow \mathbb{R}^d$ such that
	\begin{equation}
		\label{EqDefBarronRidgeletSpace}
		\Vert f \Vert_{\widetilde{\mathbb{B}}^{k,r,\gamma}_{\psi,a,b}(U;\mathbb{R}^d)} := \inf_h \mathbb{E}\left[ \left\Vert \frac{\left( 1 + \Vert a_1 \Vert^2 \right)^\frac{\gamma+k}{2} \left( 1 + \vert b_1 \vert^2 \right)^\frac{\gamma}{2}}{p_{a,b}(a_1,b_1)} (\mathfrak{R}_\psi h)(a_1,b_1) \right\Vert^r \right]^\frac{1}{r} < \infty,
	\end{equation}
	where the infimum is taken over all functions $h \in L^1(\mathbb{R}^m,\mathcal{L}(\mathbb{R}^m),du;\mathbb{R}^d)$ satisfying $\widehat{h} \in L^1(\mathbb{R}^m,\mathcal{L}(\mathbb{R}^m),du;\mathbb{C}^d)$ and $h = f$ a.e.~on $U$. 
\end{definition}

\begin{remark}
	Note that \eqref{EqDefBarronRidgeletSpace} implicitly imposes some integrability on the random initializations $(\theta_n)_{n \in \mathbb{N}} := (a_n,b_n)_{n \in \mathbb{N}}$.
\end{remark}

In the following lemma, we show for the set of feature maps $\mathcal{G}$ defined in Remark~\ref{RemIntRepr} that $\widetilde{\mathbb{B}}^{k,r,\gamma}_{\psi,a,b}(U;\mathbb{R}^d)$ is a subset of $\mathbb{B}^r_{\mathcal{G},\theta}(W^{k,p}(U,\mathcal{L}(U),w;\mathbb{R}^d))$ introduced in Definition~\ref{DefBarronSpace}.

\begin{lemma}
	\label{LemmaIneq}
	Let Assumption~\ref{AssSobRN}+\ref{AssPDFRN} hold, let $\mathcal{G}$ be as in Remark~\ref{RemIntRepr}, let $(\psi,\rho) \in \mathcal{S}_0(\mathbb{R};\mathbb{C}) \times C^k_{pol,\gamma}(\mathbb{R})$, and define $C^{(\gamma,p)}_{U,w} > 0$ as in \eqref{EqCorARRN1}. Then, for every $f \in \widetilde{\mathbb{B}}^{k,r,\gamma}_{\psi,a,b}(U;\mathbb{R}^d)$ it holds that $\Vert f \Vert_{\mathbb{B}^r_{\mathcal{G},\theta}(W^{k,p}(U,\mathcal{L}(U),w;\mathbb{R}^d))} \leq 2^{3+\frac{1}{p}} \Vert \rho \Vert_{C^k_{pol,\gamma}(\mathbb{R})} \frac{C^{(\gamma,p)}_{U,w} m^{k/p}}{\vert C^{(\psi,\rho)}_m \vert} \Vert f \Vert_{\widetilde{\mathbb{B}}^{k,r,\gamma}_{\psi,a,b}(U;\mathbb{R}^d)}$. In particular, we have $\widetilde{\mathbb{B}}^{k,r,\gamma}_{\psi,a,b}(U;\mathbb{R}^d) \subseteq \mathbb{B}^r_{\mathcal{G},\theta}(W^{k,p}(U,\mathcal{L}(U),w;\mathbb{R}^d))$.
\end{lemma}

Now, as a corollary of Theorem~\ref{ThmAR}, we obtain the following rates to approximate any given function $f \in W^{k,p}(U,\mathcal{L}(U),w;\mathbb{R}^d) \cap \widetilde{\mathbb{B}}^{k,r,\gamma}_{\psi,a,b}(U;\mathbb{R}^d)$ by a random neural network.

\begin{corollary}[Approximation rates]
	\label{CorARRN}
	Let Assumption~\ref{AssSobRN}+\ref{AssPDFRN} hold such that
	\begin{equation}
		\label{EqCorARRN1}
		C^{(\gamma,p)}_{U,w} := \left( \int_U (1+\Vert u \Vert)^{\gamma p} w(u) du \right)^\frac{1}{p} < \infty.
	\end{equation}
	Moreover, let $(\psi,\rho) \in \mathcal{S}_0(\mathbb{R};\mathbb{C}) \times C^k_{pol,\gamma}(\mathbb{R})$ be $m$-admissible. Then, there exists a constant $C_{p,r} > 0$ (depending only on $p \in (1,\infty)$ and $r \in [1,\infty)$) such that for every $f \in W^{k,p}(U,\mathcal{L}(U),w;\mathbb{R}^d) \cap \widetilde{\mathbb{B}}^{k,r,\gamma}_{\psi,a,b}(U;\mathbb{R}^d)$ there exists a random neural network $G_N \in \mathcal{RN}^\rho_{U,d} \cap L^r(\Omega,\mathcal{F}_{a,b},\mathbb{P};W^{k,p}(U,\mathcal{L}(U),w;\mathbb{R}^d))$ with $N$ neurons satisfying
	\begin{equation}
		\label{EqCorARRN2}
		\mathbb{E}\left[ \Vert f - G_N \Vert_{W^{k,p}(U,\mathcal{L}(U),w;\mathbb{R}^d)}^r \right]^\frac{1}{r} \leq C_{p,r} \Vert \rho \Vert_{C^k_{pol,\gamma}(\mathbb{R})} \frac{C^{(\gamma,p)}_{U,w} m^\frac{k}{p}}{\left\vert C^{(\psi,\rho)}_m \right\vert} \frac{\Vert f \Vert_{\widetilde{\mathbb{B}}^{k,r,\gamma}_{\psi,a,b}(U;\mathbb{R}^d)}}{N^{1-\frac{1}{\min(2,p,r)}}}.
	\end{equation}
\end{corollary}

\begin{remark}
	Theorem~\ref{CorARRN} extends the approximation rates for random neural networks in \cite[Section~4.2]{gonon20} and \cite[Theorem~1]{gonon21} from ReLU activation functions and $L^2$-spaces (resp.~$C^0$-spaces) to more general activation functions and weighted Sobolev spaces, where the approximation of the weak derivatives is now included. Moreover, these rates are analogous to the ones for \emph{deterministic} neural networks in \cite{barron93,darken93,mhaskar95,candes98,kurkova12,bgkp17,siegel20,neufeld24}.
\end{remark}

Next, we give sufficient conditions for a function to be in a Barron-ridgelet space (see Definition~\ref{DefBarronRidgeletSpace}). For example, the solution of the heat equation (with suitable initial condition) at any fixed time belongs to such a space (see Corollary~\ref{CorHeat}.\ref{CorHeat2}). To this end, we assume that the random initializations $(a_n)_{n \in \mathbb{N}} \sim p_a$ admit a probability density function $p_a: \mathbb{R}^m \rightarrow (0,\infty)$ and that they are independent of $(b_n)_{n \in \mathbb{N}} \sim \mathbf{t}_1$, which follow a Student's $t$-distribution.\footnote{\label{FootnoteTDistribution}For $m \in \mathbb{N}$, $\mathbf{t}_m$ denotes the Student's $t$-distribution with probability density function $\mathbb{R}^m \ni a \mapsto \theta_A(a) = \frac{\Gamma((m+1)/2)}{\pi^{(m+1)/2}} \big( 1 + \Vert a \Vert^2 \big)^{-(m+1)/2} \in (0,\infty)$, where $\Gamma$ is the Gamma function (see \cite[Section~6.1]{abramowitz70}).}

\begin{proposition}
	\label{PropConst}
	Let Assumption~\ref{AssSobRN}+\ref{AssPDFRN} hold, let $(a_n,b_n)_{n \in \mathbb{N}} \sim p_a \otimes \mathbf{t}_1$ be i.i.d., and let $\psi \in \mathcal{S}_0(\mathbb{R};\mathbb{C})$ such that $\zeta_1 := \inf\big\lbrace \vert \zeta \vert : \zeta \in \supp(\widehat{\psi}) \big\rbrace > 0$. Then, there exists a constant $C_1 > 0$ (being independent of $m,d \in \mathbb{N}$) such that for every $f \in L^1(\mathbb{R}^m,\mathcal{L}(\mathbb{R}^m),du;\mathbb{R}^d)$ with $(\lceil\gamma\rceil+2)$-times differentiable Fourier transform it holds that
	\begin{equation}
		\label{EqPropConst1}
		\Vert f \Vert_{\widetilde{\mathbb{B}}^{k,r,\gamma}_{\psi,a,b}(U;\mathbb{R}^d)} \leq \frac{C_1}{\zeta_1^\frac{m}{r}} \sup_{\zeta \in \supp(\widehat{\psi})} \sum_{\beta \in \mathbb{N}^m_{0,\lceil\gamma\rceil+2}} \left( \int_{\mathbb{R}^m} \Vert \partial_\beta \widehat{f}(\xi) \Vert^r \frac{\left( 1 + \Vert \xi/\zeta \Vert^2 \right)^\frac{(2\lceil\gamma\rceil+k+3)r}{2}}{\theta_A(\xi/\zeta)^{r-1}} d\xi \right)^\frac{1}{r}.
	\end{equation}
	In particular, if $r = 2$ and$^{\text{\ref{FootnoteTDistribution}}}$ $(a_n,b_n)_{n \in \mathbb{N}} \sim \mathbf{t}_m \otimes \mathbf{t}_1$ i.i.d., then it holds for every $f \in L^1(\mathbb{R}^m,\mathcal{L}(\mathbb{R}^m),du;\mathbb{R}^d)$ with $(\lceil\gamma\rceil+2)$-times differentiable Fourier transform that
	\begin{equation}
		\label{EqPropConst2}
		\Vert f \Vert_{\widetilde{\mathbb{B}}^{k,2,\gamma}_{\psi,a,b}(U;\mathbb{R}^d)} \leq \frac{C_1}{\zeta_1^\frac{m}{2}} \frac{\pi^\frac{m+1}{4}}{\Gamma\left( \frac{m+1}{2} \right)^\frac{1}{2}} \sum_{\beta \in \mathbb{N}^m_{0,\lceil\gamma\rceil+2}} \left( \int_{\mathbb{R}^m} \big\Vert \partial_\beta \widehat{f}(\xi) \big\Vert^2 \left( 1 + \Vert \xi/\zeta_1 \Vert^2 \right)^{2\lceil\gamma\rceil+k+\frac{m+5}{2}} d\xi \right)^\frac{1}{2}.
	\end{equation}
	Hence, if the right-hand side of \eqref{EqPropConst1} or \eqref{EqPropConst2} is finite, we obtain that $f \in \widetilde{\mathbb{B}}^{k,r,\gamma}_{\psi,a,b}(U;\mathbb{R}^d)$.
\end{proposition}

Thus, for $r = 2$ and $(a_n,b_n)_{n \in \mathbb{N}} \sim \mathbf{t}_m \otimes \mathbf{t}_1$ i.i.d., we can insert \eqref{EqPropConst2} into the right-hand side of \eqref{EqCorARRN2} to conclude that the approximation rate for random neural networks is the same as the approximation rate for \emph{deterministic} neural networks proven in \cite[Theorem~3.6]{neufeld24}. 

Moreover, the following estimate holds true for the constant $C^{(\gamma,p)}_{U,w} > 0$ appearing in \eqref{EqCorARRN2}, while a lower bound for the constant $\big\vert C^{(\psi,\rho)}_m \big\vert > 0$ is given below the list of Example~\ref{ExAdm}.

\begin{lemma}[{\cite[Lemma~3.9]{neufeld24}}]
	\label{LemmaWeight}
	Let Assumption~\ref{AssSobRN} hold with the weight $U \ni u := (u_1,...,u_m)^\top \mapsto w(u) := \prod_{l=1}^m w_0(u_l) \in [0,\infty)$, where $w_0: \mathbb{R} \rightarrow [0,\infty)$ satisfies $\int_{\mathbb{R}} w_0(s) ds = 1$ and $C^{(\gamma,p)}_{\mathbb{R},w_0} := \big( \int_{\mathbb{R}} (1+\vert s \vert)^{\gamma p} w_0(s) ds \big)^{1/p} < \infty$. Then, $C^{(\gamma,p)}_{U,w} \leq C^{(\gamma,p)}_{\mathbb{R},w_0} m^{\gamma+1/p}$.
\end{lemma}

In addition, we analyze the situation when the approximation of a function by random neural networks overcomes the curse of dimensionality in the sense that the computational costs (measured as the number of neurons $N \in \mathbb{N}$) grow polynomially in both the input/output dimensions $m,d \in \mathbb{N}$ and the reciprocal of a pre-specified tolerated approximation error.

\begin{proposition}
	\label{PropCOD}
	Let Assumption~\ref{AssSobRN} hold with $r = 2$ and a weight as in Lemma~\ref{LemmaWeight}, let $(a_n,b_n)_{n \in \mathbb{N}} \sim \mathbf{t}_m \otimes \mathbf{t}_1$ be i.i.d., and let $(\psi,\rho) \in \mathcal{S}_0(\mathbb{R};\mathbb{C}) \times C^k_{pol,\gamma}(\mathbb{R})$ be a pair as in Example~\ref{ExAdm} (with $0 < \zeta_1 < \zeta_2 < \infty$). In addition, let $f \in W^{k,p}(U,\mathcal{L}(U),w;\mathbb{R}^d)$ satisfy the conditions of Proposition~\ref{PropConst} such that the right-hand side of \eqref{EqPropConst2} satisfies $\mathcal{O}\left( m^s (2/\zeta_2)^m (m+1)^{m/2} \right)$ for some $s \in \mathbb{N}_0$. Then, there exist some constants $C_2,C_3 > 0$ such that for every $m,d \in \mathbb{N}$ and every $\varepsilon > 0$ there exists a random neural network $G_N \in \mathcal{RN}^\rho_{U,d}$ with $N = \big\lceil C_2 m^{C_3} \varepsilon^{-\frac{\min(2,p,r)}{\min(2,p,r)-1}} \big\rceil$ neurons satisfying
	\begin{equation*}
		\mathbb{E}\left[ \Vert f - G_N \Vert_{W^{k,p}(U,\mathcal{L}(U),w;\mathbb{R}^d)}^2 \right]^\frac{1}{2} \leq \varepsilon.
	\end{equation*}
	Hence, in this case, random neural networks overcome the curse of dimensionality.
\end{proposition}

\section{Least squares and generalization error}
\label{SecLSGE}

In this section, we use the least squares method to learn a given function by a random feature model in the Sobolev space $W^{k,2}(U,\mathcal{L}(U),w;\mathbb{R}^d)$, where we fix some $k \in \mathbb{N}_0$, $U \subseteq \mathbb{R}^m$ (open, if $k \geq 1$), and a weight $w: U \rightarrow [0,\infty)$ that is \emph{normalized}, i.e.~$\int_U w(u) du = 1$. To this end, we assume that the set of feature maps $\mathcal{G} := \lbrace g_1,...,g_e \rbrace$ consists of finitely many $\mathcal{B}(\Theta)/\mathcal{B}(W^{k,2}(U,\mathcal{L}(U),w;\mathbb{R}^d))$-measurable maps $g_1,...,g_e: \Theta \rightarrow W^{k,2}(U,\mathcal{L}(U),w;\mathbb{R}^d)$, where $e \in \mathbb{N}$.

Moreover, we fix an i.i.d.~sequence of $U$-valued random variables $(V_j)_{j \in \mathbb{N}} \sim w$ as training data, which are independent of the random initializations $(\theta_n)_{n \in \mathbb{N}}$. From this, we define the $\sigma$-algebra $\mathcal{F}_{\theta,V} := \sigma(\lbrace \theta_n, V_n: n \in \mathbb{N} \rbrace)$ satisfying $\mathcal{F}_\theta \subseteq \mathcal{F}_{\theta,V} \subseteq \mathcal{F}$. 

In addition, for every fixed $N \in \mathbb{N}$, we denote by $\mathcal{Y}_N$ the vector space of all $\mathbb{R}^{e \times N}$-valued random variables $y := (y_{i,n})_{i=1,...,e}^{n=1,...,N}$, which are $\mathcal{F}_{\theta,V}/\mathcal{B}(\mathbb{R}^{e \times N})$-measurable. Then, for every $y \in \mathcal{Y}_N$, we define the corresponding random feature model as
\begin{equation}
	\label{EqDefRFMY}
	\Omega \ni \omega \quad \mapsto \quad G_N^y(\omega) := \sum_{n=1}^N \sum_{l=1}^e y_{l,n}(\omega) g_l(\theta_n(\omega)) \in W^{k,2}(U,\mathcal{L}(U),w;\mathbb{R}^d).
\end{equation}
Note that \eqref{EqDefRFMY} slightly differs from Definition~\ref{DefRN} as the linear readout $y \in \mathcal{Y}_N$ is now measurable with respect to $\mathcal{F}_{\theta,V}$ (instead of $\mathcal{F}_\theta$) and can therefore only be trained after the training data $(V_j)_{j \in \mathbb{N}}$ has been drawn. Moreover, we denote by $\mathcal{RG}^V$ the set of all random feature models of the form \eqref{EqDefRFMY}.

For some fixed $J \in \mathbb{N}$, we now approximate a given $k$-times weakly differentiable function $f: U \rightarrow \mathbb{R}^d$ by the least squares method on the training data $(V_j)_{j=1,...,J}$. To this end, we aim for the random feature model $G_N^{y^{(J)}} \in \mathcal{RG}^V$ with linear readout $y^{(J)} := \big( y^{(J)}_{i,n} \big)_{i=1,...,e}^{n=1,...,N} \in \mathcal{Y}_N$ that minimizes the empirical (weighted) mean squared error (MSE), i.e.~we set
\begin{equation}
	\label{EqEmpMSE}
	y^{(J)}(\omega) = \argmin_{y \in \mathcal{Y}_N} \left( \frac{1}{J} \sum_{j=1}^J \sum_{\alpha \in \mathbb{N}^m_{0,k}} c_{\alpha}^2 \left\Vert \partial_\alpha f(V_j(\omega)) - \partial_\alpha G^y_N(\omega)(V_j(\omega)) \right\Vert^2 \right)
\end{equation}
for all $\omega \in \Omega$. Hereby, the constants $\mathbf{c} := (c_\alpha)_{\alpha \in \mathbb{N}^m_{0,k}} \subset (0,\infty)$ control the contribution of the derivatives, e.g.,~$c_\alpha := m^{-\vert \alpha \vert}$, $\alpha \in \mathbb{N}^m_{0,k}$, means equal contribution of each order. Moreover, for $\mathbf{c} := (c_\alpha)_{\alpha \in \mathbb{N}^m_{0,k}}$, we define the number $\kappa(\mathbf{c}) := \frac{\max_{\alpha \in \mathbb{N}^m_{0,k}} c_\alpha}{\min_{\alpha \in \mathbb{N}^m_{0,k}} c_\alpha}$.

\begin{algorithm}[!ht]
	\DontPrintSemicolon
	\begin{small}
		\KwInput{$J,N \in \mathbb{N}$ and a $k$-times weakly differentiable function $f = (f_1,...,f_d)^\top: U \rightarrow \mathbb{R}^d$.}
		\KwOutput{$G_N^{y^{(J)}} \in \mathcal{RG}^V$ with linear readout $y^{(J)} := \big( y^{(J)}_{l,n} \big)_{l=1,...,e}^{n=1,...,N} \in \mathcal{Y}_N$ solving~\eqref{EqEmpMSE}.}
		
		Generate i.i.d.~random variables (RVs) $(\theta_n)_{n=1,...,N}$ satisfying Assumption~\ref{AssPDF}.\label{Alg1}
		
		Generate i.i.d.~random variables (RVs) $(V_j)_{j = 1,...,J} \sim w$ independent of $(\theta_n)_{n=1,...,N}$.\label{Alg2}
		
		Define the $\mathbb{R}^{(J \cdot \vert \mathbb{N}^m_{0,k} \vert \cdot d) \times (e \cdot N)}$-valued RV $\mathbf{G} = (\mathbf{G}_{(j,\alpha,i),(l,n)})_{(j,\alpha,i) \in \lbrace 1,...,J \rbrace \times \mathbb{N}^m_{0,k} \times \lbrace 1,...,d \rbrace}^{(l,n) \in \lbrace 1,...,e \rbrace \times \lbrace 1,...,N \rbrace}$ with $\mathbf{G}_{(j,\alpha,i),(l,n)} := c_\alpha \partial_\alpha g_{l,i}(\theta_n)(V_j)$ for $(j,\alpha,i) \in \lbrace 1,...,J \rbrace \times \mathbb{N}^m_{0,k} \times \lbrace 1,...,d \rbrace$ and $(l,n) \in \lbrace 1,...,e \rbrace \times \lbrace 1,...,N \rbrace$, where $g_l(\vartheta)(u) := (g_{l,i}(\vartheta)(u))_{i=1,...,d}^\top \in \mathbb{R}^d$.\label{Alg3}
		
		Define the $\mathbb{R}^{J \cdot \vert \mathbb{N}^m_{0,k} \vert \cdot d}$-valued RV $Z := (c_\alpha \partial_\alpha f_i(V_j))_{(j,\alpha,i) \in \lbrace 1,...,J \rbrace \times \mathbb{N}^m_{0,k} \times \lbrace 1,...,d \rbrace}^\top$.\label{Alg4}
		
		Solve the least squares problem $\mathbf{G}^\top \mathbf{G} \vec{y}^{(J)} = \mathbf{G}^\top Z$ for $\vec{y}^{(J)} := \big( y^{(J)}_{(l,n)} \big)_{(l,n) \in \lbrace 1,...,e \rbrace \times \lbrace 1,...,N \rbrace}^\top$ via Cholesky decomposition and forward/backward substitution (see \cite[Section~2.2.2]{bjoerck96}).\label{Alg5}
		
		Return $\Omega \ni \omega \mapsto G_N^{y^{(J)}}(\omega) := \sum_{n=1}^N \sum_{l=1}^e y^{(J)}_{l,n}(\omega) g_l(\theta_n(\omega)) \in W^{k,2}(U,\mathcal{L}(U),w;\mathbb{R}^d)$.\label{Alg6}
	\end{small}
	\caption{Least squares method to learn a random feature model}
	\label{Alg}
\end{algorithm}

To analyze the complexity of Algorithm~\ref{Alg} in the following result, we count every elementary operation, every function evaluation of $g_1,...,g_e \in \mathcal{G}$, and each generation of a random variable as one unit. Then, we show the following result whose proof is given in Section~\ref{SecProofsAlg}.

\begin{proposition}
	\label{PropAlg}
	For $e \in \mathbb{N}$, let $\mathcal{G} := \lbrace g_1,...,g_e \rbrace$ consist of maps $g_1,...,g_e: \Theta \rightarrow W^{k,2}(U,\mathcal{L}(U),w;\mathbb{R}^d)$ that are $\mathcal{B}(\Theta)/\mathcal{B}(W^{k,2}(U,\mathcal{L}(U),w;\mathbb{R}^d))$-measurable, let Assumption~\ref{AssPDF} hold, let $J,N \in \mathbb{N}$, $(c_\alpha)_{\alpha \in \mathbb{N}^m_{0,k}} \subset (0,\infty)$, and $f: U \rightarrow \mathbb{R}^d$ be $k$-times weakly differentiable. Then, Algorithm~\ref{Alg} terminates and is correct, i.e.~returns $G_N^{y^{(J)}} \in \mathcal{RG}^V$ with $y^{(J)} \in \mathcal{Y}_N$ solving~\eqref{EqEmpMSE}. Moreover, the complexity of Algorithm~\ref{Alg} is of order $\mathcal{O}\big( J m^k d (eN)^2 + (eN)^3 \big)$.
\end{proposition}

For fixed $k \in \mathbb{N}_0$, Proposition~\ref{PropAlg} shows that the computational costs for learning a given $k$-times weakly differentiable function by a random feature model including the derivatives up to order $k$ scales polynomially in $J,N,m,d \in \mathbb{N}$.

Now, we bound the generalization error for learning a function by the random feature model $G^{y^{(J)}}_N \in \mathcal{RG}^V$ obtained from Algorithm~\ref{Alg}. Since the linear readout $y^{(J)} \in \mathcal{Y}_N$ minimizes the empirical MSE in \eqref{EqEmpMSE}, the random feature model $G^{y^{(J)}}_N \in \mathcal{RG}^V$ is the best choice on the training data $(V_j)_{j=1,...,J}$. In order to bound the error beyond the training data, we combine the approximation rate in Theorem~\ref{ThmAR} with a result on non-parametric function regression (see \cite[Theorem~11.3]{gyoerfi02}). Hereby, we introduce the truncation $\mathbb{R}^d \ni z := (z_1,...,z_d)^\top \mapsto T_L(z) := \left( \max(\min(z_i,L),-L) \right)_{i=1,...,d}^\top \in \mathbb{R}^d$. The proof can be found in Section~\ref{SecProofsGenErr}.

\begin{theorem}[Generalization error]
	\label{ThmGenErr}
	For $e \in \mathbb{N}$, let $\mathcal{G} := \lbrace g_1,...,g_e \rbrace$ consist of $\mathcal{B}(\Theta)/$ $\mathcal{B}(W^{k,2}(U,\mathcal{L}(U),w;\mathbb{R}^d))$-measurable maps $g_1,...,g_e: \Theta \rightarrow W^{k,2}(U,\mathcal{L}(U),w;\mathbb{R}^d)$ and let Assumption~\ref{AssPDF} hold. Then, there exists a constant $C_4 > 0$ (being independent of $m,d \in \mathbb{N}$) such that for every $J,N \in \mathbb{N}$, $L > 0$, and $f := (f_1,...,f_d)^\top \in \mathbb{B}^2_{\mathcal{G},\theta}(W^{k,2}(U,\mathcal{L}(U),w;\mathbb{R}^d))$, Algorithm~\ref{Alg} returns a random feature model $G^{y^{(J)}}_N \in \mathcal{RG}^V$ with $N$ features being a strongly $(\mathbb{P},\mathcal{F}_{\theta,V})$-measurable map $G^{y^{(J)}}_N: \Omega \rightarrow W^{k,2}(U,\mathcal{L}(U),w;\mathbb{R}^d)$ such that
	\begin{equation}
		\label{EqThmGenErr1}
		\begin{aligned}
			& \mathbb{E}\left[ \sum_{\alpha \in \mathbb{N}^m_{0,k}} \int_U \left\Vert T_L\left(\partial_\alpha f(u) - \partial_\alpha G^{y^{(J)}}_N(\cdot)(u)\right) \right\Vert^2 w(u) du \right]^\frac{1}{2} \\
			& \quad\quad \leq C_4 L m^\frac{k}{2} d^\frac{1}{2} \sqrt{\frac{(\ln(J)+1) N}{J}} + C_4 \kappa(\mathbf{c}) \frac{\Vert f \Vert_{\mathbb{B}^2_{\mathcal{G},\theta}(W^{k,2}(U,\mathcal{L}(U),w;\mathbb{R}^d))}}{\sqrt{N}}.
		\end{aligned}
	\end{equation}
	Moreover, if $\vert \partial_\alpha f_i(u) \vert \leq L$ for all $\alpha \in \mathbb{N}^m_{0,k}$, $i = 1,...,d$, and $u \in U$, then 
	\begin{equation}
		\label{EqThmGenErr2}
		\begin{aligned}
			& \mathbb{E}\left[ \sum_{\alpha \in \mathbb{N}^m_{0,k}} \int_U \left\Vert \partial_\alpha f(u) - T_L\left( \partial_\alpha G^{y^{(J)}}_N(\cdot)(u)\right) \right\Vert^2 w(u) du \right]^\frac{1}{2} \\
			& \quad\quad \leq C_4 L m^\frac{k}{2} d^\frac{1}{2} \sqrt{\frac{(\ln(J)+1) N}{J}} + C_4 \kappa(\mathbf{c}) \frac{\Vert f \Vert_{\mathbb{B}^2_{\mathcal{G},\theta}(W^{k,2}(U,\mathcal{L}(U),w;\mathbb{R}^d))}}{\sqrt{N}}.
		\end{aligned}
	\end{equation}
\end{theorem}

Hence, the generalization errors \eqref{EqThmGenErr1}+\eqref{EqThmGenErr2} can be made arbitrarily small by first choosing the number of features $N \in \mathbb{N}$ large enough and then by increasing the sample size $J \in \mathbb{N}$. 

\begin{remark}	
	Choosing $N = \mathcal{O}\big( \sqrt{J/\ln(J)} \big)$ random features in \eqref{EqThmGenErr1} is optimal to obtain a generalization error of $\mathcal{O}\big( (\ln(J)/J)^{1/4} \big)$ in terms of the number of samples $J \in \mathbb{N}$. Note that \eqref{EqThmGenErr1}+\eqref{EqThmGenErr2} coincide (up to constants) with the $L^2$-learning rate of $\mathcal{O}\big( \sqrt{\ln(J) N/J} + 1/N^{1/4} \big)$ in \cite[Theorem~3]{barron94} for deterministic neural networks trained via constrained regression, which is however difficult to compute in practice. 
	
	Note that the global $L^2$-generalization bounds in \eqref{EqThmGenErr1}+\eqref{EqThmGenErr2} provide different generalization guarantees than the high-probability bounds in \cite{rahimi08,rudi17}. The latter imply for any $\delta \in (0,1]$ that
	\begin{equation}
		\label{eq:hp}
		\mathbb{P}\left[ \left\Vert f - G^{y^{(J)}}_N \right\Vert_{L^2(U,\mathcal{L}(U),w;\mathbb{R}^d)} \leq r_{\delta,J,N} \right] \geq 1-\delta,
	\end{equation}
	with $r_{\delta,J,N} := \mathcal{O}\left( \big( 1/J^{1/4} + 1/N^{1/4} \big) \ln(1/\delta)^{1/4} \right)$ in \cite[Theorem~1]{rahimi08} (using constrained regression) and with $r_{\delta,J,N} := \mathcal{O}\left( \ln(1/\delta)/J^{1/4} \right)$ in \cite[Theorem~1]{rudi17} using $N = \mathcal{O}(\sqrt{J} \ln(J/\delta))$ random features (trained with ridge regression). Note that \cite[Theorem~1]{rudi17} needs a larger optimal number $N = \mathcal{O}(\sqrt{J} \ln(J))$ of random features to ensure $\omega$-uniform generalization on a set with large probability, whereas our global $L^2$-bounds~\eqref{EqThmGenErr1}+\eqref{EqThmGenErr2} require a lower optimal number $N = \mathcal{O}\big( \sqrt{J/\ln(J)} \big)$ of random features to guarantee $L^2$-generalization.
\end{remark}

\subsection{Random neural networks}

In this section, we now consider random neural networks as particular instance of random feature learning, where $(\theta_n)_{n \in \mathbb{N}} := (a_n,b_n)_{n \in \mathbb{N}}: \Omega \rightarrow \mathbb{R}^m \times \mathbb{R}$ is the random initialization. To this end, we fix some $\gamma \in [0,\infty)$ and an activation function $\rho \in C^k_{pol,\gamma}(\mathbb{R})$. Then, for every fixed $N \in \mathbb{N}$, we denote by $\mathcal{Y}_N$ be the vector space of all $\mathcal{F}_{a,b,V}/\mathcal{B}(\mathbb{R}^{d \times N})$-measurable random variables $y := (y_n)_{n=1,...,N}^\top := (y_{i,n})_{i=1,...,d}^{n=1,...,N}$, where $\mathcal{F}_{a,b,V} := \mathcal{F}_{\theta,V}$. Then, for any $y \in \mathcal{Y}_N$, we define the corresponding random neural network as
\begin{equation}
	\label{EqDefRFMYRN}
	\Omega \ni \omega \quad \mapsto \quad G_N^y(\omega) := \sum_{n=1}^N y_n \rho\left( a_n^\top \cdot - b_n \right) \in C^k_{pol,\gamma}(U;\mathbb{R}^d).
\end{equation}
Moreover, we denote by $\mathcal{RN}^{\rho,V}_{U,d}$ the set of all random neural networks of the form \eqref{EqDefRFMYRN}.

For some fixed $J \in \mathbb{N}$, we now approximate a given $k$-times weakly differentiable function $f: U \rightarrow \mathbb{R}^d$ by the least squares method on the training data $(V_j)_{j=1,...,J}$, which corresponds to the random neural network $G_N^{y^{(J)}} \in \mathcal{RN}^{\rho,V}_{U,d}$ with linear readout $y^{(J)} \in \mathcal{Y}_N$ solving \eqref{EqEmpMSE}. In this case, we obtain the following version of Algorithm~\ref{Alg} for random neural networks.

\begin{algorithm}[!ht]
	\DontPrintSemicolon
	\begin{small}
		\KwInput{$J,N \in \mathbb{N}$ and $k$-times weakly differentiable function $f = (f_1,...,f_d)^\top: U \rightarrow \mathbb{R}^d$.}
		\KwOutput{$G_N^{y^{(J)}} \in \mathcal{RN}^{\rho,V}_{U,d}$ with linear readout $y^{(J)} := \big( y^{(J)}_{i,n} \big)_{i=1,...,d}^{n=1,...,N} \in \mathcal{Y}_N$ solving~\eqref{EqEmpMSE}.}
		
		Generate i.i.d.~random variables (RVs) $(a_n,b_n)_{n=1,...,N} \sim p_{a,b}$ (see Assumption~\ref{AssPDFRN}).\label{AlgRN1}
		
		Generate i.i.d.~random variables (RVs) $(V_j)_{j = 1,...,J} \sim w$ independent of $(a_n,b_n)_{n=1,...,N}$.\label{AlgRN2}
		
		Define the $\mathbb{R}^{(J \cdot \vert \mathbb{N}^m_{0,k} \vert) \times N}$-valued RV $\mathbf{G} = (\mathbf{G}_{(j,\alpha),n})_{(j,\alpha) \in \lbrace 1,...,J \rbrace \times \mathbb{N}^m_{0,k}}^{n = 1,...,N}$ with $\mathbf{G}_{(j,\alpha),n} := c_\alpha \rho^{(\vert \alpha \vert)}\big( a_n^\top V_j - b_n \big) a_n^\alpha$ for $(j,\alpha) \in \lbrace 1,...,J \rbrace \times \mathbb{N}^m_{0,k}$ and $n = 1,...,N$.\label{AlgRN3}
		
		\For{$i = 1,...,d$\label{AlgRN4}}{
			
			Define the $\mathbb{R}^{J \cdot \vert \mathbb{N}^m_{0,k} \vert}$-valued random variable $Z_i := (c_\alpha \partial_\alpha f_i(V_j))_{(j,\alpha) \in \lbrace 1,...,J \rbrace \times \mathbb{N}^m_{0,k}}^\top$.\label{AlgRN5}
			
			Solve the least squares problem $\mathbf{G}^\top \mathbf{G} y_i^{(J)} = \mathbf{G}^\top Z_i$ for $y_i^{(J)} := \big( y^{(J)}_{(i,n)} \big)_{n=1,...,N}^\top$ via Cholesky decomposition and forward/backward substitution (see \cite[Section~2.2.2]{bjoerck96}).\label{AlgRN6}
			
		}
		
		Return $\Omega \ni \omega \mapsto G_N^{y^{(J)}}(\omega) := \sum_{n=1}^N y^{(J)}_n(\omega) \rho\big( a_n(\omega)^\top \cdot - b_n(\omega) \big) \in C^k_{pol,\gamma}(U;\mathbb{R}^d)$.
	\end{small}
	\caption{Least squares method to learn a random neural network}
	\label{AlgRN}
\end{algorithm}

\begin{corollary}[Generalization error]
	\label{CorGenErrRN}
	Let $w: U \rightarrow [0,\infty)$ be a normalized weight such that the constant $C^{(\gamma,2)}_{U,w} > 0$ defined in \eqref{EqCorARRN1} is finite. Moreover, let $(\psi,\rho) \in \mathcal{S}_0(\mathbb{R};\mathbb{C}) \times C^k_{pol,\gamma}(\mathbb{R})$ be $m$-admissible and let Assumption~\ref{AssPDFRN} hold. Then, there exists some $C_5 > 0$ (being independent of $m,d \in \mathbb{N}$) such that for every $J,N \in \mathbb{N}$, $L > 0$, and $f := (f_1,...,f_d)^\top \in W^{k,2}(U,\mathcal{L}(U),w;\mathbb{R}^d) \cap \widetilde{\mathbb{B}}^{k,2,\gamma}_{\psi,a,b}(U;\mathbb{R}^d)$, Algorithm~\ref{AlgRN} returns a random neural network $G^{y^{(J)}}_N \in \mathcal{RN}^\rho_{U,d}$ with $N$ neurons being a strongly $(\mathbb{P},\mathcal{F}_{a,b,V})$-measurable map $G^{y^{(J)}}_N: \Omega \rightarrow W^{k,2}(U,\mathcal{L}(U),w;\mathbb{R}^d)$ such that
	\begin{equation}
		\label{EqCorGenErrRN1}
		\begin{aligned}
			& \mathbb{E}\left[ \sum_{\alpha \in \mathbb{N}^m_{0,k}} \int_U \left\Vert T_L\left( \partial_\alpha f(u) - \partial_\alpha G^{y^{(J)}}_N(\cdot)(u)\right) \right\Vert^2 w(u) du \right]^\frac{1}{2} \\
			& \quad\quad \leq C_5 L m^\frac{k}{2} d^\frac{1}{2} \sqrt{\frac{(\ln(J)+1) N}{J}} + C_5 \kappa(\mathbf{c}) \Vert \rho \Vert_{C^k_{pol,\gamma}(\mathbb{R})} \frac{C^{(\gamma,2)}_{U,w} m^\frac{k}{2}}{\left\vert C^{(\psi,\rho)}_m \right\vert} \frac{\Vert f \Vert_{\widetilde{\mathbb{B}}^{k,2,\gamma}_{\psi,a,b}(U;\mathbb{R}^d)}}{\sqrt{N}}.
		\end{aligned}
	\end{equation}
	Moreover, if $\vert \partial_\alpha f_i(u) \vert \leq L$ for all $\alpha \in \mathbb{N}^m_{0,k}$, $i = 1,...,d$, and $u \in U$, then 
	\begin{equation}
		\label{EqCorGenErrRN2}
		\begin{aligned}
			& \mathbb{E}\left[ \sum_{\alpha \in \mathbb{N}^m_{0,k}} \int_U \left\Vert \partial_\alpha f(u) - T_L\left(\partial_\alpha G^{y^{(J)}}_N(\cdot)(u)\right) \right\Vert^2 w(u) du \right]^\frac{1}{2} \\
			& \quad\quad \leq C_5 L m^\frac{k}{2} d^\frac{1}{2} \sqrt{\frac{(\ln(J)+1) N}{J}} + C_5 \kappa(\mathbf{c}) \Vert \rho \Vert_{C^k_{pol,\gamma}(\mathbb{R})} \frac{C^{(\gamma,2)}_{U,w} m^\frac{k}{2}}{\left\vert C^{(\psi,\rho)}_m \right\vert} \frac{\Vert f \Vert_{\widetilde{\mathbb{B}}^{k,2,\gamma}_{\psi,a,b}(U;\mathbb{R}^d)}}{\sqrt{N}}.
		\end{aligned}
	\end{equation}
\end{corollary}

\begin{remark}
	Corollary~\ref{CorGenErrRN} extends the generalization error in \cite[Theorem~4.1]{gonon21} for random neural networks with ReLU activation function to more general activation functions and includes the approximation of the weak derivatives. Moreover, \eqref{EqCorGenErrRN1}+\eqref{EqCorGenErrRN2} coincide (up to constants) with the generalization error in \cite[Theorem~3]{barron94} for \emph{deterministic} neural networks trained via constrained regression, which is however difficult to compute in practice.
\end{remark}

\section{Numerical examples}
\label{SecNE}

In this section, we illustrate in two numerical examples\footnotemark~that random feature models (including random neural networks) can be effectively used to learn the solution of a partial differential equation (PDE) of the form
\begin{equation}
	\label{eq:pde}
	\left\lbrace
	\begin{matrix*}[l]
		\frac{\partial f}{\partial t}(t,u) - (\mathcal{A} f)(t,u) & = 0, & \quad\quad (t,u) & \in (0,\infty) \times \mathbb{R}^m, \\
		\quad\quad\quad\quad\quad\;\;\;\, f(0,u) & = f_0(u), & \quad\quad\quad\;\; u & \in \mathbb{R}^m,
	\end{matrix*}
	\right.
\end{equation}
where $\mathcal{A}$ is a differential operator on a suitable function space. We first consider the heat equation as an example of a linear PDE, followed by the non-linear Fokker-Planck equation. For both PDEs, we consider the following two complementary learning settings:
\begin{itemize}
	\item \textbf{Fixed-time learning:} We learn the solution $f(t,\cdot)$ of the PDE~\eqref{eq:pde} at a fixed time $t > 0$, which is a function regression problem and requires the prior knowledge of its solution.
	\item \textbf{Spatio-temporal learning:} We learn the solution $f$ of the PDE~\eqref{eq:pde} over a time interval $[0,T]$, where $T > 0$. Instead of knowing the PDE solution a-priori, we assume that $\mathcal{A}$ is the generator of a strongly continuous (possibly non-linear) convex monotone semigroup $(S_t)_{t \in [0,\infty)}$ on a suitable Banach space (e.g., $S_t = e^{t\mathcal{A}}$ if $\mathcal{A}$ is linear), which allows us to define the feature map $G(t,u) := (S_t g_\theta)(u)$ solving the PDE in the interior (in an appropriate (mild/viscosity) sense; see \cite{kato67,brezis71,crandall71,barbu10,blessing25}), where $(g_\theta)_{\theta \in \Theta}$ are parametrized functions. Hence, we only need to minimize the residual with respect to the initial condition $f_0$.
\end{itemize}
We compare both approaches in Section~\ref{SecRL} with related random feature methods from the literature and summarize our main contributions.

\footnotetext{The numerical experiments have been implemented in \texttt{Python} and were executed on a high-performing computing (HPC) cluster of ETH Zurich. The code can be found under the following link: https://github.com/psc25/RandomNeuralNetworks}

\subsection{Linear heat equation}

First, we consider the linear heat equation, which describes the evolution of a given quantity throughout time. More precisely, we consider the PDE
\begin{equation}
	\label{EqDefHeat}
	\left\lbrace
	\begin{matrix*}[l]
		\frac{\partial f}{\partial t}(t,u) - \lambda \Delta f(t,u) & = 0, & \quad\quad (t,u) & \in (0,\infty) \times \mathbb{R}^m, \\
		\quad\quad\quad\quad\quad\;\;\; f(0,u) & = f_0(u), & \quad\quad\quad\;\; u & \in \mathbb{R}^m,
	\end{matrix*}
	\right.
\end{equation}
where $\Delta f(t,u) := \sum_{l=1}^m \frac{\partial^2 f}{\partial u_l^2}(t,u)$ denotes the Laplacian, and where $f_0: \mathbb{R}^m \rightarrow \mathbb{R}$ represents the initial condition. If $f_0: \mathbb{R}^m \rightarrow \mathbb{R}$ is bounded and continuous, then
\begin{equation}
	\label{EqSolHeat}
	[0,\infty) \times \mathbb{R}^m \ni (t,u) \quad \mapsto \quad f(t,u) = \frac{1}{(4\pi \lambda t)^\frac{m}{2}} \int_{\mathbb{R}^m} e^{-\frac{\Vert u-v \Vert^2}{4\lambda t}} f_0(v) dv \in \mathbb{R}
\end{equation}
is the unique solution of \eqref{EqDefHeat} with initial condition $f_0: \mathbb{R}^m \rightarrow \mathbb{R}$ (see \cite[Theorem~2.3.1]{evans10}).

\subsubsection{Fixed-time learning}
\label{SecHeatSuper}

For some fixed $t \in [0,T]$, we now learn the function $f(t,\cdot)$ by random trigonometric feature models, random neural networks, and their deterministic counterparts. We omit random Fourier regression as it coincides for real-valued functions with random trigonometric models. Moreover, we provide sufficient conditions that the approximation of $f(t,\cdot)$ by random feature models overcomes the curse of dimensionality in the sense that the computational costs (measured as the number of feature maps $N \in \mathbb{N}$; for $\mathcal{RT}_{\mathbb{R}^m,1}$, we use $\lceil N/2 \rceil$ cosine and $\lceil N/2 \rceil$ sine features) grow polynomially in both the input/output dimensions $m,d \in \mathbb{N}$ and the reciprocal of a pre-specified tolerated approximation error $\varepsilon > 0$. To this end, we apply the approximation rates in Corollary~\ref{CorARTrigo}+\ref{CorARRN}.

\begin{corollary}
	\label{CorHeat}
	For $\lambda,t,\kappa_m,\sigma \in (0,\infty)$ and $f_0(u) := \frac{\kappa_m}{(2\pi \sigma^2)^{m/2}} \exp\left( -\Vert u \Vert^2/(2\sigma^2) \right)$, $u \in \mathbb{R}^m$, let $f(t,\cdot): \mathbb{R}^m \rightarrow \mathbb{R}$ be the solution of \eqref{EqDefHeat} at time $t$. Moreover, let $p \in (1,\infty)$, $\gamma \in [0,\infty)$, and $w: \mathbb{R}^m \rightarrow [0,\infty)$ be as in Lemma~\ref{LemmaWeight}. Then, the following holds:
	\begin{enumerate}
		\item\label{CorHeat1} Let $(\theta_n)_{n \in \mathbb{N}} \sim t_m$ be i.i.d. Then, there exists $C_6,C_7 > 0$ (being independent of $m \in \mathbb{N}$) such that for every $N \in \mathbb{N}$ there exists a random trigonometric feature model $G_N \in \mathcal{RT}_{\mathbb{R}^m,1}$ with $N$ features satisfying
		\begin{equation}
			\label{EqCorHeat1}
			\mathbb{E}\left[ \Vert f(t,\cdot) - G_N \Vert_{L^p(\mathbb{R}^m,\mathcal{L}(\mathbb{R}^m),w)}^2 \right]^\frac{1}{2} \leq C_6 m^{C_7} \frac{\kappa_m \left( \frac{(\zeta_2/\zeta_1)^2}{\pi \sigma} \right)^\frac{m}{2}}{N^{1-\frac{1}{\min(2,p)}}}.
		\end{equation}
		In particular, if $\kappa_m \left( (\zeta_2/\zeta_1)^2/ (\pi \sigma) \right)^{m/2} = \mathcal{O}\left( m^\nu \right)$ for some $\nu > 0$, then there exists $C_8,C_9 > 0$ such that for every $m \in \mathbb{N}$ and $\varepsilon > 0$ there exists a random trigonometric feature model $G_N \in \mathcal{RT}_{\mathbb{R}^m,1}$ with $N = \big\lceil C_8 m^{C_9} \varepsilon^{-\frac{\min(2,p)}{\min(2,p)-1}} \big\rceil$ features satisfying
		\begin{equation}
			\label{EqCorHeat2}
			\mathbb{E}\left[ \Vert f(t,\cdot) - G_N \Vert_{L^p(\mathbb{R}^m,\mathcal{L}(\mathbb{R}^m),w)}^2 \right]^\frac{1}{2} \leq \varepsilon.
		\end{equation}
		\item\label{CorHeat2} Let $(a_n,b_n)_{n \in \mathbb{N}} \sim t_m \otimes t_1$ be i.i.d.~and let $(\psi,\rho) \in \mathcal{S}_0(\mathbb{R};\mathbb{C}) \times C^0_{pol,\gamma}(\mathbb{R})$ be as in Example~\ref{ExAdm} (with $0 < \zeta_1 < \zeta_2 < \infty$). Then, $f(t,\cdot) \in \widetilde{\mathbb{B}}^{0,2,\gamma}_{\psi,a,b}(\mathbb{R}^m)$ and there exist $C_{10},C_{11} > 0$ (being independent of $m \in \mathbb{N}$) such that for every $N \in \mathbb{N}$ there exists a random neural network $G_N \in \mathcal{RN}^\rho_{\mathbb{R}^m,1}$ with $N$ neurons satisfying
		\begin{equation}
			\label{EqCorHeat3}
			\mathbb{E}\left[ \Vert f(t,\cdot) - G_N \Vert_{L^p(\mathbb{R}^m,\mathcal{L}(\mathbb{R}^m),w)}^2 \right]^\frac{1}{2} \leq C_{10} m^{C_{11}} \frac{\kappa_m \left( \frac{(\zeta_2/\zeta_1)^2}{\pi \sigma} \right)^\frac{m}{2}}{N^{1-\frac{1}{\min(2,p)}}}.
		\end{equation}
		In particular, $\kappa_m \left( (\zeta_2/\zeta_1)^2/ (\pi \sigma) \right)^{m/2} = \mathcal{O}\left( m^\nu \right)$ for some $\nu > 0$, then there exist $C_{12},C_{13} > 0$ such that for every $m \in \mathbb{N}$ and $\varepsilon > 0$ there exists a random neural network $G_N \in \mathcal{RN}^\rho_{\mathbb{R}^m,1}$ with $N = \big\lceil C_{12} m^{C_{13}} \varepsilon^{-\frac{\min(2,p)}{\min(2,p)-1}} \big\rceil$ neurons satisfying
		\begin{equation}
			\label{EqCorHeat4}
			\mathbb{E}\left[ \Vert f(t,\cdot) - G_N \Vert_{L^p(\mathbb{R}^m,\mathcal{L}(\mathbb{R}^m),w)}^2 \right]^\frac{1}{2} \leq \varepsilon.
		\end{equation}
	\end{enumerate}
\end{corollary}

Note that the approximation bounds in \eqref{EqCorHeat1} and \eqref{EqCorHeat3} can be incorporated into Theorem~\ref{ThmGenErr} and Corollary~\ref{CorGenErrRN}, showing that the generalization error also does not suffer from the curse of dimensionality. For the numerical experiment, we choose $t = 1$, $\lambda = 0.02$, the initial condition $f_0(u) := \kappa_m (2\pi \sigma^2)^{-m/2} \exp\left( -\Vert u \Vert^2/(2\sigma^2) \right)$ with $\sigma = 1$ and $\kappa_m = 10^{-4} m^6 1.7724^m$, and the weight $w(u) := (2\pi \sigma_w^2)^{-m/2} \exp\left( -\Vert u \Vert^2/(2 \sigma_w^2) \right)$ with $\sigma_w := 0.5$ satisfying the conditions of Lemma~\ref{LemmaWeight}. Then, we generate $J = 2^{12},2^{13},...,2^{18}$ i.i.d.~data $(V_j)_{j=1,...,J} \sim w$, split up into $80\%/20\%$ for training/testing, and minimize the empirical $L^2$-error
\begin{equation}
	\label{EqDefMSE}
	\bigg( \frac{1}{J} \sum_{j=1}^J \left\vert f(1,V_j) - G_N(\cdot)(V_j) \right\vert^2 \bigg)^\frac{1}{2}
\end{equation}
over random trigonometric feature models $G_N \in \mathcal{RT}_{\mathbb{R}^m,1}$, random neural networks $G_N \in \mathcal{RN}^{\tanh}_{\mathbb{R}^m,1}$, and their deterministic counterparts, all of them with $N = 10 \sqrt{J/\ln(J)}$ features.

For the random feature models, we randomly initialize $(\theta_n)_{n \in \mathbb{N}} \sim \mathbf{t}_m$ (for $\mathcal{RT}_{\mathbb{R}^m,1}$) and $(a_n,b_n)_{n \in \mathbb{N}} \sim \mathbf{t}_m \otimes \mathbf{t}_1$ (for $\mathcal{RN}^{\tanh}_{\mathbb{R}^m,1}$), and train the models by minimizing \eqref{EqDefMSE} according to Algorithm~1+2. As $\kappa_m \left( (\zeta_2/\zeta_1)^2/(\pi\sigma) \right)^{m/2} \leq 10^{-4} m^6$, Corollary~\ref{CorHeat} shows that the approximation of $f(t,\cdot)$ by the random feature models overcomes the curse of dimensionality.

For the deterministic feature models, we minimize \eqref{EqDefMSE} over the deterministic trigonometric feature models $G_N \in \mathcal{T}_{\mathbb{R}^m,1} := \linspan_\mathbb{R}\big(\left\lbrace \mathbb{R}^m \ni u \mapsto h\big( \vartheta^\top u \big) \in \mathbb{R}: h \in \lbrace \cos, \sin \rbrace \right\rbrace\big)$ and the deterministic neural networks $G_N \in \mathcal{N}^{\tanh}_{\mathbb{R}^m,1} := \linspan_\mathbb{R}\big(\big\lbrace \mathbb{R}^m \ni u \mapsto \tanh\big( \vartheta_1^\top u - \vartheta_2 \big) \in \mathbb{R}: (\vartheta_1,\vartheta_2) \in \mathbb{R}^m \times \mathbb{R} \big\rbrace\big)$. Hereby, we apply the Adam algorithm (see \cite{kingma15}) over $5000$ epochs (with learning rate $10^{-4}$ and batchsize $1000$) to train the deterministic feature models.

Figure~\ref{FigHeat} empirically demonstrates that random trigonometric features and random neural networks are indeed able to learn the solution of the heat equation \eqref{EqDefHeat}. While random feature models typically require a larger $N$ to achieve a reasonable accuracy, they significantly outperform their deterministic counterparts in terms of computational time (see also Table~\ref{TabHeat}).

\begin{figure}[h]
	\begin{minipage}[t]{0.48\textwidth}
		\centering
		\includegraphics[height = 5.2cm]{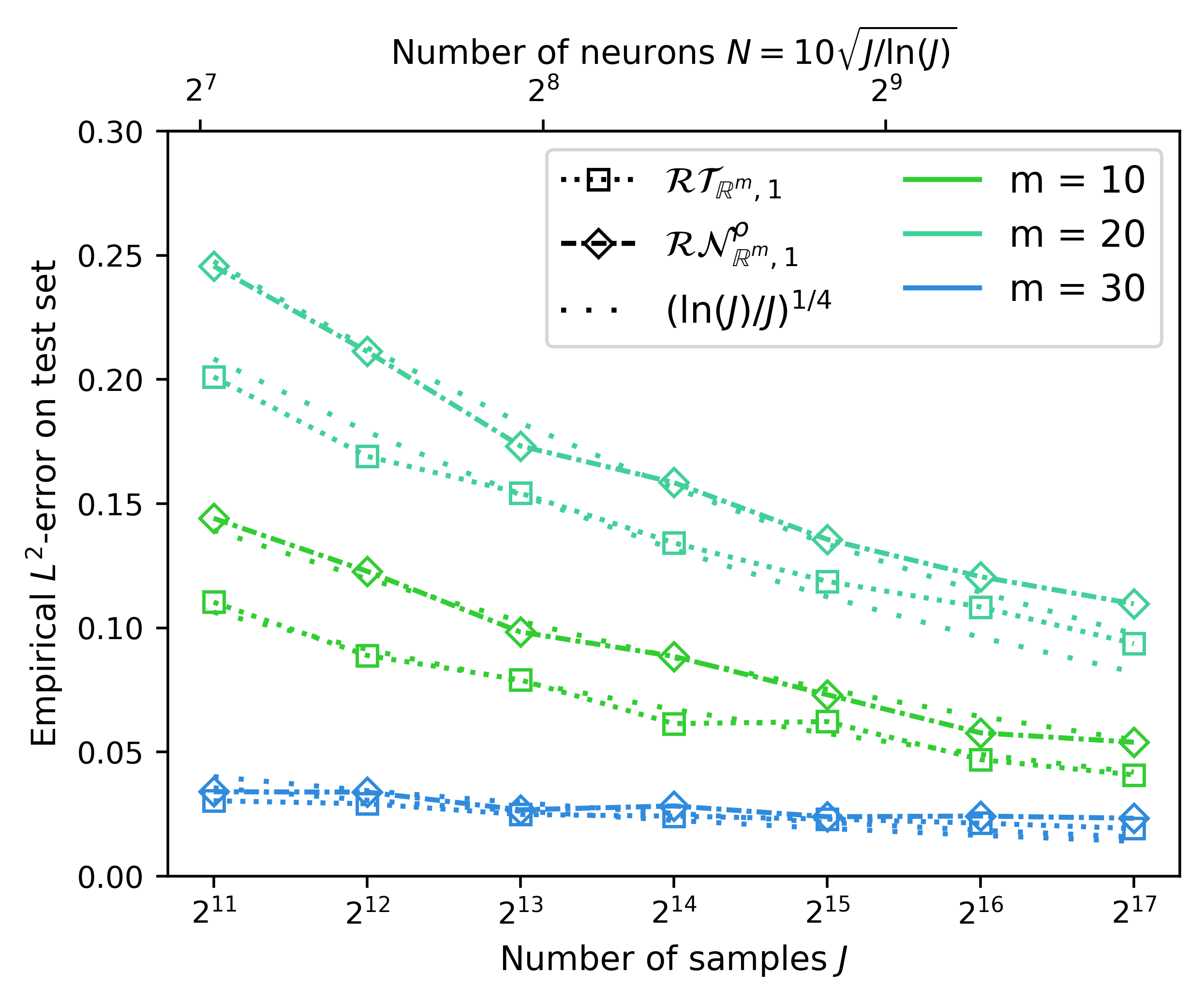}
		\subcaption{Empirical $L^2$-error defined in \eqref{EqDefMSE}.}
	\end{minipage}
	\hspace{0.02\textwidth}
	\begin{minipage}[t]{0.48\textwidth}
		\centering
		\includegraphics[height = 5.2cm]{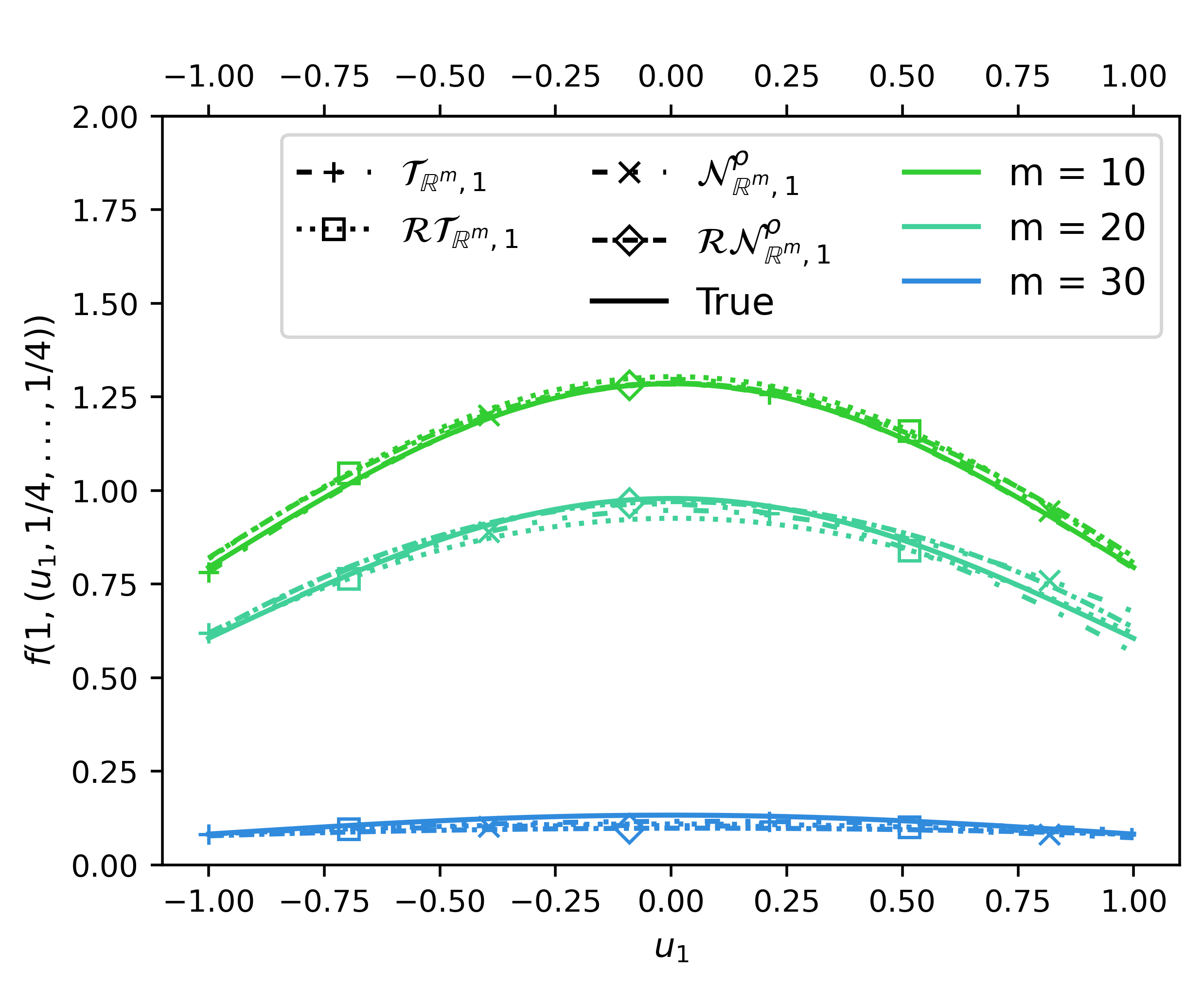}
		\subcaption{Approximation of the function $\mathbb{R} \ni u_1 \mapsto f(1,(u_1,0.4,...,0.4)) \in \mathbb{R}$.}
	\end{minipage}
	\caption{Fixed-time learning of the heat equation \eqref{EqDefHeat}.}
	\label{FigHeat}
\end{figure}

\begin{table}
	\centering
	\begin{scriptsize}
		\begin{tabular}{ll|R{1.2cm}|R{1.2cm}|R{1.2cm}|R{1.2cm}|R{1.2cm}|R{1.2cm}|R{1.2cm}}
			\multicolumn{2}{r|}{$N = 10 \sqrt{J \ln(J)}$} & \multicolumn{1}{c|}{$J = 2^{12}$} & \multicolumn{1}{c|}{$J = 2^{13}$} & \multicolumn{1}{c|}{$J = 2^{14}$} & \multicolumn{1}{c|}{$J = 2^{15}$} & \multicolumn{1}{c|}{$J = 2^{16}$} & \multicolumn{1}{c}{$J = 2^{17}$} & \multicolumn{1}{c}{$J = 2^{18}$} \\
			\hline
			\multirow{8}{*}{$m = 10$} & \multirow{2}{*}{$\mathcal{T}_{\mathbb{R}^m,1}$} & \textbf{0.098}  & \textbf{0.073}  & \textbf{0.047}  & \textbf{0.028}  & \textbf{0.026}  & \textbf{0.024}  & \textbf{0.020} \\ 
 & & \textit{35.95} s  & \textit{59.82} s  & \textit{122.96} s  & \textit{244.66} s  & \textit{455.58} s  & \textit{922.78} s  & \textit{2031.82} s \\ 
& \multirow{2}{*}{$\mathcal{RT}_{\mathbb{R}^m,1}$} & \cellcolor{gray!40} \textbf{0.110}  & \cellcolor{gray!40} \textbf{0.089}  & \cellcolor{gray!40} \textbf{0.079}  & \cellcolor{gray!40} \textbf{0.061}  & \cellcolor{gray!40} \textbf{0.062}  & \cellcolor{gray!40} \textbf{0.047}  & \cellcolor{gray!40} \textbf{0.041} \\ 
 & & \cellcolor{gray!40} \textit{0.03} s & \cellcolor{gray!40} \textit{0.09} s & \cellcolor{gray!40} \textit{0.31} s & \cellcolor{gray!40} \textit{1.24} s & \cellcolor{gray!40} \textit{4.34} s & \cellcolor{gray!40} \textit{16.00} s & \cellcolor{gray!40} \textit{56.92} s \\ 
& \multirow{2}{*}{$\mathcal{N}_{\mathbb{R}^m,1}$} & \textbf{0.106}  & \textbf{0.095}  & \textbf{0.081}  & \textbf{0.062}  & \textbf{0.042}  & \textbf{0.032}  & \textbf{0.029} \\ 
 & & \textit{31.11} s  & \textit{53.13} s  & \textit{107.37} s  & \textit{213.35} s  & \textit{419.86} s  & \textit{805.53} s  & \textit{1730.47} s \\ 
& \multirow{2}{*}{$\mathcal{RN}_{\mathbb{R}^m,1}$} & \cellcolor{gray!40} \textbf{0.144}  & \cellcolor{gray!40} \textbf{0.123}  & \cellcolor{gray!40} \textbf{0.098}  & \cellcolor{gray!40} \textbf{0.089}  & \cellcolor{gray!40} \textbf{0.073}  & \cellcolor{gray!40} \textbf{0.058}  & \cellcolor{gray!40} \textbf{0.054} \\ 
 & & \cellcolor{gray!40} \textit{0.04} s & \cellcolor{gray!40} \textit{0.10} s & \cellcolor{gray!40} \textit{0.34} s & \cellcolor{gray!40} \textit{1.32} s & \cellcolor{gray!40} \textit{4.51} s & \cellcolor{gray!40} \textit{15.80} s & \cellcolor{gray!40} \textit{72.68} s \\ 
\hline 
\multirow{8}{*}{$m = 20$} & \multirow{2}{*}{$\mathcal{T}_{\mathbb{R}^m,1}$} & \textbf{0.137}  & \textbf{0.139}  & \textbf{0.122}  & \textbf{0.108}  & \textbf{0.097}  & \textbf{0.087}  & \textbf{0.072} \\ 
 & & \textit{37.75} s  & \textit{71.77} s  & \textit{145.13} s  & \textit{283.43} s  & \textit{501.72} s  & \textit{1001.23} s  & \textit{2041.21} s \\ 
& \multirow{2}{*}{$\mathcal{RT}_{\mathbb{R}^m,1}$} & \cellcolor{gray!40} \textbf{0.201}  & \cellcolor{gray!40} \textbf{0.169}  & \cellcolor{gray!40} \textbf{0.154}  & \cellcolor{gray!40} \textbf{0.134}  & \cellcolor{gray!40} \textbf{0.119}  & \cellcolor{gray!40} \textbf{0.108}  & \cellcolor{gray!40} \textbf{0.094} \\ 
 & & \cellcolor{gray!40} \textit{0.25} s & \cellcolor{gray!40} \textit{0.09} s & \cellcolor{gray!40} \textit{0.29} s & \cellcolor{gray!40} \textit{1.23} s & \cellcolor{gray!40} \textit{4.49} s & \cellcolor{gray!40} \textit{16.44} s & \cellcolor{gray!40} \textit{61.34} s \\ 
& \multirow{2}{*}{$\mathcal{N}_{\mathbb{R}^m,1}$} & \textbf{0.169}  & \textbf{0.151}  & \textbf{0.135}  & \textbf{0.130}  & \textbf{0.116}  & \textbf{0.107}  & \textbf{0.092} \\ 
 & & \textit{29.64} s  & \textit{52.85} s  & \textit{105.27} s  & \textit{212.85} s  & \textit{478.87} s  & \textit{877.64} s  & \textit{1709.79} s \\ 
& \multirow{2}{*}{$\mathcal{RN}_{\mathbb{R}^m,1}$} & \cellcolor{gray!40} \textbf{0.246}  & \cellcolor{gray!40} \textbf{0.211}  & \cellcolor{gray!40} \textbf{0.173}  & \cellcolor{gray!40} \textbf{0.159}  & \cellcolor{gray!40} \textbf{0.136}  & \cellcolor{gray!40} \textbf{0.121}  & \cellcolor{gray!40} \textbf{0.110} \\ 
 & & \cellcolor{gray!40} \textit{0.07} s & \cellcolor{gray!40} \textit{0.13} s & \cellcolor{gray!40} \textit{0.43} s & \cellcolor{gray!40} \textit{1.54} s & \cellcolor{gray!40} \textit{4.66} s & \cellcolor{gray!40} \textit{17.25} s & \cellcolor{gray!40} \textit{61.72} s \\ 
\hline 
\multirow{8}{*}{$m = 30$} & \multirow{2}{*}{$\mathcal{T}_{\mathbb{R}^m,1}$} & \textbf{0.025}  & \textbf{0.025}  & \textbf{0.022}  & \textbf{0.022}  & \textbf{0.022}  & \textbf{0.021}  & \textbf{0.020} \\ 
 & & \textit{39.37} s  & \textit{72.97} s  & \textit{146.84} s  & \textit{258.97} s  & \textit{489.56} s  & \textit{958.23} s  & \textit{1786.83} s \\ 
& \multirow{2}{*}{$\mathcal{RT}_{\mathbb{R}^m,1}$} & \cellcolor{gray!40} \textbf{0.030}  & \cellcolor{gray!40} \textbf{0.029}  & \cellcolor{gray!40} \textbf{0.025}  & \cellcolor{gray!40} \textbf{0.024}  & \cellcolor{gray!40} \textbf{0.023}  & \cellcolor{gray!40} \textbf{0.021}  & \cellcolor{gray!40} \textbf{0.019} \\ 
 & & \cellcolor{gray!40} \textit{0.31} s & \cellcolor{gray!40} \textit{0.09} s & \cellcolor{gray!40} \textit{0.39} s & \cellcolor{gray!40} \textit{1.49} s & \cellcolor{gray!40} \textit{4.76} s & \cellcolor{gray!40} \textit{16.66} s & \cellcolor{gray!40} \textit{63.43} s \\ 
& \multirow{2}{*}{$\mathcal{N}_{\mathbb{R}^m,1}$} & \textbf{0.036}  & \textbf{0.028}  & \textbf{0.024}  & \textbf{0.022}  & \textbf{0.022}  & \textbf{0.021}  & \textbf{0.019} \\ 
 & & \textit{35.42} s  & \textit{65.60} s  & \textit{108.97} s  & \textit{206.92} s  & \textit{509.77} s  & \textit{909.82} s  & \textit{1662.40} s \\ 
& \multirow{2}{*}{$\mathcal{RN}_{\mathbb{R}^m,1}$} & \cellcolor{gray!40} \textbf{0.034}  & \cellcolor{gray!40} \textbf{0.034}  & \cellcolor{gray!40} \textbf{0.027}  & \cellcolor{gray!40} \textbf{0.028}  & \cellcolor{gray!40} \textbf{0.024}  & \cellcolor{gray!40} \textbf{0.024}  & \cellcolor{gray!40} \textbf{0.023} \\ 
 & & \cellcolor{gray!40} \textit{0.04} s & \cellcolor{gray!40} \textit{0.11} s & \cellcolor{gray!40} \textit{0.37} s & \cellcolor{gray!40} \textit{1.56} s & \cellcolor{gray!40} \textit{4.82} s & \cellcolor{gray!40} \textit{17.02} s & \cellcolor{gray!40} \textit{66.71} s 
		\end{tabular}
	\end{scriptsize}
	\vspace{0.3cm}
	\caption{Fixed-time learning of the heat equation \eqref{EqDefHeat}: Empirical $L^2$-error defined in \eqref{EqDefMSE} on the test set (in bold letters) and running time (in italic letters; in seconds).}
	\label{TabHeat}
\end{table}

\subsubsection{Spatio-temporal learning}

In this section, we now consider a spatio-temporal learning approach to solve the heat equation~\eqref{EqDefHeat}. To this end, we first observe that the feature map
\begin{equation*}
	\mathbb{R}^m \times \mathbb{R} \ni (a,b) \quad \mapsto \quad \left( (t,u) \mapsto g(a,b)(t,u) := e^{-\lambda t \Vert a \Vert^2} \cos\left( a^\top u - b \right) \right)
\end{equation*}
solves the heat equation~\eqref{EqDefHeat} in the interior. Thus, random feature models of the form
\begin{equation}
	\label{EqDefRNHeat}
	\Omega \ni \omega \,\,\, \mapsto \,\,\, \left( (t,u) \mapsto G_N(\omega)(t,u) := \sum_{n=1}^N y_n(\omega) e^{-\lambda t \Vert a_n(\omega) \Vert^2} \cos\left( a_n(\omega)^\top u - b_n(\omega) \right) \right),
\end{equation}
solve for any fixed $\omega \in \Omega$ the heat equation \eqref{EqDefHeat} in the interior, where $(a_n,b_n)_{n=1,...,N} \sim \mathbf{t}_m \otimes \mathbf{t}_1$ are the random initializations and $(y_n)_{n=1,...,N}$ are the $\mathcal{F}_{a,b}$-measurable linear readouts. We denote by $\mathcal{R}\lbrace g \rbrace_{\mathbb{R}^m,1}$ all random feature models of the form \eqref{EqDefRNHeat}, while $\lbrace g \rbrace_{\mathbb{R}^m,1} := \linspan\big\lbrace (t,u) \mapsto e^{-\lambda t \Vert a \Vert^2} \cos\left( a^\top u - b \right): (a,b) \in \mathbb{R}^m \times \mathbb{R} \big\rbrace$ is the deterministic counterpart.

Hence, we train the linear readouts $(y_n)_{n=1,...,N}$ so that the corresponding random feature model $G_N$ minimizes the empirical $L^2$-error with respect to the initial condition $f_0$, i.e.
\begin{equation}
	\label{EqDefMSE1}
	\bigg( \frac{1}{J} \sum_{j=1}^J \left\vert f_0(V_j) - G_N(\cdot)(0,V_j) \right\vert^2 \bigg)^\frac{1}{2},
\end{equation}
where $(V_j)_{j=1,...,J} \sim w$ for some weight $w: \mathbb{R}^m \rightarrow [0,\infty)$. For $T > 0$ and a partition $0 = t_0 < t_1 < ... < t_K = T$, we then test the performance of $G_N$ on the global $L^2$-error
\begin{equation}
	\label{EqDefMSE2}
	\bigg( \frac{1}{J(K+1)} \sum_{j=1}^J \sum_{k=0}^K \left\vert f(t_k,V_j) - G_N(\cdot)(t_k,V_j) \right\vert^2 \bigg)^\frac{1}{2}.
\end{equation}
This procedure only requires a feature map solving the heat equation~\eqref{EqDefHeat} in the interior. Note that, unlike other PDE learning methods based on deterministic neural networks (see, e.g., \cite{sirignano18}), it is difficult to minimize the residual of a linear PDE, e.g., $\partial_t f - \mathcal{L} f$ with random neural networks as the only trainable parameter (the linear readout) factors out linearly. If $\partial_t f$ and $\mathcal{L} f$ have different distributions, it forces the linear readout towards zero, which might fail to represent the true PDE solution. However, our method only minimizes the residual with respect to the initial condition, i.e., without the interior, and therefore circumvents this issue.

Now, we choose $T = 1$, $K = 20$, the initial condition $g$, the weight $w$, and $(V_j)_{j=1,...,J} \sim w$ as in Section~\ref{SecHeatSuper}. Then, we train the random feature model $G_N \in \mathcal{R}\lbrace g \rbrace_{\mathbb{R}^m,1}$ of the form \eqref{EqDefRNHeat} by minimizing \eqref{EqDefMSE1} according to Algorithm~1. On the other hand, we also consider the corresponding deterministic feature models $G_N \in \lbrace g \rbrace_{\mathbb{R}^m,1}$, which are trained with the Adam algorithm over $5000$ epochs (with learning rate $10^{-4}$ and batchsize $1000$).

Figure~\ref{FigHeatUSV} empirically demonstrates that specific random feature models and their deterministic counterparts can learn the heat equation~\eqref{EqDefHeat}. However, the random feature models outperform the deterministic models in terms of computational time (see also Table~\ref{TabHeatUSV}).

\begin{figure}[h]
	\begin{minipage}[t]{0.48\textwidth}
		\centering
		\includegraphics[height = 5.2cm]{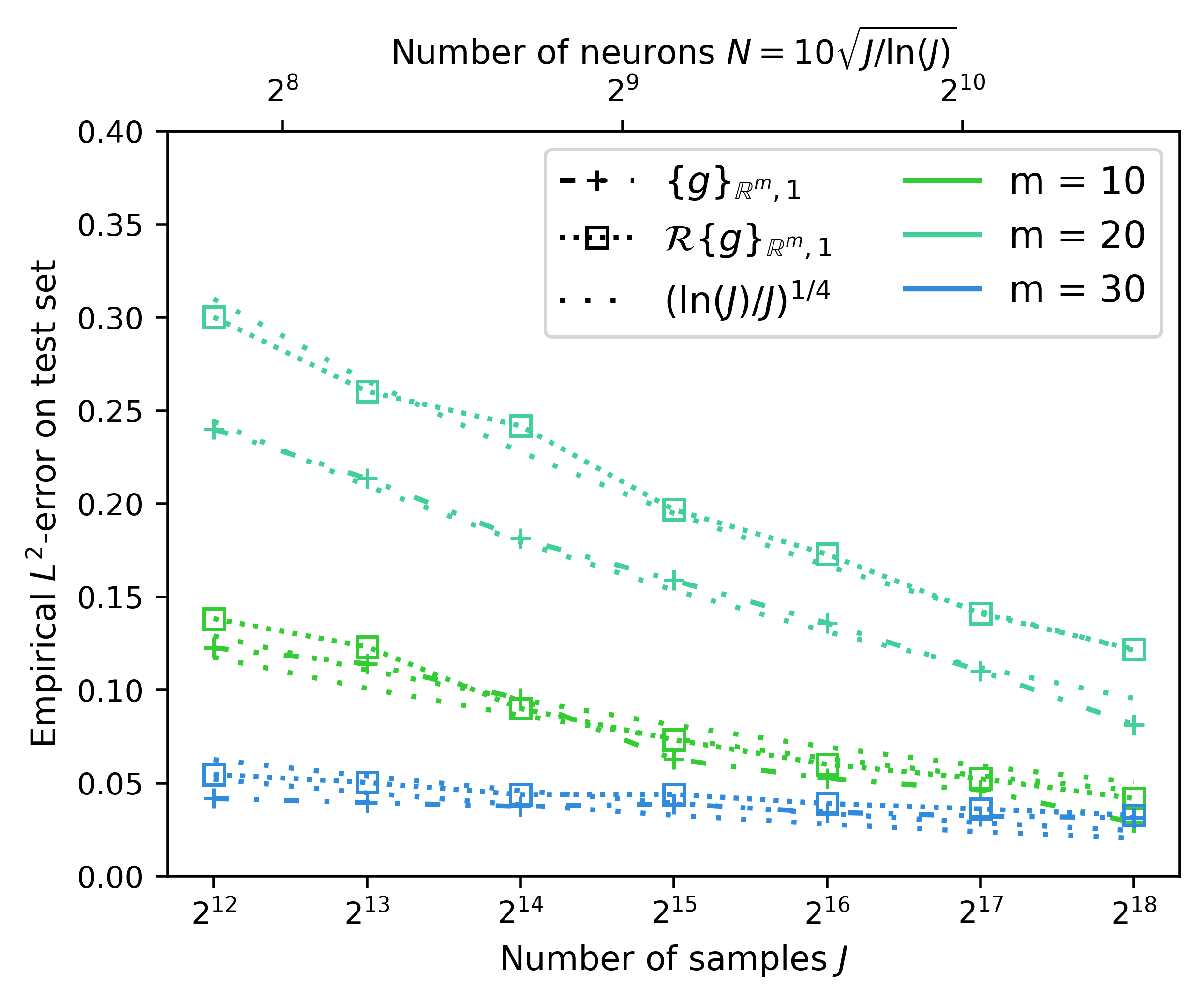}
		\subcaption{Empirical $L^2$-error defined in \eqref{EqDefMSE2}.}
	\end{minipage}
	\hspace{0.02\textwidth}
	\begin{minipage}[t]{0.48\textwidth}
		\centering
		\includegraphics[height = 5.5cm, trim = {0.8cm 0 0 0.4cm}, clip]{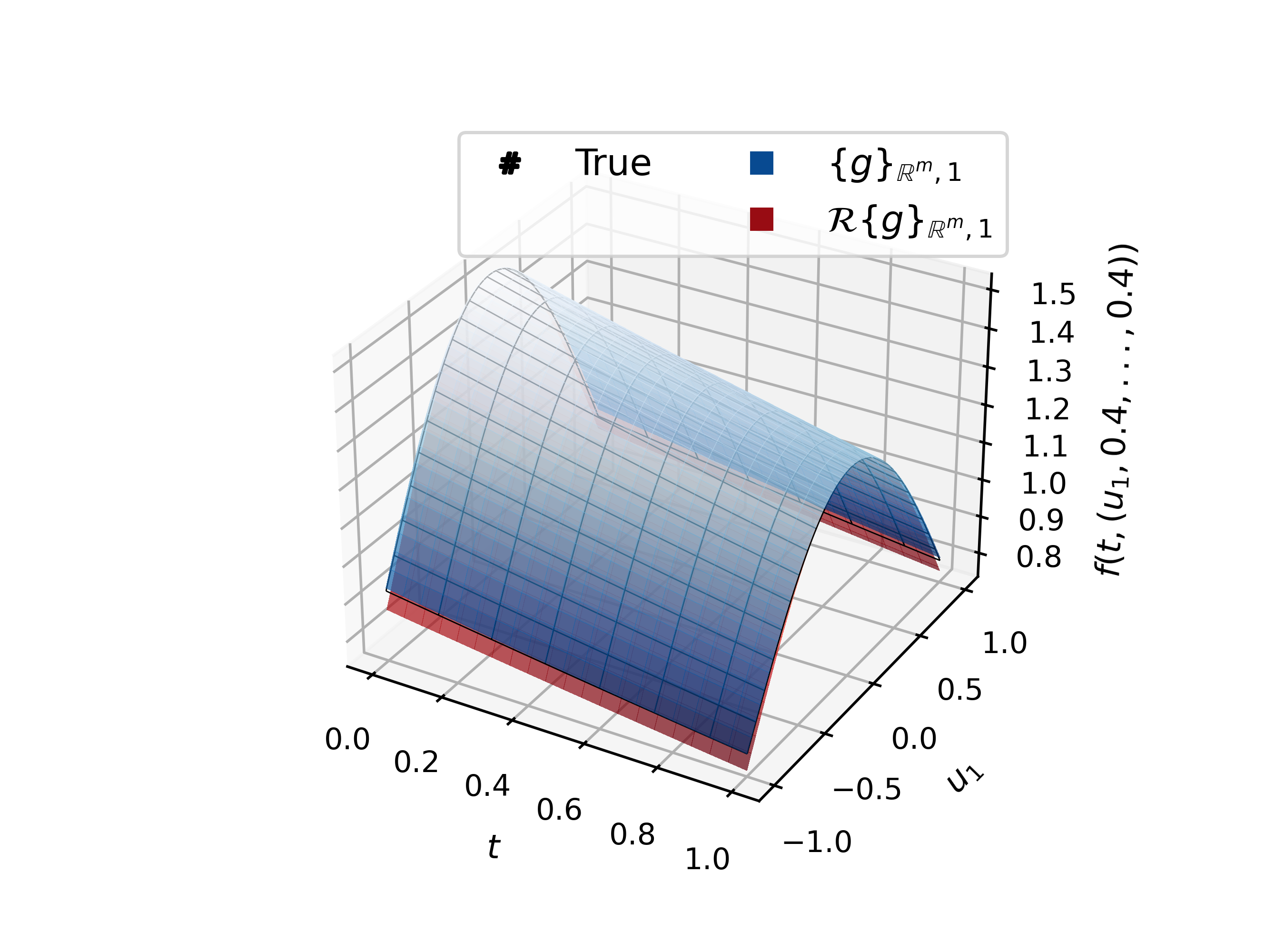}
		\subcaption{Approximation of the function $[0,T] \times \mathbb{R} \ni (t,u_1) \mapsto f(t,(u_1,0.4,...,0.4)) \in \mathbb{R}$ for $m = 10$.}
	\end{minipage}
	\begin{minipage}[t]{0.48\textwidth}
		\centering
		\includegraphics[height = 5.5cm, trim = {0.8cm 0 0 0.4cm}, clip]{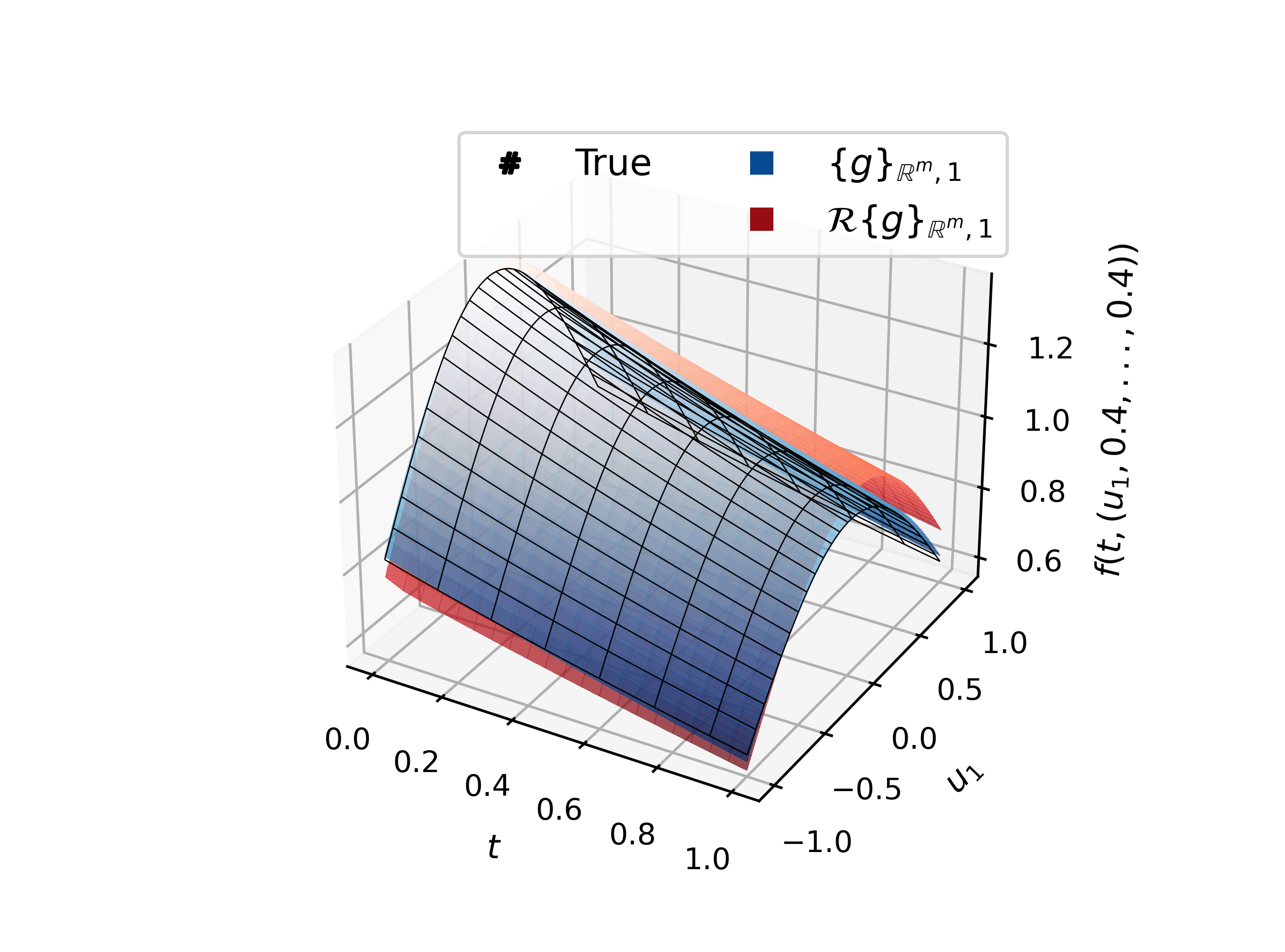}
		\subcaption{Approximation of the function $[0,T] \times \mathbb{R} \ni (t,u_1) \mapsto f(t,(u_1,0.4,...,0.4)) \in \mathbb{R}$ for $m = 20$.}
	\end{minipage}
	\hspace{0.02\textwidth}
	\begin{minipage}[t]{0.48\textwidth}
		\centering
		\includegraphics[height = 5.5cm, trim = {0.8cm 0 0 0.4cm}, clip]{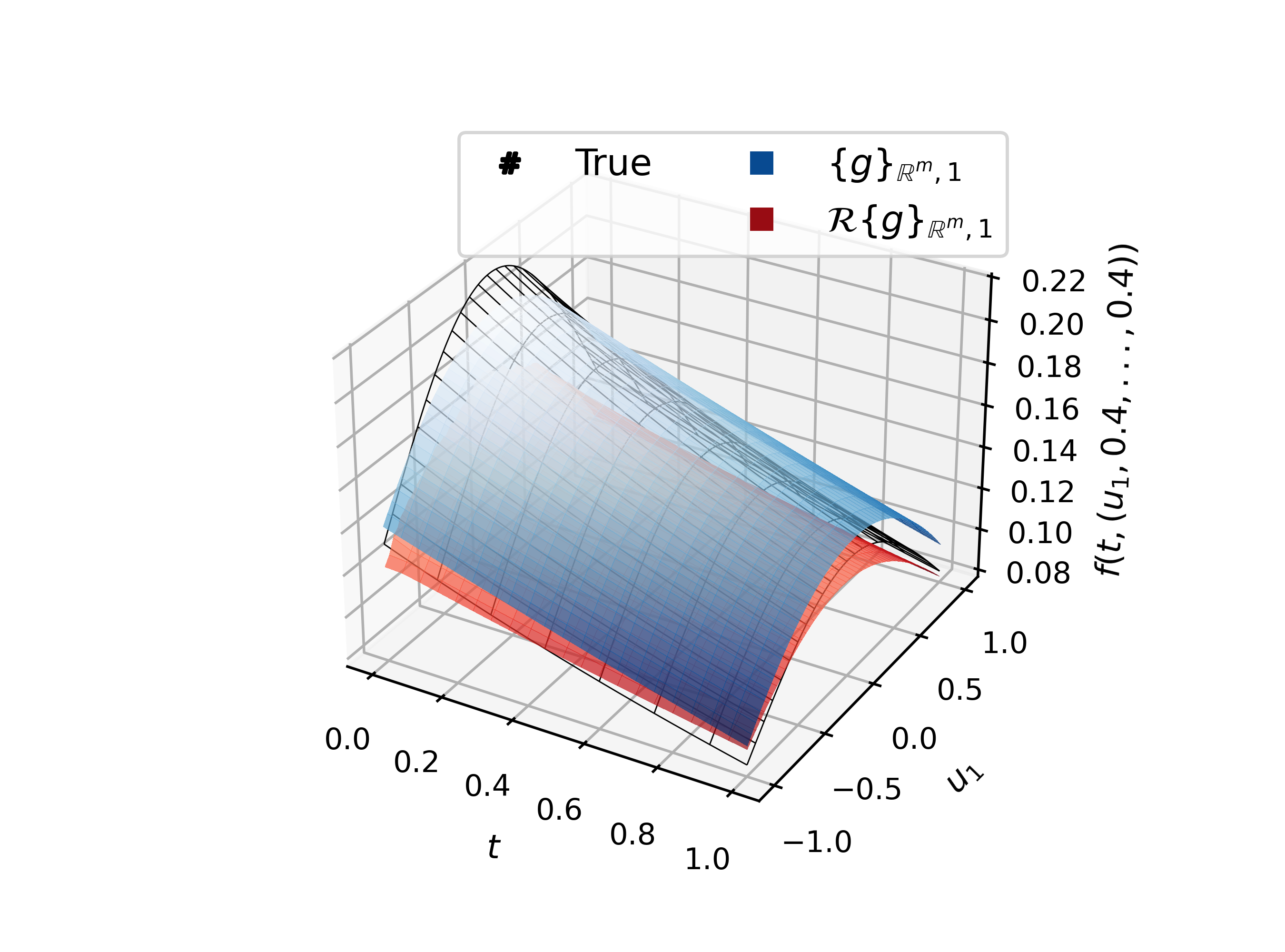}
		\subcaption{Approximation of the function $[0,T] \times \mathbb{R} \ni (t,u_1) \mapsto f(t,(u_1,0.4,...,0.4)) \in \mathbb{R}$ for $m = 30$.}
	\end{minipage}
	\caption{Spatio-temporal learning of the heat equation \eqref{EqDefHeat}.}
	\label{FigHeatUSV}
\end{figure}

\begin{table}
	\centering
	\begin{scriptsize}
		\begin{tabular}{ll|R{1.2cm}|R{1.2cm}|R{1.2cm}|R{1.2cm}|R{1.2cm}|R{1.2cm}|R{1.2cm}}
			\multicolumn{2}{r|}{$N = 10 \sqrt{J \ln(J)}$} & \multicolumn{1}{c|}{$J = 2^{12}$} & \multicolumn{1}{c|}{$J = 2^{13}$} & \multicolumn{1}{c|}{$J = 2^{14}$} & \multicolumn{1}{c|}{$J = 2^{15}$} & \multicolumn{1}{c|}{$J = 2^{16}$} & \multicolumn{1}{c}{$J = 2^{17}$} & \multicolumn{1}{c}{$J = 2^{18}$} \\
			\hline
			\multirow{4}{*}{$m = 10$} & \multirow{2}{*}{$\lbrace g \rbrace_{\mathbb{R}^m,1}$} & \textbf{0.123}  & \textbf{0.114}  & \textbf{0.095}  & \textbf{0.063}  & \textbf{0.053}  & \textbf{0.046}  & \textbf{0.029} \\ 
 & & \textit{37.25} s  & \textit{68.73} s  & \textit{127.17} s  & \textit{239.23} s  & \textit{560.76} s  & \textit{1007.78} s  & \textit{1871.33} s \\ 
& \multirow{2}{*}{$\mathcal{R}\lbrace g \rbrace_{\mathbb{R}^m,1}$} & \cellcolor{gray!40} \textbf{0.138}  & \cellcolor{gray!40} \textbf{0.123}  & \cellcolor{gray!40} \textbf{0.090}  & \cellcolor{gray!40} \textbf{0.073}  & \cellcolor{gray!40} \textbf{0.060}  & \cellcolor{gray!40} \textbf{0.052}  & \cellcolor{gray!40} \textbf{0.042} \\ 
 & & \cellcolor{gray!40} \textit{0.11} s & \cellcolor{gray!40} \textit{0.30} s & \cellcolor{gray!40} \textit{0.87} s & \cellcolor{gray!40} \textit{2.67} s & \cellcolor{gray!40} \textit{7.49} s & \cellcolor{gray!40} \textit{27.02} s & \cellcolor{gray!40} \textit{318.27} s \\ 
\hline 
\multirow{4}{*}{$m = 20$} & \multirow{2}{*}{$\lbrace g \rbrace_{\mathbb{R}^m,1}$} & \textbf{0.240}  & \textbf{0.213}  & \textbf{0.181}  & \textbf{0.159}  & \textbf{0.136}  & \textbf{0.110}  & \textbf{0.081} \\ 
 & & \textit{33.54} s  & \textit{79.33} s  & \textit{140.58} s  & \textit{216.41} s  & \textit{542.13} s  & \textit{1006.61} s  & \textit{1937.19} s \\ 
& \multirow{2}{*}{$\mathcal{R}\lbrace g \rbrace_{\mathbb{R}^m,1}$} & \cellcolor{gray!40} \textbf{0.300}  & \cellcolor{gray!40} \textbf{0.260}  & \cellcolor{gray!40} \textbf{0.242}  & \cellcolor{gray!40} \textbf{0.197}  & \cellcolor{gray!40} \textbf{0.173}  & \cellcolor{gray!40} \textbf{0.141}  & \cellcolor{gray!40} \textbf{0.122} \\ 
 & & \cellcolor{gray!40} \textit{0.13} s & \cellcolor{gray!40} \textit{0.33} s & \cellcolor{gray!40} \textit{1.05} s & \cellcolor{gray!40} \textit{2.79} s & \cellcolor{gray!40} \textit{8.40} s & \cellcolor{gray!40} \textit{35.78} s & \cellcolor{gray!40} \textit{368.16} s \\ 
\hline 
\multirow{4}{*}{$m = 30$} & \multirow{2}{*}{$\lbrace g \rbrace_{\mathbb{R}^m,1}$} & \textbf{0.042}  & \textbf{0.040}  & \textbf{0.037}  & \textbf{0.039}  & \textbf{0.034}  & \textbf{0.032}  & \textbf{0.031} \\ 
 & & \textit{34.04} s  & \textit{78.49} s  & \textit{127.62} s  & \textit{283.79} s  & \textit{474.76} s  & \textit{946.81} s  & \textit{1940.34} s \\ 
& \multirow{2}{*}{$\mathcal{R}\lbrace g \rbrace_{\mathbb{R}^m,1}$} & \cellcolor{gray!40} \textbf{0.055}  & \cellcolor{gray!40} \textbf{0.050}  & \cellcolor{gray!40} \textbf{0.044}  & \cellcolor{gray!40} \textbf{0.044}  & \cellcolor{gray!40} \textbf{0.039}  & \cellcolor{gray!40} \textbf{0.036}  & \cellcolor{gray!40} \textbf{0.033} \\ 
 & & \cellcolor{gray!40} \textit{0.12} s & \cellcolor{gray!40} \textit{0.31} s & \cellcolor{gray!40} \textit{0.81} s & \cellcolor{gray!40} \textit{2.56} s & \cellcolor{gray!40} \textit{7.56} s & \cellcolor{gray!40} \textit{27.21} s & \cellcolor{gray!40} \textit{341.02} s 
		\end{tabular}
	\end{scriptsize}
	\vspace{0.3cm}
	\caption{Spatio-temporal learning of the heat equation \eqref{EqDefHeat}: Empirical $L^2$-error defined in \eqref{EqDefMSE2} on the test set (in bold letters) and running time (in italic letters; in seconds).}
	\label{TabHeatUSV}
\end{table}

\subsection{Non-linear Fokker-Planck equation}

Second, we consider a non-linear Fokker-Planck equation describing the forward equation of the corresponding (density-dependent) McKean-Vlasov stochastic differential equation (SDE). More precisely, we consider the PDE
\begin{equation}
	\label{EqDefFokker}
	\left\lbrace
	\begin{matrix*}[l]
		\frac{\partial f}{\partial t}(t,u) - \nabla \cdot \big( D \nabla f(t,u) + f(t,u) \nabla \mathcal{L}(f(t,\cdot))(u) \big) & = 0, & \quad (t,u) & \in (0,\infty) \times \mathbb{R}^m, \\
		\quad\quad\quad\quad\quad\quad\quad\quad\quad\quad\quad\quad\quad\quad\quad\quad\quad\quad\, f(0,u) & = f_0(u), & \quad\quad\,\,\, u & \in \mathbb{R}^m,
	\end{matrix*}
	\right.
\end{equation}
where $\nabla \cdot F := \big( u \mapsto \sum_{l=1}^m \frac{\partial F_l}{\partial u_l}(u) \big)$ for $F := (F_1,...,F_m)^\top \in C^1(\mathbb{R}^m;\mathbb{R}^m)$, where $D \in \mathbb{S}^m_{++}$, where\footnote{Here, $(V * h)(u) := \int_{\mathbb{R}^m} V(u-y) h(y) dy$ denotes the convolution of $V$ with $h$.} $\mathcal{L}(h)(u) := U(u) + (V * h)(u)$ for some potential $U \in C^2(\mathbb{R}^m)$ and mean-field interaction term $V \in C^2(\mathbb{R}^m)$, and where $f_0: \mathbb{R}^m \rightarrow [0,\infty)$ is the initial condition satisfying $\int_{\mathbb{R}^m} f_0(u) du = 1$. If $f: [0,\infty) \times \mathbb{R}^m \rightarrow \mathbb{R}$ is a sufficiently regular weak solution of \eqref{EqDefFokker}, then $f(t,\cdot)$ describes the probability density function at time $t$ of the (weak) solution $(Z_t)_{t \in [0,\infty)}$ to the (density-dependent) McKean-Vlasov SDE
\begin{equation}
	\label{EqDefMcKV}
	\left\lbrace
	\begin{matrix*}[l]
		dZ_t & = -\nabla \mathcal{L}(f(t,\cdot)) dt + \sqrt{2D} dW_t, & \quad t & \in (0,\infty), \\
		\,\, Z_0 & \sim f_0(u) du, &
	\end{matrix*}
	\right.
\end{equation}
where $\sqrt{2D} \in \mathbb{S}^m_{++}$ denotes the matrix square root and $(W_t)_{t \in [0,\infty)}$ is an $m$-dimensional Brownian motion. Indeed, \eqref{EqDefMcKV} follows from \eqref{EqDefFokker} by applying Ito's formula to $\varphi(Z_t)$ and verifying that $\frac{d}{dt} \mathbb{E}\left[ \varphi(Z_t) \right] = \frac{d}{dt} \int_{\mathbb{R}^m} f(t,u) \varphi(u) du = \int_{\mathbb{R}^m} \frac{\partial f}{\partial t}(t,u) \varphi(u) du$ (see also \cite{sznitman91,frank05}).

Now, if $f_0(u) := (2\pi)^{-m/2} (\det(\Sigma_0))^{-1/2} \exp\left( -0.5 (u-\mu_0)^\top \Sigma_0^{-1} (u-\mu_0) \right)$, i.e.~$Z_0 \sim \mathcal{N}_m(\mu_0,\Sigma_0)$ is Gaussian, and $U$ and $V$ are quadratic, i.e.~$U(u) = \frac{1}{2} u^\top C_1 u + c_1^\top u + d_1$ and $V(u) = \frac{1}{2} u^\top C_2 u + c_2^\top u + d_2$ with $C_1, C_2 \in \mathbb{S}^m_{++}$, $c_1,c_2 \in \mathbb{R}^m$, and $d_1,d_2 \in \mathbb{R}$ (implying that $\nabla \mathcal{L}(h)(u) = (C_1+C_2) u - C_2 \mu_h + c_2$, where $\mu_h := \int_{\mathbb{R}^m} u h(u) du$), then $Z_t \sim \mathcal{N}_m(\mu_t,\Sigma_t)$ is Gaussian with mean process $(\mu_t)_{t \in [0,\infty)}$ and covariance process $(\Sigma_t)_{t \in [0,\infty)}$ satisfying
\begin{equation*}
	\left\lbrace
	\begin{matrix*}[l]
		\frac{\partial \mu_t}{\partial t} & = \frac{d}{dt} \mathbb{E}[Z_t] & = -C_1 \mu_t - c_2, \\
		\frac{\partial \Sigma_t}{\partial t} & = \frac{d}{dt} \mathbb{E}\left[ (Z_t-\mu_t) (Z_t-\mu_t)^\top \right] & = -(C_1+C_2) \Sigma_t - \Sigma_t (C_1+C_2) + 2 D.
	\end{matrix*}
	\right.
\end{equation*}
The solutions of the latter ordinary differential equations are for every $t \in [0,\infty)$ given by
\begin{equation}
	\label{EqFokkerMeanCov}
	\left\lbrace
	\begin{matrix*}[l]
		\mu_t & = e^{-t C_1} \mu_0 - \int_0^t e^{-(t-s) C_1} c_2 ds, \\
		\Sigma_t & = e^{-t (C_1+C_2)} \Sigma_0 e^{-t (C_1+C_2)} + \int_0^t e^{-(t-s) (C_1+C_2)} 2D e^{-(t-s) (C_1+C_2)} ds.
	\end{matrix*}
	\right.
\end{equation}
Hence, the solution $f: [0,\infty) \times \mathbb{R}^m \rightarrow \mathbb{R}$ of \eqref{EqDefFokker} is Gaussian given by
\begin{equation}
	\label{EqSolFokker}
	[0,\infty) \times \mathbb{R}^m \ni (t,u) \,\, \mapsto \,\, f(t,u) := \frac{1}{\sqrt{(2\pi)^m \det(\Sigma_t)}} \exp\left( -\frac{1}{2} (u - \mu_t)^\top \Sigma_t^{-1} (u-\mu_t) \right) \in \mathbb{R},
\end{equation}
which describes the probability density function of \eqref{EqDefMcKV}.

\subsubsection{Fixed-time learning}

For some fixed $t \in [0,T]$, we now learn the function $f(t,\cdot)$ by random trigonometric feature models, random neural networks, and their deterministic counterparts. Moreover, we provide sufficient conditions that the approximation of $f(t,\cdot)$ by random feature models overcomes the curse of dimensionality. To this end, we denote by $\lambda_{\min}(\Sigma_t) > 0$ the smallest eigenvalue of the matrix $\Sigma_t \in \mathbb{S}^m_{++}$.

\begin{corollary}
	\label{CorFokker}
	Let $t \in (0,\infty)$ and assume that $f(t,\cdot): \mathbb{R}^m \rightarrow \mathbb{R}$ is the solution of \eqref{EqDefFokker} given in \eqref{EqSolFokker}. Moreover, let $p \in (1,\infty)$, $\gamma \in [0,\infty)$, and $w: \mathbb{R}^m \rightarrow [0,\infty)$ be as in Lemma~\ref{LemmaWeight}. Then, the following holds true:
	\begin{enumerate}
		\item\label{CorFokker1} Let $(\theta_n)_{n \in \mathbb{N}} \sim t_m$. Then, there exists $C_{14},C_{15} > 0$ (being independent of $m \in \mathbb{N}$) such that for every $N \in \mathbb{N}$ there exists a random trigonometric feature model $G_N \in \mathcal{RT}_{\mathbb{R}^m,1}$ with $N$ features satisfying
		\begin{equation}
			\label{EqCorFokker1}
			\mathbb{E}\left[ \Vert f(t,\cdot) - G_N \Vert_{L^p(\mathbb{R}^m,\mathcal{L}(\mathbb{R}^m),w)}^2 \right]^\frac{1}{2} \leq C_{14} m^{C_{15}} \frac{\left( \frac{2}{\lambda_{\min}(\Sigma_t) \pi^2} \right)^\frac{m}{4}}{N^{1-\frac{1}{\min(2,p)}}}.
		\end{equation}
		In particular, if $\lambda_{\min}(\Sigma_t) \geq 2/\pi^2$, then there exists $C_{16},C_{17} > 0$ such that for every $m \in \mathbb{N}$ and $\varepsilon > 0$ there exists a random trigonometric feature model $G_N \in \mathcal{RT}_{\mathbb{R}^m,1}$ with $N = \big\lceil C_{16} m^{C_{17}} \varepsilon^{-\frac{\min(2,p)}{\min(2,p)-1}} \big\rceil$ features satisfying
		\begin{equation}
			\label{EqCorFokker2}
			\mathbb{E}\left[ \Vert f(t,\cdot) - G_N \Vert_{L^p(\mathbb{R}^m,\mathcal{L}(\mathbb{R}^m),w)}^2 \right]^\frac{1}{2} \leq \varepsilon.
		\end{equation}
		\item\label{CorFokker2} Let $(a_n,b_n)_{n \in \mathbb{N}} \sim t_m \otimes t_1$ be i.i.d., and let $(\psi,\rho) \in \mathcal{S}_0(\mathbb{R};\mathbb{C}) \times C^0_{pol,\gamma}(\mathbb{R})$ be as in Example~\ref{ExAdm} (with $0 < \zeta_1 < \zeta_2 < \infty$). Then, $f(t,\cdot) \in \widetilde{\mathbb{B}}^{0,2,\gamma}_{\psi,a,b}(\mathbb{R}^m)$ and there exist $C_{18},C_{19} > 0$ (being independent of $m \in \mathbb{N}$) such that for every $N \in \mathbb{N}$ there exists a random neural network $G_N \in \mathcal{RN}^\rho_{\mathbb{R}^m,1}$ with $N$ neurons satisfying
		\begin{equation}
			\label{EqCorFokker3}
			\mathbb{E}\left[ \Vert f(t,\cdot) - G_N \Vert_{L^p(\mathbb{R}^m,\mathcal{L}(\mathbb{R}^m),w)}^2 \right]^\frac{1}{2} \leq C_{18} m^{C_{19}} (1 + \Vert \mu_t \Vert + \Vert \Sigma_t \Vert)^{\lceil\gamma\rceil+2} \frac{\left( \frac{2}{\lambda_{\min}(\Sigma_t) \pi^2} \right)^\frac{m}{4}}{N^{1-\frac{1}{\min(2,p)}}}.
		\end{equation}
		In particular, $\lambda_{\min}(\Sigma_t) \geq 2/\pi^2$ and $1+\Vert \mu_t \Vert+\Vert \Sigma_t \Vert = \mathcal{O}(m^\eta)$ for some $\eta > 0$, then there exist $C_{20},C_{21} > 0$ such that for every $m \in \mathbb{N}$ and $\varepsilon > 0$ there exists a random neural network $G_N \in \mathcal{RN}^\rho_{\mathbb{R}^m,1}$ with $N = \big\lceil C_{20} m^{C_{21}} \varepsilon^{-\frac{\min(2,p)}{\min(2,p)-1}} \big\rceil$ neurons satisfying
		\begin{equation}
			\label{EqCorFokker4}
			\mathbb{E}\left[ \Vert f(t,\cdot) - G_N \Vert_{L^p(\mathbb{R}^m,\mathcal{L}(\mathbb{R}^m),w)}^2 \right]^\frac{1}{2} \leq \varepsilon.
		\end{equation}
	\end{enumerate}
\end{corollary}

Now, we choose $t = 1$, the initial condition $f_0(u) := (2\pi \sigma_0^2)^{m/2} \exp\left( -\Vert u \Vert^2/(2 \sigma_0^2) \right)$ with $\sigma_0 := 0.5$, the coefficients $C_1 := C_2 := 0.1 m^{-1/2} I_m \in \mathbb{S}^m_{++}$, $c_1 := 0 \in \mathbb{R}^m$, $c_2 := 0.1 m^{-1} (1,...,1)^\top \in \mathbb{R}^m$, $d_1 := d_2 := 0 \in \mathbb{R}$ and $D := 0.2 m^{-1/2} I_m \in \mathbb{S}^m_{++}$, and the weight function $w(u) := (2\pi \sigma_w^2)^{m/2} \exp\left( -\Vert u \Vert^2/(2 \sigma_w^2) \right)$ with $\sigma_w := 0.3$ satisfying the conditions of Lemma~\ref{LemmaWeight}. Then, we generate $J = 2^{12},2^{13},...,2^{18}$ i.i.d.~data $(V_j)_{j=1,...,J} \sim w$, split up into $80\%/20\%$ for training/testing, and minimize the empirical $L^2$-error \eqref{EqDefMSE} over random trigonometric feature models $G_N \in \mathcal{RT}_{\mathbb{R}^m,1}$, random neural networks $G_N \in \mathcal{RN}^{\tanh}_{\mathbb{R}^m,1}$, and their deterministic counterparts, all of them having $N = 10 \sqrt{J/\ln(J)}$ features. For training the deterministic models, we apply the Adam algorithm over $5000$ epochs (with learning rate $5 \cdot 10^{-5}$ and batchsize $1000$).

Moreover, by using that $\mu_0 := 0 \in \mathbb{R}^m$ and $\Sigma_0 := \operatorname{diag}(\sigma_0,...,\sigma_0) \in \mathbb{S}^m_{++}$ and inserting the coefficients $C_1, C_2 \in \mathbb{S}^m_{++}$, $c_1, c_2 \in \mathbb{R}^m$, $d_1, d_2 \in \mathbb{R}$ into \eqref{EqFokkerMeanCov}, it holds for every $t \in [0,\infty)$ that
\begin{equation*}
	\Vert \mu_t \Vert \leq e^{-\lambda_{\min}(C_1) t} \Vert \mu_0 \Vert + \frac{\Vert c_2 \Vert}{\lambda_{\min}(C_1)} \left( 1 \!-\! e^{-\lambda_{\min}(C_1) t} \right) \leq \Vert \mu_0 \Vert + \frac{\Vert c_2 \Vert}{\lambda_{\min}(C_1)} \leq \frac{0.1 m^{-\frac{1}{2}}}{0.1 m^{-\frac{1}{2}}} \leq 1.
\end{equation*}
By similar arguments, we have for every $t \in [0,\infty)$ that
\begin{equation*}
	\begin{aligned}
		\Vert \Sigma_t \Vert & \leq e^{-\lambda_{\min}(C_1+C_2) t} \Vert \Sigma_0 \Vert e^{-\lambda_{\min}(C_1+C_2) t} + \frac{\Vert D \Vert}{\lambda_{\min}(C_1+C_2)} \left( 1-e^{-2\lambda_{\min}(C_1+C_2) t} \right) \\
		& \leq \Vert \Sigma_0 \Vert + \frac{\Vert D \Vert}{\lambda_{\min}(C_1+C_2)} = 0.5 + \frac{0.2 m^{-\frac{1}{2}}}{0.2 m^{-\frac{1}{2}}} \leq 1.5.
	\end{aligned}
\end{equation*}
In addition, we observe for every $t \in [0,T]$ with $T = 1$ that
\begin{equation*}
	\begin{aligned}
		& \lambda_{\min}(\Sigma_t) \geq e^{-\lambda_{\max}(C_1+C_2) t} \lambda_{\min}(\Sigma_0) e^{-\lambda_{\max}(C_1+C_2) t} \!+\! \frac{\lambda_{\min}(D)}{\lambda_{\max}(C_1 \!+\! C_2)} \left( 1 \!-\! e^{-2\lambda_{\max}(C_1+C_2) t} \right) \\
		& \,\,\, \geq \frac{\lambda_{\min}(D)}{\lambda_{\max}(C_1+C_2)} \left( 1-e^{-2\lambda_{\max}(C_1+C_2) T} \right) = \frac{0.2 m^{-\frac{1}{2}}}{0.2 m^{-\frac{1}{2}}} \left( 1 - e^{-2 \cdot 0.2 m^{-1/2} \cdot 1} \right) \approx 0.33 \geq \frac{2}{\pi^2},
	\end{aligned}
\end{equation*}
where $\lambda_{\max}(C_1+C_2) > 0$ denotes the largest eigenvalue of the matrix $C_1+C_2 \in \mathbb{S}^m_{++}$. Hence, Corollary~\ref{CorFokker} shows that the approximation of $f(t,\cdot)$ by random trigonometric features and random neural networks overcomes the curse of dimensionality.

Figure~\ref{FigFokker} empirically demonstrates that random trigonometric feature models as well as random neural networks are indeed able to learn the solution of the non-linear Fokker-Planck equation \eqref{EqDefFokker}. However, in terms of computational time, the random feature models significantly outperform their deterministic counterparts (see also Table~\ref{TabFokker}).

\begin{figure}[h]
	\begin{minipage}[t]{0.48\textwidth}
		\centering
		\includegraphics[height = 5.2cm]{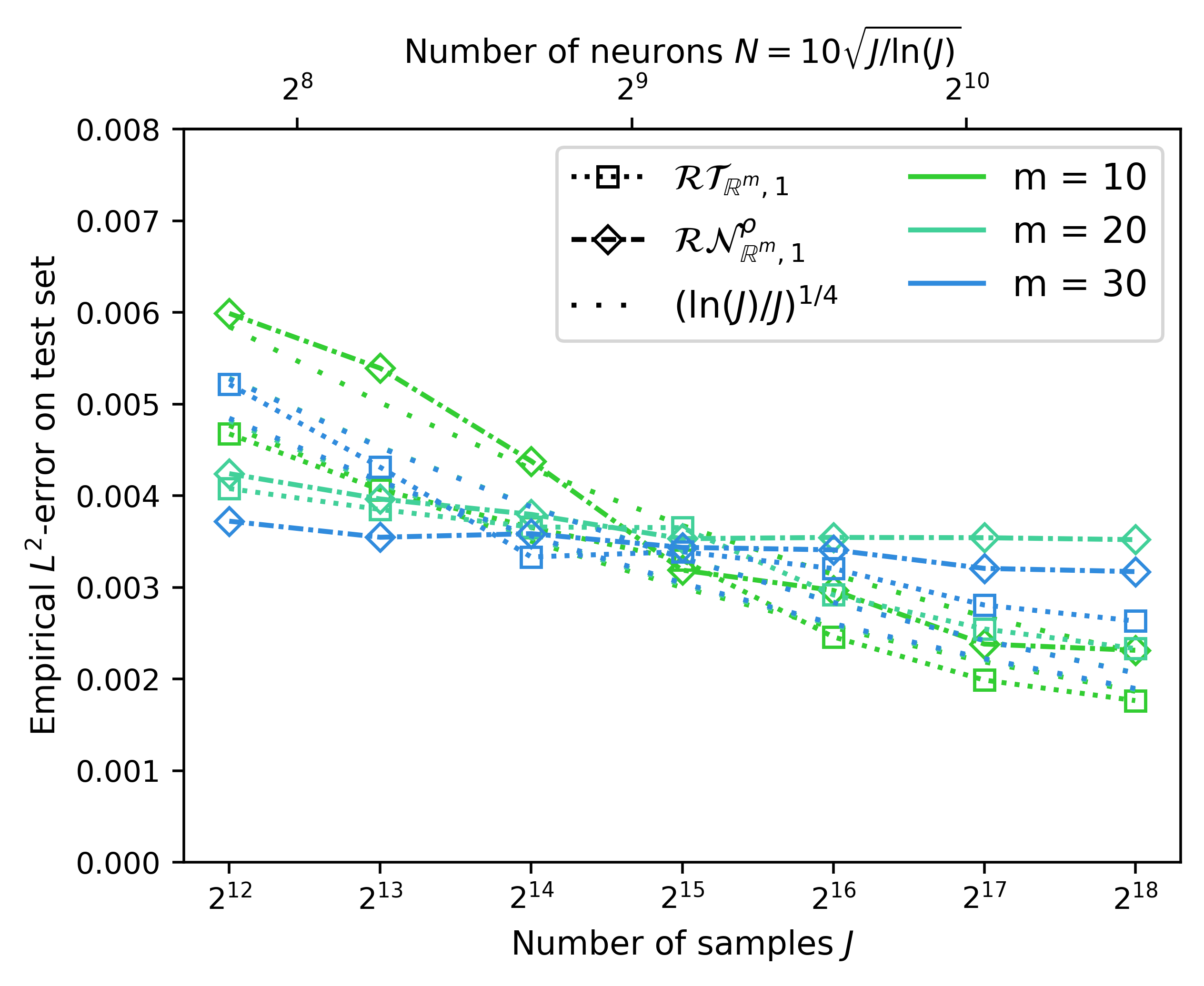}
		\subcaption{Empirical $L^2$-error defined in \eqref{EqDefMSE}.}
	\end{minipage}
	\hspace{0.02\textwidth}
	\begin{minipage}[t]{0.48\textwidth}
		\centering
		\includegraphics[height = 5.2cm]{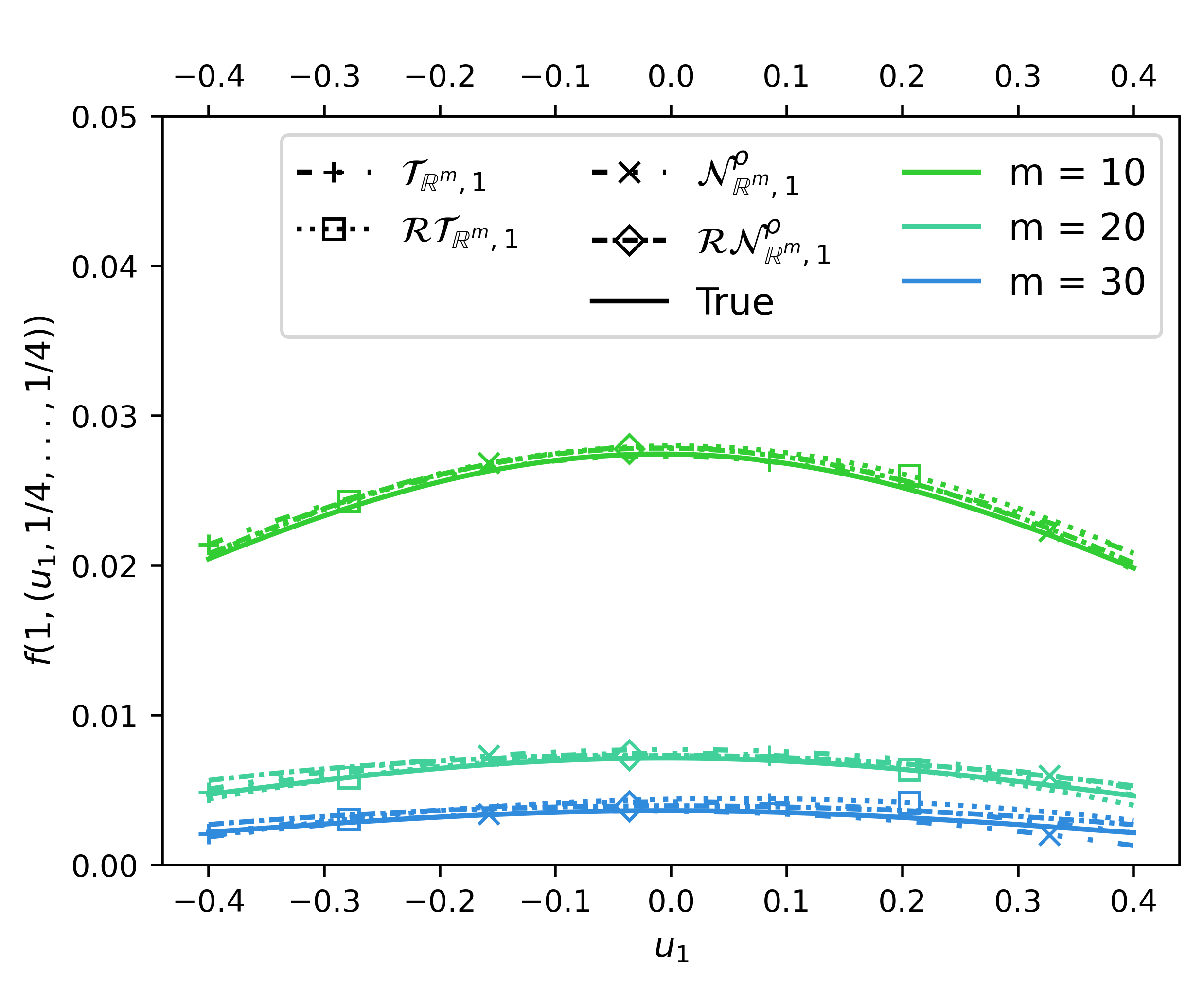}
		\subcaption{Approximation of the function $\mathbb{R} \ni u_1 \mapsto f(1,(u_1,1/4,...,1/4)) \in \mathbb{R}$.}
	\end{minipage}
	\caption{Fixed-time learning of the non-linear Fokker-Planck equation \eqref{EqDefFokker}.}
	\label{FigFokker}
\end{figure}

\begin{table}
	\centering
	\begin{scriptsize}
		\begin{tabular}{ll|R{1.2cm}|R{1.2cm}|R{1.2cm}|R{1.2cm}|R{1.2cm}|R{1.2cm}|R{1.2cm}}
			\multicolumn{2}{r|}{$N = 10 \sqrt{J \ln(J)}$} & \multicolumn{1}{c|}{$J = 2^{12}$} & \multicolumn{1}{c|}{$J = 2^{13}$} & \multicolumn{1}{c|}{$J = 2^{14}$} & \multicolumn{1}{c|}{$J = 2^{15}$} & \multicolumn{1}{c|}{$J = 2^{16}$} & \multicolumn{1}{c}{$J = 2^{17}$} & \multicolumn{1}{c}{$J = 2^{18}$} \\
			\hline
			\multirow{8}{*}{$m = 10$} & \multirow{2}{*}{$\mathcal{T}_{\mathbb{R}^m,1}$} & \textbf{0.006}  & \textbf{0.005}  & \textbf{0.006}  & \textbf{0.006}  & \textbf{0.005}  & \textbf{0.004}  & \textbf{0.004} \\ 
 & & \textit{32.50} s  & \textit{77.89} s  & \textit{111.28} s  & \textit{226.13} s  & \textit{455.12} s  & \textit{958.22} s  & \textit{1845.27} s \\ 
& \multirow{2}{*}{$\mathcal{RT}_{\mathbb{R}^m,1}$} & \cellcolor{gray!40} \textbf{0.005}  & \cellcolor{gray!40} \textbf{0.004}  & \cellcolor{gray!40} \textbf{0.004}  & \cellcolor{gray!40} \textbf{0.003}  & \cellcolor{gray!40} \textbf{0.002}  & \cellcolor{gray!40} \textbf{0.002}  & \cellcolor{gray!40} \textbf{0.002} \\ 
 & & \cellcolor{gray!40} \textit{0.03} s & \cellcolor{gray!40} \textit{0.10} s & \cellcolor{gray!40} \textit{0.32} s & \cellcolor{gray!40} \textit{1.21} s & \cellcolor{gray!40} \textit{4.28} s & \cellcolor{gray!40} \textit{15.09} s & \cellcolor{gray!40} \textit{54.51} s \\ 
& \multirow{2}{*}{$\mathcal{N}_{\mathbb{R}^m,1}$} & \textbf{0.010}  & \textbf{0.007}  & \textbf{0.006}  & \textbf{0.006}  & \textbf{0.006}  & \textbf{0.005}  & \textbf{0.005} \\ 
 & & \textit{25.36} s  & \textit{47.45} s  & \textit{125.98} s  & \textit{195.95} s  & \textit{401.79} s  & \textit{813.33} s  & \textit{1635.54} s \\ 
& \multirow{2}{*}{$\mathcal{RN}_{\mathbb{R}^m,1}$} & \cellcolor{gray!40} \textbf{0.006}  & \cellcolor{gray!40} \textbf{0.005}  & \cellcolor{gray!40} \textbf{0.004}  & \cellcolor{gray!40} \textbf{0.003}  & \cellcolor{gray!40} \textbf{0.003}  & \cellcolor{gray!40} \textbf{0.002}  & \cellcolor{gray!40} \textbf{0.002} \\ 
 & & \cellcolor{gray!40} \textit{0.03} s & \cellcolor{gray!40} \textit{0.10} s & \cellcolor{gray!40} \textit{0.34} s & \cellcolor{gray!40} \textit{1.28} s & \cellcolor{gray!40} \textit{4.88} s & \cellcolor{gray!40} \textit{20.06} s & \cellcolor{gray!40} \textit{60.06} s \\ 
\hline 
\multirow{8}{*}{$m = 20$} & \multirow{2}{*}{$\mathcal{T}_{\mathbb{R}^m,1}$} & \textbf{0.004}  & \textbf{0.004}  & \textbf{0.004}  & \textbf{0.004}  & \textbf{0.004}  & \textbf{0.004}  & \textbf{0.004} \\ 
 & & \textit{31.84} s  & \textit{57.27} s  & \textit{133.30} s  & \textit{223.47} s  & \textit{471.27} s  & \textit{937.77} s  & \textit{1894.66} s \\ 
& \multirow{2}{*}{$\mathcal{RT}_{\mathbb{R}^m,1}$} & \cellcolor{gray!40} \textbf{0.004}  & \cellcolor{gray!40} \textbf{0.004}  & \cellcolor{gray!40} \textbf{0.004}  & \cellcolor{gray!40} \textbf{0.004}  & \cellcolor{gray!40} \textbf{0.003}  & \cellcolor{gray!40} \textbf{0.003}  & \cellcolor{gray!40} \textbf{0.002} \\ 
 & & \cellcolor{gray!40} \textit{0.03} s & \cellcolor{gray!40} \textit{0.08} s & \cellcolor{gray!40} \textit{0.31} s & \cellcolor{gray!40} \textit{1.15} s & \cellcolor{gray!40} \textit{4.03} s & \cellcolor{gray!40} \textit{14.10} s & \cellcolor{gray!40} \textit{50.85} s \\ 
& \multirow{2}{*}{$\mathcal{N}_{\mathbb{R}^m,1}$} & \textbf{0.005}  & \textbf{0.005}  & \textbf{0.004}  & \textbf{0.004}  & \textbf{0.004}  & \textbf{0.004}  & \textbf{0.004} \\ 
 & & \textit{29.34} s  & \textit{53.52} s  & \textit{101.96} s  & \textit{211.52} s  & \textit{408.84} s  & \textit{844.67} s  & \textit{1651.71} s \\ 
& \multirow{2}{*}{$\mathcal{RN}_{\mathbb{R}^m,1}$} & \cellcolor{gray!40} \textbf{0.004}  & \cellcolor{gray!40} \textbf{0.004}  & \cellcolor{gray!40} \textbf{0.004}  & \cellcolor{gray!40} \textbf{0.004}  & \cellcolor{gray!40} \textbf{0.004}  & \cellcolor{gray!40} \textbf{0.004}  & \cellcolor{gray!40} \textbf{0.004} \\ 
 & & \cellcolor{gray!40} \textit{0.04} s & \cellcolor{gray!40} \textit{0.10} s & \cellcolor{gray!40} \textit{0.33} s & \cellcolor{gray!40} \textit{1.31} s & \cellcolor{gray!40} \textit{4.30} s & \cellcolor{gray!40} \textit{18.30} s & \cellcolor{gray!40} \textit{56.88} s \\ 
\hline 
\multirow{8}{*}{$m = 30$} & \multirow{2}{*}{$\mathcal{T}_{\mathbb{R}^m,1}$} & \textbf{0.004}  & \textbf{0.003}  & \textbf{0.004}  & \textbf{0.004}  & \textbf{0.003}  & \textbf{0.004}  & \textbf{0.004} \\ 
 & & \textit{32.10} s  & \textit{68.04} s  & \textit{122.23} s  & \textit{223.63} s  & \textit{466.83} s  & \textit{933.89} s  & \textit{1790.32} s \\ 
& \multirow{2}{*}{$\mathcal{RT}_{\mathbb{R}^m,1}$} & \cellcolor{gray!40} \textbf{0.005}  & \cellcolor{gray!40} \textbf{0.004}  & \cellcolor{gray!40} \textbf{0.003}  & \cellcolor{gray!40} \textbf{0.003}  & \cellcolor{gray!40} \textbf{0.003}  & \cellcolor{gray!40} \textbf{0.003}  & \cellcolor{gray!40} \textbf{0.003} \\ 
 & & \cellcolor{gray!40} \textit{0.03} s & \cellcolor{gray!40} \textit{0.09} s & \cellcolor{gray!40} \textit{0.31} s & \cellcolor{gray!40} \textit{1.22} s & \cellcolor{gray!40} \textit{4.08} s & \cellcolor{gray!40} \textit{13.26} s & \cellcolor{gray!40} \textit{47.93} s \\ 
& \multirow{2}{*}{$\mathcal{N}_{\mathbb{R}^m,1}$} & \textbf{0.005}  & \textbf{0.004}  & \textbf{0.003}  & \textbf{0.004}  & \textbf{0.004}  & \textbf{0.004}  & \textbf{0.004} \\ 
 & & \textit{29.44} s  & \textit{61.95} s  & \textit{121.96} s  & \textit{200.03} s  & \textit{418.05} s  & \textit{827.91} s  & \textit{1691.33} s \\ 
& \multirow{2}{*}{$\mathcal{RN}_{\mathbb{R}^m,1}$} & \cellcolor{gray!40} \textbf{0.004}  & \cellcolor{gray!40} \textbf{0.004}  & \cellcolor{gray!40} \textbf{0.004}  & \cellcolor{gray!40} \textbf{0.003}  & \cellcolor{gray!40} \textbf{0.003}  & \cellcolor{gray!40} \textbf{0.003}  & \cellcolor{gray!40} \textbf{0.003} \\ 
 & & \cellcolor{gray!40} \textit{0.04} s & \cellcolor{gray!40} \textit{0.10} s & \cellcolor{gray!40} \textit{0.34} s & \cellcolor{gray!40} \textit{1.33} s & \cellcolor{gray!40} \textit{4.30} s & \cellcolor{gray!40} \textit{16.66} s & \cellcolor{gray!40} \textit{66.72} s 
		\end{tabular}
	\end{scriptsize}
	\vspace{0.3cm}
	\caption{Fixed-time learning of the non-linear Fokker-Planck equation \eqref{EqDefFokker}: Empirical $L^2$-error defined in \eqref{EqDefMSE} on the test set (in bold letters) and running time (in italic letters; in seconds).}
	\label{TabFokker}
\end{table}

\subsubsection{Spatio-temporal learning}

In this section, we again consider the spatio-temporal learning approach to numerically solve the non-linear Fokker-Planck equation~\eqref{EqDefFokker}. To this end, we first observe that the feature map
\begin{equation*}
	\mathbb{R}^m \times \mathbb{S}^m_{++} \ni (\mu_0,\Sigma_0) \quad \mapsto \quad \left( (t,u) \mapsto g(\mu_0,\Sigma_0)(t,u) := \frac{e^{-\frac{1}{2} (u-\mu_t)^\top \Sigma_t^{-1} (u-\mu_t)}}{(2\pi)^{m/2} \det(\Sigma_t)^{1/2}} \right)
\end{equation*}
solves the non-linear Fokker-Planck equation~\eqref{EqDefFokker} in the interior, where $(\mu_t)_{t \in [0,T]}$ and $(\Sigma_t)_{t \in [0,T]}$ follow the dynamics \eqref{EqFokkerMeanCov} with initial values $\mu_0 \in \mathbb{R}^m$ and $\Sigma_0 \in \mathbb{S}^m_{++}$, respectively. Thus, random feature models of the form
\begin{equation}
	\label{EqDefRNFokker}
	\Omega \ni \omega \quad \mapsto \quad \left( (t,u) \mapsto G_N(\omega)(t,u) := \sum_{n=1}^N y_n(\omega) \frac{e^{-\frac{1}{2} (u-\mu_{n,t}(\omega))^\top \Sigma_{n,t}(\omega)^{-1} (u-\mu_{n,t}(\omega))}}{(2\pi)^{m/2} \det(\Sigma_{n,t}(\omega))^{1/2}} \right),
\end{equation}
solve for any fixed $\omega \in \Omega$ the non-linear Fokker-Planck equation~\eqref{EqDefFokker} in the interior, where $(\mu_{n,0},\Sigma_{n,0})_{n \in \mathbb{N}}: \Omega \rightarrow \mathbb{R}^m \times \mathbb{S}^m_{++}$ are some random initializations inducing $(\mu_{n,t}(\omega))_{t \in [0,T]}$ and $(\Sigma_{n,t}(\omega))_{t \in [0,T]}$ according to \eqref{EqFokkerMeanCov}, and where $(y_n)_{n \in \mathbb{N}}$ are the $\mathcal{F}_{a,b}$-measurable linear readouts. We denote by $\mathcal{R}\lbrace g \rbrace_{\mathbb{R}^m,1}$ all random feature models of the form \eqref{EqDefRNFokker}, while $\lbrace g \rbrace_{\mathbb{R}^m,1} := \linspan\big\lbrace (t,u) \mapsto g(\mu_0,\Sigma_0)(t,u): (\mu_0,\Sigma_0) \in \mathbb{R}^m \times \mathbb{S}^m_{++} \big\rbrace$ represents the deterministic counterpart. Hence, we aim to train the linear readouts $(y_n)_{n \in \mathbb{N}}$ so that the corresponding random feature model $G_N$ minimizes the empirical $L^2$-error \eqref{EqDefMSE1}, where $(V_j)_{j=1,...,J} \sim w$ for some weight $w: \mathbb{R}^m \rightarrow [0,\infty)$. For $T > 0$ and a partition $0 = t_0 < t_1 < ... < t_K = T$, we then test the performance of $G_N$ on the global $L^2$-error \eqref{EqDefMSE2}.

Now, we choose $T = 1$, $K = 20$, and the initial condition $g$, the weight $w$, as well as $(V_j)_{j=1,...,J} \sim w$ as in Section~\ref{SecHeatSuper}, except $J = 2^9,2^{10},...,2^{13}$. Then, we train the random feature model $G_N \in \mathcal{R}\lbrace g \rbrace_{\mathbb{R}^m,1}$ of the form \eqref{EqDefRNHeat} by minimizing \eqref{EqDefMSE}. On the other hand, we also train the corresponding deterministic feature models $G_N \in \lbrace g \rbrace_{\mathbb{R}^m,1}$ with the Adam algorithm over $8000$ epochs (with learning rate $\gamma = 10^{-5}$ and batchsize $1000$).

Figure~\ref{FigFokkerUSV} empirically demonstrates that specific random feature models and their deterministic counterparts can learn the non-linear Fokker-Planck equation~\eqref{EqDefFokker}. Note that learning the dependence of the deterministic features on the parameter $\Sigma_0 \in \mathbb{S}^m_{++}$ can be challenging, especially if $\lambda_{\max}(\Sigma_0)$ is small (meaning that $\Sigma_0 \in \mathbb{S}^m_{++}$ is close to being singular). However, random neural networks, for which only the linear readouts need to be trained, overcome this issue and are considerably faster (see Table~\ref{TabFokkerUSV}).

\begin{figure}[h]
	\begin{minipage}[t]{0.48\textwidth}
		\centering
		\includegraphics[height = 5.2cm]{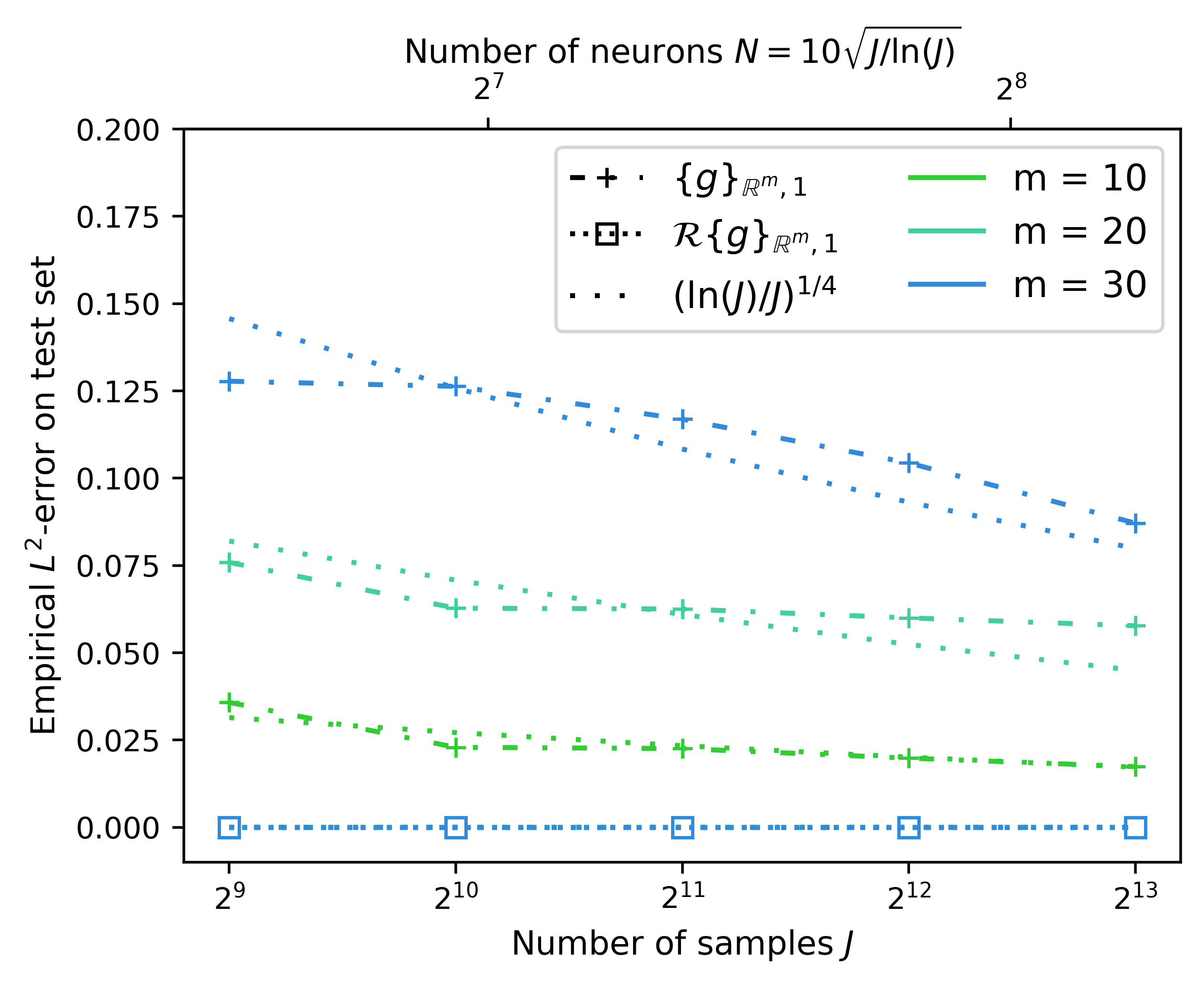}
		\subcaption{Empirical $L^2$-error defined in \eqref{EqDefMSE2}.}
	\end{minipage}
	\hspace{0.02\textwidth}
	\begin{minipage}[t]{0.48\textwidth}
		\centering
		\includegraphics[height = 5.5cm, trim = {0.8cm 0 0 0.4cm}, clip]{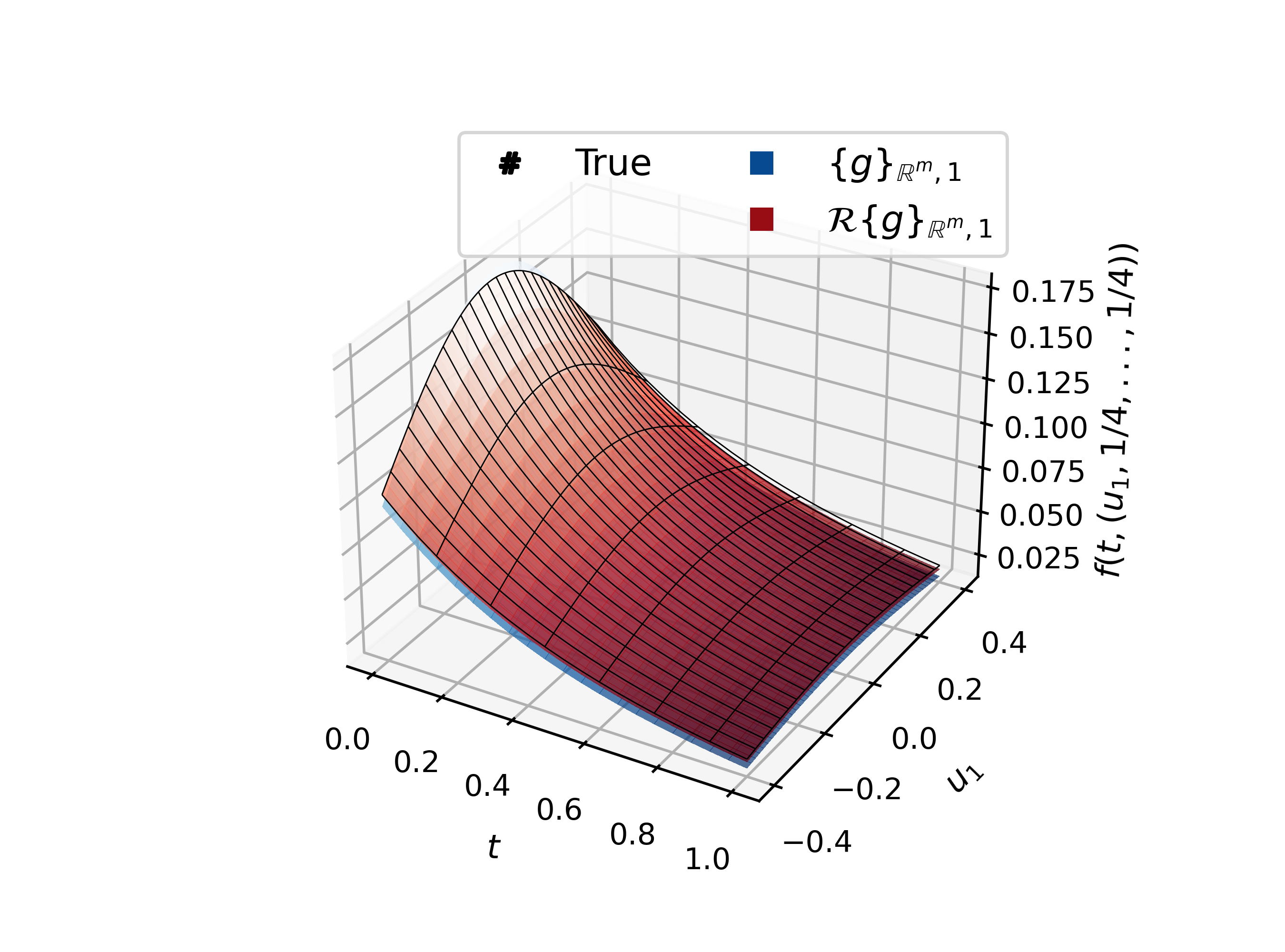}
		\subcaption{Approximation of the function $[0,T] \times \mathbb{R} \ni (t,u_1) \!\mapsto\! f(t,(u_1,1/4,...,1/4)) \in \mathbb{R}$ for $m = 10$.}
	\end{minipage}
	\begin{minipage}[t]{0.48\textwidth}
		\centering
		\includegraphics[height = 5.5cm, trim = {0.8cm 0 0 0.4cm}, clip]{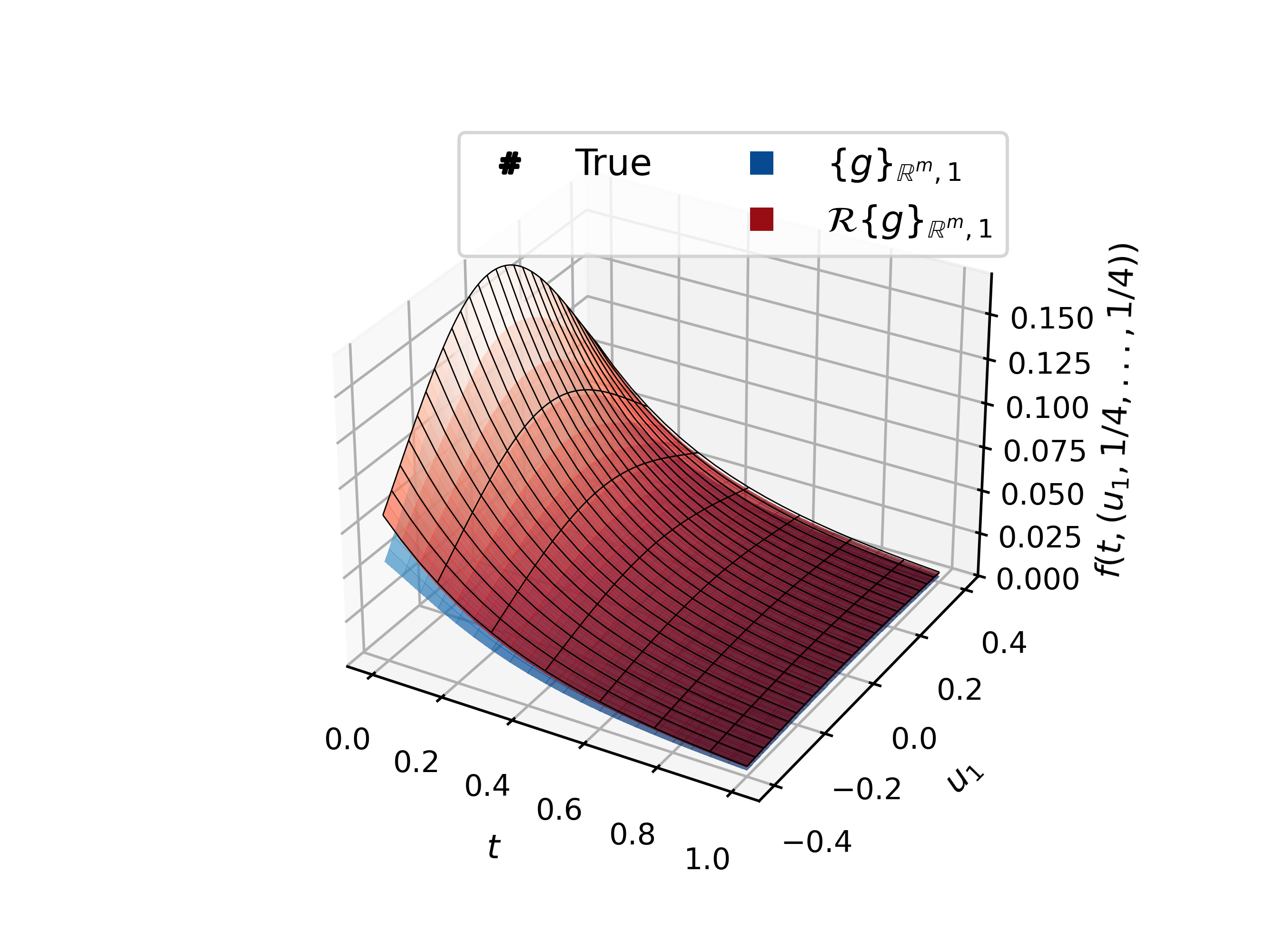}
		\subcaption{Approximation of the function $[0,T] \times \mathbb{R} \ni (t,u_1) \!\mapsto\! f(t,(u_1,1/4,...,1/4)) \in \mathbb{R}$ for $m = 20$.}
	\end{minipage}
	\hspace{0.02\textwidth}
	\begin{minipage}[t]{0.48\textwidth}
		\centering
		\includegraphics[height = 5.5cm, trim = {0.8cm 0 0 0.4cm}, clip]{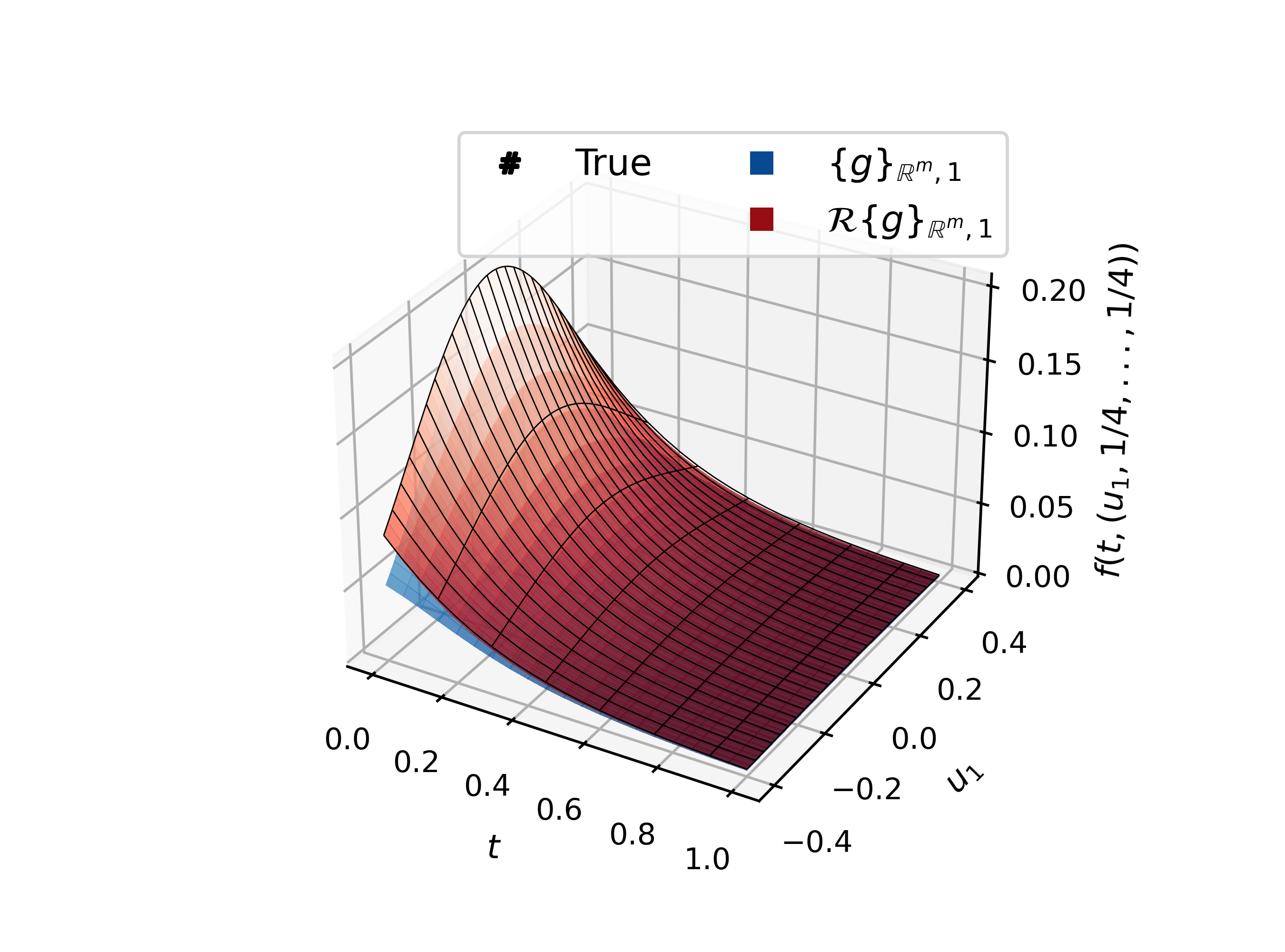}
		\subcaption{Approximation of the function $[0,T] \times \mathbb{R} \ni (t,u_1) \!\mapsto\! f(t,(u_1,1/4,...,1/4)) \in \mathbb{R}$ for $m = 30$.}
	\end{minipage}
	\caption{Spatio-temporal learning of the non-linear Fokker-Planck equation \eqref{EqDefFokker}.}
	\label{FigFokkerUSV}
\end{figure}

\begin{table}
	\centering
	\begin{scriptsize}
		\begin{tabular}{ll|R{1.6cm}|R{1.6cm}|R{1.6cm}|R{1.6cm}|R{1.6cm}}
			\multicolumn{2}{r|}{$N = 10 \sqrt{J \ln(J)}$} & \multicolumn{1}{c|}{$J = 2^9$} & \multicolumn{1}{c|}{$J = 2^{10}$} & \multicolumn{1}{c|}{$J = 2^{11}$} & \multicolumn{1}{c|}{$J = 2^{13}$} & \multicolumn{1}{c}{$J = 2^{15}$} \\
			\hline
			\multirow{4}{*}{$m = 10$} & \multirow{2}{*}{$\lbrace g \rbrace_{\mathbb{R}^m,1}$} & \textbf{3.578 $\cdot$ 10$^{\text{-2}}$}  & \textbf{2.291 $\cdot$ 10$^{\text{-2}}$}  & \textbf{2.255 $\cdot$ 10$^{\text{-2}}$}  & \textbf{1.982 $\cdot$ 10$^{\text{-2}}$}  & \textbf{1.738 $\cdot$ 10$^{\text{-2}}$} \\ 
 & & \textit{20.16} s  & \textit{27.59} s  & \textit{69.04} s  & \textit{167.88} s  & \textit{375.78} s \\ 
& \multirow{2}{*}{$\mathcal{R}\lbrace g \rbrace_{\mathbb{R}^m,1}$} & \cellcolor{gray!40} \textbf{1.939 $\cdot$ 10$^{\text{-8}}$}  & \cellcolor{gray!40} \textbf{1.939 $\cdot$ 10$^{\text{-8}}$}  & \cellcolor{gray!40} \textbf{1.940 $\cdot$ 10$^{\text{-8}}$}  & \cellcolor{gray!40} \textbf{1.939 $\cdot$ 10$^{\text{-8}}$}  & \cellcolor{gray!40} \textbf{1.939 $\cdot$ 10$^{\text{-8}}$} \\ 
 & & \cellcolor{gray!40} \textit{0.75} s & \cellcolor{gray!40} \textit{1.43} s & \cellcolor{gray!40} \textit{2.90} s & \cellcolor{gray!40} \textit{5.82} s & \cellcolor{gray!40} \textit{11.83} s \\ 
\hline 
\multirow{4}{*}{$m = 20$} & \multirow{2}{*}{$\lbrace g \rbrace_{\mathbb{R}^m,1}$} & \textbf{7.588 $\cdot$ 10$^{\text{-2}}$}  & \textbf{6.277 $\cdot$ 10$^{\text{-2}}$}  & \textbf{6.254 $\cdot$ 10$^{\text{-2}}$}  & \textbf{5.998 $\cdot$ 10$^{\text{-2}}$}  & \textbf{5.773 $\cdot$ 10$^{\text{-2}}$} \\ 
 & & \textit{22.70} s  & \textit{38.96} s  & \textit{87.57} s  & \textit{216.28} s  & \textit{590.72} s \\ 
& \multirow{2}{*}{$\mathcal{R}\lbrace g \rbrace_{\mathbb{R}^m,1}$} & \cellcolor{gray!40} \textbf{6.894 $\cdot$ 10$^{\text{-8}}$}  & \cellcolor{gray!40} \textbf{6.901 $\cdot$ 10$^{\text{-8}}$}  & \cellcolor{gray!40} \textbf{6.775 $\cdot$ 10$^{\text{-8}}$}  & \cellcolor{gray!40} \textbf{6.877 $\cdot$ 10$^{\text{-8}}$}  & \cellcolor{gray!40} \textbf{6.925 $\cdot$ 10$^{\text{-8}}$} \\ 
 & & \cellcolor{gray!40} \textit{1.12} s & \cellcolor{gray!40} \textit{2.32} s & \cellcolor{gray!40} \textit{4.60} s & \cellcolor{gray!40} \textit{8.77} s & \cellcolor{gray!40} \textit{22.05} s \\ 
\hline 
\multirow{4}{*}{$m = 30$} & \multirow{2}{*}{$\lbrace g \rbrace_{\mathbb{R}^m,1}$} & \textbf{1.278 $\cdot$ 10$^{\text{-1}}$}  & \textbf{1.263 $\cdot$ 10$^{\text{-1}}$}  & \textbf{1.169 $\cdot$ 10$^{\text{-1}}$}  & \textbf{1.044 $\cdot$ 10$^{\text{-1}}$}  & \textbf{8.712 $\cdot$ 10$^{\text{-2}}$} \\ 
 & & \textit{32.91} s  & \textit{49.81} s  & \textit{116.20} s  & \textit{307.15} s  & \textit{758.80} s \\ 
& \multirow{2}{*}{$\mathcal{R}\lbrace g \rbrace_{\mathbb{R}^m,1}$} & \cellcolor{gray!40} \textbf{1.452 $\cdot$ 10$^{\text{-7}}$}  & \cellcolor{gray!40} \textbf{1.452 $\cdot$ 10$^{\text{-7}}$}  & \cellcolor{gray!40} \textbf{1.445 $\cdot$ 10$^{\text{-7}}$}  & \cellcolor{gray!40} \textbf{1.451 $\cdot$ 10$^{\text{-7}}$}  & \cellcolor{gray!40} \textbf{1.448 $\cdot$ 10$^{\text{-7}}$} \\ 
 & & \cellcolor{gray!40} \textit{1.82} s & \cellcolor{gray!40} \textit{3.38} s & \cellcolor{gray!40} \textit{6.42} s & \cellcolor{gray!40} \textit{12.65} s & \cellcolor{gray!40} \textit{28.19} s 
		\end{tabular}
	\end{scriptsize}
	\vspace{0.3cm}
	\caption{Spatio-temporal learning of the non-linear Fokker-Planck equation \eqref{EqDefFokker}: Empirical $L^2$-error defined in \eqref{EqDefMSE2} on the test set (in bold letters) and running time (in italic letters; in seconds).}
	\label{TabFokkerUSV}
\end{table}

\subsection{Related literature and conclusion}
\label{SecRL}

In recent years, fully trained neural networks have been successfully applied to learn the solution of a partial differential equation (PDE). Based on their universal approximation property, \cite{e17,han18,beck21} among others developed deep learning algorithms to numerically solve high-dimensional parabolic PDEs and BSDEs via stochastic representations. Moreover, the Deep Ritz Method~\cite{e18} reformulated the variational problem as neural network optimization tasks, while the Deep Galerkin Method~\cite{sirignano18} approximated the solution of PDEs by minimizing the residual of the PDE and the boundary condition. This line of work led to physics-informed neural networks (PINNs), which incorporate the governing PDE directly into the loss function (see, e.g., \cite{raissi19,karniadakis21} and the references therein). In contrast to PINNs, our spatio-temporal learning approach only requires fitting the initial condition since the feature map already satisfies the PDE in the interior.

Subsequently, also random feature models have been applied to learn the solution of PDEs. For example, \cite{gonon21} used random neural networks with ReLU activation function to approximate the solution of the Black-Scholes-type PDEs at fixed time and \cite{dong21,wang23} developed closely related methods under the framework of extreme learning machines. Moreover, during the revision process, \cite{deryck25} introduced physics-informed random neural networks to solve PDEs over a given time interval, while \cite{liao26} used random features. Our numerical experiments confirm the efficiency of random feature models for both settings: at fixed time (function regression) and over a given time interval (spatio-temporal learning).

Based on the theoretical results in Section~\ref{SecUAT}+\ref{SecAR}, we have empirically demonstrated in this section that random feature models are able to learn the solution of a given PDE. In particular, they outperform their deterministic counterparts in terms of stability as well as computational efficiency. Indeed, while the deterministic models are trained via a non-convex optimization problem using an iterative training method (e.g., stochastic gradient descent; see \cite{goodfellow16}), the random feature models require only training the linear readout, yielding a convex optimization problem that can be efficiently solved via least squares method and allows to control the optimization error (see, e.g., \cite{wu24}).

\section{Proof of results in Section~\ref{SecUAT}}
\label{SecProofs}

\subsection{Proof of Theorem~\ref{ThmUAT}}
\label{SecProofsUAT}

For the following proofs in this section, we denote the closed ball of radius $r > 0$ around $\vartheta_0 \in \Theta$ by $\overline{B_r(\vartheta_0)} := \left\lbrace \vartheta \in \Theta: d_\Theta(\vartheta,\vartheta_0) \leq r \right\rbrace$.

\begin{proof}[Proof of Theorem~\ref{ThmUAT}]
	First, by using that $\mathcal{G}(\Theta) := \linspan_\mathbb{K}(\lbrace g(\vartheta): g \in \mathcal{G}, \, \vartheta \in \Theta \rbrace)$ is by assumption dense in $X$ together with \cite[Lemma~1.2.19~(i)]{hytoenen16}, i.e.~that $\mathcal{I}_{\mathcal{F}_\theta} \otimes X := \linspan_\mathbb{K}(\lbrace \Omega \ni \omega \mapsto \mathds{1}_E(\omega) x \in X: E \in \mathcal{F}_\theta, \, x \in X \rbrace)$ is dense in $L^r(\Omega,\mathcal{F}_\theta,\mathbb{P};X)$, we obtain that $\mathcal{I}_{\mathcal{F}_\theta} \otimes \mathcal{G}(\Theta) := \linspan_\mathbb{K}(\left\lbrace \Omega \ni \omega \mapsto \mathds{1}_E(\omega) g(\vartheta) \in X: E \in \mathcal{F}_\theta, \, g \in \mathcal{G}, \, \vartheta \in \Theta \right\rbrace)$ is dense in $L^r(\Omega,\mathcal{F}_\theta,\mathbb{P};X)$. Hence, it suffices to show the approximation of any map of the form $\Omega \ni \omega \mapsto \mathds{1}_E(\omega) g(\vartheta_0) \in X$, with $E \in \mathcal{F}_\theta$, $g \in \mathcal{G}$, and $\vartheta_0 \in \Theta$, by some $G \in \mathcal{RG} \cap L^r(\Omega,\mathcal{F}_\theta,\mathbb{P};X)$ with respect to $\Vert \cdot \Vert_{L^r(\Omega,\mathcal{F},\mathbb{P};X)}$. To this end, we fix some $E \in \mathcal{F}_\theta$, $g \in \mathcal{G}$, $\vartheta_0 \in \Theta$, and $\varepsilon > 0$. Moreover, for every $M,n \in \mathbb{N}$, we define the map
	\begin{equation*}
		\Omega \ni \omega \quad \mapsto \quad G_{M,n}(\omega) := y_{M,n}(\omega) g(\theta_n(\omega)) \in X
	\end{equation*}
	with $\mathcal{F}_\theta$-measurable linear readout $\Omega \ni \omega \mapsto y_{M,n}(\omega) := C_M^{-1} \mathds{1}_{\overline{B_{1/M}(\vartheta_0)}}(\theta_n(\omega)) \in \mathbb{R}$, where $C_M := \mathbb{P}\big[ \big\lbrace \omega \in \Omega: \theta_1(\omega) \in \overline{B_{1/M}(\vartheta_0)} \big\rbrace \big] > 0$ due to Assumption~\ref{AssCDF}.
	
	Now, we show that the sequence $\big(\mathds{1}_{\overline{B_{1/M}(\vartheta_0)}}(\theta_n(\omega)) g(\vartheta_0) - C_M G_{M,n}(\omega)\big)_{M \in \mathbb{N}}$ converges uniformly in $\omega \in \Omega$ and $n \in \mathbb{N}$ to $0 \in X$ with respect to $\Vert \cdot \Vert_X$. To this end, we fix some $\varepsilon > 0$. Then, by using that $g \in \mathcal{G} \subseteq C^0(\Theta;X)$ is continuous, there exists some $\delta > 0$ such that for every $\vartheta \in B_\delta(\vartheta_0)$ it holds that
	\begin{equation*}
		\Vert g(\vartheta_0) - g(\vartheta) \Vert_X < \varepsilon.
	\end{equation*}
	Hence, by choosing $M_0 \in \mathbb{N}$ large enough such that $M_0 > \delta^{-1}$, it follows for every $M \in \mathbb{N} \cap [M_0,\infty)$ that
	\begin{equation*}
		\begin{aligned}
			& \sup_{\omega \in \Omega} \sup_{n \in \mathbb{N}} \left\Vert \mathds{1}_{\overline{B_{1/M}(\vartheta_0)}}(\theta_n(\omega)) g(\vartheta_0) - C_M G_{M,n}(\omega) \right\Vert_X \\
			& \quad\quad = \sup_{\omega \in \Omega} \sup_{n \in \mathbb{N}} \left\Vert \mathds{1}_{\overline{B_{1/M}(\vartheta_0)}}(\theta_n(\omega)) \big( g(\vartheta_0) - g(\theta_n(\omega)) \big) \right\Vert_X \\
			& \quad\quad = \sup_{\vartheta \in \overline{B_{1/M}(\vartheta_0)}} \left\Vert g(\vartheta_0) - g(\vartheta) \right\Vert_X \\
			& \quad\quad \leq \sup_{\vartheta \in \overline{B_{1/M_0}(\vartheta_0)}} \left\Vert g(\vartheta_0) - g(\vartheta) \right\Vert_X < \varepsilon.
		\end{aligned}
	\end{equation*}
	Since $\varepsilon > 0$ was chosen arbitrarily, this shows that the sequence $\big(\mathds{1}_{\overline{B_{1/M}(\vartheta_0)}}(\theta_n(\omega)) g(\vartheta_0) - C_M G_{M,n}(\omega)\big)_{M \in \mathbb{N}}$ converges uniformly in $\omega \in \Omega$ and $n \in \mathbb{N}$ to $0 \in X$ with respect to $\Vert \cdot \Vert_X$.
	
	Next, we show for every fixed $M,n \in \mathbb{N}$ that $G_{M,n} \in L^r(\Omega,\mathcal{F}_\theta,\mathbb{P};X)$. Indeed, by using that $\Omega \ni \omega \mapsto (y_{M,n}(\omega),\theta_n(\omega)) \in \mathbb{R} \times \Theta$ is $\mathcal{F}_\theta/\mathcal{B}(\mathbb{R} \times \Theta)$-measurable and that $g \in \mathcal{G} \subseteq C^0(\Theta;X)$ is continuous, it follows that the concatenation $\Omega \ni \omega \mapsto y_{M,n}(\omega) g(\theta_n(\omega)) \in X$ is $\mathcal{F}_\theta/\mathcal{B}(X)$-measurable. Hence, by using that $(X,\Vert \cdot \Vert_X)$ is separable, we can apply \cite[Theorem 1.1.6+1.1.20]{hytoenen16} to conclude that $G_{M,n}: \Omega \rightarrow X$ is strongly $(\mathbb{P},\mathcal{F}_\theta)$-measurable. Moreover, by using Minkowski's inequality and that the sequence $\big(\mathds{1}_{\overline{B_{1/M}(\vartheta_0)}}(\theta_n(\omega)) g(\vartheta_0) - C_M G_{M,n}(\omega)\big)_{M \in \mathbb{N}}$ is by the previous step uniformly bounded in $\omega \in \Omega$ and $n \in \mathbb{N}$, we have
	\begin{equation*}
		\begin{aligned}
			& \Vert G_{M,n} \Vert_{L^r(\Omega,\mathcal{F},\mathbb{P};X)} = \mathbb{E}\left[ \Vert G_{M,n} \Vert_X^r \right]^\frac{1}{r} = \frac{1}{C_M} \mathbb{E}\left[ \Vert C_M G_{M,n} \Vert_X^r \right]^\frac{1}{r} \\
			& \quad\quad \leq \frac{1}{C_M} \mathbb{E}\left[ \left\Vert \mathds{1}_{\overline{B_{1/M}(\vartheta_0)}}(\theta_n) g(\vartheta_0) \right\Vert_X^r \right]^\frac{1}{r} + \frac{1}{C_M} \mathbb{E}\left[ \left\Vert \mathds{1}_{\overline{B_{1/M}(\vartheta_0)}}(\theta_n) g(\vartheta_0) - C_M G_{M,n} \right\Vert_X^r \right]^\frac{1}{r} \\
			& \quad\quad \leq \frac{1}{C_M} \Vert g(\vartheta_0) \Vert_X + \frac{1}{C_M} \sup_{\omega \in \Omega} \left\Vert \mathds{1}_{\overline{B_{1/M}(\vartheta_0)}}(\theta_n(\omega)) g(\vartheta_0) - C_M G_{M,n} \right\Vert_X < \infty,
		\end{aligned}
	\end{equation*}
	which shows that $G_{M,n} \in L^r(\Omega,\mathcal{F}_\theta,\mathbb{P};X)$ for all $M,n \in \mathbb{N}$.	
	
	Now, we show that there exists some $M_1 \in \mathbb{N}$ such that the constant maps $\left( \omega \mapsto g(\vartheta_0) \right) \in L^r(\Omega,\mathcal{F}_\theta,\mathbb{P};X)$ and $\left( \omega \mapsto \mathbb{E}[G_{M_1,1}] \right) \in L^r(\Omega,\mathcal{F}_\theta,\mathbb{P};X)$ are $\frac{\varepsilon}{2}$-close to each other with respect to $\Vert \cdot \Vert_{L^r(\Omega,\mathcal{F},\mathbb{P};X)}$. Indeed, by using that $\big( \mathds{1}_{\overline{B_{1/M}(\vartheta_0)}}(\theta_n(\omega)) g(\vartheta_0) - C_M G_{M,n}(\omega) \big)_{M \in \mathbb{N}}$ converges uniformly in $\omega \in \Omega$ and $n \in \mathbb{N}$ to $0 \in X$ with respect to $\Vert \cdot \Vert_X$, there exists some $M_1 \in \mathbb{N}$ such that
	\begin{equation}
		\label{EqPropSimplebyRNProof3}
		\sup_{n \in \mathbb{N}} \sup_{\omega \in \Omega} \left\Vert \mathds{1}_{\overline{B_{1/M_1}(\vartheta_0)}}(\theta_n(\omega)) g(\vartheta_0) - C_{M_1} G_{M_1,n}(\omega) \right\Vert_X < \frac{\varepsilon}{2}.
	\end{equation}
	Hence, by using that $\mathbb{E}\big[ \mathds{1}_{\overline{B_{1/M_1}(\vartheta_0)}}(\theta_1) \big] = \mathbb{P}\big[ \big\lbrace \omega \in \Omega: \theta_1(\omega) \in \overline{B_{1/M_1}(\vartheta_0)} \big\rbrace \big] = C_{M_1} > 0$, and \cite[Proposition~1.2.2]{hytoenen16}, it follows that
	\begin{equation}
		\label{EqPropSimplebyRNProof4}
		\begin{aligned}
			& \left\Vert g(\vartheta_0) - \mathbb{E}[G_{M_1,1}] \right\Vert_{L^r(\Omega,\mathcal{F},\mathbb{P};X)} = \mathbb{E}\left[ \left\Vert g(\vartheta_0) - \mathbb{E}[G_{M_1,1}] \right\Vert_X^r \right]^\frac{1}{r} = \left\Vert g(\vartheta_0) - \mathbb{E}[G_{M_1,1}] \right\Vert_X \\
			& \quad\quad = \left\Vert \mathbb{E}\left[ \frac{1}{C_{M_1}} \mathds{1}_{\overline{B_{1/M_1}(\vartheta_0)}}(\theta_1) g(\vartheta_0) - G_{M_1,1} \right] \right\Vert_X \\
			& \quad\quad \leq \mathbb{E}\left[ \frac{\mathds{1}_{\overline{B_{1/M_1}(\vartheta_0)}}(\theta_1)}{C_{M_1}} \left\Vert \mathds{1}_{\overline{B_{1/M_1}(\vartheta_0)}}(\theta_1) g(\vartheta_0) - C_{M_1} G_{M_1,1} \right\Vert_X \right] \\
			& \quad\quad \leq \underbrace{\frac{\mathbb{E}\left[ \mathds{1}_{\overline{B_{1/M_1}(\vartheta_0)}}(\theta_1) \right]}{C_{M_1}}}_{=1} \sup_{\omega \in \Omega} \left\Vert \mathds{1}_{\overline{B_{1/M_1}(\vartheta_0)}}(\theta_1) g(\vartheta_0) - C_{M_1} G_{M_1,1}(\omega) \right\Vert_X < \frac{\varepsilon}{2}.
		\end{aligned}
	\end{equation}
	This shows that the constant maps $(\omega \mapsto  g(\vartheta_0)) \in L^r(\Omega,\mathcal{F}_\theta,\mathbb{P};X)$ and $\left( \omega \mapsto \mathbb{E}[G_{M_1,1}] \right) \in L^r(\Omega,\mathcal{F}_\theta,\mathbb{P};X)$ are $\frac{\varepsilon}{2}$-close to each other with respect to $\Vert \cdot \Vert_{L^r(\Omega,\mathcal{F},\mathbb{P};X)}$.
	
	Finally, we approximate the constant random variable $\left( \omega \mapsto \mathbb{E}\left[ G_{M_1,1} \right] \right) \in L^1(\Omega,\mathcal{F}_\theta,\mathbb{P};X)$ by the average of the i.i.d.~sequence $(G_{M_1,n})_{n \in \mathbb{N}} \subseteq L^1(\Omega,\mathcal{F}_\theta,\mathbb{P};X)$. Indeed, by applying the strong law of large numbers for Banach space-valued random variables in \cite[Theorem~3.3.10]{hytoenen16} with Banach space $(X,\Vert \cdot \Vert_X)$, we conclude that
	\begin{equation}
		\label{EqPropSimplebyRNProof5}
		\frac{1}{N} \sum_{n=1}^N G_{M_1,n} \quad \overset{N \rightarrow \infty}{\longrightarrow} \quad \mathbb{E}\left[ G_{M_1,1} \right] \quad \text{in } L^1(\Omega,\mathcal{F}_\theta,\mathbb{P};X) \text{ and } \mathbb{P}\text{-a.s.}
	\end{equation}
	Moreover, if $r \in (1,\infty)$, we generalize the convergence in \eqref{EqPropSimplebyRNProof5} to $L^r(\Omega,\mathcal{F}_\theta,\mathbb{P};X)$. To this end, we define the sequence of real-valued random variables $(Z_N)_{N \in \mathbb{N}}$ by $Z_N(\omega) := \big\Vert \mathbb{E}\big[ G_{M_1,1} \big] - \frac{1}{N} \sum_{n=1}^N G_{M_1,n} \big\Vert_X^r$, for $\omega \in \Omega$ and $N \in \mathbb{N}$. Then, by using \cite[Proposition~1.2.2]{hytoenen16} and \eqref{EqPropSimplebyRNProof3}, it follows for every $N \in \mathbb{N}$ that
	\begin{equation*}
		\begin{aligned}
			\sup_{\omega \in \Omega} Z_N(\omega) & \leq \sup_{\omega \in \Omega} \left( \left\Vert \mathbb{E}\left[ G_{M_1,1} \right] \right\Vert_X + \frac{1}{N} \sum_{n=1}^N \left\Vert G_{M_1,n}(\omega) \right\Vert_X \right)^r \\
			& \leq \sup_{n \in \mathbb{N}} \sup_{\omega \in \Omega} \left( \mathbb{E}\left[ \left\Vert G_{M_1,1} \right\Vert_X \right] + \left\Vert G_{M_1,n}(\omega) \right\Vert_X \right)^r \\
			& \leq \frac{2^r}{C_{M_1}^r} \sup_{n \in \mathbb{N}} \sup_{\omega \in \Omega} \left\Vert C_{M_1} G_{M_1,n}(\omega) \right\Vert_X^r \\
			& \leq \frac{2^r}{C_{M_1}^r} \sup_{n \in \mathbb{N}} \sup_{\omega \in \Omega} \bigg( \left\Vert \mathds{1}_{\overline{B_{1/M_1}(\vartheta_0)}}(\theta_1(\omega)) g(\vartheta_0) \right\Vert_X + \\
			& \quad\quad\quad\quad\quad\quad\quad\quad + \left\Vert \mathds{1}_{\overline{B_{1/M_1}(\vartheta_0)}}(\theta_1(\omega)) g(\vartheta_0) - C_{M_1} G_{M_1,n}(\omega) \right\Vert_X \bigg)^r \\
			& < \frac{2^r}{C_{M_1}^r} \left( \left\Vert g(\vartheta_0) \right\Vert_X + \frac{\varepsilon}{2} \right)^r =: C_Z < \infty.
		\end{aligned}
	\end{equation*}
	Hence, $\sup_{N \in \mathbb{N}} \mathbb{E}\left[ \vert Z_N \vert \mathds{1}_{\left\lbrace \vert Z_N \vert > C_Z \right\rbrace} \right] = 0$, which implies that the family of random variables $(Z_N)_{N \in \mathbb{N}}$ is uniformly integrable (see \cite[Definition~A.3.1]{hytoenen16}). Thus, by using that $Z_N \rightarrow 0$, $\mathbb{P}$-a.s., as $N \rightarrow \infty$ (see \eqref{EqPropSimplebyRNProof5}), and Vitali's convergence theorem (see \cite[Proposition~A.3.5]{hytoenen16}), we have
	\begin{equation}
		\label{EqPropSimplebyRNProof6}
		\lim_{N \rightarrow \infty} \mathbb{E}\left[ \left\Vert \mathbb{E}\left[ G_{M_1,1} \right] - \frac{1}{N} \sum_{n=1}^N G_{M_1,n} \right\Vert_X^r \right] = \lim_{N \rightarrow \infty} \mathbb{E}[Z_N] = 0.
	\end{equation}
	Thus, either by \eqref{EqPropSimplebyRNProof5} (if $r = 1$) or \eqref{EqPropSimplebyRNProof6} (if $r \in (1,\infty)$) there exists some $N_0 \in \mathbb{N}$ such that
	\begin{equation}
		\label{EqPropSimplebyRNProof7}
		\mathbb{E}\left[ \left\Vert \mathbb{E}\left[ G_{M_1,1} \right] - \frac{1}{N_0} \sum_{n=1}^{N_0} G_{M_1,n} \right\Vert_X^r \right]^\frac{1}{r} < \frac{\varepsilon}{2}.
	\end{equation}
	Finally, we define $G := \big( \omega \mapsto \frac{1}{N_0} \sum_{n=1}^{N_0} \mathds{1}_E(\omega) G_{M_1,n}(\omega) \big) \in \mathcal{RG} \cap L^r(\Omega,\mathcal{F}_\theta,\mathbb{P};X)$. Hence, by combining \eqref{EqPropSimplebyRNProof4} and \eqref{EqPropSimplebyRNProof7} with Minkowski's inequality, it follows that
	\begin{equation*}
		\begin{aligned}
			\left\Vert \mathds{1}_E g(\vartheta_0) - G \right\Vert_{L^r(\Omega,\mathcal{F},\mathbb{P};X)} & = \mathbb{E}\left[ \left\Vert \mathds{1}_E g(\vartheta_0) - G \right\Vert^r \right]^\frac{1}{r} \\
			& = \mathbb{E}\Bigg[ \underbrace{\mathds{1}_E}_{\leq 1} \left\Vert g(\vartheta_0) - \frac{1}{N_0} \sum_{n=1}^{N_0} G_{M_1,n} \right\Vert^r \Bigg]^\frac{1}{r} \\
			& \leq \mathbb{E}\left[ \left\Vert g(\vartheta_0) - \frac{1}{N_0} \sum_{n=1}^{N_0} G_{M_1,n} \right\Vert^r \right]^\frac{1}{r} \\
			& \leq \left\Vert g(\vartheta_0) - \mathbb{E}\left[ G_{M_1,n} \right] \right\Vert_X + \mathbb{E}\left[ \left\Vert \mathbb{E}\left[ G_{M_1,n} \right] - \frac{1}{N_0} \sum_{n=1}^{N_0} G_{M_1,n} \right\Vert^r \right]^\frac{1}{r} \\
			& < \frac{\varepsilon}{2} + \frac{\varepsilon}{2} = \varepsilon.
		\end{aligned}
	\end{equation*}
	Since $\varepsilon > 0$, $g \in \mathcal{G}$, and $\vartheta_0 \in \Theta$ were chosen arbitrarily, this shows that $\Omega \ni \omega \mapsto \mathds{1}_E(\omega) g(\vartheta_0) \in X$ can be approximated by a random feature model $G \in \mathcal{RG} \cap L^r(\Omega,\mathcal{F}_\theta,\mathbb{P};X)$ with respect to $\Vert \cdot \Vert_{L^r(\Omega,\mathcal{F},\mathbb{P};X)}$. Combining this together with the first step of the proof, i.e.~that $\mathcal{I}_{\mathcal{F}_\theta} \otimes \mathcal{G}(\Theta)$ is dense in $L^r(\Omega,\mathcal{F}_\theta,\mathbb{P};X)$, we obtain the conclusion.
\end{proof}

\subsection{Proof of Corollary~\ref{CorUATTrigo}+\ref{CorUATFourier}+\ref{CorUATRN}}
\label{SecProofsUATCor}

\begin{proof}[Proof of Corollary~\ref{CorUATTrigo}]
	We aim to apply Theorem~\ref{ThmUAT} with Banach space $(X,\Vert \cdot \Vert_X) := (C^0(U),\Vert \cdot \Vert_{C^0(U)})$. To this end, we first observe that $(C^0(U),\Vert \cdot \Vert_{C^0(U)})$ is by \cite[Problem~24]{brezis11} separable. Moreover, we choose $\Theta := \mathbb{R}^m$ and let $(\theta_n)_{n \in \mathbb{N}}: \Omega \rightarrow \mathbb{R}^m$ be an i.i.d.~sequence satisfying Assumption~\ref{AssCDF}. In addition, we define
	\begin{equation*}
		\mathcal{G} := \left\lbrace \mathbb{R}^m \ni \vartheta \mapsto h\left( \vartheta^\top \cdot \right) \in C^0(U): h \in \lbrace \cos, \sin \rbrace \right\rbrace.
	\end{equation*}
	Then, for both $h \in \lbrace \cos, \sin \rbrace$, we use that $\mathbb{R}^m \times U \ni (\vartheta,u) \mapsto h\big( \vartheta^\top u \big) \in \mathbb{R}$ is continuous to conclude that $K \times U \ni (\vartheta,u) \mapsto h\big( \vartheta^\top u \big) \in \mathbb{R}$ is uniformly continuous, for all compact subsets $K \subset \mathbb{R}^m$. Hence, the map $\mathbb{R}^m \ni \vartheta \mapsto h\big( \vartheta^\top \cdot \big) \in C^0(U)$ is continuous, which shows that $\mathcal{G} \subseteq C^0(\Theta;X)$. Moreover, by using the trigonometric identities $\cos(s) \cos(t) = (\cos(s-t)+\cos(s+t))/2$, $\sin(s) \sin(t) = (\cos(s-t)-\cos(s+t))/2$, and $\cos(s) \sin(t) = (\sin(s+t)-\sin(s-t))/2$ for any $s,t \in \mathbb{R}$, we observe that
	\begin{equation*}
		\linspan_\mathbb{R}(\mathcal{G}(\Theta)) = \linspan_\mathbb{R}\left( \left\lbrace U \ni u \mapsto h\left( \vartheta^\top u \right) \in \mathbb{R}: h \in \lbrace \cos, \sin \rbrace, \, \vartheta \in \mathbb{R}^m \right\rbrace \right)
	\end{equation*}
	is a subalgebra of $C^0(U)$, i.e.~for every $g_1,g_2 \in \linspan_\mathbb{R}(\mathcal{G}(\Theta))$ we have $g_1 + g_2 \in \linspan_\mathbb{R}(\mathcal{G}(\Theta))$ and $g_1 \cdot g_2 \in \linspan_\mathbb{R}(\mathcal{G}(\Theta))$. Moreover, $\linspan_\mathbb{R}(\mathcal{G}(\Theta))$ is point separating, i.e.~for any distinct $u_1,u_2 \in U$ there exists some $g \in \linspan_\mathbb{R}(\mathcal{G}(\Theta))$ such that $g(u_1) \neq g(u_2)$. In addition, $\linspan_\mathbb{R}(\mathcal{G}(\Theta))$ vanishes nowhere, i.e.~for every $u_0 \in U$ there exists some $g \in \linspan_\mathbb{R}(\mathcal{G}(\Theta))$ such that $g(u_0) \neq 0$. Hence, we can apply the Stone-Weierstrass theorem (see \cite{stone48}) to obtain that $\linspan_\mathbb{R}(\mathcal{G}(\Theta))$ is dense in $C^0(U)$. Thus, the conclusion follows from Theorem~\ref{ThmUAT}.
\end{proof}

\begin{proof}[Proof of Corollary~\ref{CorUATFourier}]
	We aim to apply Theorem~\ref{ThmUAT} with Banach space $(X,\Vert \cdot \Vert_X) := (C^0(U),\Vert \cdot \Vert_{C^0(U)})$. To this end, we first observe that $(C^0(U),\Vert \cdot \Vert_{C^0(U)})$ is by \cite[Problem~24]{brezis11} separable. Moreover, we choose $\Theta := \mathbb{R}^m$ and let $(\theta_n)_{n \in \mathbb{N}}: \Omega \rightarrow \mathbb{R}^m$ be an i.i.d.~sequence satisfying Assumption~\ref{AssCDF}. In addition, we define the singleton set
	\begin{equation*}
		\mathcal{G} := \left\lbrace \mathbb{R}^m \ni \vartheta \mapsto \exp\left( \mathbf{i} \vartheta^\top \cdot \right) \in C^0(U) \right\rbrace.
	\end{equation*}
	Then, we follow the proof of Corollary~\ref{CorUATRN} to conclude that $\mathbb{R}^m \ni \vartheta \mapsto \exp\big( \mathbf{i} \vartheta^\top \cdot \big) \in C^0(U)$ is continuous, which shows that $\mathcal{G} \subseteq C^0(\Theta;X)$. Moreover, by using the identities $\exp\big( \mathbf{i} \vartheta_1^\top u \big) \exp\big( \mathbf{i} \vartheta_2^\top u \big) = \exp\big( \mathbf{i} (\vartheta_1 + \vartheta_2)^\top u \big)$ and $\overline{\exp\big( \mathbf{i} \vartheta_1^\top u \big)} = \exp\big( \!\!-\!\!\mathbf{i} \vartheta_1^\top u \big) = \exp\big( \mathbf{i} (-\vartheta_1)^\top u \big)$ for any $\vartheta_1, \vartheta_2 \in \mathbb{R}^m$ and $u \in U$, we observe that
	\begin{equation*}
		\linspan_\mathbb{C}(\mathcal{G}(\Theta)) = \linspan_\mathbb{C}\left( \left\lbrace U \ni u \mapsto \exp\left( \mathbf{i} \vartheta^\top u \right) \in \mathbb{C}: \vartheta \in \mathbb{R}^m \right\rbrace \right)
	\end{equation*}
	is a subalgebra of $C^0(U)$, which is point separating, nowhere vanishing, and self-adjoint, where the latter means that for every $g \in \linspan_\mathbb{C}(\mathcal{G}(\Theta))$ the function $\Theta \ni \vartheta \mapsto \overline{g}(\vartheta) := \overline{g(\vartheta)} \in \mathbb{C}$ satisfies $\overline{g} \in \linspan_\mathbb{C}(\mathcal{G}(\Theta))$. Hence, we can apply the complex-valued Stone-Weierstrass theorem (see e.g.~\cite[p.~122]{rudin91}) to obtain that $\linspan_\mathbb{C}(\mathcal{G}(\Theta))$ is dense in $C^0(U;\mathbb{C})$. Thus, the conclusion follows from Theorem~\ref{ThmUAT}.
\end{proof}

For the proof of Corollary~\ref{CorUATRN}, we first show the following auxiliary lemma about neurons and that Banach spaces $(X,\Vert \cdot \Vert_X)$ satisfying Assumption~\ref{AssEmb} are separable.

\begin{lemma}
	\label{LemmaSep}
	Let $(X,\Vert \cdot \Vert_X)$ satisfy Assumption~\ref{AssEmb}. Then, the following holds true:
	\begin{enumerate}
		\item\label{LemmaSep1} For every $y \in \mathbb{R}^d$, $a \in \mathbb{R}^m$, $b \in \mathbb{R}$, and $\rho \in \overline{C^k_b(\mathbb{R})}^\gamma$ it holds that $y \rho\big( a^\top \cdot - b \big) \in X$.
		\item\label{LemmaSep2} For every $\rho \in \overline{C^k_b(\mathbb{R})}^\gamma$ the map $\mathbb{R}^d \times \mathbb{R}^m \times \mathbb{R} \ni (y,a,b) \mapsto y \rho\big( a^\top \cdot - b \big) \in X$ is continuous.
		\item\label{LemmaSep3} The Banach space $(X,\Vert \cdot \Vert_X)$ is separable.
	\end{enumerate}
\end{lemma}
\begin{proof}
	For Part~\ref{LemmaSep1}., we apply \cite[Lemma~2.5]{neufeld24} to conclude that $y \rho\big( a^\top \cdot - b \big) \in X$ for all $y \in \mathbb{R}^d$, $a \in \mathbb{R}^m$, $b \in \mathbb{R}$, and $\rho \in \overline{C^k_b(\mathbb{R})}^\gamma$. 
	
	For Part~\ref{LemmaSep2}., we fix some $\varepsilon > 0$ and a sequence $(y_M,a_M,b_M)_{M \in \mathbb{N}} \subseteq \mathbb{R}^d \times \mathbb{R}^m \times \mathbb{R}$ converging to $(y,a,b) \in \mathbb{R}^d \times \mathbb{R}^m \times \mathbb{R}$. Then, by using that $y_M a_M^\alpha$ converges uniformly in $\alpha \in \mathbb{N}^m_{0,k}$ to $y a^\alpha$ (where $a^\alpha := \prod_{l=1}^m a_l^{\alpha_l}$ for $a := (a_1,...,a_m)^\top \in \mathbb{R}^m$ and $\alpha := (\alpha_1,...,\alpha_m) \in \mathbb{N}^m_{0,k}$), the constant $C_{y,a} := 1 + \max_{\alpha \in \mathbb{N}^m_{0,k}} \left\Vert y a^\alpha \right\Vert + \sup_{M \in \mathbb{N}} \max_{\alpha \in \mathbb{N}^m_{0,k}} \left\Vert y_M a_M^\alpha \right\Vert > 0$ is finite. Moreover, since $(a_M,b_M)_{M \in \mathbb{N}}^\top \subseteq \mathbb{R}^m \times \mathbb{R}$ converges to $(a,b) \in \mathbb{R}^m \times \mathbb{R}$, the constant $C_{a,b} := 1 + \Vert (a,b) \Vert + \sup_{M \in \mathbb{N}} \Vert (a_M,b_M) \Vert > 0$ is finite. In addition, there exists by definition of $\overline{C^k_b(\mathbb{R})}^\gamma$ some $\widetilde{\rho} \in C^k_b(\mathbb{R})$ such that
	\begin{equation}
		\label{EqLemmaSepProof1}
		\Vert \rho - \widetilde{\rho} \Vert_{C^k_{pol,\gamma}(\mathbb{R})} := \max_{\alpha \in \mathbb{N}^m_{0,k}} \sup_{s \in \mathbb{R}} \frac{\left\vert \rho^{(\vert \alpha \vert)}(s) - \widetilde{\rho}^{(\vert \alpha \vert)}(s) \right\vert}{\left( 1 + \left\vert s \right\vert \right)^\gamma} < \frac{\varepsilon}{6 C_{y,a} C_{a,b}}.
	\end{equation}
	Now, we choose some $r > 0$ large enough such that $(1+r)^\gamma \geq 6 \varepsilon^{-1} C_{y,a} \Vert \widetilde{\rho} \Vert_{C^k_b(\mathbb{R})}$. Then, the inequality $1+\left\vert a_M^\top u - b_M \right\vert \leq 1+\Vert a_M \Vert \Vert u \Vert + \vert b_M \vert \leq (1+\Vert a_M \Vert + \vert b_M \vert) (1 + \Vert u \Vert)$ for any $u \in \mathbb{R}^m$ and \eqref{EqLemmaSepProof1} imply that
	\begin{equation}
		\label{EqLemmaSepProof2}
		\begin{aligned}
			& \max_{\alpha \in \mathbb{N}^m_{0,k}} \sup_{u \in \mathbb{R}^m \setminus \overline{B_r(0)}} \frac{\left\Vert y_M \rho^{(\vert \alpha \vert)}\left( a_M^\top u - b_M \right) a_M^\alpha \right\Vert}{(1+\Vert u \Vert)^\gamma} \\
			& \quad\quad \leq \left( \max_{\alpha \in \mathbb{N}^m_{0,k}} \left\Vert y_M a_M^\alpha \right\Vert \right) \max_{\alpha \in \mathbb{N}^m_{0,k}} \sup_{u \in \mathbb{R}^m \setminus \overline{B_r(0)}} \frac{\left\vert \rho^{(\vert \alpha \vert)}\left( a_M^\top u - b_M \right) \right\vert}{(1+\Vert u \Vert)^\gamma} \\
			& \quad\quad \leq C_{y,a} \max_{\alpha \in \mathbb{N}^m_{0,k}} \sup_{u \in \mathbb{R}^m \setminus \overline{B_r(0)}} \frac{\left\vert \rho^{(\vert \alpha \vert)}\left( a_M^\top u - b_M \right) - \widetilde{\rho}^{(\vert \alpha \vert)}\left( a_M^\top u - b_M \right) \right\vert}{(1+\Vert u \Vert)^\gamma} \\
			& \quad\quad\quad\quad + C_{y,a} \max_{\alpha \in \mathbb{N}^m_{0,k}} \sup_{u \in \mathbb{R}^m \setminus \overline{B_r(0)}} \frac{\left\vert \widetilde{\rho}^{(\vert \alpha \vert)}\left( a_M^\top u - b_M \right) \right\vert}{(1+\Vert u \Vert)^\gamma} \\
			& \quad\quad \leq C_{y,a} (1+\Vert a_M \Vert+\Vert b_M \Vert)^\gamma \max_{\alpha \in \mathbb{N}^m_{0,k}} \sup_{u \in \mathbb{R}^m} \frac{\left\vert \rho^{(\vert \alpha \vert)}\left( a_M^\top u - b_M \right) - \widetilde{\rho}^{(\vert \alpha \vert)}\left( a_M^\top u - b_M \right) \right\vert}{\left( 1 + \left\vert a_M^\top u - b_M \right\vert \right)^\gamma} \\
			& \quad\quad\quad\quad + C_{y,a} \frac{\Vert \widetilde{\rho} \Vert_{C^k_b(\mathbb{R})}}{(1+r)^\gamma} \\
			& \quad\quad \leq C_{y,a} C_{a,b} \max_{j=0,...,k} \sup_{s \in \mathbb{R}} \frac{\left\vert \rho^{(j)}(s) - \widetilde{\rho}^{(j)}(s) \right\vert}{\left( 1 + \left\vert s \right\vert \right)^\gamma} + C_{y,a} \frac{\varepsilon}{6 C_{y,a} C_{a,b}} \\
			& \quad\quad < C_{y,a} C_{a,b} \frac{\varepsilon}{6 C_{y,a} C_{a,b}} + \frac{\varepsilon}{6} = \frac{\varepsilon}{3}.
		\end{aligned}
	\end{equation}
	Analogously, we conclude that
	\begin{equation}
		\label{EqLemmaSepProof3}
		\max_{\alpha \in \mathbb{N}^m_{0,k}} \sup_{u \in \mathbb{R}^m \setminus \overline{B_r(0)}} \frac{\left\Vert y \rho^{(\vert \alpha \vert)}\left( a^\top u - b \right) a^\alpha \right\Vert}{(1+\Vert u \Vert)^\gamma} < \frac{\varepsilon}{3}.
	\end{equation}
	Moreover, we define the compact subset $K := \big\lbrace x^\top u - y: u \in \overline{B_r(0)}, \, \Vert x \Vert + \Vert y \Vert \leq C_{a,b} \big\rbrace \subseteq \mathbb{R}$. Then, by using that $\rho,\rho',...,\rho^{(k)} \in \overline{C^k_b(\mathbb{R})}^\gamma$ are continuous, thus uniformly continuous on $K$, there exists some $\delta > 0$ such that for every $j = 0,...,k$ and $s_1,s_2 \in K$ with $\vert s_1 - s_2 \vert < \delta$ it holds that
	\begin{equation}
		\label{EqLemmaSepProof4}
		\left\vert \rho^{(j)}(s_1) - \rho^{(j)}(s_2) \right\vert < \frac{\varepsilon}{6 C_{y,a}}.
	\end{equation}
	Now, we define the constant $C_{r,\rho} := 1 + \max_{j=0,...,k} \sup_{u \in \overline{B_r(0)}} \left\vert \rho^{(j)}\left( a^\top u - b \right) \right\vert > 0$. Moreover, we choose some $M_2 \in \mathbb{N}$ such that for every $M \in \mathbb{N} \cap [M_2,\infty)$ it holds that $\Vert (a-a_M,b-b_M) \Vert < \delta/(1+r)$ and that
	\begin{equation}
		\label{EqLemmaSepProof5}
		\max_{\alpha \in \mathbb{N}^m_{0,k}} \left\Vert y a^\alpha - y_M a_M^\alpha \right\Vert < \frac{\varepsilon}{6 C_{r,\rho}}.
	\end{equation}
	Then, we conclude for every $M \in \mathbb{N} \cap [M_2,\infty)$ that
	\begin{equation}
		\label{EqLemmaSepProof6}
		\begin{aligned}
			\left\vert \left( a^\top u - b \right) - \left( a_M^\top u - b_M \right) \right\vert & \leq \left\vert (a-a_M)^\top u - (b - b_M) \right\vert \\
			& \leq \Vert a-a_M \Vert \Vert u \Vert + \vert b-b_M \vert \\
			& \leq \left( \Vert a-a_M \Vert + \vert b-b_M \vert \right) (1+r) \\
			& \leq \Vert (a-a_M,b-b_M) \Vert (1+r) < \delta.
		\end{aligned}
	\end{equation}
	Hence, by using \eqref{EqLemmaSepProof5} and by combining \eqref{EqLemmaSepProof4} with \eqref{EqLemmaSepProof6}, it follows for every $M \in \mathbb{N} \cap [M_2,\infty)$ that
	\begin{equation}
		\label{EqLemmaSepProof7}
		\begin{aligned}
			& \max_{\alpha \in \mathbb{N}^m_{0,k}} \sup_{u \in \overline{B_r(0)}} \left\Vert y \rho^{(\vert \alpha \vert)}\left( a^\top u - b \right) a^\alpha - y_M \rho^{(\vert \alpha \vert)}\left( a_M^\top u - b_M \right) a_M^\alpha \right\Vert \\
			& \quad\quad \leq \max_{\alpha \in \mathbb{N}^m_{0,k}} \sup_{u \in \overline{B_r(0)}} \left\Vert y \rho^{(\vert \alpha \vert)}\left( a^\top u - b \right) a^\alpha - y_M \rho^{(\vert \alpha \vert)}\left( a^\top u - b \right) a_M^\alpha \right\Vert \\
			& \quad\quad\quad\quad + \max_{\alpha \in \mathbb{N}^m_{0,k}} \sup_{u \in \overline{B_r(0)}} \left\Vert y_M \rho^{(\vert \alpha \vert)}\left( a^\top u - b \right) a_M^\alpha - y_M \rho^{(\vert \alpha \vert)}\left( a_M^\top u - b_M \right) a_M^\alpha \right\Vert \\
			& \quad\quad \leq \max_{\alpha \in \mathbb{N}^m_{0,k}} \left\Vert y a^\alpha - y_M a_M^\alpha \right\Vert \max_{j=0,...,k} \sup_{u \in \overline{B_r(0)}} \left\vert \rho^{(j)}\left( a^\top u - b \right) \right\vert \\
			& \quad\quad\quad\quad + \max_{\alpha \in \mathbb{N}^m_{0,k}} \left\Vert y_M a_M^\alpha \right\Vert \max_{j=0,...,k} \sup_{u \in \overline{B_r(0)}} \left\vert \rho^{(j)}\left( a_M^\top u - b_M \right) - \rho^{(j)}\left( a^\top u - b \right) \right\vert \\
			& \quad\quad \leq \frac{\varepsilon}{6 C_{r,\rho}} C_{r,\rho} + C_{y,a} \frac{\varepsilon}{6 C_{y,a}} = \frac{\varepsilon}{3}.
		\end{aligned}
	\end{equation}
	Thus, by using the inequalities \eqref{EqLemmaSepProof2}+\eqref{EqLemmaSepProof3}+\eqref{EqLemmaSepProof7} and that $(1+\Vert u \Vert)^\gamma \geq 1$ for any $u \in U$, we have
	\begin{equation*}
		\begin{aligned}
			& \left\Vert y \rho\left( a^\top \cdot - b \right) - y_M \rho\left( a_M^\top \cdot - b_M \right) \right\Vert_{C^k_{pol,\gamma}(\mathbb{R}^m;\mathbb{R}^d)} \\
			& \quad\quad = \max_{\alpha \in \mathbb{N}^m_{0,k}} \sup_{u \in \mathbb{R}^m} \frac{\left\Vert y \rho^{(\vert \alpha \vert)}\left( a^\top u - b \right) a^\alpha - y_M \rho^{(\vert \alpha \vert)}\left( a_M^\top u - b_M \right) a_M^\alpha \right\Vert}{(1+\Vert u \Vert)^\gamma} \\
			& \quad\quad \leq \max_{\alpha \in \mathbb{N}^m_{0,k}} \sup_{u \in \overline{B_r(0)}} \left\Vert y \rho^{(\vert \alpha \vert)}\left( a^\top u - b \right) a^\alpha - y_M \rho^{(\vert \alpha \vert)}\left( a_M^\top u - b_M \right) a_M^\alpha \right\Vert \\
			& \quad\quad\quad\quad + \max_{\alpha \in \mathbb{N}^m_{0,k}} \sup_{u \in \mathbb{R}^m \setminus \overline{B_r(0)}} \frac{\left\Vert y \rho^{(\vert \alpha \vert)}\left( a^\top u - b \right) a^\alpha \right\Vert}{(1+\Vert u \Vert)^\gamma} \\
			& \quad\quad\quad\quad + \max_{\alpha \in \mathbb{N}^m_{0,k}} \sup_{u \in \mathbb{R}^m \setminus \overline{B_r(0)}} \frac{\left\Vert y_M \rho^{(\vert \alpha \vert)}\left( a_M^\top u - b_M \right) a_M^\alpha \right\Vert}{(1+\Vert u \Vert)^\gamma} \\
			& \quad\quad < \frac{\varepsilon}{3} + \frac{\varepsilon}{3} + \frac{\varepsilon}{3} = \varepsilon.
		\end{aligned}
	\end{equation*}
	Since $\varepsilon > 0$ was chosen arbitrarily, this shows that $\mathbb{R}^d \times \mathbb{R}^m \times \mathbb{R} \ni (y,a,b) \mapsto y \rho\big( a^\top \cdot - b \big) \in \overline{C^k_b(U;\mathbb{R}^d)}^\gamma$ is continuous. Hence, by using that $(\overline{C^k_b(\mathbb{R}^m;\mathbb{R}^d)}^\gamma,\Vert \cdot \Vert_{C^k_{pol,\gamma}(\mathbb{R}^m;\mathbb{R}^d)}) \ni f \mapsto f\vert_U \in (X,\Vert \cdot \Vert_X)$ is by Assumption~\ref{AssEmb} continuous, we obtain the conclusion in Part~\ref{LemmaSep2}.
	
	For Part~\ref{LemmaSep3}., we define the subsets
	\begin{equation*}
		\mathcal{N}^{\cos}_{U,d}[\mathbb{A}] := \left\lbrace U \ni u \mapsto \sum_{n=1}^N y_n \cos\left( a_n^\top \cdot - b_n \right) \in \mathbb{R}^d:
		\begin{matrix}
			N \in \mathbb{N}, \, y_1,...,y_N \in \mathbb{A}^d, \\
			a_1,...,a_N \in \mathbb{A}^m, \, b_1,...,b_N \in \mathbb{A}
		\end{matrix}
		\right\rbrace \subseteq X,
	\end{equation*}
	for $\mathbb{A} \in \lbrace \mathbb{Q}, \mathbb{R} \rbrace$. Then, by using that the map $\mathbb{R}^d \times \mathbb{R}^m \times \mathbb{R} \ni (y,a,b) \mapsto y \rho\big( a^\top \cdot - b \big) \in X$ is continuous (see Part~2.), we conclude that $\mathcal{N}^{\cos}_{U,d}[\mathbb{R}]$ is contained in the closure of $\mathcal{N}^{\cos}_{U,d}[\mathbb{Q}]$ with respect to $\Vert \cdot \Vert_X$. Moreover, by using that $\cos \in \overline{C^k_b(\mathbb{R})}^\gamma$ is non-polynomial, we can apply \cite[Theorem~2.8]{neufeld24} to conclude that $\mathcal{N}^{\cos}_{U,d}[\mathbb{R}]$ is dense in $X$. Hence, by combining these two arguments, we obtain that $\mathcal{N}^{\cos}_{U,d}[\mathbb{Q}]$ is also dense in $X$. Since $\mathcal{N}^{\cos}_{U,d}[\mathbb{Q}]$ is countable, this shows that $(X,\Vert \cdot \Vert_X)$ is separable.
\end{proof}

\begin{proof}[Proof of Corollary~\ref{CorUATRN}]
	We aim to apply Theorem~\ref{ThmUAT} with Banach space $(X,\Vert \cdot \Vert_X)$ satisfying Assumption~\ref{AssEmb}. To this end, we first observe that $(X,\Vert \cdot \Vert_X)$ is by Lemma~\ref{LemmaSep}.\ref{LemmaSep3} separable. Moreover, we choose $\Theta := \mathbb{R}^m \times \mathbb{R}$ and let $(\theta_n)_{n \in \mathbb{N}} := (a_n,b_n)_{n \in \mathbb{N}}: \Omega \rightarrow \mathbb{R}^m \times \mathbb{R}$ be an i.i.d.~sequence satisfying Assumption~\ref{AssCDFRN}, which implies Assumption~\ref{AssCDF}. In addition, we define
	\begin{equation*}
		\mathcal{G} := \left\lbrace \mathbb{R}^m \times \mathbb{R} \ni (\vartheta_1,\vartheta_2) \mapsto e_i \rho\left( \vartheta_1^\top \cdot - \vartheta_2 \right) \in X: i = 1,...,d \right\rbrace.
	\end{equation*}
	Since $\mathbb{R}^m \times \mathbb{R} \ni (\vartheta_1,\vartheta_2) \mapsto e_i \rho\big( \vartheta_1^\top \cdot - \vartheta_2 \big) \in X$ is by Lemma~\ref{LemmaSep}.\ref{LemmaSep2} continuous, we have $\mathcal{G} \subseteq C^0(\Theta;X)$. Moreover, we observe that
	\begin{equation*}
		\linspan_\mathbb{R}(\mathcal{G}(\Theta)) = \linspan_\mathbb{R}\left( \left\lbrace U \ni u \mapsto e_i \rho\left( \vartheta_1^\top u - \vartheta_2 \right) \in \mathbb{R}^d: (\vartheta_1,\vartheta_2) \in \mathbb{R}^m \times \mathbb{R}, \, i = 1,...,d \right\rbrace \right),
	\end{equation*}
	forms the set of deterministic (i.e.~fully trained) neural networks with activation function $\rho \in \overline{C^k_b(\mathbb{R})}^\gamma$, where $e_i \in \mathbb{R}^d$ denotes the $i$-th unit vector of $\mathbb{R}^d$. Since $\rho \in \overline{C^k_b(\mathbb{R})}^\gamma$ is non-polynomial, we can apply \cite[Theorem~2.8]{neufeld24} to conclude that $\linspan_\mathbb{R}(\mathcal{G}(\Theta))$ is dense in $X$. Hence, the conclusion follows from Theorem~\ref{ThmUAT}.
\end{proof}

\section{Proof of results in Section~\ref{SecAR}}

\subsection{Proof of Theorem~\ref{ThmAR}}
\label{SecProofsAR}

For the proof of Theorem~\ref{ThmAR}, we first show the following auxiliary lemma about Banach space types.

\begin{lemma}
	\label{LemmaBanachSpaceTypeSmaller}
	Let $(X,\Vert \cdot \Vert_X)$ be a Banach space of type $t \in [1,2]$ with constant $C_X > 0$, and let $t' \in [1,t]$. Then, $(X,\Vert \cdot \Vert_X)$ is a Banach space of type $t'$ with constant $C_X > 0$.
\end{lemma}
\begin{proof}
	Fix some $N \in \mathbb{N}$, $(x_n)_{n=1,...,N} \subseteq X$, and a Rademacher sequence $(\epsilon_n)_{n=1,...,N}$ defined on a (possibly different) probability space $(\widetilde{\Omega},\widetilde{\mathcal{F}},\widetilde{\mathbb{P}})$. Then, by using Jensen's inequality and the inequality $\big( \sum_{n=1}^N x_n \big)^{t'/t} \leq \sum_{n=1}^N x_n^{t'/t}$ for any $x_1,...,x_N \geq 0$, it follows that
	\begin{equation*}
		\widetilde{\mathbb{E}}\left[ \left\Vert \sum_{n=1}^N \epsilon_n x_n \right\Vert_X^{t'} \right]^\frac{1}{t'} \leq \widetilde{\mathbb{E}}\left[ \left\Vert \sum_{n=1}^N \epsilon_n x_n \right\Vert_X^t \right]^\frac{1}{t} \leq C_X \left( \sum_{n=1}^N \Vert x_n \Vert_X^t \right)^\frac{1}{t} \leq C_X \left( \sum_{n=1}^N \Vert x_n \Vert_X^{t'} \right)^\frac{1}{t'}.
	\end{equation*}
	This shows that $(X,\Vert \cdot \Vert_X)$ is a Banach space of type $t' \in [1,t]$ with constant $C_X > 0$. 
\end{proof}

\begin{proof}[Proof of Theorem~\ref{ThmAR}]
	Fix some $x \in \mathbb{B}^r_{\mathcal{G},\theta}(X)$ and $N \in \mathbb{N}$. Then, by definition of $\mathbb{B}^r_{\mathcal{G},\theta}(X)$, there exists a $\mathcal{B}(\Theta)/\mathcal{B}(\mathbb{R}^e)$-measurable map $y := (y_1,...,y_e)^\top: \Theta \rightarrow \mathbb{R}^e$ such that $x = \mathbb{E}\left[ \sum_{i=1}^e y_i(\theta_1) g_i(\theta_1) \right] \in X$ and
	\begin{equation}
		\label{EqThmARProof1}
		\mathbb{E}\left[ \left\Vert \sum_{i=1}^e y_i(\theta_1) g_i(\theta_1) \right\Vert_X^r \right]^\frac{1}{r} \leq 2 \Vert x \Vert_{\mathbb{B}^r_{\mathcal{G},\theta}(X)} < \infty.
	\end{equation}
	From this, we define for every fixed $n = 1,...,N$, the map
	\begin{equation*}
		\Omega \ni \omega \quad \mapsto \quad G_n(\omega) := \sum_{i=1}^e y_i(\theta_n(\omega)) g_i(\theta_n(\omega)) \in X.
	\end{equation*}
	Then, by using that $\theta_n: \Omega \rightarrow \Theta$ is by definition $\mathcal{F}_\theta/\mathcal{B}(\Theta)$-measurable and that $y := (y_1,...,y_e)^\top: \Theta \rightarrow \mathbb{R}^e$ is by definition $\mathcal{B}(\Theta)/\mathcal{B}(\mathbb{R}^e)$-measurable, the concatenation $\Omega \ni \omega \mapsto G_n := \sum_{l=1}^e y_l(\theta_n(\omega)) g_l(\theta_n(\omega)) \in X$ is $\mathcal{F}_\theta/\mathcal{B}(X)$-measurable. Hence, by using that $(X,\Vert \cdot \Vert_X)$ is separable, we can apply \cite[Theorem 1.1.6+1.1.20]{hytoenen16} to conclude that $G_n: \Omega \rightarrow X$ is strongly $(\mathbb{P},\mathcal{F}_\theta)$-measurable. Thus, \eqref{EqThmARProof1} and $\theta_n \sim \theta_1$ ensure that $G_n \in L^r(\Omega,\mathcal{F}_\theta,\mathbb{P};X)$.
	
	Now, by using that $x = \mathbb{E}\left[ \sum_{i=1}^e y_i(\theta_1) g_i(\theta_1) \right] = \mathbb{E}\left[ \sum_{i=1}^e y_i(\theta_n) g_i(\theta_n) \right] = \mathbb{E}[G_n] \in X$ for any $n = 1,...,N$, the right-hand side of \cite[Lemma~6.3]{ledoux91} for the independent mean-zero random variables $(\mathbb{E}[G_n] - G_n)_{n=1,...,N}$ with Rademacher sequence $(\epsilon_n)_{n=1,...,N}$ on $(\Omega,\mathcal{F},\mathbb{P})$ independent of $(\mathbb{E}[G_n] - G_n)_{n=1,...,N}$, the Kahane-Khintchine inequality in \cite[Theorem~3.2.23]{hytoenen16} with constant $\kappa_{r,\min(r,t)} > 0$ (depending only on $r \in [1,\infty)$ and $\min(r,t) \in [1,2]$), that $(X,\Vert \cdot \Vert_X)$ is by assumption a Banach space of type $t \in [1,2]$ (with constant $C_X > 0$), thus by Lemma~\ref{LemmaBanachSpaceTypeSmaller} of type $\min(r,t) \in (1,t]$ (with the same constant $C_X > 0$), and that $(\mathbb{E}[G_n] - G_n)_{n=1,...,N} \sim \mathbb{E}[G_1] - G_1$ are identically distributed, it follows for the random feature model $G_N := \frac{1}{N} \sum_{n=1}^N G_n \in \mathcal{RG} \cap L^r(\Omega,\mathcal{F}_\theta,\mathbb{P};X)$ that
	\begin{equation*}
		\begin{aligned}
			\mathbb{E}\left[ \left\Vert x - G_N \right\Vert_X^r \right]^\frac{1}{r} & = \frac{1}{N} \mathbb{E}\left[ \left\Vert \sum_{n=1}^N \left( \mathbb{E}[G_n] - G_n \right) \right\Vert_X^r \right]^\frac{1}{r} \\
			& \leq \frac{2}{N} \mathbb{E}\left[ \left\Vert \sum_{n=1}^N \epsilon_n \left( \mathbb{E}[G_n] - G_n \right) \right\Vert_X^r \right]^\frac{1}{r} \\
			& \leq \frac{2 \kappa_{r,\min(r,t)}}{N} \mathbb{E}\left[ \left\Vert \sum_{n=1}^N \epsilon_n \left( \mathbb{E}[G_n] - G_n \right) \right\Vert_X^{\min(r,t)} \right]^\frac{1}{\min(r,t)} \\
			& \leq \frac{2 C_X \kappa_{r,\min(r,t)}}{N} \left( \sum_{n=1}^N \mathbb{E}\left[ \left\Vert \mathbb{E}[G_n] - G_n \right\Vert_X^{\min(r,t)} \right] \right)^\frac{1}{\min(r,t)} \\
			& = \frac{2 C_X \kappa_{r,\min(r,t)}}{N^{1-\frac{1}{\min(r,t)}}} \mathbb{E}\left[ \left\Vert \mathbb{E}[G_1] - G_1 \right\Vert_X^{\min(r,t)} \right]^\frac{1}{\min(r,t)}. \\
		\end{aligned}
	\end{equation*}
	Hence, by using Jensen's inequality, Minkowski's inequality, \cite[Proposition~1.2.2]{hytoenen16}, the inequality \eqref{EqThmARProof1}, and the constant $C_{r,t} := 8 \kappa_{r,\min(r,t)} > 0$ (depending only on $r \in [1,\infty)$ and $t \in [1,2]$), we conclude for $G_N := \frac{1}{N} \sum_{n=1}^N G_n \in \mathcal{RG} \cap L^r(\Omega,\mathcal{F}_\theta,\mathbb{P};X)$ that
	\begin{equation*}
		\begin{aligned}
			\mathbb{E}\left[ \left\Vert x - G_N \right\Vert_X^r \right]^\frac{1}{r} & \leq \frac{2 C_X \kappa_{r,\min(r,t)}}{N^{1-\frac{1}{\min(r,t)}}} \mathbb{E}\left[ \left\Vert \mathbb{E}[G_1] - G_1 \right\Vert_X^r \right]^\frac{1}{r} \\
			& \leq \frac{2 C_X \kappa_{r,\min(r,t)}}{N^{1-\frac{1}{\min(r,t)}}} \left( \left\Vert \mathbb{E}[G_1] \right\Vert_X + \mathbb{E}\left[ \Vert G_1 \Vert_X^r \right]^\frac{1}{r} \right) \\
			& \leq \frac{4 C_X \kappa_{r,\min(r,t)}}{N^{1-\frac{1}{\min(r,t)}}} \left\Vert G_1 \right\Vert_{L^r(\Omega,\mathcal{F},\mathbb{P};X)} \\
			& \leq C_{r,t} C_X \frac{\Vert x \Vert_{\mathbb{B}^r_{\mathcal{G},\theta}(X)}}{N^{1-\frac{1}{\min(r,t)}}},
		\end{aligned}
	\end{equation*}
	which completes the proof.
\end{proof}

\subsection{Proof of Corollary~\ref{CorARTrigo}+\ref{CorARFourier}+\ref{CorARRN} and Proposition~\ref{PropConst}}
\label{SecProofsARCor}

\begin{proof}[Proof of Corollary~\ref{CorARTrigo}]
	We aim to apply Theorem~\ref{ThmAR} onto a fixed function $f \in W^{k,p}(U,\mathcal{L}(U),w) \cap L^1(\mathbb{R}^m,\mathcal{L}(\mathbb{R}^m),du)$ with $C_f := \big( \int_{\mathbb{R}^m} \frac{\vert \widehat{f}(\vartheta) \vert^r \left( 1+\Vert \vartheta \Vert^2 \right)^{kr/2}}{p_\theta(\vartheta)^{r-1}} d\vartheta \big)^{1/r} < \infty$. To this end, we first observe that $(W^{k,p}(U,\mathcal{L}(U),w;\mathbb{R}^d),\Vert \cdot \Vert_{W^{k,p}(U,\mathcal{L}(U),w;\mathbb{R}^d)})$ is a separable Banach space (see \cite[Lemma~4.7]{neufeld24}). Moreover, we define the linear readouts
	\begin{equation}
		\label{EqCorARTrigoProof1}
		\begin{aligned}
			\mathbb{R}^m \ni \vartheta \quad \mapsto \quad y_1(\vartheta) & := \frac{\re\big( \widehat{f}(\vartheta) \big)}{(2\pi)^m p_\theta(\vartheta)} \in \mathbb{R}, \quad\quad \text{and} \\
			\mathbb{R}^m \ni \vartheta \quad \mapsto \quad y_2(\vartheta) & := -\frac{\im\big( \widehat{f}(\vartheta) \big)}{(2\pi)^m p_\theta(\vartheta)} \in \mathbb{R},
		\end{aligned}
	\end{equation}
	which are $\mathcal{B}(\mathbb{R}^m)/\mathcal{B}(\mathbb{R})$-measurable as composition of the continuous function $\mathbb{R}^m \ni \vartheta \mapsto \widehat{f}(\vartheta) \in \mathbb{C}$ (see \cite[p.~214]{folland92}) and the $\mathcal{B}(\mathbb{C})/\mathcal{B}(\mathbb{R})$-measurable functions returning the real and imaginary part. Then, by using Jensen's inequality, it follows that
	\begin{equation*}
		\begin{aligned}
			\int_{\mathbb{R}^m} \vert \widehat{f}(\vartheta) \vert d\vartheta & \leq \int_{\mathbb{R}^m} \frac{\vert \widehat{f}(\vartheta) \vert}{p_\theta(\vartheta)} p_\theta(\vartheta) d\vartheta \leq \left( \int_{\mathbb{R}^m} \frac{\vert \widehat{f}(\vartheta) \vert^r}{p_\theta(\vartheta)^r} p_\theta(\vartheta) d\vartheta \right)^\frac{1}{r} \\
			& \leq \left( \int_{\mathbb{R}^m} \frac{\vert \widehat{f}(\vartheta) \vert^r}{p_\theta(\vartheta)^{r-1}} \left( 1+\Vert \vartheta \Vert^2 \right)^\frac{kr}{2} d\vartheta \right)^\frac{1}{r} = C_f < \infty,
		\end{aligned}
	\end{equation*}
	which shows that $\widehat{f} \in L^1(\mathbb{R}^m,\mathcal{L}(\mathbb{R}^m),du;\mathbb{C})$. Hence, we can apply the Fourier inversion theorem (see \cite[Equation~7.14]{folland92}) and use that the left-hand side is real-valued to conclude for a.e.~$u \in \mathbb{R}^m$ that
	\begin{equation}
		\label{EqCorARTrigoProof2}
		\begin{aligned}
			f(u) & = \frac{1}{(2\pi)^m} \int_{\mathbb{R}^m} \widehat{f}(\vartheta) e^{\mathbf{i} \vartheta^\top u} d\vartheta = \int_{\mathbb{R}^m} \frac{\widehat{f}(\vartheta)}{(2\pi)^m p_\theta(\vartheta)} e^{\mathbf{i} \vartheta^\top u} p_\theta(\vartheta) d\vartheta \\
			& = \mathbb{E}\left[ \frac{\widehat{f}(\theta_1)}{(2\pi)^m p_\theta(\theta_1)} \cos\left( \theta_1^\top u \right) + \frac{\widehat{f}(\theta_1)}{(2\pi)^m p_\theta(\theta_1)} \mathbf{i} \sin\left( \theta_1^\top u \right) \right] \\
			& = \mathbb{E}\left[ \frac{\re\big( \widehat{f}(\theta_1) \big)}{(2\pi)^m p_\theta(\theta_1)} \cos\left( \theta_1^\top u \right) - \frac{\im\big( \widehat{f}(\theta_1) \big)}{(2\pi)^m p_\theta(\theta_1)} \sin\left( \theta_1^\top u \right) \right] \\
			& = \mathbb{E}\left[ y_1(\theta_1) \cos\left( \theta_1^\top u \right) + y_2(\theta_1) \sin\left( \theta_1^\top u \right) \right].
		\end{aligned}
	\end{equation}
	Thus, by using integration by parts, it holds that
	\begin{equation}
		\label{EqCorARTrigoProof2b}
		\begin{aligned}
			\int_U \mathbb{E}\left[ y_1(\theta_1) \cos\left( \theta_1^\top u \right) + y_2(\theta_1) \sin\left( \theta_1^\top u \right) \right] \partial_\alpha h(u) du & = \int_U x(u) \partial_\alpha h(u) du \\
			& = (-1)^{\vert\alpha\vert} \int_U \partial_\alpha f(u) h(u) du,
		\end{aligned}
	\end{equation}
	which shows that the weak derivatives of $\mathbb{E}\left[ y_1(\theta_1) \cos\left( \theta_1^\top \cdot \right) + y_2(\theta_1) \sin\left( \theta_1^\top \cdot \right) \right]$ and $f$ coincide, and thus implies that $\mathbb{E}\left[ y_1(\theta_1) \cos\left( \theta_1^\top \cdot \right) + y_2(\theta_1) \sin\left( \theta_1^\top \cdot \right) \right] = f \in W^{k,p}(U,\mathcal{L}(U),w)$. In addition, by using that $\big\vert \vartheta^\alpha \big\vert := \big\vert \prod_{l=1}^m \vartheta_l^{\alpha_l} \big\vert = \prod_{l=1}^m \vert \vartheta_l \vert^{\alpha_l} \leq \prod_{l=1}^m \left( 1+\Vert \vartheta \Vert^2 \right)^{\alpha_l/2} = \left( 1+\Vert \vartheta \Vert^2 \right)^{k/2}$ for any $\alpha := (\alpha_1,...,\alpha_m) \in \mathbb{N}^m_{0,k}$ and $\vartheta := (\vartheta_1,...,\vartheta_m) \in \mathbb{R}^m$, that $\big\vert \mathbb{N}^m_{0,k} \big\vert = \sum_{j=0}^k m^j \leq 2m^k$, and that $w(U) := \int_U w(u) du < \infty$, we conclude for every $\vartheta \in \mathbb{R}^m$ that
	\begin{equation}
		\label{EqCorARTrigoProof3}
		\begin{aligned}
			\left\Vert \cos\left( \vartheta^\top \cdot \right) \right\Vert_{W^{k,p}(U,\mathcal{L}(U),w)} & = \left( \sum_{\alpha \in \mathbb{N}^m_{0,k}} \int_U \left\vert \partial_\alpha \left( \cos\left( \vartheta^\top u \right) \right) \right\vert^p w(u) du \right)^\frac{1}{p} \\
			& = \left( \sum_{\alpha \in \mathbb{N}^m_{0,k}} \int_U \left\vert \cos^{(\vert\alpha\vert)}\left( \vartheta^\top u \right) \vartheta^\alpha \right\vert^p w(u) du \right)^\frac{1}{p} \\
			& \leq \left\vert \mathbb{N}^m_{0,k} \right\vert^\frac{1}{p} \left( 1 + \Vert \vartheta \Vert^2 \right)^\frac{k}{2} \left( \int_U w(u) du \right)^\frac{1}{p} \\
			& \leq 2^\frac{1}{p} m^\frac{k}{p} \left( 1 + \Vert \vartheta \Vert^2 \right)^\frac{k}{2}.
		\end{aligned}
	\end{equation}
	Moreover, by using the same arguments as in \eqref{EqCorARTrigoProof3}, we also obtain that
	\begin{equation}
		\label{EqCorARTrigoProof3b}
		\left\Vert \sin\left( \vartheta^\top \cdot \right) \right\Vert_{W^{k,p}(U,\mathcal{L}(U),w)} \leq 2^\frac{1}{p} m^\frac{k}{p} \left( 1 + \Vert \vartheta \Vert^2 \right)^\frac{k}{2}.
	\end{equation}
	Hence, by using that $\mathbb{E}\left[ y_1(\theta_1) \cos\left( \theta_1^\top \cdot \right) + y_2(\theta_1) \sin\left( \theta_1^\top \cdot \right) \right] = f \in W^{k,p}(U,\mathcal{L}(U),w)$, the inequalities~\eqref{EqCorARTrigoProof3}+\eqref{EqCorARTrigoProof3b}, and that $\vert \re(z) \vert + \vert \im(z) \vert \leq \sqrt{2} \vert z \vert$ for any $z \in \mathbb{C}$, it follows that
	\begin{equation}
		\label{EqCorARTrigoProof4}
		\begin{aligned}
			& \Vert f \Vert_{\mathbb{B}^r_{\mathcal{G},\theta}(W^{k,p}(U,\mathcal{L}(U),w))} \leq \mathbb{E}\left[ \left\Vert y_1(\theta_1) \cos\left( \theta_1^\top \cdot \right) + y_2(\theta_1) \sin\left( \theta_1^\top \cdot \right) \right\Vert_{W^{k,p}(U,\mathcal{L}(U),w)}^r \right]^\frac{1}{r} \\
			& \quad\quad \leq \mathbb{E}\left[ \left( \vert y_1(\theta_1) \vert \left\Vert \cos\left( \theta_1^\top \cdot \right) \right\Vert_{W^{k,p}(U,\mathcal{L}(U),w)} + \vert y_2(\theta_1) \vert \left\Vert \sin\left( \theta_1^\top \cdot \right) \right\Vert_{W^{k,p}(U,\mathcal{L}(U),w)} \right)^r \right]^\frac{1}{r} \\
			& \quad\quad \leq \frac{2^\frac{1}{p} m^\frac{k}{p} w(U)^\frac{1}{p}}{(2\pi)^m} \mathbb{E}\left[ \frac{\Big( \big\vert \re\big( \widehat{f}(\theta_1) \big) \big\vert + \big\vert \im\big( \widehat{f}(\theta_1) \big) \big\vert \Big)^r}{p_\theta(\theta_1)^r} \left( 1 + \Vert \theta_1 \Vert^2 \right)^\frac{kr}{2} \right]^\frac{1}{r} \\
			& \quad\quad \leq \frac{2^{\frac{1}{p}+\frac{1}{2}} m^\frac{k}{p} w(U)^\frac{1}{p}}{(2\pi)^m} \left( \int_{\mathbb{R}^m} \frac{\big\vert \widehat{f}(\vartheta) \big\vert^r}{p_\theta(\vartheta)^r} \left( 1 + \Vert \vartheta \Vert^2 \right)^\frac{kr}{2} p_\theta(\vartheta) d\vartheta \right)^\frac{1}{r} \\
			& \quad\quad = \frac{2^{\frac{1}{p}+\frac{1}{2}} m^\frac{k}{p} w(U)^\frac{1}{p}}{(2\pi)^m} C_f < \infty,
		\end{aligned}
	\end{equation}
	which shows that $f \in \mathbb{B}^r_{\mathcal{G},\theta}(X)$. Thus, by using that $(W^{k,p}(U,\mathcal{L}(U),w),\Vert \cdot \Vert_{W^{k,p}(U,\mathcal{L}(U),w)})$ is a Banach space of type $t = \min(2,p)$ with constant $C_{W^{k,p}(U,\mathcal{L}(U),w)} > 0$ depending only on $p \in (1,\infty)$ (see \cite[Lemma~4.9]{neufeld24}), we can apply Theorem~\ref{ThmAR} (with constant $C_{r,t} > 0$ depending only on $r \in [1,\infty)$ and $t \in [1,2]$), insert the inequality \eqref{EqCorARTrigoProof4}, and define the constant $C_{p,r} := 2^{1/p+1/2} C_{r,t} C_{W^{k,p}(U,\mathcal{L}(U),w)} > 0$ (depending only on $p \in (1,\infty)$ and $r \in [1,\infty)$) to conclude that there exists a random trigonometric feature model $G_N \in \mathcal{RT}_{U,1} \cap L^r(\Omega,\mathcal{F}_\theta,\mathbb{P};W^{k,p}(U,\mathcal{L}(U),w))$ with $N$ features satisfying
	\begin{equation*}
		\begin{aligned}
			\mathbb{E}\left[ \Vert f - G_N \Vert_{W^{k,p}(U,\mathcal{L}(U),w)}^r \right]^\frac{1}{r} & \leq C_{r,t} C_{W^{k,p}(U,\mathcal{L}(U),w)} \frac{\Vert f \Vert_{\mathbb{B}^r_{\mathcal{G},\theta}(W^{k,p}(U,\mathcal{L}(U),w))}}{N^{1-\frac{1}{\min(2,p,r)}}} \\
			& \leq C_{p,r} \frac{m^\frac{k}{p} w(U)^\frac{1}{p}}{(2\pi)^m} \frac{C_f}{N^{1-\frac{1}{\min(2,p,r)}}},
		\end{aligned}
	\end{equation*}
	which completes the proof.
\end{proof}

\begin{proof}[Proof of Corollary~\ref{CorARFourier}]
	We aim to apply Theorem~\ref{ThmAR} onto a fixed function $f \in W^{k,p}(U,\mathcal{L}(U),w;\mathbb{C}) \cap L^1(\mathbb{R}^m,\mathcal{L}(\mathbb{R}^m),du;\mathbb{C})$ with $\big( \int_{\mathbb{R}^m} \frac{\vert \widehat{f}(\vartheta) \vert^r \left( 1+\Vert \vartheta \Vert^2 \right)^{kr/2}}{p_\theta(\vartheta)^{r-1}} d\vartheta \big)^{1/r} < \infty$. To this end, we first observe that $(W^{k,p}(U,\mathcal{L}(U),w;\mathbb{R}^d),\Vert \cdot \Vert_{W^{k,p}(U,\mathcal{L}(U),w;\mathbb{R}^d)})$ is a separable Banach space (see \cite[Lemma~4.7]{neufeld24}). Moreover, we define
	\begin{equation*}
		\mathbb{R}^m \ni \vartheta \quad \mapsto \quad y(\vartheta) := \frac{\widehat{f}(\vartheta)}{(2\pi)^m p_\theta(\vartheta)} \in \mathbb{C},
	\end{equation*}
	and follow the proof of Corollary~\ref{CorARFourier}, where the Fourier inversion theorem is applied to conclude for a.e.~$u \in \mathbb{R}^m$ that
	\begin{equation*}
		f(u) = \frac{1}{(2\pi)^m} \int_{\mathbb{R}^m} \widehat{f}(\vartheta) e^{\mathbf{i} \vartheta^\top u} d\vartheta = \int_{\mathbb{R}^m} \frac{\widehat{f}(\vartheta)}{(2\pi)^m p_\theta(\vartheta)} e^{\mathbf{i} \vartheta^\top u} p_\theta(\vartheta) d\vartheta = \mathbb{E}\left[ y(\theta_1) \exp\left( \mathbf{i} \theta_1^\top u \right) \right].
	\end{equation*}
	Hence, by using the same steps as in the proof of Corollary~\ref{CorARFourier}, we obtain the conclusion from Theorem~\ref{ThmAR}.
\end{proof}

\begin{proof}[Proof of Lemma~\ref{LemmaIneq}]
	Let $\mathcal{G}$ as in Remark~\ref{RemIntRepr}, let $\psi \in \mathcal{S}_0(\mathbb{R};\mathbb{C})$, let $\rho \in C^k_{pol,\gamma}(\mathbb{R})$, and fix some $f \in \widetilde{\mathbb{B}}^{k,r,\gamma}_{\psi,a,b}(U;\mathbb{R}^d)$. Then, there exists by definition of $\widetilde{\mathbb{B}}^{k,r,\gamma}_{\psi,a,b}(U;\mathbb{R}^d)$ some $h \in L^1(\mathbb{R}^m;\mathcal{L}(\mathbb{R}^m),du;\mathbb{R}^d)$ such that $\widehat{h} \in L^1(\mathbb{R}^m;\mathcal{L}(\mathbb{R}^m),du;\mathbb{C}^d)$, $h = f$ a.e.~on $U$, and
	\begin{equation}
		\label{EqLemmaIneqProof1}
		\mathbb{E}\left[ \left\Vert \frac{\left( 1 + \Vert a_1 \Vert^2 \right)^\frac{\gamma+k}{2} \left( 1 + \vert b_1 \vert^2 \right)^\frac{\gamma}{2}}{p_{a,b}(a_1,b_1)} (\mathfrak{R}_\psi h)(a_1,b_1) \right\Vert^r \right]^\frac{1}{r} \leq 2 \Vert f \Vert_{\widetilde{\mathbb{B}}^{k,r,\gamma}_{\psi,a,b}(U;\mathbb{R}^d)} < \infty.
	\end{equation}
	Moreover, we recall that $\mathcal{G}$ consists of the feature maps given by
	\begin{equation}
		\label{EqLemmaIneqProof2}
		\mathbb{R}^m \times \mathbb{R} \ni (a,b) \quad \mapsto \quad g_i(a,b) := e_i \rho\left( a^\top \cdot - b \right) \in W^{k,p}(U,\mathcal{L}(U),w;\mathbb{R}^d), \quad\quad i = 1,...,d,
	\end{equation}
	where $e_i \in \mathbb{R}^d$ denotes the $i$-th unit vector. In addition, we define the linear readout
	\begin{equation}
		\label{EqLemmaIneqProof3}
		\mathbb{R}^m \times \mathbb{R} \ni (a,b) \quad \mapsto \quad y(a,b) := (y_i(a,b))_{i=1,...,d}^\top := \re\left( \frac{(\mathfrak{R}_\psi f)(a,b)}{C^{(\psi,\rho)}_m p_{a,b}(a,b)} \right) \in \mathbb{R}^d.
	\end{equation}
	Then, by using Proposition~\ref{PropIntRepr} together with the fact that $h: \mathbb{R}^m \rightarrow \mathbb{R}^d$ is $\mathbb{R}^d$-valued and that $h = f$ a.e.~on $U$, it follows for a.e.~$u \in U$ that
	\begin{equation*}
		\mathbb{E}\left[ y(a_1,b_1) \rho\left( a_1^\top u - b_1 \right) \right] = h(u) = f(u),
	\end{equation*}
	which implies by following the arguments of \eqref{EqCorARTrigoProof2b} that
	\begin{equation*}
		\mathbb{E}\left[ \sum_{i=1}^d y_i(a_1,b_1) e_i \rho\left( a_1^\top \cdot - b_1 \right) \right] = \mathbb{E}\left[ y(a_1,b_1) \rho\left( a_1^\top \cdot - b_1 \right) \right] = f \in W^{k,p}(U,\mathcal{L}(U),w;\mathbb{R}^d).
	\end{equation*}
	Hence, by using that $\big\vert a^\alpha \big\vert := \big\vert \prod_{l=1}^m a_l^{\alpha_l} \big\vert = \prod_{l=1}^m \vert a_l \vert^{\alpha_l} \leq \big( 1+\Vert a \Vert^2 \big)^{\vert \alpha \vert/2} \leq \big( 1+\Vert a \Vert^2 \big)^{k/2}$ for any $\alpha := (\alpha_1,...,\alpha_m) \in \mathbb{N}^m_{0,k}$ and $a := (a_1,...,a_m) \in \mathbb{R}^m$, the inequality \cite[Equation~42]{neufeld24} (with constant $C^{(\gamma,p)}_{U,w} > 0$ defined in \eqref{EqCorARRN1}), that $\big\vert \mathbb{N}^m_{0,k} \big\vert = \sum_{j=0}^k m^j \leq 2m^k$, and \eqref{EqLemmaIneqProof1}, we have
	\begin{equation}
		\label{EqLemmaIneqProof4}
		\begin{aligned}
			& \Vert f \Vert_{\mathbb{B}^r_{\mathcal{G},\theta}(W^{k,p}(U,\mathcal{L}(U),w;\mathbb{R}^d))} \leq \mathbb{E}\left[ \left\Vert \sum_{i=1}^d y_i(a_1,b_1) e_i \rho\left( a_1^\top u - b_1 \right) \right\Vert_{W^{k,p}(U,\mathcal{L}(U),w;\mathbb{R}^d)}^r \right]^\frac{1}{r} \\
			& \quad\quad = \mathbb{E}\left[ \left( \sum_{\alpha \in \mathbb{N}^m_{0,k}} \int_U \left\vert \partial_\alpha \Big( y(a_1,b_1) \rho\left( a_1^\top u - b_1 \right) \Big) \right\vert^p w(u) du \right)^\frac{r}{p} \right]^\frac{1}{r} \\
			& \quad\quad \leq \mathbb{E}\left[ \left( \sum_{\alpha \in \mathbb{N}^m_{0,k}} \left\Vert \re\left( \frac{a_1^\alpha (\mathfrak{R}_\psi g)(a_1,b_1)}{C^{(\psi,\rho)}_m p_{a,b}(a_1,b_1)} \right) \right\Vert^p \int_U \left\vert \rho^{(\vert\alpha\vert)}\left( a_1^\top u - b_1 \right) \right\vert^p du \right)^\frac{r}{p} \right]^\frac{1}{r} \\
			& \quad\quad \leq 4 \Vert \rho \Vert_{C^k_{pol,\gamma}(\mathbb{R})} \frac{C^{(\gamma,p)}_{U,w} \left\vert \mathbb{N}^m_{0,k} \right\vert^\frac{1}{p}}{\left\vert C^{(\psi,\rho)}_m \right\vert} \mathbb{E}\left[ \left\Vert \frac{\left( 1 + \Vert a_1 \Vert^2 \right)^\frac{\gamma+k}{2} \left( 1 + \vert b_1 \vert^2 \right)^\frac{\gamma}{2}}{p_{a,b}(a_1,b_1)} (\mathfrak{R}_\psi g)(a_1,b_1) \right\Vert^r \right]^\frac{1}{r} \\
			& \quad\quad \leq 2^{3+\frac{1}{p}} \Vert \rho \Vert_{C^k_{pol,\gamma}(\mathbb{R})} \frac{C^{(\gamma,p)}_{U,w} m^\frac{k}{p}}{\left\vert C^{(\psi,\rho)}_m \right\vert} \Vert f \Vert_{\widetilde{\mathbb{B}}^{k,r,\gamma}_{\psi,a,b}(U;\mathbb{R}^d)} < \infty,
		\end{aligned}
	\end{equation}
	which shows that $f \in \mathbb{B}^r_{\mathcal{G},\theta}(W^{k,p}(U,\mathcal{L}(U),w;\mathbb{R}^d))$.
\end{proof}

\begin{proof}[Proof of Corollary~\ref{CorARRN}]
	We aim to apply Theorem~\ref{ThmAR} onto a fixed function $f \in W^{k,p}(U,\mathcal{L}(U),w;\mathbb{R}^d) \cap \widetilde{\mathbb{B}}^{k,r,\gamma}_{\psi,a,b}(U;\mathbb{R}^d)$, where we recall that $(W^{k,p}(U,\mathcal{L}(U),w;\mathbb{R}^d),\Vert \cdot \Vert_{W^{k,p}(U,\mathcal{L}(U),w;\mathbb{R}^d)})$ is separable. To this end, we use that there exists by definition of $\widetilde{\mathbb{B}}^{k,r,\gamma}_{\psi,a,b}(U;\mathbb{R}^d)$ some $h \in L^1(\mathbb{R}^m;\mathcal{L}(\mathbb{R}^m),du;\mathbb{R}^d)$ such that $\widehat{h} \in L^1(\mathbb{R}^m;\mathcal{L}(\mathbb{R}^m),du;\mathbb{C}^d)$, $h = f$ a.e.~on $U$, and \eqref{EqLemmaIneqProof1} holds true. Moreover, we recall that the feature maps are given by \eqref{EqLemmaIneqProof2} and define the linear $y: \mathbb{R}^m \times \mathbb{R} \rightarrow \mathbb{R}^d$ as in \eqref{EqLemmaIneqProof3}. Then, by using that $(W^{k,p}(U,\mathcal{L}(U),w;\mathbb{R}^d),\Vert \cdot \Vert_{W^{k,p}(U,\mathcal{L}(U),w;\mathbb{R}^d)})$ is a Banach space of type $t = \min(2,p)$ with constant $C_{W^{k,p}(U,\mathcal{L}(U),w;\mathbb{R}^d)} > 0$ depending only on $p \in (1,\infty)$ (see \cite[Lemma~4.9]{neufeld24}), we can use Theorem~\ref{ThmAR} (with constant $C_{r,t} > 0$ depending only on $r \in [1,\infty)$ and $t \in [1,2]$), Lemma~\ref{LemmaIneq}, and the constant $C_{p,r} := 2^{3+1/p} C_{r,t} C_{W^{k,p}(U,\mathcal{L}(U),w;\mathbb{R}^d)} > 0$ (depending only on $p \in (1,\infty)$ and $r \in [1,\infty)$) to conclude that there exists a random neural network $G_N \in \mathcal{RN}^\rho_{U,d} \cap L^r(\Omega,\mathcal{F}_{a,b},\mathbb{P};W^{k,p}(U,\mathcal{L}(U),w;\mathbb{R}^d))$ with $N$ neurons satisfying
	\begin{equation*}
		\begin{aligned}
			\mathbb{E}\left[ \Vert f - G_N \Vert_{W^{k,p}(U,\mathcal{L}(U),w;\mathbb{R}^d)}^r \right]^\frac{1}{r} & \leq C_{r,t} C_{W^{k,p}(U,\mathcal{L}(U),w;\mathbb{R}^d)} \frac{\Vert f \Vert_{\mathbb{B}^r_{\mathcal{G},\theta}(W^{k,p}(U,\mathcal{L}(U),w;\mathbb{R}^d))}}{N^{1-\frac{1}{\min(2,p,r)}}} \\
			& \leq C_{p,r} \Vert \rho \Vert_{C^k_{pol,\gamma}(\mathbb{R})} \frac{C^{(\gamma,p)}_{U,w} m^\frac{k}{p}}{\left\vert C^{(\psi,\rho)}_m \right\vert} \frac{\Vert f \Vert_{\widetilde{\mathbb{B}}^{k,r,\gamma}_{\psi,a,b}(U;\mathbb{R}^d)}}{N^{1-\frac{1}{\min(2,p,r)}}},
		\end{aligned}
	\end{equation*}
	which completes the proof.
\end{proof}

\begin{proof}[Proof of Proposition~\ref{PropConst}]
	For \eqref{EqPropConst1}, let $f \in L^1(\mathbb{R}^m,\mathcal{L}(\mathbb{R}^m),du;\mathbb{R}^d)$ with $(\lceil\gamma\rceil+2)$-times differentiable Fourier transform and fix some $c \in \lbrace 0,\lceil\gamma\rceil+2 \rbrace$. Then, by using that $b_1 \sim \mathbf{t}_1$, the inequality \cite[Equation~46]{neufeld24}, and Minkowski's integral inequality (with measure spaces $(\mathbb{R}^m,\mathcal{L}(\mathbb{R}^m),da)$ and $(\mathbb{N}^m_{0,k} \times \mathbb{R},\mathcal{P}(\mathbb{N}^m_{0,k}) \otimes \mathcal{B}(\mathbb{R}),d\mu \otimes d\zeta)$, where $\mathcal{P}(\mathbb{N}^m_{0,k})$ denotes the power set of $\mathbb{N}^m_{0,k}$, and where $\mathcal{P}(\mathbb{N}^m_{0,k}) \ni E \mapsto \mu(E) := \sum_{\alpha \in \mathbb{N}^m_{0,k}} \mathds{1}_E(\alpha) \in [0,\infty)$ is the counting measure) together with the probability distribution function of $a_1 \sim \mathbf{t}_m$, we have
	\begin{equation*}
		\begin{aligned}
			& \Vert f \Vert_{\widetilde{\mathbb{B}}^{k,r,\gamma}_{\psi,a,b}(U;\mathbb{R}^d)} \leq \mathbb{E}\left[ \left\Vert \frac{\left( 1 + \Vert a_1 \Vert^2 \right)^\frac{\gamma+k}{2} \left( 1 + \vert b_1 \vert^2 \right)^\frac{\gamma}{2}}{p_{a,b}(a_1,b_1)} (\mathfrak{R}_\psi f)(a_1,b_1) \right\Vert^r \right]^\frac{1}{r} \\
			& \leq \mathbb{E}\left[ \sup_{\widetilde{b} \in \mathbb{R}} \left\Vert \frac{\left( 1 + \Vert a_1 \Vert^2 \right)^\frac{\lceil\gamma\rceil+k}{2}}{p_a(a_1)} \left( 1 + \big\vert \widetilde{b} \big\vert^2 \right)^\frac{\lceil\gamma\rceil+2}{2} (\mathfrak{R}_\psi f)\big(a_1,\widetilde{b}\big) \right\Vert^r \right]^\frac{1}{r} \\
			& \leq 2^\frac{\lceil\gamma\rceil}{2} \frac{(\lceil\gamma\rceil+2)!}{\pi} \mathbb{E}\left[ \left\Vert \frac{\left( 1 + \Vert a_1 \Vert^2 \right)^\frac{2\lceil\gamma\rceil+k+2}{2}}{p_a(a_1)} \sum_{\beta \in \mathbb{N}^m_{0,\lceil\gamma\rceil+2}} \int_{\mathbb{R}} \big\Vert \partial_\beta \widehat{f}(\zeta a) \big\Vert \left\vert \widehat{\psi}^{(\lceil\gamma\rceil+2-\vert\beta\vert)}(\zeta) \right\vert d\zeta \right\Vert^r \right]^\frac{1}{r} \\
			& \leq 2^\frac{\lceil\gamma\rceil}{2} \frac{(\lceil\gamma\rceil+2)!}{\pi} \sum_{\beta \in \mathbb{N}^m_{0,\lceil\gamma\rceil+2}} \int_{\mathbb{R}} \left\vert \widehat{\psi}^{(\lceil\gamma\rceil+2-\vert\beta\vert)}(\zeta) \right\vert \left( \int_{\mathbb{R}^m} \big\Vert \partial_\beta \widehat{f}(\zeta a) \big\Vert^r \frac{\left( 1 + \Vert a \Vert^2 \right)^\frac{(2\lceil\gamma\rceil+k+2)r}{2}}{p_a(a)^{r-1}} da \right)^\frac{1}{r} d\zeta.
		\end{aligned}
	\end{equation*}
	Hence, by using the substitution $\xi \mapsto \zeta a$ with Jacobi determinant $d\xi = \vert \zeta \vert^m da$, that $\zeta_1 := \inf\big\lbrace \vert \zeta \vert : \zeta \in \supp(\widehat{\psi}) \big\rbrace > 0$, and $C_1 := 2^{\lceil\gamma\rceil/2} (\lceil\gamma\rceil+2)! \max_{j=0,...,\lceil\gamma\rceil+2} \int_{\mathbb{R}} \big\vert \widehat{\psi}^{(j)}(\zeta) \big\vert d\zeta > 0$ (depending only on $\gamma \in [0,\infty)$ and $\psi \in \mathcal{S}(\mathbb{R};\mathbb{C})$), it follows that
	\begin{equation*}
		\begin{aligned}
			& \Vert f \Vert_{\widetilde{\mathbb{B}}^{k,r,\gamma}_{\psi,a,b}(U;\mathbb{R}^d)} \\
			& \quad \leq \frac{(\lceil\gamma\rceil+2)!}{\pi} \sum_{\beta \in \mathbb{N}^m_{0,\lceil\gamma\rceil+2}} \int_{\mathbb{R}} \frac{\left\vert \widehat{\psi}^{(\lceil\gamma\rceil+2-\vert\beta\vert)}(\zeta) \right\vert}{\zeta^\frac{m}{2}} \left( \int_{\mathbb{R}^m} \big\Vert \partial_\beta \widehat{f}(\xi) \big\Vert^r \frac{\left( 1 + \Vert \xi/\zeta \Vert^2 \right)^\frac{(2\lceil\gamma\rceil+k+2)r}{2}}{p_a(\xi/\zeta)^{r-1}} d\xi \right)^\frac{1}{r} d\zeta \\
			& \quad \leq \frac{C_1}{\zeta_1^\frac{m}{2}} \sup_{\zeta \in \supp(\widehat{\psi})} \sum_{\beta \in \mathbb{N}^m_{0,\lceil\gamma\rceil+2}} \left( \int_{\mathbb{R}^m} \big\Vert \partial_\beta \widehat{f}(\xi) \big\Vert^r \frac{\left( 1 + \Vert \xi/\zeta \Vert^2 \right)^\frac{(2\lceil\gamma\rceil+k+2)r}{2}}{p_a(\xi/\zeta)^{r-1}} d\xi \right)^\frac{1}{r},
		\end{aligned}
	\end{equation*}
	which proves \eqref{EqPropConst1}. For \eqref{EqPropConst2}, we use \eqref{EqPropConst1} and that $a_1 \sim \mathbf{t}_m$ to obtain that
	\begin{equation*}
		\begin{aligned}
			\Vert f \Vert_{\widetilde{\mathbb{B}}^{k,r,\gamma}_{\psi,a,b}(U;\mathbb{R}^d)} & \leq \frac{C_1}{\zeta_1^\frac{m}{2}} \sup_{\zeta \in \supp(\widehat{\psi})} \sum_{\beta \in \mathbb{N}^m_{0,\lceil\gamma\rceil+2}} \left( \int_{\mathbb{R}^m} \big\Vert \widehat{f}(\xi) \big\Vert^r \frac{\left( 1 + \Vert \xi/\zeta \Vert^2 \right)^{2\lceil\gamma\rceil+k+2}}{p_a(\xi/\zeta)} d\xi \right)^\frac{1}{2} \\
			& = \frac{C_1}{\zeta_1^\frac{m}{2}} \frac{\pi^\frac{m+1}{4}}{\Gamma\left( \frac{m+1}{2} \right)^\frac{1}{2}} \sum_{\beta \in \mathbb{N}^m_{0,\lceil\gamma\rceil+2}} \left( \int_{\mathbb{R}^m} \big\Vert \partial_\beta \widehat{f}(\xi) \big\Vert^2 \left( 1 + \Vert \xi/\zeta_1 \Vert^2 \right)^{2\lceil\gamma\rceil+k+\frac{m+5}{2}} d\xi \right)^\frac{1}{2},
		\end{aligned}
	\end{equation*}
	which completes the proof.
\end{proof}

\begin{proof}[Proof of Proposition~\ref{PropCOD}]
	The proof is based on \cite[Proposition~3.10]{neufeld24} which has established a similar result, but with respect to \emph{deterministic} neural networks. To this end, we fix some $m,d \in \mathbb{N}$ and $\varepsilon > 0$. Moreover, let $p > 1$ and $w: U \rightarrow [0,\infty)$ be a weight as in Lemma~\ref{LemmaWeight} (with constant $C^{(\gamma,p)}_{\mathbb{R},w_0} > 0$ being independent of $m,d \in \mathbb{N}$ and $\varepsilon > 0$), let $(\psi,\rho) \in \mathcal{S}_0(\mathbb{R};\mathbb{C}) \times C^k_{pol,\gamma}(\mathbb{R})$ be a pair as in Example~\ref{ExAdm} (with $0 < \zeta_1 < \zeta_2 < \infty$ and constant $C_{\psi,\rho} > 0$ being independent of $m,d \in \mathbb{N}$ and $\varepsilon > 0$), and fix some $f \in W^{k,p}(U,\mathcal{L}(U),w;\mathbb{R}^d)$ satisfying the conditions of Proposition~\ref{PropConst} such that the right-hand side of \eqref{EqPropConst2} satisfies $\mathcal{O}\left( m^s (2/\zeta_2)^m (m+1)^{m/2} \right)$ for some $s \in \mathbb{N}_0$. Then, there exists some constant $C > 0$ (being independent of $m,d \in \mathbb{N}$ and $\varepsilon > 0$) such that for every $m,d \in \mathbb{N}$ it holds that
	\begin{equation}
		\label{EqPropCODProof1}
		\frac{C_1}{\zeta_1^\frac{m}{2}} \sum_{\beta \in \mathbb{N}^m_{0,\lceil\gamma\rceil+2}} \left( \int_{\mathbb{R}^m} \big\Vert \partial_\beta \widehat{f}(\xi) \big\Vert^2 \left( 1 + \Vert \xi/\zeta_1 \Vert^2 \right)^{2\lceil\gamma\rceil+k+\frac{m+5}{2}} d\xi \right)^\frac{1}{2} \leq C m^s \left( \frac{2}{\zeta_2} \right)^m (m+1)^\frac{m}{2}.
	\end{equation}
	Hence, by using the inequality~\eqref{EqPropConst2} in Proposition~\ref{PropConst} together with \eqref{EqPropCODProof1}, that $\Gamma(x) \geq \sqrt{2\pi/x} (x/e)^x$ for any $x \in (0,\infty)$ (see \cite[Lemma~2.4]{gonon19}), and that $\frac{\pi^{m/4} (2/\zeta_2)^m}{(2\pi/\zeta_2)^m (1/(2e))^{m/2}} = \big( \frac{2e\sqrt{\pi}}{\pi^2} \big)^{m/2} \leq 1$ for any $m \in \mathbb{N}$, we conclude that there exist some constants $C_2, C_3 > 0$ (being independent of $m,d \in \mathbb{N}$ and $\varepsilon > 0$) such that
	\begin{equation}
		\label{EqPropCODProof2}
		\begin{aligned}
			& C_p \Vert \rho \Vert_{C^k_{pol,\gamma}(\mathbb{R})} \frac{C^{(\gamma,p)}_{\mathbb{R},w_0} m^{\gamma+\frac{k+1}{p}}}{C_{\psi,\rho} \left( \frac{2\pi}{\zeta_2} \right)^m} \Vert f \Vert_{\widetilde{\mathbb{B}}^{k,r,\gamma}_{\psi,a,b}(U;\mathbb{R}^d)} \\
			& \quad\quad \leq C_p \Vert \rho \Vert_{C^k_{pol,\gamma}(\mathbb{R})} \frac{C^{(\gamma,p)}_{\mathbb{R},w_0} m^{\gamma+\frac{k+1}{p}} \pi^\frac{m+1}{4}}{C_{\psi,\rho} \left( \frac{2\pi}{\zeta_2} \right)^m \Gamma\left( \frac{m+1}{2} \right)^\frac{1}{2}} C m^s \left( \frac{2}{\zeta_2} \right)^m (m+1)^\frac{m}{2} \\
			& \quad\quad \leq C_p \Vert \rho \Vert_{C^k_{pol,\gamma}(\mathbb{R})} \frac{C^{(\gamma,p)}_{\mathbb{R},w_0} m^{\gamma+\frac{k+1}{p}} \pi^\frac{m+1}{4}}{C_{\psi,\rho} \left( \frac{2\pi}{\zeta_2} \right)^m \left( \frac{4\pi}{m+1} \right)^\frac{1}{4} \left( \frac{m+1}{2e} \right)^\frac{m+1}{2}} C m^s \left( \frac{2}{\zeta_2} \right)^m (m+1)^\frac{m}{2} \\
			& \quad\quad \leq 2 C_p \frac{C_p \Vert \rho \Vert_{C^k_{pol,\gamma}(\mathbb{R})} C^{(\gamma,p)}_{\mathbb{R},w_0} \pi^\frac{1}{4} (2e)^\frac{1}{2} C_3}{C_{\psi,\rho} (4\pi)^\frac{1}{4}} C m^s \\
			& \quad\quad \leq \left( C_2 m^{C_3} \right)^{1-\frac{1}{\min(2,p)}}.
		\end{aligned}
	\end{equation}
	Hence, by using that $f \in \widetilde{\mathbb{B}}^{k,r,\gamma}_{\psi,a,b}(U;\mathbb{R}^d)$ (see Proposition~\ref{PropConst}), we can apply Theorem~\ref{ThmAR} with $N = \Big\lceil C_2 m^{C_3} \varepsilon^{-\frac{\min(2,p)}{\min(2,p)-1}} \Big\rceil$ and insert the inequality~\eqref{EqPropCODProof2} to obtain a random neural network $G_N \in \mathcal{RN}^\rho_{U,d}$ with $N$ neurons satisfying
	\begin{equation*}
		\begin{aligned}
			\mathbb{E}\left[ \Vert f - G_N \Vert_{W^{k,p}(U,\mathcal{L}(U),w;\mathbb{R}^d)}^r \right]^\frac{1}{r} & \leq C_{p,r} \Vert \rho \Vert_{C^k_{pol,\gamma}(\mathbb{R})} \frac{C^{(\gamma,p)}_{U,w} m^\frac{k}{p}}{\left\vert C^{(\psi,\rho)}_m \right\vert} \frac{\Vert f \Vert_{\widetilde{\mathbb{B}}^{k,r,\gamma}_{\psi,a,b}(U;\mathbb{R}^d)}}{N^{1-\frac{1}{\min(2,p,r)}}} \\
			& \leq \frac{\left( C_2 m^{C_3} \right)^{1-\frac{1}{\min(2,p)}}}{N^{1-\frac{1}{\min(2,p,r)}}} \leq \varepsilon,
		\end{aligned}
	\end{equation*}
	which completes the proof.
\end{proof}

\section{Proof of results in Section~\ref{SecLSGE}}

\subsection{Proof of Proposition~\ref{PropAlg}}
\label{SecProofsAlg}

\begin{proof}[Proof of Proposition~\ref{PropAlg}]
	Fix some $J,N \in \mathbb{N}$ and a $k$-times weakly differentiable function $f := (f_1,...,f_d)^\top: U \rightarrow \mathbb{R}^d$. Moreover, in order to ease notation, we define $\widetilde{m} := J \vert \mathbb{N}^m_{0,k} \vert d \in \mathbb{N}$ and $\widetilde{n} := e N \in \mathbb{N}$. Then, by using the definition of the Euclidean norm, we first observe that \eqref{EqEmpMSE} is equivalent to
	\begin{equation}
		\label{EqPropAlgProof1}
		y^{(J)}(\omega) = \argmin_{y \in \mathcal{Y}_N} \left( \frac{1}{J} \sum_{j=1}^J \sum_{\alpha \in \mathbb{N}^m_{0,k}} \sum_{i=1}^d c_{\alpha}^2 \left\vert \partial_\alpha f_i(V_j(\omega)) - \partial_\alpha G^y_{N,i}(\omega)(V_j(\omega)) \right\vert^2 \right),
	\end{equation}
	where $G^y_N(\omega) := (G^y_{N,1}(\omega),...,G^y_{N,d}(\omega))^\top \in W^{k,2}(U,\mathcal{L}(U),w;\mathbb{R}^d)$ is defined in \eqref{EqDefRFMY}. Hence, for every fixed $\omega \in \Omega$, the least squares problem \eqref{EqPropAlgProof1} is by \cite[Theorem~1.1.2]{bjoerck96} equivalent to the normal equations $\mathbf{G}(\omega)^\top \mathbf{G}(\omega) \vec{y}^{(J)}(\omega) = \mathbf{G}(\omega)^\top Z(\omega)$ stated in Line~\ref{Alg5} of Algorithm~\ref{Alg}, where $\vec{y}^{(J)}(\omega) := \big( y^{(J)}_{(l,n)}(\omega) \big)_{(l,n) \in \lbrace 1,...,e \rbrace \times \lbrace 1,...,N \rbrace}^\top$ denotes the vectorized version of $y^{(J)}(\omega) := \big( y^{(J)}_{l,n}(\omega) \big)_{l=1,...,e}^{n = 1,...,N}$. Thus, the problem \eqref{EqPropAlgProof1} admits by \cite[Theorem~1.2.10]{bjoerck96} a solution $y^{(J)}(\omega) := \big( y^{(J)}_{l,n}(\omega) \big)_{l=1,...,e}^{n = 1,...,N} \in \mathbb{R}^{e \times N}$, which proves that Algorithm~\ref{Alg} terminates.	
	
	Next, we show that Algorithm~\ref{Alg} is correct. To this end, we first prove that the $\mathbb{R}^{e \times N}$-valued random variable $y^{(J)} := \big( y^{(J)}_{l,n} \big)_{l=1,...,e}^{n = 1,...,N}$ defined in \eqref{EqPropAlgProof1} is $\mathcal{F}_{\theta,V}/\mathcal{B}(\mathbb{R}^{\widetilde{n}})$-measurable. Let us define the function
	\begin{equation}
		\label{EqPropAlgProof2}
		(\mathbb{R}^{\widetilde{m} \times \widetilde{n}} \times \mathbb{R}^{\widetilde{m}}) \times \mathbb{R}^{\widetilde{n}} \ni ((A,b),y) \quad \mapsto \quad \Vert Ay - b \Vert^2 \in \mathbb{R},
	\end{equation}
	whose epigraphical mapping $(\mathbb{R}^{\widetilde{m} \times \widetilde{n}} \times \mathbb{R}^{\widetilde{m}}) \times \mathbb{R}^{\widetilde{n}} \times \mathbb{R} \ni ((A,b),y,t) \mapsto \big\lbrace ((A,b),y,t) \in (\mathbb{R}^{\widetilde{m} \times \widetilde{n}} \times \mathbb{R}^{\widetilde{m}}) \times \mathbb{R}^{\widetilde{n}} \times \mathbb{R}: \Vert Ay - b \Vert^2 \leq t \big\rbrace$ is closed-valued and measurable (see \cite[Definition~14.1]{rockafellar97} for the definition of the latter). This shows that \eqref{EqPropAlgProof2} is a normal integrand in the sense of \cite[Definition~14.27]{rockafellar97}. Hence, we can apply \cite[Theorem~14.36]{rockafellar97} to conclude that there exists a $\mathcal{B}((\mathbb{R}^{\widetilde{m} \times \widetilde{n}} \times \mathbb{R}^{\widetilde{m}}) \times \mathbb{R}^{\widetilde{n}})/\mathcal{B}(\mathbb{R}^{\widetilde{n}})$-measurable map $\Upsilon: \mathbb{R}^{\widetilde{m} \times \widetilde{n}} \times \mathbb{R}^{\widetilde{m}} \rightarrow \mathbb{R}^{\widetilde{n}}$ returning a minimizer, i.e.~such that for every $(A,b) \in \mathbb{R}^{\widetilde{m} \times \widetilde{n}} \times \mathbb{R}^{\widetilde{m}}$ it holds that
	\begin{equation*}
		\Vert A \Upsilon(A,b) - b \Vert^2 = \min_{y \in \mathbb{R}^{\widetilde{n}}} \Vert A y - b \Vert^2.
	\end{equation*}
	Moreover, by using that $(\theta_n)_{n \in \mathbb{N}}: \Omega \rightarrow \Theta$ are by definition $\mathcal{F}_{\theta,V}/\mathcal{B}(\Theta)$-measurable, that $(V_j)_{j \in \mathbb{N}}: \Omega \rightarrow U$ are by definition $\mathcal{F}_{\theta,V}/\mathcal{B}(U)$-measurable, and that the feature maps $g_1,...,g_e: \Theta \rightarrow W^{k,2}(U,\mathcal{L}(U),w;\mathbb{R}^d)$ are by assumption $\mathcal{B}(\Theta)/\mathcal{B}(W^{k,2}(U,\mathcal{L}(U),w;\mathbb{R}^d))$-measurable, the $\mathbb{R}^{\widetilde{m} \times \widetilde{n}}$-valued random variable $\mathbf{G} = (\mathbf{G}_{(j,\alpha,i),(l,n)})_{(j,\alpha,i) \in \lbrace 1,...,J \rbrace \times \mathbb{N}^m_{0,k} \times \lbrace 1,...,d \rbrace}^{(l,n) \in \lbrace 1,...,e \rbrace \times \lbrace 1,...,N \rbrace}$ with $\mathbf{G}_{(j,\alpha,i),(l,n)} := c_\alpha \partial_\alpha g_{l,i}(\theta_n)(V_j)$, for $(j,\alpha,i) \in \lbrace 1,...,J \rbrace \times \mathbb{N}^m_{0,k} \times \lbrace 1,...,d \rbrace$ and $(l,n) \in \lbrace 1,...,e \rbrace \times \lbrace 1,...,N \rbrace$, is $\mathcal{F}_{\theta,V}/\mathcal{B}(\mathbb{R}^{\widetilde{m} \times \widetilde{n}})$-measurable. In addition, by using that $(V_j)_{j \in \mathbb{N}}: \Omega \rightarrow U$ are by definition $\mathcal{F}_{\theta,V}/\mathcal{B}(U)$-measurable and that $f: U \rightarrow \mathbb{R}^d$ is $k$-times weakly differentiable, the $\mathbb{R}^{\widetilde{m}}$-valued random variable $Z := (c_\alpha \partial_\alpha f_i(V_j))_{(j,\alpha,i) \in \lbrace 1,...,J \rbrace \times \mathbb{N}^m_{0,k} \times \lbrace 1,...,d \rbrace}$ is $\mathcal{F}_{\theta,V}/\mathcal{B}(\mathbb{R}^{\widetilde{m}})$-measurable. Thus, by combining this with the $\mathcal{B}((\mathbb{R}^{\widetilde{m} \times \widetilde{n}} \times \mathbb{R}^{\widetilde{m}}) \times \mathbb{R}^{\widetilde{n}})/\mathcal{B}(\mathbb{R}^{\widetilde{n}})$-measurable map $\Upsilon: \mathbb{R}^{\widetilde{m} \times \widetilde{n}} \times \mathbb{R}^{\widetilde{m}} \rightarrow \mathbb{R}^{\widetilde{n}}$, it follows that
	\begin{equation*}
		\Omega \ni \omega \quad \mapsto \quad \vec{y}^{(J)}(\omega) := \Upsilon(\mathbf{G}(\omega),Z(\omega)) \in \mathbb{R}^{e \cdot N}
	\end{equation*}
	is $\mathcal{F}_{\theta,V}/\mathcal{B}(\mathbb{R}^{\widetilde{n}})$-measurable, which shows that $y^{(J)} \in \mathcal{Y}_N$. Since $\vec{y}^{(J)}(\omega) = \Upsilon(\mathbf{G}(\omega),Z(\omega)) = \min_{y \in \mathbb{R}^{\widetilde{n}}} \Vert \mathbf{G}(\omega) y - Z(\omega) \Vert^2$ is by \cite[Theorem~1.1.2]{bjoerck96} equivalent to the normal equations $\mathbf{G}(\omega)^\top \mathbf{G}(\omega) \vec{y}^{(J)}(\omega) = \mathbf{G}(\omega)^\top Z(\omega)$ in Line~\ref{Alg5}, we obtain that the algorithm is correct.
	
	Finally, we compute the complexity of Algorithm~\ref{Alg}. In Line~\ref{Alg1}, we generate $N$ random variables $(\theta_n)_{n=1,...,N}$, which costs $N$ units. In Line~\ref{Alg2}, we generate $J$ random variables $(V_j)_{j=1,...,J} \sim w$, which requires $J$ units. In Line~\ref{Alg3}, we compute the $\mathbb{R}^{\widetilde{m} \times \widetilde{n}}$-valued random variable $\mathbf{G} = (\mathbf{G}_{(j,\alpha,i),(l,n)})_{(j,\alpha,i) \in \lbrace 1,...,J \rbrace \times \mathbb{N}^m_{0,k} \times \lbrace 1,...,d \rbrace}^{(l,n) \in \lbrace 1,...,e \rbrace \times \lbrace 1,...,N \rbrace}$ with $\mathbf{G}_{(j,\alpha,i),(l,n)} := c_\alpha \partial_\alpha g_l(\theta_n)_i(V_j)$, for $(j,\alpha,i) \in \lbrace 1,...,J \rbrace \times \mathbb{N}^m_{0,k} \times \lbrace 1,...,d \rbrace$ and $(l,n) \in \lbrace 1,...,e \rbrace \times \lbrace 1,...,N \rbrace$, which needs $2 J \big\vert \mathbb{N}^m_{0,k} \big\vert d e N$ units. In Line~\ref{Alg4}, we compute the $\mathbb{R}^{\widetilde{m}}$-valued random variable $Z := (c_\alpha \partial_\alpha f_i(V_j))_{(j,\alpha,i) \in \lbrace 1,...,J \rbrace \times \mathbb{N}^m_{0,k} \times \lbrace 1,...,d \rbrace}$, which requires $2 J \big\vert \mathbb{N}^m_{0,k} \big\vert d$ units. In Line~\ref{Alg5}, we solve the least squares problem via Cholesky decomposition and forward/backward substitution (see \cite[Section~2.2.2]{bjoerck96}), which needs
	\begin{equation*}
		\frac{1}{2} \widetilde{m} \widetilde{n}^2 + \frac{1}{6} \widetilde{n}^3 + \mathcal{O}(\widetilde{m} \widetilde{n}) = \frac{1}{2} \left( J \left\vert \mathbb{N}^m_{0,k} \right\vert d \right) (eN)^2 + \frac{1}{6} (eN)^3 + \mathcal{O}\left( J \left\vert \mathbb{N}^m_{0,k} \right\vert d e N \right)
	\end{equation*}
	units (see \cite[p.~45]{bjoerck96}). Hence, by summing the computational costs and by using that $\big\vert \mathbb{N}^m_{0,k} \big\vert = \sum_{j=0}^k m^j \leq 2 m^k$, the complexity of Algorithm~\ref{Alg} is of order
	\begin{equation*}
		\begin{aligned}
			& N + J + 2 J \left\vert \mathbb{N}^m_{0,k} \right\vert d e N + \frac{1}{2} \left( J \left\vert \mathbb{N}^m_{0,k} \right\vert d \right) (eN)^2 + \frac{1}{6} (eN)^3 + \mathcal{O}\left( J \left\vert \mathbb{N}^m_{0,k} \right\vert d e N \right) \\
			& \quad\quad \leq N + J + 4 J m^k d e N + 2 J m^k d (eN)^2 + \frac{1}{6} (eN)^3 + \mathcal{O}\left( J m^k d e N \right) \\
			& \quad\quad = \mathcal{O}\left( J m^k d (eN)^2 + (eN)^3 \right),
		\end{aligned}
	\end{equation*}
	which completes the proof.	
\end{proof}

\subsection{Proof of Theorem~\ref{ThmGenErr} and Corollary~\ref{CorGenErrRN}}
\label{SecProofsGenErr}

\begin{proof}[Proof of Theorem~\ref{ThmGenErr}]
	Fix some $J,N \in \mathbb{N}$, $L > 0$, and a function $f := (f_1,...,f_d)^\top \in \mathbb{B}^2_{\mathcal{G},\theta}(W^{k,2}(U,\mathcal{L}(U),w;\mathbb{R}^d))$ satisfying $\vert \partial_\alpha f_i(u) \vert \leq L$ for all $\alpha \in \mathbb{N}^m_{0,k}$, $i=1,...,d$, and $u \in U$. Then, we apply Algorithm~\ref{Alg} to obtain some $G^{y^{(J)}}_N \in \mathcal{RG}^V$ with $\mathbb{R}^{e \times N}$-valued random variable $y^{(J)} = \big( y^{(J)}_{l,n} \big)_{l=1,...,e}^{n=1,...,N} \in \mathcal{Y}_N$ solving \eqref{EqEmpMSE}. Moreover, by using that $(\theta_n)_{n \in \mathbb{N}}: \Omega \rightarrow \Theta$ are by definition $\mathcal{F}_{\theta,V}/\mathcal{B}(\Theta)$-measurable, that the feature maps $g_1,...,g_e: \Theta \rightarrow W^{k,2}(U,\mathcal{L}(U),w;\mathbb{R}^d)$ are by assumption $\mathcal{B}(\Theta)/\mathcal{B}(W^{k,2}(U,\mathcal{L}(U),w;\mathbb{R}^d))$-measurable, and that $y^{(J)} = \big( y^{(J)}_{l,n} \big)_{l=1,...,e}^{n=1,...,N} \in \mathcal{Y}_N$ is $\mathcal{F}_{\theta,V}/\mathcal{B}(\mathbb{R}^{e \times N})$-measurable, it follows that
	\begin{equation*}
		\begin{aligned}
			& \Omega \ni \omega \quad \mapsto \quad G^{y^{(J)}}_N(\omega) := \left( u \mapsto G^{y^{(J)}}_{N,i}(\omega)(u) \right)_{i=1,...,d}^\top := \sum_{n=1}^N \sum_{l=1}^e y^{(J)}_{l,N}(\omega) g_l(\theta_n(\omega)) \\
			& \quad\quad := \left( u \mapsto \sum_{n=1}^N \sum_{l=1}^e y^{(J)}_{l,N}(\omega) g_{l,i}(\theta_n(\omega))(u) \right)_{i=1,...,d}^\top \in W^{k,2}(U,\mathcal{L}(U),w;\mathbb{R}^d)
		\end{aligned}
	\end{equation*}
	is $\mathcal{F}_{\theta,V}/\mathcal{B}(W^{k,2}(U,\mathcal{L}(U),w;\mathbb{R}^d))$-measurable. Hence, by using \cite[Lemma~4.7]{neufeld24}, i.e.~that the Banach space $(W^{k,2}(U,\mathcal{L}(U),w;\mathbb{R}^d),\Vert \cdot \Vert_{W^{k,2}(U,\mathcal{L}(U),w;\mathbb{R}^d)})$ is separable, we can apply \cite[Theorem~1.1.6+1.1.20]{hytoenen16} to conclude that $G^{y^{(J)}}: \Omega \rightarrow W^{k,2}(U,\mathcal{L}(U),w;\mathbb{R}^d)$ is a $(\mathbb{P},\mathcal{F}_{\theta,V})$-strongly measurable map.
	
	In order to show \eqref{EqThmGenErr1}, we adapt the proof of \cite[Theorem~11.3]{gyoerfi02}. To this end, we define for every $\alpha \in \mathbb{N}^m_{0,k}$ and $i = 1,...,d$ the $L^2(U,\mathcal{L}(U),w)$-valued random variable
	\begin{equation*}
		\Omega \ni \omega \quad \mapsto \quad \Delta^{y^{(J)}}_{\alpha,i,L}(\omega) := \left( u \mapsto T_L\left(\partial_\alpha f_i(u) - \partial_\alpha G^{y^{(J)}}_{N,i}(\omega)(u)\right) \right) \in L^2(U,\mathcal{L}(U),w).
	\end{equation*}
	Moreover, we define for every fixed $\alpha \in \mathbb{N}^m_{0,k}$, $i = 1,...,d$, and $\vartheta := (\vartheta_1,...,\vartheta_N) \in \bigtimes_{n=1}^N \Theta$, the $L^2(U,\mathcal{L}(U),w)$-valued random variable
	\begin{equation*}
		\Omega \ni \omega \quad \mapsto \quad \Delta^{y^{(J)},\vartheta}_{\alpha,i,L}(\omega) := \left( u \mapsto T_L\left(\partial_\alpha f_i(u) - \partial_\alpha G^{y^{(J)},\vartheta}_{N,i}(\omega)(u)\right) \right) \in L^2(U,\mathcal{L}(U),w),
	\end{equation*}
	where $\Omega \ni \omega \mapsto G^{y^{(J)},\vartheta}_{N,i}(\omega) := \sum_{n=1}^N \sum_{l=1}^e y^{(J)}_{l,n}(\omega) g_{l,i}(\vartheta_n) \in L^2(U,\mathcal{L}(U),w)$. In addition, we define the corresponding (random) empirical mean squared error $\Vert \cdot \Vert_J$ of such $L^2(U,\mathcal{L}(U),w)$-valued random variables as
	\begin{equation*}
		\begin{aligned}
			\Omega \ni \omega \quad \mapsto \quad \left\Vert \Delta^{y^{(J)}}_{\alpha,i,L}(\omega) \right\Vert_J & := \left( \frac{1}{J} \sum_{j=1}^J \left\vert \Delta^{y^{(J)}}_{\alpha,i,L}(\omega)(V_j(\omega)) \right\vert^2 \right)^\frac{1}{2} \in \mathbb{R} \quad\quad \text{and} \\
			\Omega \ni \omega \quad \mapsto \quad \left\Vert \Delta^{y^{(J)},\vartheta}_{\alpha,i,L}(\omega) \right\Vert_J & := \left( \frac{1}{J} \sum_{j=1}^J \left\vert \Delta^{y^{(J)},\vartheta}_{\alpha,i,L}(\omega)(V_j(\omega)) \right\vert^2 \right)^\frac{1}{2} \in \mathbb{R}.
		\end{aligned}
	\end{equation*}
	Then, by using the inequality $(x+y)^2 \leq 2 \left( x^2 + y^2 \right)$ for any $x,y \geq 0$, it follows that
	\begin{equation*}
		\begin{aligned}
			& \mathbb{E}\left[ \sum_{\alpha \in \mathbb{N}^m_{0,k}} \int_U \left\Vert T_L\left(\partial_\alpha f(u) - \partial_\alpha G^{y^{(J)}}_N(\cdot)(u)\right) \right\Vert^2 w(u) du \right] \\
			& \quad\quad \leq \mathbb{E}\left[ \sum_{\alpha \in \mathbb{N}^m_{0,k}} \sum_{i=1}^d \int_U \left\vert T_L\left(\partial_\alpha f_i(u) - \partial_\alpha G^{y^{(J)}}_{N,i}(\cdot)(u)\right) \right\vert^2 w(u) du \right] \\
			& \quad\quad = \sum_{\alpha \in \mathbb{N}^m_{0,k}} \sum_{i=1}^d \mathbb{E}\left[ \left( \left\Vert \Delta^{y^{(J)}}_{\alpha,i,L} \right\Vert_{L^2(U,\mathcal{L}(U),w)} - 2 \left\Vert \Delta^{y^{(J)}}_{\alpha,i,L} \right\Vert_J + 2 \left\Vert \Delta^{y^{(J)}}_{\alpha,i,L} \right\Vert_J \right)^2 \right] \\
			& \quad\quad \leq 2 \sum_{\alpha \in \mathbb{N}^m_{0,k}} \sum_{i=1}^d \mathbb{E}\left[ \max\left( \left\Vert \Delta^{y^{(J)}}_{\alpha,i,L} \right\Vert_{L^2(U,\mathcal{L}(U),w)} - 2 \left\Vert \Delta^{y^{(J)}}_{\alpha,i,L} \right\Vert_J, 0 \right)^2 + 4 \left\Vert \Delta^{y^{(J)}}_{\alpha,i,L} \right\Vert_J^2 \right]. \\
		\end{aligned}
	\end{equation*}
	Hence, by conditioning on $\mathcal{F}_\theta$, by using that $\vert \mathbb{N}^m_{0,k} \vert = \sum_{j=0}^k m^j \leq 2m^k$, that the random variables $(V_j)_{j \in \mathbb{N}}$ are independent of $(\theta_n)_{n \in \mathbb{N}}$, and the notation $\theta := (\theta_n)_{n=1,...,N}$, we have
	\begin{equation}
		\label{EqThmGenProof1}
		\begin{aligned}
			& \mathbb{E}\left[ \sum_{\alpha \in \mathbb{N}^m_{0,k}} \int_U \left\Vert T_L\left(\partial_\alpha f(u) - \partial_\alpha G^{y^{(J)}}_N(\cdot)(u)\right) \right\Vert^2 w(u) du \right] \\
			& \quad\quad \leq 2 \left\vert \mathbb{N}^m_{0,k} \right\vert d \max_{\alpha \in \mathbb{N}^m_{0,k} \atop i=1,...,d} \mathbb{E}\left[ \mathbb{E}\left[ \max\left( \left\Vert \Delta^{y^{(J)}}_{\alpha,i,L} \right\Vert_{L^2(U,\mathcal{L}(U),w)} - 2 \left\Vert \Delta^{y^{(J)}}_{\alpha,i,L} \right\Vert_J, 0 \right)^2 \Bigg\vert \mathcal{F}_\theta \right] \right] \\
			& \quad\quad\quad\quad + 8 \mathbb{E}\left[ \sum_{\alpha \in \mathbb{N}^m_{0,k}} \sum_{i=1}^d \frac{1}{J} \sum_{j=1}^J \left\Vert \Delta^{y^{(J)}}_{\alpha,i,L}(\cdot)(V_j) \right\Vert^2 \right] \\
			& \quad\quad \leq 4 m^k d \max_{\alpha \in \mathbb{N}^m_{0,k} \atop i=1,...,d} \mathbb{E}\left[ \mathbb{E}\left[ \max\left( \left\Vert \Delta^{y^{(J)},\vartheta}_{\alpha,i,L} \right\Vert_{L^2(U,\mathcal{L}(U),w)} - 2 \left\Vert \Delta^{y^{(J)},\vartheta}_{\alpha,i,L} \right\Vert_J, 0 \right)^2 \right]\Bigg\vert_{\vartheta = \theta} \right] \\
			& \quad\quad\quad\quad + 8 \mathbb{E}\left[ \sum_{\alpha \in \mathbb{N}^m_{0,k}} \sum_{i=1}^d \frac{1}{J} \sum_{j=1}^J \left\Vert \Delta^{y^{(J)}}_{\alpha,i,L}(\cdot)(V_j) \right\Vert^2 \right].
		\end{aligned}
	\end{equation}
	Moreover, we define for every fixed $\alpha \in \mathbb{N}^m_{0,k}$, $i = 1,...,d$, and $\vartheta := (\vartheta_1,...,\vartheta_N) \in \bigtimes_{n=1}^N \Theta$ the vector space of random functions
	\begin{equation*}
		\mathcal{G}^\vartheta_{\alpha,i} := \left\lbrace \Omega \ni \omega \mapsto \sum_{n=1}^N \sum_{l=1}^e y_{l,n} \partial_\alpha g_{l,i}(\vartheta)(V_j(\omega)) \in L^2(U,\mathcal{L}(U),w): y = \left( y_{l,n} \right)_{l=1,...,e}^{n=1,...,N} \in \mathcal{Y}_N \right\rbrace.
	\end{equation*}
	Then, by following \cite[p.~193]{gyoerfi02}, i.e.~by using \cite[Theorem~11.2]{gyoerfi02} (with the set $T_L\big( \mathcal{G}^\vartheta_{\alpha,i} \big) := \big\lbrace \Omega \ni \omega \mapsto \left( u \mapsto T_L(G(\omega)(u)) \right) \in L^2(U,\mathcal{L}(U),w): G \in \mathcal{G}^\vartheta_{\alpha,i} \big\rbrace$ and where $\mathcal{G}^\vartheta_{\alpha,i}$ has for fixed $a \in \mathbb{R}^{N \times m}$, $b \in \mathbb{R}^N$, $\alpha \in \mathbb{N}^m_{0,k}$, and $i = 1,...,d$ the vector space dimension $N$ in the sense of \cite[Theorem~11.1]{gyoerfi02}) together with \cite[Lemma~9.2+9.4 \& Theorem~9.5]{gyoerfi02}, it follows for every $u > 576 L^2/J$ that
	\begin{equation}
		\label{EqThmGenProof2}
		\begin{aligned}
			& \mathbb{P}\left[ \max\left( \left\Vert \Delta^{y^{(J)},\vartheta}_{\alpha,i,L} \right\Vert_{L^2(U,\mathcal{L}(U),w)} - 2 \left\Vert \Delta^{y^{(J)},\vartheta}_{\alpha,i,L} \right\Vert_J, 0 \right)^2 > u \right] \\
			& \quad\quad \leq \mathbb{P}\left[ \exists g \in T_L\left( \mathcal{G}^\vartheta_{\alpha,i} \right): \Vert g \Vert_{L^2(U,\mathcal{L}(U),w)} - 2 \Vert g \Vert_J > \frac{\sqrt{u}}{2} \right] \\
			& \quad\quad \leq 9 (12eJ)^{2(N+1)} e^{-\frac{Ju}{2304 L^2}}.
		\end{aligned}
	\end{equation}
	Hence, by using the constant $v := \frac{2304 L^2}{J} \ln\left( 9 (12eJ)^{2(N+1)} \right) > 576 L^2/J$, the inequality \eqref{EqThmGenProof2}, and that $\ln(108e) \geq 1$ together with $2304 \leq 9216 \ln(108e)$, we conclude that
	\begin{equation}
		\label{EqThmGenProof3}
		\begin{aligned}
			& \mathbb{E}\left[ \max\left( \left\Vert \Delta^{y^{(J)},\vartheta}_{\alpha,i,L} \right\Vert_{L^2(U,\mathcal{L}(U),w)} - 2 \left\Vert \Delta^{y^{(J)},\theta}_{\alpha,i,L} \right\Vert_J, 0 \right)^2 \right] \\
			& \quad\quad = \int_0^\infty \mathbb{P}\left[ \max\left( \left\Vert \Delta^{y^{(J)},\vartheta}_{\alpha,i,L} \right\Vert_{L^2(U,\mathcal{L}(U),w)} - 2 \left\Vert \Delta^{y^{(J)},\theta}_{\alpha,i,L} \right\Vert_J, 0 \right)^2 > u \right] du \\
			& \quad\quad \leq v + \int_v^\infty \mathbb{P}\left[ \max\left( \left\Vert \Delta^{y^{(J)},\vartheta}_{\alpha,i,L} \right\Vert_{L^2(U,\mathcal{L}(U),w)} - 2 \left\Vert \Delta^{y^{(J)},\vartheta}_{\alpha,i,L} \right\Vert_J, 0 \right)^2 > u \right] du \\
			& \quad\quad \leq v + 9 (12eJ)^{2(N+1)} \int_v^\infty e^{-\frac{Ju}{2304 L^2}} du \\
			& \quad\quad = \frac{2304 L^2}{J} \underbrace{\ln\left( 9 (12eJ)^{2(N+1)} \right)}_{\leq 4N \ln\left( 108e J \right)} + \frac{2304 L^2}{J} e^{-\frac{J v}{2304 L^2}} \\
			& \quad\quad \leq \frac{2304 L^2}{J} 4N \left( \ln(108e) + \ln(J) \right) + \frac{2304 L^2}{J} \\
			& \quad\quad \leq 9216 \ln(108e) L^2 \frac{(\ln(J)+1)N}{J}.
		\end{aligned}
	\end{equation}
	On the other hand, for the second term on the right-hand side of \eqref{EqThmGenProof1}, we use that $\left\vert \partial_\alpha f_i(u) \right\vert \leq L$ for any $\alpha \in \mathbb{N}^m_{0,k}$, $i=1,...,d$, and $u \in U$, that $\Vert T_L(y) \Vert \leq \Vert y \Vert$ for any $y \in \mathbb{R}^d$, and that the $\mathbb{R}^{e \times N}$-valued random variable $y^{(J)} = \big( y^{(J)}_{l,n} \big)_{l=1,...,d}^{n=1,...,N}$ solves \eqref{EqEmpMSE}, to obtain that
	\begin{equation*}
		\begin{aligned}
			& \mathbb{E}\left[ \sum_{\alpha \in \mathbb{N}^m_{0,k}} \sum_{i=1}^d \frac{1}{J} \sum_{j=1}^J \left\Vert \Delta^{y^{(J)}}_{\alpha,i,L}(\cdot)(V_j) \right\Vert^2 \right]^\frac{1}{2} \\
			& \quad\quad = \mathbb{E}\left[ \frac{1}{J} \sum_{j=1}^J \sum_{\alpha \in \mathbb{N}^m_{0,k}} \left\Vert T_L\left(\partial_\alpha f(V_j) - \partial_\alpha G^{y^{(J)}}_N(\cdot)(V_j)\right) \right\Vert^2 \right]^\frac{1}{2} \\
			& \quad\quad \leq \frac{1}{\min_{\alpha \in \mathbb{N}^m_{0,k}} c_\alpha} \mathbb{E}\left[ \frac{1}{J} \sum_{j=1}^J \sum_{\alpha \in \mathbb{N}^m_{0,k}} c_\alpha^2 \left\Vert \partial_\alpha f(V_j) - \partial_\alpha G^{y^{(J)}}_N(\cdot)(V_j) \right\Vert^2 \right]^\frac{1}{2} \\
			& \quad\quad = \frac{1}{\min_{\alpha \in \mathbb{N}^m_{0,k}} c_\alpha} \mathbb{E}\left[ \min_{y \in \mathcal{Y}_N} \left( \frac{1}{J} \sum_{j=1}^J \sum_{\alpha \in \mathbb{N}^m_{0,k}} c_\alpha^2 \left\Vert \partial_\alpha f(V_j) - \partial_\alpha G^y_N(\cdot)(V_j) \right\Vert^2 \right) \right]^\frac{1}{2} \\
			& \quad\quad \leq \frac{1}{\min_{\alpha \in \mathbb{N}^m_{0,k}} c_\alpha} \inf_{y \in \mathcal{Y}_N} \mathbb{E}\left[ \frac{1}{J} \sum_{j=1}^J \sum_{\alpha \in \mathbb{N}^m_{0,k}} c_\alpha^2 \left\Vert \partial_\alpha f(V_j) - \partial_\alpha G^y_N(\cdot)(V_j) \right\Vert^2 \right]^\frac{1}{2} \\
			& \quad\quad \leq \frac{\max_{\alpha \in \mathbb{N}^m_{0,k}} c_\alpha}{\min_{\alpha \in \mathbb{N}^m_{0,k}} c_\alpha} \inf_{y \in \mathcal{Y}_N} \mathbb{E}\left[ \sum_{\alpha \in \mathbb{N}^m_{0,k}} \int_U \left\Vert \partial_\alpha f(u) - \partial_\alpha G^y_N(\cdot)(u) \right\Vert^2 w(u) du \right]^\frac{1}{2}.
		\end{aligned}
	\end{equation*}
	Hence, by using Theorem~\ref{CorARRN} (with constants $C_{2,2} > 0$ and $C_{W^{k,2}(U,\mathcal{L}(U),w;\mathbb{R}^d)} > 0$ independent of $f: U \rightarrow \mathbb{R}^d$ and $m,d \in \mathbb{N}$, see also \cite[Lemma~4.9]{neufeld24}) together with $\mathcal{F}_\theta \subseteq \mathcal{F}_{\theta,V}$ (with $G^f_N \in \mathcal{RG} \cap L^2(\Omega,\mathcal{F}_\theta,\mathbb{P};W^{k,2}(U,\mathcal{L}(U),w;\mathbb{R}^d)$ having $\mathcal{F}_\theta/\mathcal{B}(\mathbb{R}^{e \times N})$-measurable linear readout contained in $\mathcal{Y}_N$ as $\mathcal{F}_\theta \subseteq \mathcal{F}_{\theta,V}$), we conclude that
	\begin{equation}
		\label{EqThmGenProof4}
		\begin{aligned}
			& \mathbb{E}\left[ \sum_{\alpha \in \mathbb{N}^m_{0,k}} \sum_{i=1}^d \frac{1}{J} \sum_{j=1}^J \left\Vert \Delta^{y^{(J)}}_{\alpha,i,L}(\cdot)(V_j) \right\Vert^2 \right]^\frac{1}{2} \\
			& \quad\quad \leq \kappa(\mathbf{c}) \inf_{y \in \mathcal{Y}_N} \mathbb{E}\left[ \sum_{\alpha \in \mathbb{N}^m_{0,k}} \int_U \left\Vert \partial_\alpha f(u) - \partial_\alpha G^y_N(\cdot)(u) \right\Vert^2 w(u) du \right]^\frac{1}{2} \\
			& \quad\quad \leq \kappa(\mathbf{c}) \mathbb{E}\left[ \left\Vert f - G^f_N \right\Vert_{W^{k,2}(U,\mathcal{L}(U),w,\mathbb{R}^d)}^2 \right]^\frac{1}{2} \\
			& \quad\quad \leq \kappa(\mathbf{c}) C_{2,2} C_{W^{k,2}(U,\mathcal{L}(U),w;\mathbb{R}^d)} \frac{\Vert f \Vert_{\mathbb{B}^2_{\mathcal{G},\theta}(W^{k,2}(U,\mathcal{L}(U),w;\mathbb{R}^d))}}{\sqrt{N}}.
		\end{aligned}
	\end{equation}
	Thus, by inserting \eqref{EqThmGenProof3}+\eqref{EqThmGenProof4} into \eqref{EqThmGenProof1} with the inequality $\sqrt{x+y} \leq \sqrt{x} + \sqrt{y}$ for any $x,y \geq 0$, and by using the constant $C_4 := \max\big( 2 \sqrt{9216 \ln(108e)}, \sqrt{8} C_{2,2} C_{W^{k,2}(U,\mathcal{L}(U),w;\mathbb{R}^d)} \big) > 0$ (being independent of $f: U \rightarrow \mathbb{R}^d$ and $m,d \in \mathbb{N}$), it follows that
	\begin{equation*}
		\begin{aligned}
			& \mathbb{E}\left[ \sum_{\alpha \in \mathbb{N}^m_{0,k}} \int_U \left\Vert  T_L\left( \partial_\alpha f(u) - \partial_\alpha G^{y^{(J)}}_N(\cdot)(u)\right) \right\Vert^2 w(u) du \right]^\frac{1}{2} \\
			& \quad\quad \leq 2 m^\frac{k}{2} d^\frac{1}{2} \max_{\alpha \in \mathbb{N}^m_{0,k} \atop i=1,...,d} \mathbb{E}\left[ \mathbb{E}\bigg[ \max\bigg( \left\Vert \Delta^{y^{(J)},\vartheta}_{\alpha,i,L} \right\Vert_{L^2(U,\mathcal{L}(U),w)} - 2 \left\Vert \Delta^{y^{(J)},\vartheta}_{\alpha,i,L} \right\Vert_J, 0 \bigg)^2 \bigg]\bigg\vert_{\vartheta = \theta} \right]^\frac{1}{2} \\
			& \quad\quad\quad\quad + \sqrt{8} \mathbb{E}\left[ \sum_{\alpha \in \mathbb{N}^m_{0,k}} \sum_{i=1}^d \frac{1}{J} \sum_{j=1}^J \left\Vert \Delta^{y^{(J)}}_{\alpha,i,L}(\cdot)(V_j) \right\Vert^2 \right]^\frac{1}{2} \\
			& \quad\quad \leq 2 m^\frac{k}{2} d^\frac{1}{2} \sqrt{9216 \ln(108e)} L \sqrt{\frac{(\ln(J)+1) N}{J}} \\
			& \quad\quad\quad\quad + \sqrt{8} \kappa(\mathbf{c}) C_{2,2} C_{W^{k,2}(U,\mathcal{L}(U),w;\mathbb{R}^d)} \frac{\Vert f \Vert_{\mathbb{B}^2_{\mathcal{G},\theta}(W^{k,2}(U,\mathcal{L}(U),w;\mathbb{R}^d))}}{\sqrt{N}} \\
			& \quad\quad \leq C_4 L m^\frac{k}{2} d^\frac{1}{2} \sqrt{\frac{(\ln(J)+1) N}{J}} + C_4 \kappa(\mathbf{c}) \frac{\Vert f \Vert_{\mathbb{B}^2_{\mathcal{G},\theta}(W^{k,2}(U,\mathcal{L}(U),w;\mathbb{R}^d))}}{\sqrt{N}},
		\end{aligned}
	\end{equation*}
	which shows the inequality~\eqref{EqThmGenErr1}. For \eqref{EqThmGenErr2}, we insert $\vert T_L(s-t) \vert \leq \vert s-T_L(t) \vert$ for all $s \in [-L,L]$ into \eqref{EqThmGenErr1}, which completes the proof.
\end{proof}

\begin{proof}[Proof of Corollary~\ref{CorGenErrRN}]
	Fix some $J,N \in \mathbb{N}$, $L > 0$, and some $f := (f_1,...,f_d)^\top \in W^{k,2}(U,\mathcal{L}(U),w;\mathbb{R}^d) \cap \widetilde{\mathbb{B}}^{k,2,\gamma}_{\psi,a,b}(U;\mathbb{R}^d)$ satisfying $\vert \partial_\alpha f_i(u) \vert \leq L$ for all $\alpha \in \mathbb{N}^m_{0,k}$, $i=1,...,d$, and $u \in U$. Then, we observe that Algorithm~\ref{AlgRN} is the same as Algorithm~\ref{Alg} for the special case of random neural networks with feature maps $\Theta := \mathbb{R}^m \times \mathbb{R} \ni (\vartheta_1,\vartheta_2) \mapsto e_i \rho\big( \vartheta_1^\top \cdot - \vartheta_2 \big) \in W^{k,2}(U,\mathcal{L}(U),w;\mathbb{R}^d)$, $i = 1,...,d$, that are continuous and thus $\mathcal{B}(\Theta)/\mathcal{B}(W^{k,2}(U,\mathcal{L}(U),w;\mathbb{R}^d))$-measurable (see \cite[Lemma~4.10]{neufeld24}), where $e_i \in \mathbb{R}^d$ denotes the $i$-th unit vector of $\mathbb{R}^d$. Hence, we can apply Theorem~\ref{ThmGenErr} (with constant $C_4 > 0$ independent of $f: U \rightarrow \mathbb{R}^d$ and $m,d \in \mathbb{N}$) to conclude that Algorithm~\ref{AlgRN} returns a random neural network $G^{y^{(J)}}_N \in \mathcal{RN}^\rho_{U,d}$ with $N$ neurons being a strongly $(\mathbb{P},\mathcal{F}_{a,b,V})$-measurable map $G^{y^{(J)}}_N: \Omega \rightarrow W^{k,2}(U,\mathcal{L}(U),w;\mathbb{R}^d)$ such that
	\begin{equation*}
		\begin{aligned}
			& \mathbb{E}\left[ \sum_{\alpha \in \mathbb{N}^m_{0,k}} \int_U \left\Vert \partial_\alpha f(u) - T_L\left(\partial_\alpha G^{y^{(J)}}_N(\cdot)(u)\right) \right\Vert^2 w(u) du \right]^\frac{1}{2} \\
			& \quad\quad \leq C_4 L m^\frac{k}{2} d^\frac{1}{2} \sqrt{\frac{(\ln(J)+1) N}{J}} + C_4 \kappa(\mathbf{c}) \frac{\Vert f \Vert_{\mathbb{B}^2_{\mathcal{G},\theta}(W^{k,2}(U,\mathcal{L}(U),w;\mathbb{R}^d))}}{\sqrt{N}}.
		\end{aligned}
	\end{equation*}
	Thus, by using Lemma~\ref{LemmaIneq}, it follows that
	\begin{equation*}
		\begin{aligned}
			& \mathbb{E}\left[ \sum_{\alpha \in \mathbb{N}^m_{0,k}} \int_U \left\Vert \partial_\alpha f(u) - T_L\left(\partial_\alpha G^{y^{(J)}}_N(\cdot)(u)\right) \right\Vert^2 w(u) du \right]^\frac{1}{2} \\
			& \quad\quad \leq C_4 L m^\frac{k}{2} d^\frac{1}{2} \sqrt{\frac{(\ln(J)+1) N}{J}} + C_4 \kappa(\mathbf{c}) 2^{3+\frac{1}{2}} \Vert \rho \Vert_{C^k_{pol,\gamma}(\mathbb{R})} \frac{C^{(\gamma,2)}_{U,w} m^\frac{k}{2}}{\left\vert C^{(\psi,\rho)}_m \right\vert} \frac{\Vert f \Vert_{\widetilde{\mathbb{B}}^{k,2,\gamma}_{\psi,a,b}(U;\mathbb{R}^d)}}{\sqrt{N}}.
		\end{aligned}
	\end{equation*}
	Therefore, by defining the constant $C_5 := 2^{3+\frac{1}{2}} C_4 > 0$ (being independent of $f: U \rightarrow \mathbb{R}^d$ and $m,d \in \mathbb{N}$), we obtain the inequality~\eqref{EqCorGenErrRN1}. For \eqref{EqCorGenErrRN2}, we insert $\vert T_L(s-t) \vert \leq \vert s-T_L(t) \vert$ for all $s \in [-L,L]$ into \eqref{EqCorGenErrRN1}, which completes the proof.
\end{proof}

\section{Proof of results in Section~\ref{SecNE}}
\label{SecProofsNE}

\begin{lemma}
	\label{LemmaHeat}
	For $\lambda,t,\kappa_m,\sigma \in (0,\infty)$ and $f_0(u) := \frac{\kappa_m}{(2\pi \sigma^2)^{m/2}} \exp\left( -\Vert u \Vert^2/(2\sigma^2) \right)$, $u \in \mathbb{R}^m$, let $f(t,\cdot): \mathbb{R}^m \rightarrow \mathbb{R}$ be the solution of \eqref{EqDefHeat} at time $t$. Moreover, let $c \in \mathbb{N}_0$, $s \in [0,\infty)$, and $0 < \zeta_1 \leq \zeta_2 < \infty$. Then, $\widehat{f(t,\cdot)}: \mathbb{R}^m \rightarrow \mathbb{C}$ is smooth and there exists a constant $C_{22} > 0$ (being independent of $m \in \mathbb{N}$) such that
	\begin{equation}
		\label{EqLemmaHeat1}
		\begin{aligned}
			& \frac{\zeta_2^m \pi^\frac{m+1}{4}}{\zeta_1^\frac{m}{2} (2\pi)^m \Gamma\left( \frac{m+1}{2} \right)^\frac{1}{2}} \sum_{\beta \in \mathbb{N}^m_{0,c}} \left( \int_{\mathbb{R}^m} \big\vert \partial_\beta \widehat{f(t,\cdot)}(\xi) \big\vert^2 \left( 1 + \Vert \xi/\zeta_1 \Vert^2 \right)^\frac{m+s}{2} d\xi \right)^\frac{1}{2} \\
			& \quad\quad \leq C_{22} m^\frac{5c+s}{4} \kappa_m \left( \frac{(\zeta_2/\zeta_1)^2}{\pi \sigma} \right)^\frac{m}{2}.
		\end{aligned}
	\end{equation}
\end{lemma}
\begin{proof}
	Fix some $\lambda,t,\kappa,\sigma \in (0,\infty)$, $c \in \mathbb{N}_0$, $s \in [0,\infty)$, $0 < \zeta_1 < \zeta_2 < \infty$, and define $\mathbb{R}^m \ni u \mapsto f_0(u) := \frac{\kappa_m}{(2\pi \sigma^2)^{m/2}} \exp\left( -\Vert u \Vert^2/(2\sigma^2) \right) \in \mathbb{R}$. Then, by using that $f(t,\cdot) = \phi_{\lambda,t} * f_0$ with $\phi_{\lambda,t}(y) := (4\pi \lambda t)^{-m/2} \exp\left( -\Vert y \Vert^2/(4\lambda t) \right)$ (see \eqref{EqSolHeat}), \cite[Theorem~7.8~(d)]{folland92}, \cite[Equation~7.31]{folland92}, and the constant $\eta := \lambda t + \sigma^2/2 > 0$, it follows that
	\begin{equation}
		\label{EqLemmaHeatProof0}
		\widehat{f(t,\cdot)}(\xi) = (\widehat{\phi_{\lambda,t} * f_0})(\xi) = \widehat{\phi_{\lambda,t}}(\xi) \widehat{f_0}(\xi) = e^{-\lambda t \Vert \xi \Vert^2} \kappa_m e^{-\frac{\sigma^2}{2} \Vert \xi \Vert^2} = \kappa_m e^{-\eta \Vert \xi \Vert^2}.
	\end{equation}
	Hence, by using \cite[Table~7.2.9]{folland92}, the substitution $\zeta_l \mapsto \sqrt{2\eta} \xi_l$, and the Hermite polynomials $(h_n)_{n \in \mathbb{N}}$ in \cite[Equation~22.2.15]{abramowitz70}, we have for every $\beta := (\beta_1,...,\beta_m) \in \mathbb{N}^m_{0,c}$ and $\xi \in \mathbb{R}^m$ that
	\begin{equation}
		\label{EqLemmaHeatProof1}
		\begin{aligned}
			\partial_\beta \widehat{f(t,\cdot)}(\xi) & = \kappa_m \partial_\beta \left( \prod_{l=1}^m e^{-\eta \xi_l^2} \right) \\
			& = \kappa_m (2\eta)^\frac{\vert \beta \vert}{2} \prod_{l=1}^m \frac{\partial^{\beta_l}}{\partial \zeta_l^{\beta_l}} \left( e^{-\frac{\zeta_l^2}{2}} \right) \bigg\vert_{\zeta_l = \sqrt{2\eta} \xi_l} \\
			& = \kappa_m (2\eta)^\frac{\vert \beta \vert}{2} \prod_{l=1}^m (-1)^{\beta_l} h_{\beta_l}(\zeta_l) e^{-\frac{\zeta_l^2}{2}} \bigg\vert_{\zeta_l = \sqrt{2\eta} \xi_l} \\
			& = \kappa_m (-1)^{\vert \beta \vert} (2\eta)^\frac{\vert \beta \vert}{2} \left( \prod_{l=1}^m h_{\beta_l}\left( \sqrt{2\eta} \xi_l \right) \right) e^{-\eta \Vert \xi \Vert^2}.
		\end{aligned}
	\end{equation}
	Now, we use the explicit expression of the Hermite polynomials $(h_n)_{n \in \mathbb{N}}$ given in \cite[Equation~22.3.11]{abramowitz70}, that $\vert \zeta_l \vert^{\beta_l-2j_l} \leq \left( 1+\Vert \zeta \Vert^2 \right)^{(\beta_l-2j_l)/2} \leq \left( 1+\Vert \zeta \Vert^2 \right)^{\beta_l/2}$ for any $l = 1,...,m$, $\beta := (\beta_1,...,\beta_m) \in \mathbb{N}^m_0$, $j_l=0,...,\lceil \beta_l/2 \rceil$, and $\zeta \in \mathbb{R}^m$, that $\sum_{j_l=1}^{\lceil \beta_l/2 \rceil} \frac{\beta_l!}{2^{j_l} j_l!(\beta_l-2j_l)!} \leq \max_{j_l=1,...,\lceil \beta_l/2 \rceil} \frac{(2j_l)!}{j_l!} \sum_{j_l=1}^{\lceil \beta_l/2 \rceil} \frac{\beta_l!}{(2j_l)!(\beta_l-2j_l)!} \leq \beta_l! \sum_{k_l=1}^{\beta_l} \frac{\beta_l!}{k_l!(\beta_l-k_l)!} = 2^{\beta_l} \beta_l!$ for any $l = 1,...,m$ and $\beta := (\beta_1,...,\beta_m) \in \mathbb{N}^m_0$, and that $\prod_{l=1}^m \beta_l! = \beta! \leq \vert \beta \vert! \leq c!$ for any $\beta := (\beta_1,...,\beta_m) \in \mathbb{N}^m_{0,c}$ to obtain for every $\beta := (\beta_1,...,\beta_m) \in \mathbb{N}^m_{0,c}$ and $\zeta := (\zeta_1,...,\zeta_m) \in \mathbb{R}^m$ that
	\begin{equation}
		\label{EqLemmaHeatProof2}
		\begin{aligned}
			\prod_{l=1}^m \left\vert h_{\beta_l}(\zeta_l) \right\vert & \leq \prod_{l=1}^m \left( \sum_{j_l=1}^{\lceil \beta_l/2 \rceil} \frac{\beta_l! \vert \zeta_l \vert^{\beta_l-2j_l}}{2^{j_l} j_l!(\beta_l-2j_l)!} \right) \leq \prod_{l=1}^m \left( \left( 1 + \Vert \zeta \Vert^2 \right)^\frac{\beta_l}{2} \sum_{j_l=1}^{\lceil \beta_l/2 \rceil} \frac{\beta_l!}{2^{j_l} j_l!(\beta_l-2j_l)!} \right) \\
			& \leq \left( 1+\Vert \zeta \Vert^2 \right)^\frac{\vert \beta \vert}{2} \prod_{l=1}^m \left( 2^{\beta_l} \beta_l! \right) \leq 2^c c! \left( 1+\Vert \zeta \Vert^2 \right)^\frac{c}{2}.
		\end{aligned}
	\end{equation}
	Hence, by using the inequality \eqref{EqLemmaHeatProof2} together with \eqref{EqLemmaHeatProof1} and by using the constant $C_{31} := 2^c c! \max(1,2\eta)^c \max(1,\zeta_1)^c > 0$, we conclude for every $\beta \in \mathbb{N}^m_{0,c}$ and $\xi \in \mathbb{R}^m$ that
	\begin{equation}
		\label{EqLemmaHeatProof3}
		\begin{aligned}
			\left\vert \partial_\beta \widehat{f(t,\cdot)}(\xi) \right\vert & \leq \kappa_m (2\eta)^\frac{\vert \beta \vert}{2} 2^c c! \left( 1 + \big\Vert \sqrt{2\eta} \xi \big\Vert^2 \right)^\frac{c}{2} e^{-\eta \Vert \xi \Vert^2} \\
			& \leq \kappa_m C_{31} \left( 1 + \Vert \xi/\zeta_1 \Vert^2 \right)^\frac{c}{2} e^{-\eta \Vert \xi \Vert^2}.
		\end{aligned}
	\end{equation}
	Moreover, by using that $Y := \Vert Z \Vert^2$ of $Z \sim \mathcal{N}_m(0,I_m)$ follows a $\chi^2(m)$-distribution with probability density function $[0,\infty) \ni y \mapsto \frac{y^{m/2-1} \exp(-y/2)}{2^{m/2} \Gamma(m/2)} \in [0,\infty)$, the substitution $x \mapsto y/2$, and the definition of the Gamma function in \cite[Equation~6.1.1]{abramowitz70}, we obtain for every $b \in \mathbb{N}_0$ that
	\begin{equation}
		\label{EqLemmaHeatProof4}
		\begin{aligned}
			\int_{\mathbb{R}^m} \Vert z \Vert^b \frac{e^{-\frac{\Vert z \Vert^2}{2}}}{(2\pi)^\frac{m}{2}} dz & = \mathbb{E}\left[ \Vert Z \Vert^b \right] = \mathbb{E}\left[ Y^\frac{b}{2} \right] = \int_0^\infty y^\frac{b}{2} \frac{y^{\frac{m}{2}-1} e^{-\frac{y}{2}}}{2^\frac{m}{2} \Gamma\left( \frac{m}{2} \right)} dy \\
			& = \frac{2^\frac{b+m}{2}}{2^\frac{m}{2} \Gamma\left( \frac{m}{2} \right)} \int_0^\infty x^{\frac{b+m}{2}-1} e^{-x} dx = \frac{2^\frac{b}{2} \Gamma\left( \frac{m+b}{2} \right)}{\Gamma\left( \frac{m}{2} \right)}.
		\end{aligned}
	\end{equation}
	Now, we use the inequality \eqref{EqLemmaHeatProof3}, the exponent $c_s := 2c + 2\lceil s \rceil \in \mathbb{N}_0$, the inequality $(x+y)^r \leq 2^r \left( x^r + y^r \right)$ for any $x,y \geq 0$ and $r \geq 0$, the constant $C_{32} := 2^{c_s/2} \sqrt{\pi} > 0$, the substitution $z \mapsto \sqrt{4\eta} \xi$, the constant $C_{33} := C_{32} \big( \sqrt{4\eta} \zeta_1 \big)^{-c_s} > 0$, and the identity \eqref{EqLemmaHeatProof4} with $b := 0$ and $b := m+2c_s$ to observe that
	\begin{equation}
		\label{EqLemmaHeatProof5}
		\begin{aligned}
			& \frac{\zeta_2^{2m} \pi^\frac{m+1}{2}}{\zeta_1^m (2\pi)^{2m} \Gamma\left( \frac{m+1}{2} \right)} \int_{\mathbb{R}^m} \left( 1 + \Vert \xi/\zeta_1 \Vert^2 \right)^\frac{m+c_s}{2} e^{-2 \eta \Vert \xi \Vert^2} d\xi \\
			& \quad\quad \leq C_{32} \frac{\zeta_2^{2m} (2\pi)^\frac{m}{2}}{\zeta_1^m (2\pi)^{2m} \Gamma\left( \frac{m+1}{2} \right)} \left( \int_{\mathbb{R}^m} e^{-2 \eta \Vert \xi \Vert^2} d\xi + \int_{\mathbb{R}^m} \Vert \xi/\zeta_1 \Vert^{m+c_s} e^{-2 \eta \Vert \xi \Vert^2} d\xi \right) \\
			& \quad\quad = C_{32} \frac{\zeta_2^{2m}}{\zeta_1^m (2\pi)^m \Gamma\left( \frac{m+1}{2} \right) (4\eta)^\frac{m}{2}} \int_{\mathbb{R}^m} \frac{e^{-\frac{\Vert z \Vert^2}{2}}}{(2\pi)^\frac{m}{2}} dz \\
			& \quad\quad\quad\quad + C_{33} \frac{\zeta_2^{2m}}{\zeta_1^{2m} (2\pi)^m \Gamma\left( \frac{m+1}{2} \right) (4\eta)^\frac{m}{2}} \int_{\mathbb{R}^m} \Vert z \Vert^{m+c_s} \frac{e^{-\frac{\Vert z \Vert^2}{2}}}{(2\pi)^\frac{m}{2}} dz \\
			& \quad\quad = \frac{C_{32}}{\Gamma\left( \frac{m+1}{2} \right)} \left( \frac{\zeta_2^2/\zeta_1}{4\pi \sqrt{\eta}} \right)^m + C_{33} \left( \frac{(\zeta_2/\zeta_1)^2}{4\pi \sqrt{\eta}} \right)^m \frac{2^\frac{m+c_s}{2} \Gamma\left( \frac{2m+c_s}{2} \right)}{\Gamma\left( \frac{m+1}{2} \right) \Gamma\left( \frac{m}{2} \right)}.
		\end{aligned}
	\end{equation}
	For the first term on the right-hand side of \eqref{EqLemmaHeatProof5}, we use that $C_{34} := C_{32} \sup_{m \in \mathbb{N}} \frac{(\zeta_1/(2 \sqrt{2}))^m}{\Gamma((m+1)/2)} < \infty$ to conclude that
	\begin{equation}
		\label{EqLemmaHeatProof6}
		\frac{C_{32}}{\Gamma\left( \frac{m+1}{2} \right)} \left( \frac{\zeta_2^2/\zeta_1}{4\pi \sqrt{\eta}} \right)^m \leq \frac{C_{32}}{\Gamma\left( \frac{m+1}{2} \right)} \left( \frac{\kappa^2 \zeta_1}{2 \sqrt{2}} \right)^m \left( \frac{(\zeta_2/\zeta_1)^2}{\sqrt{2\eta} \pi} \right)^m \leq C_{34} \left( \frac{(\zeta_2/\zeta_1)^2}{\sqrt{2\eta} \pi} \right)^m.
	\end{equation}
	Moreover, for the second term on the right-hand side of \eqref{EqLemmaHeatProof5}, we use that $\sqrt{2\pi/x} (x/e)^x \leq \Gamma(x) \leq \sqrt{2\pi/x} (x/e)^x e^{1/(12x)} \leq \sqrt{4\pi/x} (x/e)^x$ for any $x \in [1/2,\infty)$ (see \cite[Lemma~2.4]{gonon19}), that $(2m+c_s)^{c_s/2} \leq m^{c_s/2} (2+c_s)^{c_s/2}$ and $(2+c_s/m)^m = 2^m (1+c_s/(2m))^m \leq 2^m e^{c_s/2}$ for any $m \in \mathbb{N}$, and the constant $C_{35} := C_{33} 2^{c_s/2} \sqrt{8\pi} (2e)^{(c_s-1)/2} (4\pi)^{-1} (3c_s)^{c_s/2} e^{c_s/2} > 0$, to obtain that
	\begin{equation}
		\label{EqLemmaHeatProof7}
		\begin{aligned}
			& C_{33} \left( \frac{(\zeta_2/\zeta_1)^2}{4\pi \sqrt{\eta}} \right)^m \frac{2^\frac{m+c_s}{2} \Gamma\left( \frac{2m+c_s}{2} \right)}{\Gamma\left( \frac{m+1}{2} \right) \Gamma\left( \frac{m}{2} \right)} \\
			& \quad \leq C_{33} \left( \frac{(\zeta_2/\zeta_1)^2}{4\pi \sqrt{\eta}} \right)^m \frac{2^\frac{m+c_s}{2} \sqrt{\frac{8\pi}{2m+c_s}} \left( \frac{2m+c_s}{2e} \right)^\frac{2m+c_s}{2}}{\sqrt{\frac{4\pi}{m+1}} \left( \frac{m+1}{2e} \right)^\frac{m+1}{2} \sqrt{\frac{4\pi}{m}} \left( \frac{m}{2e} \right)^\frac{m}{2}} \\
			& \quad \leq C_{33} \left( \frac{(\zeta_2/\zeta_1)^2}{4\pi \sqrt{\eta}} \right)^m \frac{2^\frac{m+c_s}{2} \sqrt{8\pi} (2e)^\frac{c_s-1}{2}}{4\pi} \underbrace{\frac{\sqrt{m(m+1)} (2m+c_s)^\frac{c_s}{2}}{\sqrt{2m+c_s} \sqrt{m+1}}}_{\leq (2m+c_s)^\frac{c_s}{2}} \bigg( \underbrace{\frac{2m+c_s}{m}}_{\leq 2+c_s/m} \bigg)^m  \\
			& \quad \leq C_{33} \left( \frac{(\zeta_2/\zeta_1)^2}{2\pi \sqrt{2\eta}} \right)^m \frac{2^\frac{c_s}{2} \sqrt{8\pi} (2e)^\frac{c_s-1}{2}}{4\pi} (2+c_s)^\frac{c_s}{2} m^\frac{c_s}{2} 2^m e^\frac{c_s}{2} \\
			& \quad = C_{35} \left( \frac{(\zeta_2/\zeta_1)^2}{\sqrt{2\eta} \pi} \right)^m m^\frac{c_s}{2}.
		\end{aligned}
	\end{equation}
	Hence, by inserting \eqref{EqLemmaHeatProof6}+\eqref{EqLemmaHeatProof7} into \eqref{EqLemmaHeatProof5}, it follows that
	\begin{equation}
		\label{EqLemmaHeatProof8}
		\begin{aligned}
			& \left( \frac{\zeta_2^{2m} \pi^\frac{m+1}{2}}{\zeta_1^m (2\pi)^{2m} \Gamma\left( \frac{m+1}{2} \right)} \int_{\mathbb{R}^m} \left( 1 + \Vert \xi/\zeta_1 \Vert^2 \right)^\frac{m+s}{2} e^{-2\eta \Vert \xi \Vert^2} d\xi \right)^\frac{1}{2} \\
			& \quad\quad \leq \left( \frac{C_{32}}{\Gamma\left( \frac{m+1}{2} \right)} \left( \frac{\zeta_2^2/\zeta_1}{4\pi \sqrt{\eta}} \right)^m + C_{33} \left( \frac{(\kappa\zeta_2/\zeta_1)^2}{4\pi \sqrt{\eta}} \right)^m \frac{2^\frac{m+c_s}{2} \Gamma\left( \frac{2m+c_s}{2} \right)}{\Gamma\left( \frac{m+1}{2} \right) \Gamma\left( \frac{m}{2} \right)} \right)^\frac{1}{2} \\
			& \quad\quad \leq \left( C_{34} \left( \frac{(\zeta_2/\zeta_1)^2}{\sqrt{2\eta} \pi} \right)^m + C_{35} \left( \frac{(\zeta_2/\zeta_1)^2}{\sqrt{2\eta} \pi} \right)^m m^\frac{c_s}{2} \right)^\frac{1}{2} \\
			& \quad\quad \leq \sqrt{C_{34}+C_{35}} \left( \frac{(\kappa\zeta_2/\zeta_1)^2}{\sqrt{2\eta} \pi} \right)^\frac{m}{2} m^\frac{c_s}{4}.
		\end{aligned}
	\end{equation}	
	Thus, by using that $\big\vert \mathbb{N}^m_{0,c} \big\vert = \sum_{j=0}^{c} m^j \leq 2m^c$, the inequality \eqref{EqLemmaHeatProof8}, the constant $C_{22} := 2 C_{31} \sqrt{C_{34}+C_{35}} > 0$, and that $\eta = \lambda t + \sigma^2/2 \geq \sigma^2/2$, we conclude that
	\begin{equation*}
		\begin{aligned}
			& \frac{\zeta_2^m \pi^\frac{m+1}{4}}{\zeta_1^\frac{m}{2} (2\pi)^m  \Gamma\left( \frac{m+1}{2} \right)^\frac{1}{2}} \sum_{\beta \in \mathbb{N}^m_{0,c}} \left( \int_{\mathbb{R}^m} \big\vert \partial_\beta \widehat{f(t,\cdot)}(\xi) \big\vert^2 \left( 1 + \Vert \xi/\zeta_1 \Vert^2 \right)^\frac{m+s}{2} d\xi \right)^\frac{1}{2} \\
			& \quad\quad \leq C_{31} \kappa_m \left\vert \mathbb{N}^m_{0,c} \right\vert \max_{\beta \in \mathbb{N}^m_{0,c}} \left( \frac{\zeta_2^{2m} \pi^\frac{m+1}{2}}{\zeta_1^m (2\pi)^{2m} \Gamma\left( \frac{m+1}{2} \right)} \int_{\mathbb{R}^m} \left( 1 + \Vert \xi/\zeta_1 \Vert^2 \right)^\frac{m+s}{2} e^{-2\eta \Vert \xi \Vert^2} d\xi \right)^\frac{1}{2} \\
			& \quad\quad \leq 2 C_{31} \sqrt{C_{34}+C_{35}} m^c \kappa_m \left( \frac{(\zeta_2/\zeta_1)^2}{\sqrt{2\eta} \pi} \right)^\frac{m}{2} m^\frac{c_s}{4} \leq C_{11} m^\frac{5c+s}{4} \kappa_m \left( \frac{(\zeta_2/\zeta_1)^2}{\pi \sigma} \right)^\frac{m}{2},
		\end{aligned}
	\end{equation*}
	which completes the proof.
\end{proof}

\begin{proof}[Proof of Corollary~\ref{CorHeat}]
	For $\lambda,t,\kappa_m,\sigma \in (0,\infty)$ and the initial condition $f_0(u) := \frac{\kappa_m}{(2\pi \sigma^2)^{m/2}} \exp\left( -\Vert u \Vert^2/(2\sigma^2) \right)$, let $f(t,\cdot): \mathbb{R}^m \rightarrow \mathbb{R}$ be the solution of \eqref{EqDefHeat} at time $t$. Moreover, let $p \in (1,\infty)$, $\gamma \in [0,\infty)$, and $w: \mathbb{R}^m \rightarrow [0,\infty)$ satisfy the conditions of Lemma~\ref{LemmaWeight}.
	
	For Part~\ref{CorHeat1}., we fix some $N \in \mathbb{N}$. Then, by using the probability density function $p_\theta: \mathbb{R}^m \rightarrow [0,\infty)$ of the Student's $t$-distributed i.i.d.~sequence $(\theta_n)_{n \in \mathbb{N}} \sim t_m$ and Lemma~\ref{LemmaHeat} (with $\zeta_1 := \zeta_2 := 1$, $c := 0$, $s := 1$, and constant $C_{22} > 0$ independent of $m \in \mathbb{N}$), we observe that
	\begin{equation}
		\label{EqCorHeatProof1}
		\begin{aligned}
			\frac{C_{f(t,\cdot)}}{(2\pi)^m} & := \frac{1}{(2\pi)^m} \left( \int_{\mathbb{R}^m} \frac{\big\vert \widehat{f(t,\cdot)}(\vartheta) \big\vert^2}{p_\theta(\vartheta)} d\vartheta \right)^\frac{1}{2} \\
			& = \frac{\pi^\frac{m+1}{4}}{(2\pi)^m \Gamma\left( \frac{m+1}{2} \right)^\frac{1}{2}} \left( \int_{\mathbb{R}^m} \big\vert \widehat{f(t,\cdot)}(\vartheta) \big\vert^2 \left( 1 + \Vert \vartheta \Vert^2 \right)^\frac{m+1}{2} d\vartheta \right)^\frac{1}{2} \\
			& \leq C_{22} m^\frac{1}{4} \kappa_m \left( \frac{(\zeta_2/\zeta_1)^2}{\pi \sigma} \right)^\frac{m}{2} < \infty.
		\end{aligned}
	\end{equation}
	Hence, we can apply Corollary~\ref{CorARTrigo} (with constant $C_{p,2} > 0$ depending only on $p \in (1,\infty)$) to obtain a random trigonometric feature model $G_N \in \mathcal{RT}_{\mathbb{R}^m,1}$ with $N$ features satisfying
	\begin{equation*}
		\mathbb{E}\left[ \Vert f(t,\cdot) - G_N \Vert_{L^p(\mathbb{R}^m,\mathcal{L}(\mathbb{R}^m),w)}^2 \right]^\frac{1}{2} \leq C_{p,2} \frac{w(\mathbb{R}^m)^\frac{1}{p}}{(2\pi)^m} \frac{C_{f(t,\cdot)}}{N^{1-\frac{1}{\min(2,p)}}}.
	\end{equation*}
	Thus, by using that $w: \mathbb{R}^m \rightarrow [0,\infty)$ is a weight satisfying the conditions of Lemma~\ref{LemmaWeight}, i.e.~that $\mathbb{R}^m \ni u \mapsto w(u) := \prod_{l=1}^m w_0(u_l) \in [0,\infty)$ for some $w_0: \mathbb{R} \rightarrow [0,\infty)$ satisfying $\int_{\mathbb{R}} w_0(s) ds = 1$ (implying that $w(\mathbb{R}^m) = \int_{\mathbb{R}^m} w(u) du = \prod_{l=1}^m \int_{\mathbb{R}} w_0(u_l) du_l = 1$), Lemma~\ref{LemmaHeat} (with constant $C_{22} > 0$ being independent of $m \in \mathbb{N}$), and the constant $C_6 := C_{p,2} C_{22} > 0$ and $C_7 := \frac{1}{4}$ (being independent of $m \in \mathbb{N}$), we have
	\begin{equation*}
		\begin{aligned}
			\mathbb{E}\left[ \Vert f(t,\cdot) - G_N \Vert_{L^p(\mathbb{R}^m,\mathcal{L}(\mathbb{R}^m),w)}^2 \right]^\frac{1}{2} & \leq C_{p,2} \frac{w(\mathbb{R}^m)^\frac{1}{p}}{(2\pi)^m} \frac{C_{f(t,\cdot)}}{N^{1-\frac{1}{\min(2,p)}}} \\
			& \leq C_6 m^{C_7} \frac{\kappa_m \left( \frac{(\zeta_2/\zeta_1)^2}{\pi \sigma} \right)^\frac{m}{2}}{N^{1-\frac{1}{\min(2,p)}}},
		\end{aligned}
	\end{equation*}
	which proves the inequality~\eqref{EqCorHeat1}. For \eqref{EqCorHeat2}, we use that there exists by assumption some $\widetilde{C} > 0$ such that $\kappa_m \left( (\zeta_2/\zeta_1)^2/ (\pi \sigma) \right)^\frac{m}{2} \leq \widetilde{C} m^\nu$. Hence, by defining the constants $C_{8} := (C_6 \widetilde{C})^{\min(2,p)/(\min(2,p)-1)} > 0$ and $C_{9} := (C_7+\nu) \frac{\min(2,p)}{\min(2,p)-1}$ (being independent of $m \in \mathbb{N}$), it holds for $N = \big\lceil C_8 m^{C_9} \varepsilon^{-\frac{\min(2,p)}{\min(2,p)-1}} \big\rceil$ that
	\begin{equation*}
		\mathbb{E}\left[ \Vert f(t,\cdot) - G_N \Vert_{L^p(\mathbb{R}^m,\mathcal{L}(\mathbb{R}^m),w)}^2 \right]^\frac{1}{2} \leq C_6 m^{C_7} \frac{\kappa_m \left( \frac{(\zeta_2/\zeta_1)^2}{\pi \sigma} \right)^\frac{m}{2}}{N^{1-\frac{1}{\min(2,p)}}} \leq \frac{C_6 \widetilde{C} m^{C_7+\nu}}{N^{1-\frac{1}{\min(2,p)}}} \leq \varepsilon,
	\end{equation*}
	which proves the inequality~\eqref{EqCorHeat2}.
	
	For Part~\ref{CorHeat2}., we fix some $N \in \mathbb{N}$. Then, by using \eqref{EqPropConst2} and Lemma~\ref{LemmaHeat} (with $0 < \zeta_1 < \zeta_2 < \infty$, $c := \lceil\gamma\rceil+2$, $s := 4\lceil\gamma\rceil+5$, and constant $C_{22} > 0$ independent of $m \in \mathbb{N}$), we have
	\begin{equation}
		\label{EqCorHeatProof3}
		\begin{aligned}
			& \frac{\zeta_2^m}{(2\pi)^m} \Vert f(t,\cdot) \Vert_{\widetilde{\mathbb{B}}^{0,2,\gamma}_{\psi,a,b}(\mathbb{R}^m)} \\
			& \quad\quad \leq C_1 \frac{\zeta_2^m \pi^\frac{m+1}{4}}{ \zeta_1^\frac{m}{2} (2\pi)^m \Gamma\left( \frac{m+1}{2} \right)} \sum_{\beta \in \mathbb{N}^m_{0,c}} \left( \int_{\mathbb{R}^m} \big\vert \partial_\beta \widehat{f(t,\cdot)}(\xi) \big\vert^2 \left( 1 + \Vert \xi/\zeta_1 \Vert^2 \right)^{2\lceil\gamma\rceil+\frac{m+5}{2}} d\xi \right)^\frac{1}{2} \\
			& \quad\quad \leq C_{22} m^\frac{9\lceil\gamma\rceil+15}{4} \kappa_m \left( \frac{(\zeta_2/\zeta_1)^2}{\pi \sigma} \right)^\frac{m}{2} < \infty.
		\end{aligned}
	\end{equation}
	Therefore, by using that
	\begin{equation*}
		\begin{aligned}
			\Vert f(t,\cdot) \Vert_{L^1(\mathbb{R}^m,\mathcal{L}(\mathbb{R}^m),du)} & = \Vert \phi_{\lambda,t} * f_0 \Vert_{L^1(\mathbb{R}^m,\mathcal{L}(\mathbb{R}^m),du)} \\
			& \leq \underbrace{\Vert \phi_{\lambda,t} \Vert_{L^1(\mathbb{R}^m,\mathcal{L}(\mathbb{R}^m),du)}}_{=1} \underbrace{\Vert f_0 \Vert_{L^1(\mathbb{R}^m,\mathcal{L}(\mathbb{R}^m),du)}}_{=\kappa_m} = \kappa_m < \infty,
		\end{aligned}
	\end{equation*}
	i.e.~$f(t,\cdot) \in L^1(\mathbb{R}^m,\mathcal{L}(\mathbb{R}^m),du)$, and that $\widehat{f(t,\cdot)} \in L^1(\mathbb{R}^m,\mathcal{L}(\mathbb{R}^m),du)$ (see \eqref{EqLemmaHeatProof1}), we conclude from the definition of $\widetilde{\mathbb{B}}^{0,2,\gamma}_{\psi,a,b}(\mathbb{R}^m)$ that $f(t,\cdot) \in \widetilde{\mathbb{B}}^{0,2,\gamma}_{\psi,a,b}(\mathbb{R}^m)$. Hence, we can use Corollary~\ref{CorARRN} (with constant $C_{p,2} > 0$ depending only on $p \in (1,\infty)$) to obtain a random neural network $G_N \in \mathcal{RN}^\rho_{\mathbb{R}^m,1}$ with $N$ neurons satisfying
	\begin{equation*}
		\mathbb{E}\left[ \Vert f(t,\cdot) - G_N \Vert_{L^p(\mathbb{R}^m,\mathcal{L}(\mathbb{R}^m),w)}^2 \right]^\frac{1}{2} \leq C_{p,2} \Vert \rho \Vert_{C^0_{pol,\gamma}(\mathbb{R})} \frac{C^{(\gamma,p)}_{U,w}}{\left\vert C^{(\psi,\rho)}_m \right\vert} \frac{\Vert f(t,\cdot) \Vert_{\widetilde{\mathbb{B}}^{0,2,\gamma}_{\psi,a,b}(\mathbb{R}^m)}}{N^{1-\frac{1}{\min(2,p)}}}.
	\end{equation*}
	Thus, by using Lemma~\ref{LemmaWeight} (with constant $C^{(\gamma,p)}_{\mathbb{R},w_0}$ depending only on $\gamma \in [0,\infty)$, $p \in (1,\infty)$, and $w_0: \mathbb{R} \rightarrow [0,\infty)$), Example~\ref{ExAdm} (with constant $C_{\psi,\rho} > 0$ depending only on $\psi \in \mathcal{S}_0(\mathbb{R};\mathbb{C})$ and $\rho \in C^0_{pol,\gamma}(\mathbb{R})$), Lemma~\ref{LemmaHeat} (with constant $C_{22} > 0$ being independent of $m \in \mathbb{N}$), and the constants $C_{10} := \frac{C_{p,2} C^{(\gamma,p)}_{\mathbb{R},w_0} C_{22}}{C_{\psi,\rho}} > 0$ and $C_{11} := \gamma + \frac{1}{p} + \frac{9\lceil\gamma\rceil+15}{4} > 0$ (being independent of $m \in \mathbb{N}$), it holds that
	\begin{equation*}
		\begin{aligned}
			\mathbb{E}\left[ \Vert f(t,\cdot) - G_N \Vert_{L^p(\mathbb{R}^m,\mathcal{L}(\mathbb{R}^m),w)}^2 \right]^\frac{1}{2} & \leq C_{p,2} \Vert \rho \Vert_{C^0_{pol,\gamma}(\mathbb{R})} \frac{C^{(\gamma,p)}_{U,w}}{\left\vert C^{(\psi,\rho)}_m \right\vert} \frac{\Vert f(t,\cdot) \Vert_{\widetilde{\mathbb{B}}^{0,2,\gamma}_{\psi,a,b}(\mathbb{R}^m)}}{N^{1-\frac{1}{\min(2,p)}}} \\
			& \leq \frac{C_{p,2} C^{(\gamma,p)}_{\mathbb{R},w_0}}{C_{\psi,\rho}} m^{\gamma+\frac{1}{p}} \frac{\zeta_2^m}{(2\pi)^m} \frac{\Vert f(t,\cdot) \Vert_{\widetilde{\mathbb{B}}^{0,2,\gamma}_{\psi,a,b}(\mathbb{R}^m)}}{N^{1-\frac{1}{\min(2,p)}}} \\
			& \leq C_{10} m^{C_{11}} \frac{\kappa_m \left( \frac{(\zeta_2/\zeta_1)^2}{\pi \sigma} \right)^\frac{m}{2}}{N^{1-\frac{1}{\min(2,p)}}},
		\end{aligned}
	\end{equation*}
	which proves the inequality \eqref{EqCorHeat3}. For \eqref{EqCorHeat4}, we use that $p \in (1,\infty)$ and that there exists by assumption some $\widetilde{C} > 0$ such that $\kappa_m \left( (\zeta_2/\zeta_1)^2/ (\pi \sigma) \right)^\frac{m}{2} \leq \widetilde{C} m^\nu$. Hence, by defining the constants $C_{12} := (C_{10} \widetilde{C})^{\min(2,p)/(\min(2,p)-1)} > 0$ and $C_{13} := (C_{11}+\nu) \frac{\min(2,p)}{\min(2,p)-1}$ (being independent of $m \in \mathbb{N}$), it holds for $N = \big\lceil C_{12} m^{C_{13}} \varepsilon^{-\frac{\min(2,p)}{\min(2,p)-1}} \big\rceil$ that
	\begin{equation*}
		\mathbb{E}\left[ \Vert f(t,\cdot) - G_N \Vert_{L^p(\mathbb{R}^m,\mathcal{L}(\mathbb{R}^m),w)}^2 \right]^\frac{1}{2} \leq C_{10} m^{C_{11}} \frac{\kappa_m \left( \frac{(\zeta_2/\zeta_1)^2}{\pi \sigma} \right)^\frac{m}{2}}{N^{1-\frac{1}{\min(2,p)}}} \leq \frac{C_{10} \widetilde{C} m^{C_{11}+\nu}}{N^{1-\frac{1}{\min(2,p)}}} \leq \varepsilon,
	\end{equation*}
	which proves the inequality~\eqref{EqCorHeat4}.
\end{proof}

\begin{lemma}
	\label{LemmaFokker}
	For $t \in (0,\infty)$, let $f(t,\cdot): \mathbb{R}^m \rightarrow \mathbb{R}$ be the solution of \eqref{EqDefFokker} at time $t$ given in \eqref{EqSolFokker}. Moreover, let $c \in \mathbb{N}_0$, $s \in [0,\infty)$, and $0 < \zeta_1 \leq \zeta_2 < \infty$. Then, $\widehat{f(t,\cdot)}: \mathbb{R}^m \rightarrow \mathbb{C}$ is smooth and there exists a constant $C_{23},C_{24} > 0$ (being independent of $m \in \mathbb{N}$) such that
	\begin{equation}
		\label{EqLemmaFokker1}
		\begin{aligned}
			& \frac{\zeta_2^m \pi^\frac{m+1}{4}}{\zeta_1^\frac{m}{2} (2\pi)^m \Gamma\left( \frac{m+1}{2} \right)^\frac{1}{2}} \sum_{\beta \in \mathbb{N}^m_{0,c}} \left( \int_{\mathbb{R}^m} \big\vert \partial_\beta \widehat{f(t,\cdot)}(\xi) \big\vert^2 \left( 1 + \Vert \xi/\zeta_1 \Vert^2 \right)^\frac{m+s}{2} d\xi \right)^\frac{1}{2} \\
			& \quad\quad \leq  C_{23} m^{C_{24}} \left( 1 + \Vert \Sigma_t \Vert + \Vert \mu_t \Vert \right)^c \left( \frac{\sqrt{2 \Vert \Sigma^{-1} \Vert} (\zeta_2/\zeta_1)^2}{\pi} \right)^\frac{m}{2}.
		\end{aligned}
	\end{equation}
\end{lemma}
\begin{proof}
	Fix some $t \in (0,\infty)$, $c \in \mathbb{N}_0$, $s \in [0,\infty)$, and $0 < \zeta_1 \leq \zeta_2 < \infty$. Then, by using that $f(t,\cdot): \mathbb{R}^m \rightarrow \mathbb{R}$ is Gaussian (see \eqref{EqSolFokker}), we conclude for every $\xi \in \mathbb{R}^m$ that
	\begin{equation*}
		\widehat{f(t,\cdot)}(\xi) = e^{-\mathbf{i} \mu_t^\top \xi - \frac{1}{2} \xi^\top \Sigma_t \xi}.
	\end{equation*}
	Now, for every fixed $\beta \in \mathbb{N}^m_{0,c}$, we apply the Fa\`a di Bruno formula and using the notation $\beta^{(l)} = e_{i_l}$ for some $i_l = 1,...,m$ (if $\vert \beta^{(l)} \vert = 1$) and $\beta^{(l)} = e_{i_l} + e_{j_l}$ for some $i_l,j_l = 1,...,m$ (if $\vert \beta^{(l)} \vert = 2$) to obtain for every $\xi \in \mathbb{R}^m$ that
	\begin{equation*}
		\partial_\beta \widehat{f(t,\cdot)}(\xi) = e^{-\mathbf{i} \mu_t^\top \xi - \frac{1}{2} \xi^\top \Sigma_t \xi} \sum_{j=1}^{\vert \beta \vert} \sum_{\beta^{(1)},...,\beta^{(j)} \in \mathbb{N}^m_0 \setminus \lbrace 0 \rbrace \atop \beta^{(1)}+...+\beta^{(j)}=\beta} \prod_{l=1 \atop \vert \beta^{(l)} \vert = 1}^j \! \left( -(\Sigma_t \xi + \mathbf{i} \mu_t)_{i_l} \right) \! \prod_{l=1 \atop \vert \beta^{(l)} \vert = 2}^j \! \left( -(\Sigma_t)_{i_l,j_l} \right).
	\end{equation*}
	Hence, by using the notation $\vert \beta^{(1:j)} \vert_k := \left\vert \left\lbrace \beta^{(l)}: j = 1,...,j, \, \vert \beta^{(l)} \vert = k \right\rbrace \right\vert$, $k \in \lbrace 1,2 \rbrace$, together with $\vert (\Sigma_t \xi + \mathbf{i} \mu_t)_{i_l} \vert \leq \Vert \Sigma_t \xi + \mathbf{i} \mu_t \Vert \leq 1+\Vert \Sigma_t \Vert \Vert \xi \Vert + \Vert \mu_t \Vert$ and $\vert (\Sigma_t)_{i_l,j_l} \vert \leq \Vert \Sigma_t \Vert \leq 1+\Vert \Sigma_t \Vert$, that $\sum_{j=1}^{\vert \beta \vert} \sum_{\beta^{(1)},...,\beta^{(j)} \in \mathbb{N}^m_0 \setminus \lbrace 0 \rbrace, \, \beta^{(1)}+...+\beta^{(j)}=\beta} 1 = 2^{\vert \beta \vert} \vert \beta \vert!$, that $(x+y)^2 \leq 2 \left( x^2+y^2 \right)$ for any $x,y \geq 0$, and that $e^{-0.5 \xi^\top \Sigma_t \xi} \leq e^{-0.5 \lambda_{\min} \Vert \xi \Vert^2}$ with $\lambda_{\min} := \lambda_{\min}(\Sigma_t)$, there exists a constant $C_{36} > 0$ (independent of $m \in \mathbb{N}$) such that for every $\xi \in \mathbb{R}^m$ it holds that
	\begin{equation*}
		\begin{aligned}
			\left\vert \partial_\beta \widehat{f(t,\cdot)}(\xi) \right\vert & \leq e^{-\frac{1}{2} \xi^\top \Sigma_t \xi} \sum_{j=1}^{\vert \beta \vert} \sum_{\beta^{(1)},...,\beta^{(j)} \in \mathbb{N}^m_0 \setminus \lbrace 0 \rbrace \atop \beta^{(1)}+...+\beta^{(j)}=\beta} \left( 1 + \Vert \Sigma_t \Vert \Vert \xi \Vert + \Vert \mu_t \Vert \right)^{\vert \beta^{(1:j)} \vert_1} \left( 1 + \Vert \Sigma_t \Vert \right)^{\vert \beta^{(1:j)} \vert_2} \\
			& \leq 2^{\vert\beta\vert} \vert \beta \vert! \left( 1 + \Vert \Sigma_t \Vert \Vert \xi \Vert + \Vert \mu_t \Vert \right)^{\vert\beta\vert} e^{-\frac{1}{2} \xi^\top \Sigma_t \xi} \\
			& \leq C_{36} \left( 1 + \Vert \Sigma_t \Vert + \Vert \mu_t \Vert \right)^c \left( 1 + \Vert \xi/\zeta_1 \Vert^2 \right)^\frac{c}{2} e^{-\frac{1}{2} \lambda_{\min} \Vert \xi \Vert^2}.
		\end{aligned}
	\end{equation*}
	Thus, by following the inequalities~\eqref{EqLemmaHeatProof5}-\eqref{EqLemmaHeatProof8} (with $\frac{1}{2} \lambda_{\min}$ instead of $2\lambda t$), there exist constants $C_{37}, C_{38} > 0$ (independent of $m \in \mathbb{N}$) such that for every $\xi \in \mathbb{R}^m$ it holds that
	\begin{equation*}
		\begin{aligned}
			& \left( \frac{\zeta_2^{2m} \pi^\frac{m+1}{2}}{\zeta_1^m (2\pi)^{2m} \Gamma\left( \frac{m+1}{2} \right)} \int_{\mathbb{R}^m} \big\vert \partial_\beta \widehat{f(t,\cdot)}(\xi) \big\vert^2 \left( 1 + \Vert \xi/\zeta_1 \Vert^2 \right)^\frac{m+s}{2} d\xi \right)^\frac{1}{2} \\
			& \leq C_{36} \left( 1 + \Vert \Sigma_t \Vert + \Vert \mu_t \Vert \right)^c \left( \frac{\zeta_2^{2m} \pi^\frac{m+1}{2}}{\zeta_1^m (2\pi)^{2m} \Gamma\left( \frac{m+1}{2} \right)} \int_{\mathbb{R}^m} \left( 1 + \Vert \xi/\zeta_1 \Vert^2 \right)^\frac{m+s+2c}{2} e^{-\frac{1}{2} \lambda_{\min} \Vert \xi \Vert^2} d\xi \right)^\frac{1}{2} \\
			& = C_{37} \left( 1 + \Vert \Sigma_t \Vert + \Vert \mu_t \Vert \right)^c \left( \frac{(\zeta_2/\zeta_1)^2}{\pi \sqrt{\lambda_{\min}/2}} \right)^\frac{m}{2} m^{C_{38}}.
		\end{aligned}
	\end{equation*}
	Finally, by using this, that $\vert \mathbb{N}^m_{0,c} \vert = \sum_{j=0}^c m^j \leq 2m^c$, and the constants $C_{23} := 2 C_{37} > 0$ and $C_{24} := c+C_{38} > 0$ (independent of $m \in \mathbb{N}$), we conclude that
	\begin{equation*}
		\begin{aligned}
			& \frac{\zeta_2^m \pi^\frac{m+1}{4}}{\zeta_1^\frac{m}{2} (2\pi)^m \Gamma\left( \frac{m+1}{2} \right)^\frac{1}{2}} \sum_{\beta \in \mathbb{N}^m_{0,c}} \left( \int_{\mathbb{R}^m} \big\vert \partial_\beta \widehat{f(t,\cdot)}(\xi) \big\vert^2 \left( 1 + \Vert \xi/\zeta_1 \Vert^2 \right)^\frac{m+s}{2} d\xi \right)^\frac{1}{2} \\
			& \leq \left\vert \mathbb{N}^m_{0,c} \right\vert \left( 1 \!+\! \Vert \Sigma_t \Vert \!+\! \Vert \mu_t \Vert \right)^c \max_{\beta \in \mathbb{N}^m_{0,c}} \left( \frac{\zeta_2^{2m} \pi^\frac{m+1}{2}}{\zeta_1^m (2\pi)^{2m} \Gamma\left( \frac{m+1}{2} \right)} \int_{\mathbb{R}^m} \!\! \big\vert \partial_\beta \widehat{f(t,\cdot)}(\xi) \big\vert^2 \left( 1 \!+\! \Vert \xi/\zeta_1 \Vert^2 \right)^\frac{m+s}{2} \! d\xi \right)^\frac{1}{2} \\
			& \leq 2 m^c \left( 1 + \Vert \Sigma_t \Vert + \Vert \mu_t \Vert \right)^c C_{37} \left( \frac{(\zeta_2/\zeta_1)^2}{\pi \sqrt{\lambda_{\min}/2}} \right)^\frac{m}{2} m^{C_{38}} \\
			& = C_{23} m^{C_{24}} \left( 1 + \Vert \Sigma_t \Vert + \Vert \mu_t \Vert \right)^c \left( \frac{\sqrt{2 \Vert \Sigma^{-1} \Vert} (\zeta_2/\zeta_1)^2}{\pi} \right)^\frac{m}{2},
		\end{aligned}
	\end{equation*}
	which completes the proof.
\end{proof}

\begin{proof}[Proof of Corollary~\ref{CorFokker}]
	The conclusion is obtained by following exactly the same steps as in the proof of Corollary~\ref{CorHeat}, but where now the inequality~\eqref{EqLemmaFokker1} in Lemma~\ref{LemmaFokker} instead of the inequality~\eqref{EqLemmaHeat1} in Lemma~\ref{LemmaHeat} is inserted.
\end{proof}

\bibliographystyle{plain}
\bibliography{mybib}

\end{document}